\documentclass{article}
\PassOptionsToPackage{dvipsnames}{xcolor}
\PassOptionsToPackage{numbers,sort&compress}{natbib}

\usepackage[main, final]{neurips_2026}

\usepackage[utf8]{inputenc} % allow utf-8 input
\usepackage[T1]{fontenc}    % use 8-bit T1 fonts
\usepackage{hyperref}       % hyperlinks
\usepackage{url}            % simple URL typesetting
\usepackage{booktabs}       % professional-quality tables
\usepackage{amsfonts}       % blackboard math symbols
\usepackage{nicefrac}       % compact symbols for 1/2, etc.
\usepackage{microtype}      % microtypography
\usepackage{xcolor}         % colors
\usepackage{enumitem}

\usepackage{graphicx}
\usepackage{subcaption}

\usepackage[ruled,vlined]{algorithm2e}
\DontPrintSemicolon
\SetKwInOut{Input}{Input}
\SetKwInOut{Output}{Output}

\usepackage{my_symbols}

\title{Inference and Uncertainty Quantification for Streaming $r$-PCA}

\author{%
  Haoshu Xu \\
  University of Pennsylvania\\
  \texttt{haoshuxu@sas.upenn.edu} \\
  \And
  Hongzhe Li \\
  University of Pennsylvania \\
  \texttt{hongzhe@upenn.edu} \\
}

\begin{document}

\maketitle

\begin{abstract}
  We address two open questions in streaming PCA via Oja's algorithm: sharp operator-norm convergence for general rank under sub-Gaussian data, and distributional inference for the resulting subspace estimator. Existing convergence analyses, even in the rank-one case, either assume bounded data or leave non-vanishing remainder terms that prevent adaptation to a polynomially vanishing tail spectrum, while existing distributional results are confined to the rank-one case. Our convergence theory removes these remainder terms and yields a sharp rate. In the dense-tail spiked covariance regime, this rate matches the minimax rate up to logarithmic factors. More generally, we prove a matching lower bound, up to logarithmic factors, across both dense-tail and sparse-tail regimes under a mild nondegeneracy condition. The analysis yields a linearization of Oja's iterates, which in turn enables a high-dimensional Gaussian approximation for the general-rank subspace estimation error with an explicit limiting covariance. We also establish a row-wise Gaussian approximation over convex sets for the aligned difference, recovering prior rank-one results as special cases. For practical inference, we develop an online multiplier bootstrap algorithm and prove its consistency. Beyond streaming PCA, our techniques contribute to Gaussian approximation and bootstrap inference for nonconvex stochastic approximation.
\end{abstract}

\section{Introduction}

Streaming statistical methods are increasingly important in modern data analysis \citep{online_sparse_regression,dong_tensor,nips_sparsePCA,nips_LowPrecisionStreamingPCA}. In some applications, data arrive sequentially and updates must be performed online \citep{recommender_system}; in others, the full dataset is available but too large to store in memory or to process efficiently with batch methods \cite{langford}. Streaming principal component analysis (PCA) is a canonical problem in this setting: the goal is to estimate the leading eigenspace of the population covariance matrix from a single pass over the data. Beyond its classical role in online dimension reduction, streaming PCA has recently appeared in memory-efficient LLM training, where low-rank structure is used to compress optimizer states \cite{MemoryEfficientLLM}. 
Oja's algorithm \citep{oja82}, a standard method for streaming PCA, has also motivated differentially private variants of PCA \cite{DP_PCA}.

Starting from random initialization $U_0 \sim \Haar{\OO_{d,r}}$, Oja's iteration for streaming PCA is given by
\begin{align}
  U_t = \qr{U_{t-1} + \eta X_t X_t^\top U_{t-1}}, \label{eqn:oja_update}
\end{align}
where $\qr{\cdot}$ denotes the QR factorization. Although Oja's algorithm dates to 1982, non-asymptotic theory for its finite-sample behavior has emerged only recently \cite{zhiyuan17,colt21,jain16,han17,nearopt22,opt23,nips21_oja}, and key questions about rate optimality and uncertainty quantification remain open.

% The learning rate $\eta$ can be either fixed or varying with $t$. In this paper, we fix the learning rate $\eta$, similar to [nips21, low precision, 17han, 22near opt].

\paragraph{Prior convergence analysis of Oja's algorithm.} 
Most existing analyses of Oja's algorithm \cite{zhiyuan17,colt21,jain16,nips21_oja} yield error bounds that do not vanish even when $\Sigma = \EE XX^\top$ is exactly rank-$r$. For example, one seminal work \cite{jain16} derived a bound at least of order $\frac{\lambda_1}{\Delta_{\min}\sqrt n}$
in the rank-one case; see \cref{tab:opnorm-comparison}. This rate is tight for a generic perturbation argument based on matrix Bernstein and Wedin's theorem, but it is not statistically sharp for PCA. Indeed, fluctuations within the leading eigenspace do not affect subspace estimation; the statistical difficulty comes from the interaction between the signal and tail eigenspaces. This distinction is reflected in the statistical minimax rates, both in Frobenius norm \cite{vu13} and in operator norm \cite{anru}.

Motivated by this observation, several recent works \cite{han17,nearopt22,opt23} attempted to match the minimax rate under sub-Gaussian assumptions. \cite{han17} focuses on the rank-one case, while \cite{nearopt22,opt23} consider the general-rank case in Frobenius norm. However, all these analyses still contain remainder terms that do not vanish in the small-tail regime. For example, \cite[][Theorem 3]{han17} requires $\lambda_2/\lambda_1$ to be bounded away from zero, which conflicts with the goal of obtaining rates adaptive to a vanishing tail spectrum. 
Moreover, although \cite{nearopt22} proves a Haar-initialization theorem for the constant-step-size case, its bound still contains a non-vanishing remainder term in the small-tail regime; in particular, the constants may depend implicitly on \(\lambda_d^{-1}\). In addition,
the argument does not establish the required no-exit/stability guarantee ensuring that the iterates remain in the contractive region throughout the run. Closing this gap appears to require additional sample-size or stepsize control.
% [nearopt] studies the constant-step-size case and proves a Haar-initialization theorem, but the proof appears to rely on an implicit no-exit/stability argument for keeping the iterates in a contractive region; this step is not made explicit and seems to require additional sample-size/stepsize control.
Later, \cite{opt23} considered the varying-step-size case and circumvented the stability issue by invoking an initialization procedure from \cite{colt21}, but the associated burn-in period again imposes a strong implicit requirement on the sample size, and the resulting guarantee still contains a non-vanishing component in the small-tail regime.
In view of these issues, it remains an open question whether the minimax rate can be achieved without any non-vanishing remainder term, even in the rank-one case.

\paragraph{Uncertainty quantification for SGD and PCA.} This paper is also motivated by recent advances in uncertainty quantification for SGD and high-dimensional PCA. In the strongly convex setting, \cite{gaussian25} established Gaussian approximation results for the last iterate of SGD.
Related Gaussian approximation and bootstrap results for stochastic approximation have been developed by \cite{shao2022berry,wu2026statistical,sheshukova2026gaussian,butyrin2025improved}.
This naturally raises the question of whether analogous distributional guarantees can be obtained for nonconvex problems, with Oja's algorithm for streaming PCA serving as a canonical example. For Oja's algorithm, \cite{nips21_oja} established a high-dimensional Gaussian approximation for the rank-one case based on a Hoeffding decomposition of the unnormalized matrix product. Later, \cite{uai25} extended this approach to entrywise Gaussian approximation for the rank-one iterates. However, due to the noncommutative nature of matrix multiplication, this Hoeffding-decomposition approach does not readily extend to the general-rank case. 
In the offline PCA setting, Gaussian approximation and bootstrap inference for spectral projectors have been extensively studied \cite{koltchinskii2017normal,naumov2019bootstrap,jirak2024quantitative}, with related Gaussian comparison and anti-concentration techniques developed by \cite{gotze2019large}. In particular, \cite{yuling} established distributional guarantees for heteroscedastic PCA under a spiked covariance model. These developments motivate the present work: can one obtain inference guarantees for streaming PCA in the general-rank case, beyond the spiked covariance model?

\subsection{Contributions}
Motivated by these open questions, we develop a refined analysis of Oja's dynamics. This analysis yields sharp operator-norm convergence rates and enables Gaussian approximation and inference.

\begin{table}[!htbp]
\centering
\begin{tabular}{cccc}
\toprule
Reference & Rank & Tail assumption & Operator-norm error rate \\
\midrule
\cite{jain16}   & $r=1$ & Bounded          & $\tfrac{\lambda_1 (1 \vee \nu_2)}{\Delta_{\min}}\sqrt{\tfrac{1}{n}}$ \\[0.6em]
\cite{han17}  & $r=1$ & Sub-Gaussian & $\tfrac{\lambda_1 \nu_2}{\Delta_{\min}}\sqrt{\tfrac{1}{n}}$ \\[0.6em]
\cite{nips21_oja} & $r=1$ & Finite moments           & $\tfrac{\lambda_1(1 \vee \nu_2^2)}{\Delta_{\min}}\sqrt{\tfrac{1}{n}}$ \\[0.6em]
\cite{colt21} & $r\geq1$ & Bounded          & $\tfrac{\lambda_1 (1 \vee \nu_2^2)}{\Delta_{\min}}\sqrt{\tfrac{r^2}{n}}$ \\[0.6em]
\textbf{Ours} & $r\geq1$ & Sub-Gaussian & $\tfrac{\lambda_1 (\nu_\infty\vee \nu_2)}{\Delta_{\min}}\sqrt{\tfrac{r}{n}}$ \\
\bottomrule
\end{tabular}
\smallskip
\caption{Comparison of operator-norm guarantees for  $\|\sin\Theta(U_n,U_*)\|$. 
Existing results are translated into our notation under  Assumption~\ref{assumption:X}; for bounds involving auxiliary  distribution-dependent or boundedness parameters, we display their sharp  sub-Gaussian specialization justified in Appendix~\ref{sec:cmp}. 
Logarithmic factors are suppressed. 
The displayed rate for our method concerns the nonzero-tail regime of Theorem~\ref{thm:convergence_rate_optimal}. When $\rank(\Sigma)=r$, Theorem~\ref{thm:rankr} gives geometric convergence to zero.}
\label{tab:opnorm-comparison}
\end{table}

\paragraph{Sharp operator-norm convergence.} \cref{tab:opnorm-comparison} summarizes our rate alongside prior work, where $\nu_2 := (\lambda_{r+1 \sim d} \big/ \lambda_{1\sim r})^{1/2}$ and $\nu_\infty := (\lambda_{r+1}/\lambda_1)^{1/2}$ denote tail-to-signal ratios (formally defined in \cref{sec:setup}). Under sub-Gaussian data and Haar initialization, our bound contains no additional non-vanishing remainder term, in contrast to \cite{han17,nearopt22,opt23}, whose residuals prevent adaptivity to a vanishing tail spectrum.  
The leading factor $\nu_\infty \vee \nu_2$ separates two qualitatively different regimes: the dense-tail regime ($\nu_2$ dominates), which covers the spiked covariance model, and the sparse-tail regime ($\nu_\infty$ dominates), arising when the tail spectrum is highly concentrated. This dichotomy is sharp: a matching lower bound shows the dependence on $\nu_\infty \vee \nu_2$ is unimprovable. As corollaries, our rate matches the minimax rate of \cite{anru} for the spiked covariance model up to a logarithmic factor, and reduces to the sharpest known rank-one rates of \cite{han17} when $r=1$ (where $\nu_\infty \le \nu_2$).

The improvement rests on two ingredients: a refined one-step decomposition, sharper than \cite{nearopt22}; and matrix Freedman's inequality in place of the matrix Azuma used by \cite{nearopt22,opt23}, which ignores conditional variance. 

\paragraph{Gaussian approximation and inference.} Building on a linearization result, we establish a high-dimensional Gaussian approximation for Oja's error with an explicit limiting covariance. Unlike \cite{nips21_oja}, whose Hoeffding-decomposition approach does not extend beyond rank one, our proof handles general $r \ge 1$ and recovers the rank-one result of \cite{nips21_oja}.
We further establish a row-wise Gaussian approximation over convex sets for the aligned difference, which in the spiked covariance model recovers the covariance structure of \cite{yuling} up to a logarithmic factor. To make these results operational, we develop an online multiplier bootstrap procedure that consistently estimates the limiting law.

\subsection{Notation.}

For a positive integer $n$, we write $[n]=\{1,\ldots,n\}$. For any $a,b\in\RR$, let $a\wedge b:=\min\{a,b\}$ and $a\vee b:=\max\{a,b\}$. We write $f(n)\lesssim g(n)$ if there exists a constant $C>0$ such that $|f(n)|\le Cg(n)$ for all sufficiently large $n$.
The sub-Gaussian norm of a random variable $X \in \RR$ is defined as $\|X\|_{\psi_2}:= \inf \{t >0: \EE \exp(X^2/t^2)\leq~2\}$. For a random vector $X \in \RR^d$, we define $\|X\|_{\psi_2}:=\sup_{\norm{v}=1} \|v^\top X\|_{\psi_2}$.
For any matrix $M$, let $\norm{M}=\Onorm{M}$ denote the operator norm, and let $\Fnorm{M}$ denote the Frobenius norm. We denote by $\OO_{d,r}$ the set of $d\times r$ matrices with orthonormal columns, and by $\Haar{\OO_{d,r}}$ the Haar measure on $\OO_{d,r}$.
For $U,U_* \in \OO_{d,r}$, let $\Theta(U,U_*)=\diag\rbr{\theta_i}_{i\in[r]}$,  $\theta_i=\arccos(\sigma_i(U^\top U_*))$ denote the principal angles, where $\sigma_i(\cdot)$ is the $i$-th largest singular value, and set $\sin\Theta(U,U_*):=\diag\rbr{\sin\theta_i}_{i\in[r]}$. Finally, for any set $\cA \subseteq \RR^r$, denote $\cN(\mu, \Sigma)\{\cA\}:= \PP(W \in \cA)$ where $W \sim \cN(\mu, \Sigma)$.

\section{Problem formulation and preliminaries}
\label{sec:setup}
This section sets up the streaming PCA model and introduces the notation used in Oja's dynamics and limiting distributions. We impose the following standard sub-Gaussian assumption, as in \cite{anru,han17,nearopt22,opt23}.
\begin{assumption}
    $X_1,\ldots,X_n$ are i.i.d. copies of $X=\Sigma^{1/2}\widetilde X$ where $\Sigma \in \RR^{d\times d}$ is positive semi-definite, $\tX \in \RR^d$ satisfies $\EE \tX= 0$, $\EE \tX \tX^\top = I_d$ and $\|\tX\|_{\psi_2}\leq C_{\psi}\lesssim 1$.
    \label{assumption:X}
\end{assumption}

Let $\Sigma = U_* \Lambda_* U_*^\top \!\!\!+\! U_{*,\perp}\Lambda_{*,\perp}U_{*,\perp}^\top$ be the eigendecomposition with $U_*\!\!\in\!\OO_{d,r}$, $U_{*,\perp}\!\!\in\!\!\OO_{d,d-r}$, $\Lambda_*=\operatorname{diag}(\lambda_1,\ldots,\lambda_r)$,
$\Lambda_{*,\perp}=\operatorname{diag}(\lambda_{r+1},\ldots,\lambda_d)$, and
$\lambda_1\geq\cdots\geq\lambda_d\geq 0$. 
The target of estimation is the column space of $U_*$. We use the shorthand $\lambda_{a\sim b}:= \sum_{i=a}^b \lambda_i$. Define the eigengaps $\Delta_{lj}:= \lambda_{j} - \lambda_{r+l}$ for $l \in [d-r], j \in [r]$ with extremes $\Delta_{\max}:=\lambda_1-\lambda_d$ and $\Delta_{\min}:=\lambda_r-\lambda_{r+1}$. Throughout, we assume $\Delta_{\min}>0$.

\paragraph{Scale-invariant quantities.}
A useful structural property of Oja's iteration is its invariance under the following simultaneous rescaling:
\begin{align}
  \text{scale transformation:} \qquad X \gets c X, \quad \eta \gets c^{-2} \eta, \quad c>0. \label{eqn:scale_transform}
\end{align}
Indeed, under~\eqref{eqn:scale_transform},
$\eta X_t X_t^\top$ is unchanged, and hence the entire Oja trajectory $\{U_t\}_{t\ge0}$.
Consequently, any subspace estimation error bound should depend on the problem only through quantities that are invariant under \eqref{eqn:scale_transform}. 
We refer to such quantities as \emph{scale invariant}. The \emph{scale-invariant} quantities used throughout the paper are
\begin{align}
  \textstyle \teta := \eta \Delta_{\min}, \;\; \nu_2 := \sbr{\frac{\lambda_{r+1 \sim d}}{\lambda_{1\sim r}}}^{1/2}\!\!, \;\; \nu_\infty := \sbr{\frac{\lambda_{r+1}}{\lambda_1}}^{1/2}\!\!, \;\; \bnu := \max(\nu_2, \nu_\infty), \;\; \kappa_* := \frac{\lambda_1}{\Delta_{\min}}. \label{eqn:scale_invariant_quantities}
\end{align}
Here $\nu_2$ and $\nu_\infty$ are respectively the Schatten-$2$ and Schatten-$\infty$ cases of the general tail-to-signal ratio $\nu_p:= \|\Lambda_{*,\perp}^{1/2}\|_{\cS_p} \|\Lambda_{*}^{1/2}\|_{\cS_p}^{-1}$
where $\|\cdot\|_{\cS_p}$ denotes the Schatten-$p$ norm. We have $\nu_\infty\leq 1$ and $\nu_\infty/\sqrt{r} \leq \nu_2 \leq \nu_\infty \sqrt{d-r}$ directly from the definitions.
All these quantities are unchanged under~\eqref{eqn:scale_transform}, since eigenvalues of $\Sigma$ would scale by~$c^2$.

\paragraph{Tangent matrices.} Our analysis tracks the iterates $U_t$ through an associated \emph{tangent matrix} $\cT_t$, which serves both to bound the subspace estimation error and to characterize the row-wise aligned difference.
For any $U \in \OO_{d,r}$, define 
the \emph{cosine}, \emph{sine}, and \emph{tangent matrix of $U$ (relative to $U_*$)} by
\begin{equation}
  \cC_U := U_*^\top U \in \RR^{r\times r}, \quad \cS_U := U_{*,\perp}^\top U \in \RR^{(d-r)\times r}, \quad \cT_U := \cS_U \cC_U^{-1} \text{ (when $\cC_U$ is invertible)}. \label{eqn:cos_sin_tan_defn}
\end{equation}
When $r=1$, these reduce to the usual cosine, sine, and tangent of the angle between $U$ and $U_*$. For iterates $U_t$, we use the shorthand $\cT_t := \cT_{U_t}$. Let $\HH:=\RR^{(d-r)\times r}$ equipped with the Frobenius inner product $\Finner{\cdot}{\cdot}$, write $\cI$ for the identity on $\HH$, and define the \emph{gap operator} $\cL:\HH\to\HH$ by
\begin{align}
  [\cL M]_{lj}:=\Delta_{lj}M_{lj}, \qquad \text{for any } M \in \HH. \label{eqn:gap_operator}
\end{align}
The operator $\cL$ enters the analysis through the contraction $\cI-\eta\cL$ in the recursion $\cT_{t+1}=(\cI-\eta\cL)\cT_t+\text{noise}+\text{remainder}$.

\paragraph{Signal--noise coordinates and covariance of the outer product.}
For each observation $X=\Sigma^{1/2}\tX$, define the signal/noise coordinates and their whitened counterparts by
\begin{align}
  Y:= U_*^\top X, \quad Z:= U_{*,\perp}^\top X, \quad \tY:= U_*^\top \tX, \quad \tZ:= U_{*,\perp}^\top \tX. \label{eqn:Y_Z_defn}
\end{align}
Then $Y=\Lambda_*^{1/2}\tY$, $Z=\Lambda_{*,\perp}^{1/2}\tZ$, $\EE \tY \tY^\top = I_r$, $\EE \tZ \tZ^\top = I_{d-r}$ and $\EE ZY^\top = \EE \tZ \tY^\top = 0$.

The random elements $ZY^\top, \tZ \tY^\top \in \HH=\RR^{(d-r)\times r}$ play a vital role in the Gaussian approximation. We therefore define their covariance operators $\fC, \tfC: \HH\to \HH$ by
\begin{align}
  \fC(B):=\EE [ \Finner[\big]{B}{ZY^\top} ZY^\top], \quad \tfC(B):=\EE [\langle B, \tZ\tY^\top \rangle_{\mathrm{F}} \tZ\tY^\top], \qquad \forall B \in \HH. \label{eqn:cov_defn}
\end{align}
Equivalently, $\fC$ and $\tfC$ can be identified with the covariance matrices of the vectorizations $\vecc(ZY^\top)$ and $\vecc(\tZ\tY^\top)$. We adopt this linear-operator formulation to emphasize the Hilbert-space geometry of $\HH$ and to keep the notation coordinate-free. This perspective is also standard in functional data analysis \cite{hsing}. For a linear operator $A:\HH\to \HH$, $\|A\|_{\mathrm{HS}}$ denotes its Hilbert--Schmidt norm, equivalently the Frobenius norm of its matrix representation under the standard basis of $\HH$.

\begin{example}[Spiked covariance model] \label{ex:spiked}
  Consider the spiked covariance model $\Sigma \!=\! U_* S^2 U_*^\top \!+\! \sigma^2 I_d$ where $U_* \!\in\! \OO_{d,r}$, $S\!=\!\diag (s_1,\ldots,s_r)$ with $s_1\!\!\geq\!\! \ldots \!\!\geq \!\!s_r\!\! > \!\!0$, $\sigma\!\! >\!\! 0$. Then $\lambda_1\! =\! s_1^2 + \sigma^2$, $\lambda_{1 \sim r}\! = \!\sum_{i=1}^r (s_i^2 + \sigma^2)$, $\lambda_{r+1} \!=\! \sigma^2$ and $\lambda_{r+1 \sim d}\! =\! (d-r) \sigma^2$. Under Gaussian $X$, we have $Y\!\perp\!Z$ and $\tfC\!=\!\cI$.
  % When $\lambda_1 \lesssim \lambda_r$ and $r \lesssim d-r$, 
  % \red{anru} established the minimax rate for subspace estimation is
  % \begin{align}
  %   \inf_{\hU} \sup_{\Sigma} \EE \norm[\big]{\sin\Theta(\hU, U_*)} \asymp \sqrt{\frac{d}{n}} \rbr{\frac{\sigma}{s_r}+ \frac{\sigma^2}{s_r^2}} \wedge 1.
  %   \label{eqn:anru_minimax_rate}
  % \end{align}
  % \red{$\Sigma$ class}

  % \red{yuling}
  % \begin{align*}
  %     \Sigma_{U,m} = n^{-1} \sigma^2S^{-2}\sbr{I_r + \sigma^2 S^{-2}}
  % \end{align*}
\end{example}

\section{Main results}
This section presents our main theoretical results. Section~\ref{sec:main_convergence}
establishes a sharp operator-norm convergence rate, a nearly matching lower bound
(up to logarithmic factors, under a mild nondegeneracy condition), and a
linearization of the final iterate. Section~\ref{sec:main_clt} uses this
linearization to establish Gaussian approximations for the subspace error,
both globally in Frobenius norm and row-wise. Section~\ref{sec:main_bs}
develops an online multiplier bootstrap and proves its consistency.

\subsection{Sharp convergence and linearization}
\label{sec:main_convergence}
Recall $\teta$, $\bnu$ and $\kappa_*$ from \eqref{eqn:scale_invariant_quantities}. Throughout this section, define the noise floor $\tau_* = \kappa_*\bar\nu\sqrt{\widetilde\eta\,r\log n}$
Under the stepsize $\teta=c_\eta(\log n)/n$, this becomes
\begin{align}
  \textstyle \tau_* = \frac{\lambda_1}{\Delta_{\min}} (\nu_\infty\vee\nu_2) \sqrt{\frac{c_\eta r\log^2 n}{n}}. \label{eqn:tau_star}
\end{align}

\begin{theorem}[Sharp operator-norm convergence] \label{thm:convergence_rate_optimal}
Under Assumption~\ref{assumption:X}, assume $U_0 \sim \Haar{\OO_{d,r}}$ and
\begin{align}
    \nu_\infty &\ge c_{\nu} n^{-C_\nu} \quad\text{for constants } C_\nu \ge 0, c_\nu >0. \label{eqn:cond_nu}
\end{align}
Fix $\varepsilon \in [0,1/2)$ and set 
\begin{align}
  \teta = c_\eta \tfrac{\log n}{n}, \quad\text{where}\quad c_\eta\geq c_\eta^* := \rbr[\big]{1  - \varepsilon + \tfrac{3 C_\nu}{2}} \rbr[\big]{1 + \tfrac{\log^{-3} n + \log 2}{\varepsilon \log n + \log\log n}}. \label{eqn:cond_step_sz}
\end{align}
For any $\delta_p \in (0, 1/e)$, if 
\begin{align}
    n^{1-2\varepsilon} &\ge C_1 c_\eta \kappa_*^2 r^3 (1 + \nu_\infty d)\, (\log n)^4 \tfrac{\log(1/\delta_p)}{\delta_p^2} \label{eqn:cond_sample_sz}
\end{align}
for some constant $C_1 > 0$ sufficiently large, then with probability at least $ 1 - \delta_p$,
\begin{align}
  \norm{\sin\Theta(U_k, U_*)} \lesssim  \tau_* = \tfrac{\lambda_1}{\Delta_{\min}} (\nu_\infty \vee \nu_2) \sqrt{\tfrac{c_\eta r \log^2 n}{n}}, \qquad \text{for all } \floor{\tfrac{c_\eta^*}{c_\eta} n} \leq k \leq n.
  \label{eqn:convergence_rate}
\end{align}
\end{theorem}
We highlight several implications of Theorem~\ref{thm:convergence_rate_optimal}.

\begin{itemize}[leftmargin=*]
  \item \textbf{Rate optimality.} The factor $\bnu = \nu_\infty \vee \nu_2$ in~\eqref{eqn:convergence_rate} 
  captures the worst-case contribution of the tail spectrum across two 
  regimes. The \emph{dense-tail regime}, where $\nu_2$ dominates, includes 
  the spiked covariance model (\cref{ex:spiked}) commonly studied in the 
  statistical literature \cite{anru,yuling}; here $\nu_2 \gtrsim \nu_\infty$ 
  whenever $\lambda_1 \lesssim \lambda_r$ and $r \lesssim d-r$. The \emph{sparse-tail regime}, where 
  $\nu_\infty$ dominates, arises when the tail spectrum is highly 
  concentrated, e.g. $\lambda_1=\cdots=\lambda_r>\lambda_{r+1}>0=\lambda_{r+2}=\cdots=\lambda_d$,
  in which case $\nu_\infty = \sqrt{r}\,\nu_2$.
  Similar regime-dependent behavior has also been observed in heteroscedastic PCA \cite{anru} and high-dimensional linear regression \cite{zijian}. The dependence on $\bnu$ is unimprovable: \cref{prop:lower_bd} establishes a matching lower bound, up to a factor of $\sqrt{\log n}$, under a mild nondegeneracy condition.
  In the spiked covariance setting (\cref{ex:spiked}) with $\lambda_1 \lesssim \lambda_r$ and $r \lesssim d-r$, our rate matches the minimax rate $\textstyle \sqrt{\tfrac{d}{n}} \rbr{\frac{\sigma}{s_r}+ \frac{\sigma^2}{s_r^2}}$ from \cite{anru} to a logarithmic factor. 
  Since $\nu_\infty \leq \sqrt{r}\nu_2$ in general, the regime 
  distinction collapses when $r = O(1)$ and yields $\bnu \asymp \nu_2$; in 
  particular, the rank-one case ($r=1$) is governed entirely by $\nu_2$ (see \cref{tab:opnorm-comparison}).
  
  \item \textbf{Dependence on $C_\nu$.} The step size requirement~\eqref{eqn:cond_step_sz} depends on $C_\nu$ through the lower bound~\eqref{eqn:cond_nu} on $\nu_\infty$. This is intuitive: Haar initialization yields $\norm{\cT_0}\lesssim \delta_p^{-1}\sqrt{rd\log(1/\delta_p)}$ with probability $\geq 1-\delta_p$, so reaching the target rate within $n$ samples requires a larger step size when $\nu_\infty$ is smaller. In the limit $\nu_\infty=\nu_2=0$, every update lies in the signal subspace, yet Haar initialization always leaves nonzero error in $U_n$ for any finite $n$; this is reflected in $c_\eta^{*}\to\infty$ as $C_\nu\to\infty$. This $C_\nu$ dependence is implicit in prior work but not made explicit there: \citep[Theorem~3]{han17} requires $\nu_\infty\gtrsim 1$ with $c_\eta=2$ at $r=1$, and \citep[Remark~4.4]{nearopt22} notes that the remainder term is non-negligible and the constant $C^{*}$ in their Theorem~4.6 with $c_\eta=3/2$ depends on $\lambda_d^{-1}$. Earlier works \cite{jain16,colt21,nips21_oja} avoid this dependence because their target rate effectively sets $C_\nu=0$ for $\nu_\infty \asymp 1$ (see \cref{tab:opnorm-comparison}), and hence are not adaptive to a small tail spectrum.

  \item \textbf{Dependence on $c_\eta$.} The error bound~\eqref{eqn:convergence_rate} is monotone in $c_\eta$, matching the standard intuition that smaller step size yields smaller stochastic error. The step size lower bound \eqref{eqn:cond_step_sz} $c_\eta^* \to 1-\epsilon + \tfrac{3C_\nu}{2}$ as $n \to \infty$. When $C_\nu=0$, as assumed in \cite{han17}, one can take $c_\eta = 1$.
  
  \item \textbf{Weaker sample-size requirement.} The condition~\eqref{eqn:cond_sample_sz} is weaker than the corresponding rank-one condition
  in \cite{nearopt22}: the dependence on $d$ enters only through $d\nu_\infty$, which can be much smaller than $d$ when $\nu_\infty$ is small. Ignoring logarithmic factors, when $r,\kappa_*,c_\eta\lesssim 1$, our condition reduces to $n^{1-2\varepsilon}\gtrsim(1+d\nu_\infty)\delta_p^{-2}$, an improvement over $n^{1-2\varepsilon}\gtrsim d\,\delta_p^{-2}$ in \cite{nearopt22} when $\nu_\infty$ is small.
\end{itemize}

When $\rank(\Sigma)=r$, we have $\tau_\ast=0$. However, for any well-defined $\cT_0\neq 0$, the error remains nonzero at every finite iteration. Thus, Theorem~\ref{thm:convergence_rate_optimal} cannot extend to the exact-rank case by simply setting $\nu_2=\nu_\infty=0$. Instead, the error converges geometrically to zero.
\begin{theorem}[Exact-rank geometric convergence] \label{thm:rankr}
Under Assumption~\ref{assumption:X}, suppose that $\operatorname{rank}(\Sigma)=r$ and
$U_0\sim\operatorname{Haar}(\mathbb O_{d,r})$.
Set $\eta=\frac{c_\eta\log n}{n\lambda_r}$ for $c_\eta\ge \frac12$.
For any $\varepsilon\in[0,1/2)$ and $\delta_p\in(0,1/e)$, if
\begin{equation}
  n^{1-2\varepsilon} \ge C_2 c_\eta\kappa_\ast^2 r^2(\log n)^4 \tfrac{\log(1/\delta_p)}{\delta_p^2} \label{eqn:cond_sample_sz_rankr_thm}
\end{equation}
for some constant $C_2>0$ sufficiently large, then, with probability at least $1-\delta_p$,
\[
    \|\sin\Theta(U_k,U_\ast)\|
    \lesssim
    \tfrac{\sqrt d\,n^{1/2-\varepsilon}}
    {\sqrt{c_\eta}\,\kappa_\ast\sqrt r\,(\log n)^2}
    \exp\big\{
        -(1-\tfrac{1}{n^\varepsilon\log n})
        k\widetilde\eta
    \big\},
    \qquad k=0,1,\ldots,n,
\]
where $\teta=\eta\lambda_r=c_\eta(\log n)/n$. In particular,
\[
    \|\sin\Theta(U_n,U_\ast)\|
    \lesssim
    \tfrac{\sqrt d}
    {\sqrt{c_\eta}\,\kappa_\ast\sqrt r\,(\log n)^2}
    n^{
        1/2-\varepsilon
        -c_\eta\left(1-(n^\varepsilon\log n)^{-1}\right)
    }.
\]

\end{theorem}

\begin{remark}[No finite-time exact recovery under exact rank]
  The distinction between Theorem~\ref{thm:convergence_rate_optimal} and Theorem~\ref{thm:rankr} is intrinsic. When $\rank(\Sigma)=r$, we have $Z_k=0$ almost surely, and the tangent recursion derived in Appendix~\ref{sec:oja_bdd} gives $\cT_{k+1} = \cT_k (I_r+\eta Y_{k+1}Y_{k+1}^\top)^{-1}$. Since $I_r+\eta Y_{k+1}Y_{k+1}^\top \succ 0$ for $\eta >0$, $\cT_k=0$ if and only if $\cT_0=0$ for every finite $k$. Thus, under a generic initialization, the error cannot reach zero in finitely many iterations even though $\tau_\ast=0$. Theorem~\ref{thm:rankr} therefore gives the appropriate form of the result: geometric convergence of the error to zero.
\end{remark}

We now return to the nonzero-tail setting of Theorem~\ref{thm:convergence_rate_optimal} and establish a linearization of the final tangent iterate. This representation underlies the distributional results in Section~\ref{sec:main_clt}--\ref{sec:main_bs}.
Recall from \eqref{eqn:cos_sin_tan_defn}, \eqref{eqn:gap_operator} and \eqref{eqn:Y_Z_defn} 
the tangent matrix $\cT_n$, the signal--noise coordinates
$Y_i=U_*^\top X_i$, $Z_i=U_{*,\perp}^\top X_i$, and the gap operator
$\cL$ given by $[\mathcal L M]_{\ell j}=(\lambda_j-\lambda_{r+\ell})M_{\ell j}$.

\begin{proposition}[Linearization] \label{prop:linearization}
  Assume the conditions in Theorem~\ref{thm:convergence_rate_optimal} hold with $c_\eta \geq c_\eta^* + \frac{1}{2}$. Then, with probability at least $1-\delta_p$,
  \begin{align*}
    \textstyle \big\|\cT_n - \eta \sum_{i=1}^n (\cI-\eta \cL)^{n-i} (Z_iY_i^\top)\big\| \lesssim \tau_* \sbr[\big]{\kappa_* \sqrt{r \teta \log n}\; (1+\bnu)} \lesssim \frac{\tau_*}{n^{\varepsilon} \log n}.
  \end{align*}
\end{proposition}

Under the additional conditions stated below, the linearization in Proposition~\ref{prop:linearization} also yields a nearly matching lower bound.
\begin{proposition}[Lower bound] \label{prop:lower_bd}
  Assume the assumptions in \cref{prop:linearization} hold. If $\kappa_* \lesssim 1$ and $Y \perp Z$, then with probability at least $\frac{1}{12}\sbr{1+o(1)}-\delta_p$, $\|\sin\Theta(U_n, U_*)\| \gtrsim \frac{\tau_*}{\sqrt{\log n}}$.
\end{proposition}
\cref{prop:lower_bd} matches \cref{thm:convergence_rate_optimal} up to logarithmic factors, and shows the dependence on $\bnu$ is unimprovable. The independence condition can be relaxed to a weaker nondegeneracy condition $\lambda_{\min}(\tfC)\gtrsim 1$ where $\lambda_{\min}(\tfC)$ is the minimum eigenvalue of $\tfC$ \eqref{eqn:cov_defn}.

% ------------------------------------------------------------
\subsection{Gaussian approximation}
\label{sec:main_clt}

We now establish high-dimensional Gaussian approximations for the global subspace estimation error. Recall the \emph{sine} and \emph{tangent} matrices $\cS_n$ and $\cT_n$ from \eqref{eqn:cos_sin_tan_defn}. 
The intuition is simple. In the rank-one case, $\sin\theta \approx \tan\theta$ when $|\tan \theta|$ is small. An analogous approximation holds in the rank-$r$ case: $\|\cS_n\|_{\mathrm{F}} \approx \|\cT_n\|_{\mathrm{F}}$ when $\norm{\cT_n}$ is small. Together with the identity $\|\sin \Theta(U_n, U_*)\|_{\mathrm{F}} = \|\cS_n\|_{\mathrm{F}}$ and the linearization in \cref{prop:linearization}, this gives the heuristic approximation
\begin{align}
  \textstyle \Fnorm{\sin \Theta(U_n, U_*)} = \Fnorm{\cS_n} \approx \Fnorm{\cT_n} \approx \Fnorm{\eta \sum_{i=1}^n (\cI-\eta \cL)^{n-i} (Z_iY_i^\top)}. \label{eqn:highD_clt_heuristic}
\end{align}
\cref{thm:highD_clt} makes this intuition rigorous by applying a high-dimensional Gaussian approximation to the normalized linearized error $\eta^{1/2} \sum_{i=1}^n (\cI-\eta \cL)^{n-i} (Z_iY_i^\top)$.

Recall $\lambda_{a \sim b}\!\!=\!\! \sum_{i=a}^b \lambda_i$, the covariance operator $\fC\!=\!\cov(ZY^{\top})$ from~\eqref{eqn:cov_defn} and its Hilbert-Schmidt norm $\HS{\fC}$. Define the cross-spectral term $\lambda_{\times}\!\!:=\!\! [\lambda_{1\sim r} \lambda_{r+1 \sim d}]^{1/2}$ and effective dimension $d_{\fC}\!:=\!\frac{\lambda_\times^4}{\HS{\fC}^2}$. 
Under Assumption~\ref{assumption:X}, $d_{\fC}\gtrsim 1$; see Lemma~\ref{lem:matrix_kth_moment}.
% of $\fC$ by
% \begin{align*}
%   \textstyle \lambda_{\times}:= [\lambda_{1\sim r} \lambda_{r+1 \sim d}]^{1/2}, \qquad \qquad d_{\fC}:=\frac{\lambda_\times^4}{\HS{\fC}^2}
% \end{align*}

\begin{theorem}[Gaussian approximation for the subspace error] \label{thm:highD_clt}
  Assume $\HS{\fC}>0$ and the conditions of \cref{prop:linearization} hold with $\delta_p\to0$. Let $\mathsf{G} \sim \cN(0, \Gamma)$ be a Gaussian random element in $\RR^{(d-r)\times r}$ with $\Gamma = \int_0^\infty e^{-t\cL} \fC e^{-t\cL} dt$.
  If
  \begin{align}
    \textnormal{(i)}\quad \teta \, \kappa_*^4 r^{3} d_{\fC}^{3/2} (\log n)^{3} \to 0, \qquad 
    \textnormal{(ii)} \quad \teta\, \kappa_*^4  r^{3} d_{\fC} \bnu^2 (\log n)^{3} \to 0, \label{eqn:highD_clt_cond}
  \end{align}
  then \;  $\textstyle \sup_{t \in \RR} \abs[\big]{\PP\cbr[\big]{ \eta^{-1} \Fnorm{\sin \Theta(U_n, U_*)}^2 \leq t} - \PP\cbr{\|\mathsf{G}\|_{\mathrm{F}}^2 \leq t}} \to 0$.
\end{theorem}

\begin{remark}[Gaussian case]
  When $Y \perp Z$ (in particular for Gaussian $X$), $\fC = \Lambda_* \otimes \Lambda_{*,\perp}$, $\HS{\fC} = \|\Lambda_*\|_{\mathrm{F}} \|\Lambda_{*,\perp}\|_{\mathrm{F}}$, and $d_{\fC}=  \frac{[\tr \Lambda_*]^2}{\Fnorm{\Lambda_*}^2}  \frac{[\tr \Lambda_{*,\perp}]^2}{\Fnorm{\Lambda_{*,\perp}}^2} \leq r(d-r)$.
  % \begin{align*}
  %   d_{\fC}=  \sbr[\bigg]{\frac{\tr \Lambda_*}{\Fnorm{\Lambda_*}}}^2  \sbr[\bigg]{\frac{\tr \Lambda_{*,\perp}}{\Fnorm{\Lambda_{*,\perp}}}}^2 \leq r(d-r).
  % \end{align*}
  Thus, when $c_\eta,\kappa_\ast,r \lesssim 1$, condition~\eqref{eqn:highD_clt_cond} allows $d^{3/2} = o(n)$ up to logarithmic factors.
\end{remark}

\begin{remark}[Comparison with the rank-one result of \citep{nips21_oja}]
  When $r=1$, $\fC$ reduces to the matrix $\fC = \EE \sbr{ZZ^\top (Y^2)} \in \RR^{(d-1)\times (d-1)}$,
  which coincides with the matrix $\MM$ in \citep[Theorem~1]{nips21_oja}. Three differences are worth highlighting.
  \begin{itemize}[leftmargin=*]
    \item \emph{Explicit limiting distribution}. In contrast to the $n$-dependent covariance $\bar \VV_n$ in \citep{nips21_oja}, our limit $\Gamma=\int_0^\infty e^{-t\cL} \fC e^{-t\cL} dt$ is explicit and is the unique solution to the Lyapunov equation $\cL\Gamma+\Gamma \cL=\fC$. The same covariance structure appears in constant-step-size SGD for convex optimization \citep{gaussian25}; here it arises from the nonconvex streaming PCA dynamics.

    \item \emph{Scale-invariant assumptions}. One key assumption \citep[Eq.~(9)]{nips21_oja} reads, in our notation, $\EE\|ZY^\top\|_{\mathrm{F}}^6/\HS{\fC}^3 = o(n)$. Under Assumption~\ref{assumption:X}, the moment estimates in Appendix~\ref{sec:tech_lemmas} give $\EE \|ZY^\top\|_{\rm F}^6 \lesssim \lambda_\times^6$, with equality up to constants in the Gaussian case. 
    Since $d_{\fC}^{3/2}=\lambda_\times^6/\|\fC\|_{\rm HS}^3$, this condition requires $d_{\fC}^{3/2}=o(n)$. When $r=1$ and $\kappa_\ast\lesssim1$, our condition~\eqref{eqn:highD_clt_cond}(i), with $\teta\asymp(\log n)/n$, requires $d_{\fC}^{3/2}(\log n)^4=o(n)$, matching the dependence on $d_{\fC}$ up to logarithmic factors. Another assumption in~\cite[Eq.~(8)]{nips21_oja} is $\|\fC\|_{\rm F}\gtrsim1$, which is not scale invariant under transformation~\eqref{eqn:scale_transform}; our conditions are stated entirely in scale-invariant terms. The remaining condition~\eqref{eqn:highD_clt_cond}(ii), which couples $d_{\fC}$ and $\bar\nu$, is implied by the assumptions of~\cite[Theorem~1]{nips21_oja}.

    \item \emph{Proof technique}. \citep{nips21_oja} relies on a Hoeffding decomposition of the unnormalized matrix product $[\prod_{i=1}^n (I_d + \eta X_i X_i^\top)] U_0$.
    Our proof linearizes the normalized iteration via \cref{prop:linearization}, which keeps the QR  step inside the analysis. For general $r>1$, the unnormalized Hoeffding approach can formally lead to the same Gaussian limit, but making it rigorous appears difficult: without normalization at the level of the expansion, the argument produces exponentially growing terms that are hard to control.
  \end{itemize}
\end{remark}

We next turn to row-wise Gaussian approximation. 
To go beyond the Frobenius norm, we must account for the rotational non-identifiability of the leading eigenspace: for any $R\in \OO_{r,r}$, the matrices $U$ and $UR$ span the same subspace, so without further eigenvalue separation within the signal block one can only estimate
$U_*$ up to a global rotation. We therefore study the
Procrustes-aligned error. For any nonsingular matrix $H\in\mathbb R^{r\times r}$
with SVD $H=V_1 \Lambda_H V_2^\top$, denote $\sgn(H):=V_1 V_2^\top \in \OO_{r,r}$. Then $\sgn(U^\top U_*) = \argmin_{R \in \OO_{r,r}} \Fnorm{UR-U_*}^2$, and the aligned difference $U\sgn(U^\top U_*)-U_*$ is the natural object for row-wise inference \citep{without2inf,abbe,mc_inference}.

The intuition is again based on the tangent linearization (\cref{prop:linearization}). Recall the \emph{cosine}, \emph{sine} and \emph{tangent} matrices $\cC_n$, $\cS_n$ and $\cT_n$ from \eqref{eqn:cos_sin_tan_defn}. By definition, $U_n^\top U_* = \cC_n^\top$, $U_n^\top U_{*,\perp} = \cS_n^\top$ and $\cS_n = \cT_n \cC_n$, which gives the exact decomposition
\begin{align*}
  U_n \sgn(U_n^\top U_*)- U_* = U_{*,\perp} \cT_n - [U_* + U_{*,\perp} \cT_n] [I_r  - \cC_n \cC_n^\top] + U_n [\sgn(\cC_n^\top) - \cC_n^\top]
\end{align*}
In the rank-one case, $|1-\cos \theta| \lesssim \tan^2 \theta$ when $|\tan \theta|$ is small. Analogously, in the rank-$r$ case: $\|I_r-\cC_n \cC_n^\top\|\lesssim \|\cT_n\|^2$ and $\|\sgn(\cC_n^\top) - \cC_n^\top\| \lesssim \|\cT_n\|^2$ when $\|\cT_n\|$ is small. Thus the aligned difference is heuristically \(U_{*,\perp} \cT_n\). 
Together with the linearization in \cref{prop:linearization}, this gives the heuristic approximation
\begin{align*}
  U_n \sgn(U_n^\top U_*)- U_*  \approx U_{*,\perp} \cT_n \approx U_{*,\perp} \eta \sum_{i=1}^n (\cI-\eta \cL)^{n-i} (Z_iY_i^\top).
\end{align*}
This motivates the following limiting row-wise covariance.

\paragraph{Limiting covariance.} 
For $u\in\RR^{d-r}$, let $\Gamma_u\in\RR^{r\times r}$ denote the $u$-projection of $\Gamma$, with entries
\begin{equation}\label{eqn:Gamma_u}
  \textstyle [\Gamma_u]_{j'j} \; = \sum_{l,l'=1}^{d-r}\frac{u_{l'}u_l\,\EE \big[[Z]_l\,[Y]_j\,[Z]_{l'}\,[Y]_{j'}\big]}{\Delta_{lj}+\Delta_{l'j'}}, \qquad \quad \forall j,j' \in [r].
\end{equation}
Here $[v]_k$ denotes the $k$-th coordinate of vector $v$; see Appendix~\ref{sec:clt_setup} for properties of $\Gamma_u$.
Let $\lambda_{\min}(\tfC)$ denote the minimum eigenvalue of $\tfC$ from~\eqref{eqn:cov_defn}. 
For $m\in[d]$, set $u_m:=U_{*,\perp}^\top e_m \in \RR^{d-r}$ and $\Sigma^*_{U,m}:=\eta\Gamma_{u_m}$.

\begin{theorem}[Row-wise Gaussian approximation] \label{thm:row_clt}
  Assume the conditions of \cref{prop:linearization} hold with $\delta_p \to 0$. Fix $m \in [d]$. If $\Lambda_{*,\perp}^{1/2}u_m \neq 0$, $\lambda_{\min}(\tfC) >0$,  $\teta \kappa_*^{6} r^{7/2} \lambda_{\min}^{-3}(\tfC) \to 0$
  % \begin{align*}
  %   &\sqrt{\teta} \kappa_*^{3} r^{7/4} \lambda_{\min}^{-3/2}(\tfC) \to 0
  % \end{align*}
  and
  \begin{align}
    \teta \kappa_*^4  r^{5/2} (\log n)^2 \bnu^2 \;(1+\bnu)^2  \frac{\lambda_1 \max (\norm{u_m}^2, \|U_*^\top e_m\|^2)}{\|\Lambda_{*,\perp}^{1/2} u_m \|^2} \lambda_{\min}^{-1}(\tfC) \to 0, \label{eqn:row_clt_cond1}
  \end{align}
  then \;  $\sup_{\cA \in \mathscr{A}^r} \abs{\PP \cbr{[U_n \sgn(U_n^\top U_*) - U_*]_{m,\cdot} \in \cA} - \cN(0, \Sigma^*_{U,m})\cbr{\cA}} \to 0$
  where $\mathscr{A}^r$  is the family of convex sets in $\RR^r$. 
\end{theorem}

\begin{remark}[Comparison with existing row-wise inference results]\;
  % We compare \cref{thm:row_clt} with existing distributional results for PCA.
  \begin{itemize}[leftmargin=*]
    \item \emph{Nondegeneracy.} Condition~\eqref{eqn:row_clt_cond1}, involving $\lambda_{\min}^{-1}(\tfC)$, is scale invariant under~\eqref{eqn:scale_transform} since $\tfC=\cov(\tZ\tY^\top)$ is defined in whitened coordinates. This condition is mild in standard models: for Gaussian data, $\tY\perp\tZ$ so $\tfC=\cI$ and $\lambda_{\min}(\tfC)=1$. It ensures that the leading term dominates higher-order remainders; related lower bound conditions appear in \citep{mc_inference,TC_inference,uai25}.   In \citep[Remark~2]{yuling}, an analogous nondegeneracy is encoded by assuming the noise variances are bounded away from zero, which prevents the limiting Gaussian covariance from degenerating.

    \item \emph{Rank-one case.} When $r=1$, $\fC$ and $\Gamma$ reduce to $(d-1)\times(d-1)$ matrices, with entries $\Gamma_{ll'} = \fC_{ll'}/(\Delta_{1l} + \Delta_{1l'})$, and \eqref{eqn:Gamma_u} collapses to $\Sigma_{U,m}^* = \eta e_m^\top U_{*,\perp} \Gamma U_{*,\perp}^\top e_m \in \RR$,
    recovering the limit of \citep[Proposition~1]{uai25}.

    \item \emph{Spiked covariance.} Under \cref{ex:spiked} with Gaussian design, so that $\tY \perp \tZ$, the independence formula (\cref{sec:clt_setup}) gives, for $\eta = \frac{c_\eta \log n}{n\Delta_{\min}}$,
    \begin{align*}
      \Sigma_{U,m}^* &= \tfrac{c_\eta (\log n) \|u_m\|^2}{2n } \cdot \tfrac{\sigma^2}{s_r^{2}} \sbr{I_r + \sigma^2 S^{-2}}.
    \end{align*}
    The row-wise covariance in~\citep{yuling} is $n^{-1}\sigma^2S^{-2}[I_r + \sigma^2 S^{-2}]$. Thus, for the $j$th coordinate, the ratio of our covariance to that of~\cite{yuling} is $\frac{c_\eta\log n}{2}\|u_m\|^2\frac{s_j^2}{s_r^2}$. Since $\|u_m\|\le 1$ and $\frac{s_j^2}{s_r^2} \le \kappa_\Delta$, this ratio is at most $\frac{c_\eta\log n}{2}\kappa_\Delta$. Here, the logarithmic factor comes from the constant step size, while $\kappa_\Delta$ arises from calibrating the step size to the smallest eigengap $\Delta_{\min}=s_r^2$.
  \end{itemize}
\end{remark}

\subsection{Multiplier bootstrap}
\label{sec:main_bs}
The Gaussian approximations in \cref{sec:main_clt} depend on unknown covariance operators. We now give an online multiplier bootstrap algorithm that estimates the laws without explicitly estimating the covariances. \cref{algm:bootstrap} runs, in parallel with Oja's algorithm, perturbed Oja recursions driven by independent multiplier weights $w_t^{(b)}$. 
This construction is different from the rank-one bootstrap of~\citep{nips21_oja}, and is motivated by linearization (\cref{prop:linearization}) as follows.

Defining \emph{cosine}, \emph{sine} and \emph{tangent} matrices for the bootstrap iterates relative to $U_*$ by $\cC_t^{(b)}=U_*^\top U_t^{(b)}$, $\cS_t^{(b)}=U_{*,\perp}^\top U_t^{(b)}$ and $\cT_t^{(b)}=\cS_t^{(b)} (\cC_t^{(b)})^{-1}$, a bootstrap analogue of \cref{prop:linearization} gives $\cT_n^{(b)}\approx \eta \sum_{i=1}^n w_i^{(b)}(\cI-\eta \cL)^{n-i} (Z_iY_i^\top)$. Since both \(U_n\) and \(U_n^{(b)}\) are close to \(U_*\), their relative subspace error is controlled by the difference of their tangent matrices:
\begin{align*}
  \textstyle \|\sin\Theta(U_n^{(b)},U_n)\|_{\mathrm{F}} \approx
  \|\cT_n^{(b)}-\cT_n\|_{\mathrm{F}} \approx \|\eta\sum_{i=1}^n (w_i^{(b)}-1)(\cI-\eta\mathcal L)^{n-i}(Z_iY_i^\top) \|_{\mathrm{F}} .
\end{align*}
Thus, conditionally on the data, the multiplier weights reproduce the same linearized fluctuation that appears in~\eqref{eqn:highD_clt_heuristic} (\cref{thm:bs_highD_clt}). An analogous argument for the aligned row-wise difference yields the row-wise bootstrap result (\cref{thm:bs_row_clt}).

\begin{algorithm}[!htb] 
\caption{Online multiplier bootstrap for Oja's algorithm}
\Input{$\cbr{X_t}_{t \in [n]}$, stepsize $\eta$, number of bootstrap replicates $B$, weights $\{w_{t}^{(b)}\}_{t \in [n], b \in [B]}$.}
\Output{Oja's solution $U$ and bootstrapped versions
$\{U^{(b)}\}_{b \in [B]}$.}

Draw $U_0 \sim \Haar{\OO_{d,r}}$\;
Initialize $U, U^{(1)},\ldots,U^{(B)} \gets U_0$\;

\For{$t=1,\ldots,n$}{
    Update $U \gets \qr{U + \eta X_t (X_t^\top U)}$\;
    \For{$b=1,\ldots,B$}{
        Update $U^{(b)} \gets \qr{U^{(b)} + w_{t}^{(b)} \eta X_t (X_t^\top U^{(b)})}$\;
    }
}
\label{algm:bootstrap}
\end{algorithm}

\begin{assumption}
  The multiplier weights $\{w_{t}^{(b)}\}_{t \in [n], b \in [B]}$ are i.i.d., independent of $\cbr{X_i}_{i\in[n]}$, and satisfy $\EE (w_{t}^{(b)}) = 1$, $\Var (w_{t}^{(b)}) = 1$ and $w_{t}^{(b)}\in [0,C_w]$ almost surely
  for some constant $C_w > 0$.
  \label{assumption:w}
\end{assumption}
For example, the two-point multiplier $w_t^{(b)}\sim 2\,\mathrm{Bernoulli}(1/2)$ satisfies Assumption~\ref{assumption:w}. Our theory is stated for bounded multipliers. Unbounded choices such as standard exponential weights can also be handled by a truncation argument, at the cost of additional logarithmic factors. In practice, we found that bounded multipliers performed better.

Let $\cF_X := \sigma(U_0,X_1,\ldots,X_n)$, and let $\PP_X$ denote the law of $(U_0,X_1,\ldots,X_n)$. Conditional on $\cF_X$, let $\PP^*$ and $\EE^*$ denote probability
and expectation with respect to the bootstrap multipliers. Let $U_n^{\bs}$ denote a generic bootstrap iterate generated by Algorithm~\ref{algm:bootstrap} with an independent multiplier sequence $\{w_t^{\bs}\}_{t\in[n]}$ satisfying Assumption~\ref{assumption:w}.

\begin{theorem}[Bootstrap consistency for the subspace error] \label{thm:bs_highD_clt}
  Assume Assumption~\ref{assumption:w} holds and suppose the conditions in \cref{thm:highD_clt} are satisfied. Then
  \begin{align*}
    \sup_{t \in \RR} \big|\PP^*\cbr[\big]{\eta^{-1} \|\sin \Theta(U_n^{\mathsf{bs}}, U_n)\|_{\mathrm{F}}^2 \leq t} - \PP\cbr[\big]{ \eta^{-1} \|\sin \Theta(U_n, U_*)\|_{\mathrm{F}}^2 \leq t}\big| \xrightarrow{\PP_X} 0.
  \end{align*}
\end{theorem}

\begin{theorem}[Row-wise bootstrap consistency] \label{thm:bs_row_clt}
  Assume Assumption~\ref{assumption:w} holds and suppose the conditions in \cref{thm:row_clt} are satisfied. Denote $R_n := \sgn(U_n^\top U_*)$. Then for any fixed $m \in [d]$,
  \begin{align*}
    \textstyle \sup_{\cA \in \mathscr{A}^r} \big| \PP^* \cbr[\big]{\sbr{\rbr{U_n^{\mathsf{bs}} \sgn(U_n^{\bs \top} U_n) - U_n} R_n}_{m,\cdot} \in \cA} - \PP \cbr{[U_n R_n - U_*]_{m,\cdot} \in \cA} \big| \xrightarrow{\PP_X} 0,
  \end{align*}
  where $\mathscr{A}^r$  denotes the set of all convex sets in $\RR^r$.
\end{theorem}

\begin{remark}[Non-asymptotic Berry--Esseen bound]
  Although Theorem~\ref{thm:bs_row_clt} is stated asymptotically, its proof can be strengthened to yield a non-asymptotic Berry--Esseen-type bound over convex sets. The Berry--Esseen step in Appendix~\ref{sec:prf_bs_row_clt} is already non-asymptotic; the $o_{\PP_X}(1)$ terms arise mainly from empirical moment and covariance bounds obtained via Markov's inequality. Replacing these with suitable Bernstein-type concentration bounds would give explicit finite-sample errors, at the cost of additional logarithmic factors in the conditions.
\end{remark}

\section{Main proof ideas} \label{sec:main_prf_sketch}
The proofs of our main results are built around a common analysis of Oja's dynamics in tangent coordinates. The key ideas include a refined convergence argument based on variance-sensitive concentration, a linearization of the final iterate that underlies both the lower bound and Gaussian approximation, and a corresponding linearization for the multiplier bootstrap. We provide a high-level overview of these arguments and their connections in Appendix~\ref{sec:prf_sketch}.

\section{Experiments} \label{sec:experiments}
We draw $d$-dimensional Gaussian random vectors $X_1,\ldots,X_n \overset{iid}{\sim} \cN(0,\Sigma)$ where $(U_*, U_{*,\perp})\sim \Haar{\OO_{d,d}}$, $ \lambda_1 =\ldots=\lambda_r = 1$ and $\lambda_{r+l} = (l+1)^{-\beta}$ for $l \in [d-r]$, for some $\beta>0$. We let the multiplier weights be i.i.d. from the two-point distribution $2\mathrm{Bernoulli}(1/2)$ and set $\eta = \frac{c_\eta \log n}{n\Delta_{\min}}$ with~$c_\eta = 2$.

\paragraph{Validation of \cref{thm:bs_highD_clt}.} We set $n=5000$, $d=500$, $r=3$, and consider $\beta \in \cbr{1,4}$. Following~\cite{nips21_oja}, we fix the Haar initialization $U_0$. For the sampling distribution, we generate $200$ independent datasets and compute $\eta^{-1}\|\sin\Theta(U_n,U_\ast)\|_{\rm F}^2$. For the bootstrap distribution, we fix a dataset $\{X_i:i\in[n]\}$ and run Algorithm~\ref{algm:bootstrap} with $B=200$ independent bootstrap sequences, computing $\eta^{-1}\|\sin\Theta(U_n^{(b)},U_n)\|_{\rm F}^2$. 
Figure~\ref{fig:highD_qq} compares the corresponding empirical quantiles through Q--Q plots. A larger \(\beta\) yields faster tail decay and hence smaller values of
\(\eta^{-1}\|\sin\Theta(U_n,U_*)\|_{\mathrm F}^2\). The bootstrap approximation is visibly closer for the faster-decaying tail.

\begin{figure}[!htb]
    \centering

    \begin{subfigure}[t]{0.48\textwidth}
        \centering
        \includegraphics[width=\textwidth]{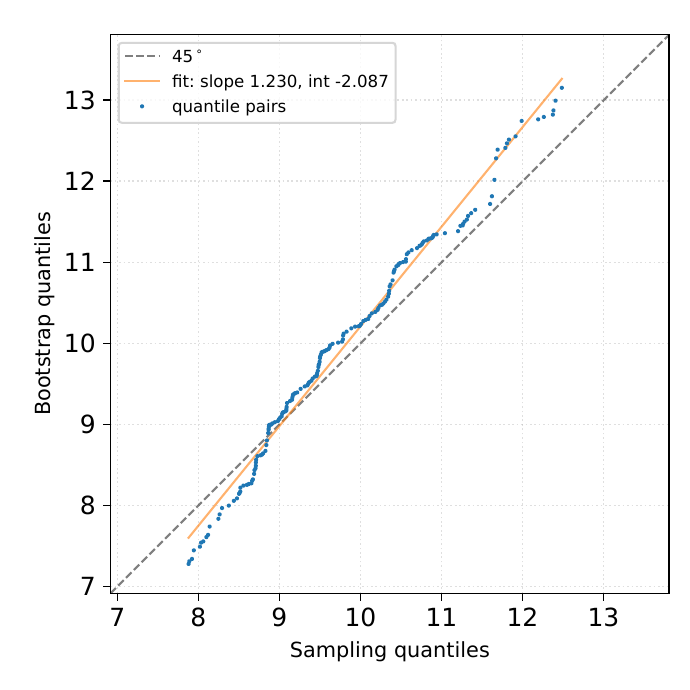}
        \caption{$\beta=1$.}
        \label{fig:first}
    \end{subfigure}
    \hfill
    \begin{subfigure}[t]{0.48\textwidth}
        \centering
        \includegraphics[width=\textwidth]{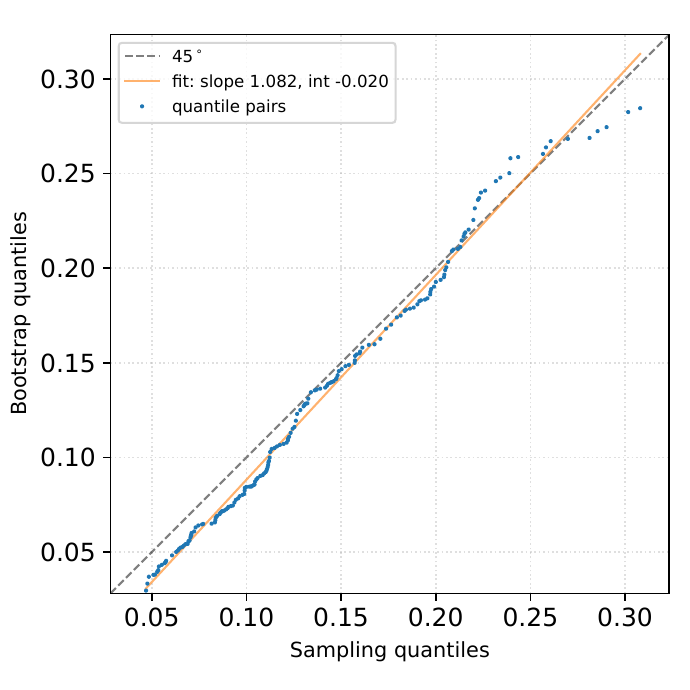}
        \caption{$\beta=4$.}
        \label{fig:second}
    \end{subfigure}

    \caption{Q--Q plots comparing the bootstrap and sampling distributions for  $n=5000$, $d=500$, $r=3$, $c_\eta = 2$ and $\beta \in \{1,4\}$.}
    \label{fig:highD_qq}
\end{figure}

\paragraph{Validation of \cref{thm:bs_row_clt}.} We set $n=2000$, $d=100$, $r=3$  and $\beta \in \{0.5,1,4\}$. We follow the protocol of \cite{yuling}. For each of 200 independent Monte Carlo trials, we run Oja's algorithm
together with $B=200$ multiplier-bootstrap trajectories. For each row $m\in[d]$, we align the bootstrap iterates to \(U_n\), form the bootstrap residuals via $[U_n^{(b)}\sgn(U_n^{(b)\top}U_n)-U_n]_{m,\cdot}$,
and estimate their covariance by the sample covariance $\hSigma_{m}^{\bs}$ across bootstrap replicates. We then construct the row-wise ellipsoidal confidence region
$\textstyle
  \mathrm{CR}_{m,1-\alpha}^{\bs} =
  \{z\in\mathbb R^r:
  (z-[U_n]_{m,\cdot})^\top [\hSigma_{m}^{\bs}]^{-1}
  (z-[U_n]_{m,\cdot}) \le \chi^2_{r,1-\alpha}
  \}$,
with \(\alpha=0.05\). We report the empirical probability that $[U_* \sgn(U_*^\top U_n)]_{m,\cdot}$ is covered by $\mathrm{CR}_{m,1-\alpha}^{\bs}$.
Table~\ref{tab:row_bs_avg} reports the average coverage rates; row-specific coverage and Q--Q diagnostics are provided in Appendix~\ref{sec:additional_experiments}.

\begin{table}[!htb]
\centering
\caption{Empirical coverage rates for different \(\beta\)'s over 200 Monte Carlo trials.}
\label{tab:row_bs_avg}
\begin{tabular}{cccc}
\toprule
 & \(\beta=0.5\) & \(\beta=1\) & \(\beta=4\) \\
\midrule
Mean$\pm$Std coverage & $0.914 \pm 0.027$ & $0.930 \pm 0.028$ & $0.948 \pm 0.021$ \\
\bottomrule
\end{tabular}
\end{table}

\section{Discussion}
We established sharp operator-norm convergence guarantees for streaming \(r\)-PCA via Oja's algorithm under sub-Gaussian data. The rate adapts to both dense-tail and sparse-tail regimes, and is matched by a lower bound up to logarithmic factors. A key ingredient is a linearization of the tangent matrix \(\cT_n\), which enables high-dimensional Gaussian approximation for the global subspace error and a row-wise Gaussian approximation over convex sets for the aligned difference. To make these distributional results operational, we developed an online multiplier bootstrap procedure that consistently approximates the limiting laws without explicitly estimating the covariance operators.

Several directions remain open. First, our current theory is limited to the constant-step-size case;
it would be useful to extend the theory to broader stepsize schedules. Second, the linearization and bootstrap techniques developed here may be applicable to streaming PCA under additional structural constraints, such as sparsity \citep{nips_sparsePCA}, Markovian dependence \citep{nips_markovian_pca}, or quantization \citep{nips_LowPrecisionStreamingPCA}. Finally, it would be interesting to develop analogous inference procedures for other nonconvex stochastic approximation problems, such as online tensor learning \citep{dong_tensor}.

\section*{Acknowledgments}
This research is supported by NIH grants. We thank the reviewers for many helpful comments. 

% \section*{References}

% References follow the acknowledgments in the camera-ready paper. Use unnumbered first-level heading for
% the references. Any choice of citation style is acceptable as long as you are
% consistent. It is permissible to reduce the font size to \verb+small+ (9 point)
% when listing the references.
% Note that the Reference section does not count towards the page limit.
\medskip

{
\small

\bibliographystyle{plainnat}
\bibliography{reference_neurips}
}

%%%%%%%%%%%%%%%%%%%%%%%%%%%%%%%%%%%%%%%%%%%%%%%%%%%%%%%%%%%%

\appendix

\numberwithin{equation}{section}
% \numberwithin{lemma}{section}

\section{Additional comparison results}
\label{sec:cmp}
\subsection{Error bounds from previous work under sub-Gaussian assumption}
In this section, we justify the entries in Table~\ref{tab:opnorm-comparison}. 
The results in \citep{nips21_oja,jain16,colt21,han17} are stated in terms of 
different problem-dependent quantities. We first translate these quantities into 
the notation of this paper, and then derive tight sub-Gaussian upper bounds for 
them under Assumption~\ref{assumption:X}. These bounds are matched, up to constants, in the 
Gaussian case.

\begin{lemma} \label{lem:compare}
    Suppose Assumption~\ref{assumption:X} holds and let $\lambda_1 \geq \ldots \geq \lambda_d \geq 0$ be the eigenvalues of $\Sigma$. 

    \begin{enumerate}[label=\textup{(\roman*)}, leftmargin=*]
        \item Denote $\cV= \|\EE [(XX^\top -\Sigma)^2]\|$. Then $\cV \lesssim C_\psi^4 \lambda_1 \tr \Sigma$. The bound is tight for Gaussian distribution.
        
        \item For any $\alpha \geq 1$, denote $\cM_\alpha = \EE [\|XX^\top - \Sigma\|^\alpha]^{1/\alpha}$. Then $\cM_1 \leq \cM_2 \lesssim C_\psi^2 [\tr \Sigma]$. The bound is tight for Gaussian distribution.
    \end{enumerate}

\end{lemma}
\begin{proof}
    See Appendix~\ref{sec:prf_compare}.
\end{proof}

\paragraph{\citep{nips21_oja}} Using the bound for $\cM_2$ from \cref{lem:compare}, the bound in \citep[Corollary~1]{nips21_oja} when $r=1$ is
\begin{align*}
    \norm{\sin \Theta(U_n, U_*)} \lesssim \frac{\cM_2}{\Delta_{\min}} \sqrt{\frac{\log n}{n}} \leq \frac{\lambda_1(1 \vee \nu_2^2)}{\Delta_{\min}} \sqrt{\frac{\log n}{n}}.
\end{align*}

\paragraph{\citep{jain16}} Using the bound for $\cV$ from \cref{lem:compare}, their error bound is 
\begin{align*}
    \norm{\sin \Theta(U_n, U_*)} \lesssim \frac{\cV^{1/2}}{\Delta_{\min}} \sqrt{\frac{1}{n}} &\asymp \frac{\lambda_1 (1 \vee \nu_2^2)^{1/2}}{\Delta_{\min}} \sqrt{\frac{1}{n}}= \frac{\lambda_1 ( 1 \vee \nu_2)}{\Delta_{\min}} \sqrt{\frac{1}{n}}.
\end{align*}

\paragraph{\citep{colt21}} Let $\widetilde{\cM}$ denote the almost-sure boundedness parameter in \citep{colt21}, namely
\begin{align*}
    \widetilde{\cM} \geq \sup_{P \in \OO_{d,r}} \Fnorm{P^\top (XX^\top - \Sigma)} \qquad \text{a.s.}
\end{align*}
Their operator-norm guarantee, suppressing logarithmic factors, takes the form
\begin{align*}
    \norm{\sin \Theta(U_n, U_*)} \lesssim \frac{\widetilde{\cM}}{\Delta_{\min}} \sqrt{\frac{1}{n}}
\end{align*}
We now relate $\widetilde{\cM}$ to the covariance scale. Since $\widetilde{\cM} \geq \|XX^\top - \Sigma\|$, we have $\widetilde{\cM} \geq \cM_1$. Hence their error bound is no better than
\begin{align*}
    \frac{\lambda_1 (1 \vee \nu_2^2)}{\Delta_{\min}} \sqrt{\frac{r^2}{n}}.
\end{align*}

\paragraph{\citep{han17}} The error bound follows directly from their Theorem 3.

\subsection{Proof of Lemma~\ref{lem:compare}}
\label{sec:prf_compare}

Without loss of generality, assume $\Sigma$ is diagonal.

\paragraph{(i)}
Since $XX^\top = \Sigma^{1/2} \tX \tX^\top \Sigma^{1/2}$,  we have $XX^\top  - \Sigma = \Sigma^{1/2} (\tX \tX^\top - I_d) \Sigma^{1/2}$, so
\begin{align*}
    \EE \sbr[\big]{\rbr{XX^\top - \Sigma}^2} = \Sigma^{1/2} \underbrace{\EE \sbr[\big]{(\tX \tX^\top - I_d) \Sigma (\tX \tX^\top - I_d)}}_{=:\tilde \cV} \Sigma^{1/2}
\end{align*}
Therefore $\cV = \|\Sigma^{1/2} \tilde \cV \Sigma^{1/2}\|$, and it suffices to analyze $\tilde \cV$.
Expanding $\tilde \cV$ gives
\begin{align*}
    \tilde \cV &= \EE \sbr[\big]{\tX \tX^\top \Sigma \tX \tX^\top} - \Sigma \EE \sbr[\big]{\tX \tX^\top} - \EE \sbr[\big]{\tX \tX^\top} \Sigma + \Sigma= \EE \sbr[\big]{\tX \tX^\top \Sigma \tX \tX^\top} - \Sigma
\end{align*}

\textit{Sub-Gaussian upper bound.} Since $\tilde \cV$ is symmetric and positive semi-definite, $\|\tilde \cV\| = \sup_{\norm{u}=1} u^\top \tilde \cV u$, we have for any fixed unit vector $u \in \RR^d$,
    \begin{align*}
        u^\top \tilde \cV u &\leq \EE \sbr[\big]{(u^\top \tX)^2 (\tX^\top \Sigma \tX)} \leq \sqrt{\EE (u^\top \tX)^4} \sqrt{\EE (\tX^\top \Sigma \tX)^2} \lesssim C_\psi^4 \tr \Sigma
    \end{align*}
    Here the last inequality follows since Lemma~\ref{lem:kth_moment} implies that $\EE (\tX^\top \Sigma \tX)^2 = \EE \|\Sigma^{1/2}\tX\|^4 \lesssim C_\psi^4 \sbr{\tr \Sigma}^2$ and $\EE (u^\top \tX)^4 \lesssim C_\psi^4$. Taking supremum over $u$ gives $\|\tilde \cV\| \lesssim C_\psi^4 \tr \Sigma$, and thus
    \begin{align*}
        \cV = \norm[\big]{\Sigma^{1/2} \tilde \cV \Sigma^{1/2}} \leq \norm[\big]{\Sigma^{1/2}}^2 \norm[\big]{\tilde \cV} \lesssim C_\psi^4 \lambda_1 \tr \Sigma
    \end{align*}

\textit{Gaussian lower bound.} Assume $\tX \sim \cN(0,I_d)$. Since $\Sigma$ is diagonal, the $(i,l)$-entry of $\tilde \cV$ is 
    \begin{align*}
        [\tilde \cV]_{il} = \tX_i \rbr[\Big]{\sum_{j=1}^d \tX_j^2 \Sigma_{jj}} \tX_l - \Sigma_{il}
    \end{align*}
    By Isserlis' theorem, we have
    \begin{align*}
        \EE \sbr[\big]{\tX_j^2 \tX_i \tX_l}= \EE \sbr[\big]{\tX_j^2} \EE \sbr[\big]{\tX_i \tX_l} + 2 \EE \sbr[\big]{\tX_j \tX_i} \EE \sbr[\big]{\tX_j \tX_l} = \delta_{il} + 2 \delta_{ij} \delta_{jl}
    \end{align*}
    Summing over $j$ gives
    \begin{align*}
        \tX_i \rbr[\Big]{\sum_{j=1}^d \tX_j^2 \Sigma_{jj}} \tX_l = \sum_{j} \Sigma_{jj}(\delta_{il} + 2 \delta_{ij} \delta_{jl}) = \delta_{il} \tr \Sigma + 2 \delta_{il}\Sigma_{il}
    \end{align*}
    Since $\Sigma$ is diagonal, we have
    \begin{align*}
        \tilde \cV = (\tr \Sigma) I_d + 2 \Sigma - \Sigma = (\tr \Sigma) I_d + \Sigma
    \end{align*}
    Therefore,
    \begin{align*}
        \cV = \norm[\big]{\Sigma^{1/2} \tilde \cV \Sigma^{1/2}} = \norm{(\tr \Sigma) \Sigma + \Sigma^2} = \lambda_1^2 + \lambda_1 \tr \Sigma \asymp \lambda_1 \tr \Sigma
    \end{align*}
    The last equality follows since $\Sigma$ is diagonal and $\lambda_1$ is the largest eigenvalue.

\paragraph{(ii)} Cauchy-Schwarz inequality gives $\cM_1 \leq \cM_2$. It suffices to show $\cM_2 \lesssim C_\psi^2 [\tr \Sigma]$ and in the Gaussian case $\cM_1 \gtrsim \tr \Sigma$.

\textit{Sub-Gaussian upper bound for $\cM_2$.} 

Since $\|XX^\top - \Sigma\| \leq \max \cbr{\|XX^\top\|, \|\Sigma\|}$, in view of Lemma~\ref{lem:kth_moment}, we have
    \begin{align*}
        \cM_2^2 \lesssim \EE \max \cbr[\big]{\norm{X}^4, \lambda_1^2} \lesssim C_\psi^4 \sbr{\tr \Sigma}^2 + \lambda_1^2 \lesssim C_\psi^4 \sbr{\tr \Sigma}^2
    \end{align*}
    Here the last inequality follows since $C_\psi \geq 1$ and $\tr \Sigma \geq \lambda_1$.

\textit{Gaussian lower bound for $\cM_1$.} Assume $X \sim \cN(0,\Sigma)$. Without loss of generality, assume $\Sigma$ is diagonal: $\Sigma= \diag(\lambda_1, \ldots, \lambda_d)$.
\begin{itemize}[leftmargin=*]
    \item Since $\norm{XX^\top - \Sigma} \geq \norm{X}^2 - \norm{\Sigma}$, we have
    \begin{align*}
        \EE \norm{XX^\top - \Sigma} \geq \EE \norm{X}^2 - \norm{\Sigma} = \tr \Sigma - \lambda_1
    \end{align*}
    \item Let $e_1$ be the first standard basis vector. Since
    \begin{align*}
        \norm{XX^\top - \Sigma} \geq \abs{e_1^\top (XX^\top - \Sigma) e_1} = \abs{X_1^2 -\lambda_1} = \lambda_1 \abs{\tX_1^2 - 1},
    \end{align*}
    we have
    \begin{align*}
        \EE \norm{XX^\top -\Sigma} \geq \lambda_1 \EE \abs{\tX_1^2 - 1} \geq c_0 \lambda_1
    \end{align*}
    for some constant $c_0 > 0$.
\end{itemize}
Combining the two bounds above gives
\begin{align*}
    \EE \norm{XX^\top -\Sigma} \geq \max \cbr{\tr \Sigma - \lambda_1, c_0 \lambda_1}
\end{align*}
If $\lambda_1 \leq \tr \Sigma /2$, then $\tr \Sigma -\lambda_1 \geq \frac{1}{2} \tr \Sigma$.
If $\lambda_1 > \tr \Sigma /2$, then $c_0 \lambda_1 > c_0 \tr \Sigma /2$. In either case, we have
\begin{align*}
    \EE \norm{XX^\top -\Sigma} \gtrsim \tr \Sigma.
\end{align*}

\section{Main proof ideas}
\label{sec:prf_sketch}
This appendix provides a high-level overview for the proofs of the main results in Section~3. 
The analysis is organized around the tangent matrix $\cT$, which converts Oja's subspace dynamics into a recursion consisting of a deterministic contraction, a centered stochastic fluctuation, and a higher-order remainder.  
Variance-sensitive control of this recursion first yields sharp convergence of the iterates. 
Once the iterates are localized near $U_\ast$, the same recursion admits a first-order linearization driven by the signal--tail interaction $Z_iY_i^\top$. 
This linear representation is the main bridge from convergence to the lower bound, Gaussian approximation, and multiplier bootstrap.

We follow the organization of Sections~\ref{sec:main_convergence}--\ref{sec:main_bs}. 
Appendix~\ref{sec:prf_sketch_convergence} explains the sharp convergence and linearization arguments, including the auxiliary bounded-observation analysis, its transfer to the original sub-Gaussian model through joint truncation, and the separate treatment of the exactly rank-$r$ case. 
Appendix~\ref{sec:prf_sketch_gaussian_approx} shows how the linearization yields the global and row-wise Gaussian approximations, while Appendix~\ref{sec:prf_sketch_bs} develops the corresponding bootstrap linearization and conditional Gaussian approximation. 
Technical details are deferred to the subsequent appendices: Theorems~\ref{thm:convergence_rate_optimal}--\ref{thm:rankr} and Propositions~\ref{prop:linearization}--\ref{prop:lower_bd} are proved in Appendix~\ref{sec:prf_sec_convergence}, Theorems~\ref{thm:highD_clt}--\ref{thm:row_clt} in Appendix~\ref{sec:clt_analysis}, and Theorems~\ref{thm:bs_highD_clt}--\ref{thm:bs_row_clt} in Appendix~\ref{sec:bootstrap_analysis}.

\subsection{Sharp convergence and linearization}
\label{sec:prf_sketch_convergence}

We first explain the core convergence and linearization arguments in an
auxiliary bounded-observation setting, where
\[
    \|Y\|_2
    \leq
    C_{\psi,n}\sqrt{\lambda_{1\sim r}},
    \qquad
    \|Z\|_2
    \leq
    C_{\psi,n}\sqrt{\lambda_{r+1\sim d}}
\]
almost surely, with $C_{\psi,n}\asymp C_\psi\sqrt{\log n}$
(Assumption~\ref{assumption:X_truncation}). We then explain how these arguments are transferred to
the original sub-Gaussian model (Assumption~\ref{assumption:X}) through a joint truncation
of the signal and tail coordinates.

\paragraph{Convergence under bounded observations.}
The starting point is the exact tangent recursion (Lemma~\ref{lem:recurrence}). After
separating the conditional mean from the centered fluctuation, one Oja
update can be written as
\begin{equation}
  \cT_{k+1} = (\cI-\eta\cL)\cT_k+D_{k+1}+R_{k+1}, \label{eqn:main_onestep}
\end{equation}
where $\cI-\eta\cL$ gives the deterministic contraction,
$D_{k+1}$ is conditionally centered, and $R_{k+1}$ is a higher-order
drift remainder.

Matrix Freedman's inequality, together with the control of the drift
remainder, leads to a quadratic function
\[
    \varphi(\tau)=\bar a\tau^2+\bar b\tau+\bar c
\]
that describes the effective evolution of the error (Lemmas~\ref{lem:one_step_bound}--\ref{lem:1epoch_bound}) starting from $\|\cT_k\|$ of order $\tau$. After sufficiently many iterations the deterministic contraction suppresses the initial error and $\|\cT_k\|$ drops to order $\varphi(\tau)$. Repeating this argument drives the error toward the stable fixed-point of $\varphi$, which is comparable to the noise floor $\tau_\ast$. The remaining argument verifies that the total number of iterations required to reach this error level is at most $n$ (Lemmas~\ref{lem:epochs}--\ref{lem:convergence_rate_optimal}), yielding the convergence guarantee in Theorem~\ref{thm:convergence_rate_optimal}.

\paragraph{Linearization under bounded observations.}
Once the convergence argument has localized the iterates to $\|\cT_k\|\lesssim \tau_\ast$, we revisit the one-step decomposition~\ref{eqn:main_onestep} and isolate its leading stochastic term. Among the terms in the centered increment $D_{k+1}$, $\eta Z_{k+1}Y_{k+1}^\top$ is the only one that contains neither a factor of $\cT_k$ nor an additional power of $\eta$. This suggests the first-order recursion
\[
\cT_{k+1} \approx (\cI-\eta\cL)\cT_k+\eta Z_{k+1}Y_{k+1}^\top,
\]
and hence the linear representation
\[
\cT_n \approx \eta\sum_{i=1}^n (\cI-\eta\cL)^{n-i}Z_iY_i^\top.
\]
Lemma~\ref{lem:linearization} formalizes this approximation by splitting the trajectory at the end of the burn-in period. The pre-burn-in contribution is exponentially damped, while thereafter $\|\cT_k\|\lesssim\tau_\ast$ makes the higher-order terms negligible relative to the leading fluctuation, yielding Proposition~\ref{prop:linearization}.

\paragraph{From bounded to sub-Gaussian observations.}
To transfer the preceding convergence and linearization arguments to Assumption~\ref{assumption:X}, we jointly truncate the whitened signal $\tY$ and tail $\tZ$ coordinates (Lemma~\ref{lem:reduction_to_bounded}). The truncated coordinates remain centered and sub-Gaussian and satisfy the almost-sure signal and tail bounds required by the bounded-observation analysis. The truncation slightly perturbs their joint covariance, by a matrix $H$ with $\|H\|\ll 1$. Consequently, relative to the original basis $(U_\ast,U_{\ast,\perp})$, both diagonal covariance blocks are perturbed and the truncated signal and tail coordinates need no longer be uncorrelated.

Rather than re-diagonalizing the perturbed covariance, we retain the original contraction $\cI-\eta\cL$ and absorb the induced drift perturbation into the remainder. Lemma~\ref{lem:one_step_bound-truncation} shows that the resulting recursion satisfies the same one-step almost-sure, conditional-variance, and drift-remainder bounds as in the bounded setting, up to constant factors. Consequently, both the convergence and linearization arguments carry over to the truncated process; for the latter, the only additional term comes from the nonzero signal--tail cross-covariance, whose post-burn-in contribution is absorbed into the modified remainder and whose pre-burn-in contribution is exponentially damped. This yields Lemmas~\ref{lem:convergence_rate_optimal-trunc} and~\ref{lem:linearization-truncation}. Finally, the truncation coupling (Lemma~\ref{lem:reduction_to_bounded}) transfers these guarantees to the original sub-Gaussian observations, giving Theorem~\ref{thm:convergence_rate_optimal} and Proposition~\ref{prop:linearization}.

\paragraph{Exact-rank case.}
When $\operatorname{rank}(\Sigma)=r$, the tail coordinates $Z$ vanish identically and $\tau_\ast=0$, so the analysis above is replaced by a direct contraction argument. The tangent recursion reduces to
\[
\cT_{k+1}
=
\cT_k(I_r+\eta Y_{k+1}Y_{k+1}^\top)^{-1},
\]
which preserves the column space of $\cT_k$ and makes $\|\cT_k\|$ nonincreasing. This column-space preservation allows the matrix martingale arising over each contraction block to be compressed from dimension $(d-r)\times r$ to at most $r\times r$. Hence the dimension prefactor in matrix Freedman's inequality is at most $2r$, rather than of ambient order $d$. Combining the resulting one-block contraction with a block-chaining argument yields the geometric decay in Lemmas~\ref{lem:rankr} and~\ref{lem:rankr-trunc}, and hence Theorem~\ref{thm:rankr}.

\subsection{Gaussian approximation}
\label{sec:prf_sketch_gaussian_approx}

\paragraph{Global Gaussian approximation.}
Proposition~\ref{prop:linearization} reduces the distributional analysis of $\cT_n$ to the weighted sum $\eta\sum_{i=1}^n (\cI-\eta\cL)^{n-i}Z_iY_i^\top$.
After normalization by $\sqrt{\eta}$, this is a sum of independent random matrices. A high-dimensional Gaussian approximation, together with convergence of its finite-sample covariance, therefore gives a Gaussian approximation for $\eta^{-1/2}\|\cT_n\|_{\rm F}$. Finally, the local equivalence between the tangent and sine errors transfers this approximation to $\eta^{-1}\|\sin\Theta(U_n,U_\ast)\|_{\rm F}^2$, yielding Theorem~\ref{thm:highD_clt}.

\paragraph{Row-wise Gaussian approximation.}
The same linearization also drives the row-wise result. After Procrustes alignment, the leading term of $U_n\sgn(U_n^\top U_\ast)-U_\ast$ is $U_{\ast,\perp}\cT_n$. Projecting onto a fixed row therefore reduces the problem to a sum of independent $r$-dimensional random vectors, to which a multivariate Berry--Esseen bound applies. The convergence and linearization bounds make the remaining alignment terms negligible, yielding Theorem~\ref{thm:row_clt}.

\subsection{Multiplier bootstrap}
\label{sec:prf_sketch_bs}

\paragraph{Bootstrap linearization.}
The bootstrap tangent recursion has the same contractive structure as the original recursion, with the bounded multiplier weights affecting only the constants in the one-step bounds. The convergence and linearization arguments therefore extend to the bootstrap iterates, giving
\[
\cT_n^{\bs}-\cT_n
\approx
\eta\sum_{i=1}^n
(w_i^{\bs}-1)
(\cI-\eta\cL)^{n-i}Z_iY_i^\top.
\]
Thus, conditionally on the data, the leading bootstrap fluctuation is obtained by attaching independent centered multipliers to the same summands that drive the sampling fluctuation in Proposition~\ref{prop:linearization}.

\paragraph{Conditional Gaussian approximation.}
Given the bootstrap linearization above, the remaining argument parallels the Gaussian approximation in Appendix~\ref{sec:prf_sketch_gaussian_approx}, now conditionally on the data. The linearized bootstrap error is a sum of independent centered multiplier terms, whose empirical covariance converges to the same limiting covariance as the sampling linearization. Conditional high-dimensional Gaussian approximation, together with a multivariate Berry--Esseen bound after row-wise projection, therefore yields the same global and row-wise Gaussian approximations as in Theorems~\ref{thm:highD_clt} and~\ref{thm:row_clt}. The same tangent-to-subspace and alignment arguments then transfer these approximations to the bootstrap statistics. Comparing the bootstrap and sampling Gaussian approximations yields Theorems~\ref{thm:bs_highD_clt} and~\ref{thm:bs_row_clt}.

\section{Analysis of Oja's algorithm: bounded observations}
\label{sec:oja_bdd}
% This section proves the high-probability convergence, linearization, and lower-bound results for Oja's algorithm under a bounded-coordinate version of the sub-Gaussian model and a high-probability initialization condition. Specifically, under Assumptions~\ref{assumption:X_truncation}--\ref{assumption:init} below, Lemmas~\ref{lem:convergence_rate_optimal}--\ref{lem:lower_bound} give the counterparts of \cref{thm:convergence_rate_optimal}, Proposition~\ref{prop:linearization}, and Proposition~\ref{prop:lower_bd}. The reduction from the original sub-Gaussian setting to these bounded-coordinate assumptions is handled by a truncation argument in \cref{sec:prf_sec_convergence}; The reduction from Haar initialization to Assumption~\ref{assumption:init} follows from a conditioning argument;  Assumption~\ref{assumption:d}  is implicit in the assumptions of \cref{thm:convergence_rate_optimal}.

% We assume the following conditions.

This appendix develops the core convergence and linearization analysis under Assumptions~\ref{assumption:X_truncation}--\ref{assumption:init}. In particular, the signal and tail coordinates $Y$ and $Z$ are almost surely bounded, the dimension grows at most polynomially with $n$, and the initial tangent matrix satisfies a deterministic bound. 
Under Assumptions~\ref{assumption:X_truncation}--\ref{assumption:init}, we establish convergence to the noise floor (\cref{lem:convergence_rate_optimal}) and linearization (\cref{lem:linearization}) with polynomially small failure probability. 
We also treat the exactly rank-$r$ case separately: when $\rank(\Sigma)=r$, the
noise floor vanishes, and under Assumptions~\ref{assumption:X_truncation}--\ref{assumption:d}, we prove geometric contraction of the tangent error (\cref{lem:rankr}).
Appendix~\ref{sec:oja_subG} extends these results to the original sub-Gaussian observations
through truncation, and Appendix~\ref{sec:prf_sec_convergence} uses them to prove the convergence results in Section~\ref{sec:main_convergence}.

% This appendix develops the core convergence and linearization analysis under a bounded-coordinate version of the sub-Gaussian model. Assumptions~\ref{assumption:dimension} and~\ref{assumption:initialization} are maintained throughout Appendices~B and~C, while Assumption~\ref{assumption:X-bounded} is specific to the bounded analysis in this appendix. Under these assumptions, Lemmas~8--10 establish the bounded-observation counterparts of Theorem~1 and Propositions~2--3.

\begin{assumption}[Bounded signal and tail coordinates] \label{assumption:X_truncation}
    $X_1, \ldots, X_n \in \RR^d$ are i.i.d. copies of $X = \Sigma^{1/2} \tX$ where $\EE \tX=0$, $\EE \tX \tX = I_d$, $\|\tX\|_{\psi_2} \leq C_\psi \lesssim 1$. Moreover, almost surely,
    \begin{align*}
        \norm[\big]{U_*^\top X} \leq C_{\psi,n} \sqrt{\lambda_{1\sim r}} \quad \text{and}\quad \norm[\big]{U_{*,\perp}^\top X}  &\leq C_{\psi,n} \sqrt{\lambda_{r+1\sim d}},
    \end{align*}
    where $C_{\psi,n}:= C_{\mathsf{noise}} C_{\psi} \sqrt{\log n}$ for some absolute constant $C_{\mathsf{noise}} > 0$.
\end{assumption}
\begin{remark}
    It follows from the definition of $\|\cdot\|_{\psi_2}$ and $\EE \tX \tX^\top = I_d$ that $C_\psi \geq 1$. Moreover, recalling from \eqref{eqn:Y_Z_defn} that $Y = U_*^\top X$ and $Z = U_{*,\perp}^\top X$, the two almost-sure conditions in Assumption~\ref{assumption:X_truncation} can equivalently be written as
    \begin{align*}
        \norm{Y} \leq C_{\psi,n} \sqrt{\lambda_{1\sim r}} \quad \text{and}\quad \norm{Z} \leq C_{\psi,n} \sqrt{\lambda_{r+1\sim d}}.
    \end{align*}
\end{remark}

\begin{assumption}[Dimension] \label{assumption:d}
    $d \lesssim n^{C_d}$ for some constant $C_d > 0$.
\end{assumption}
\begin{remark}[Role of Assumption~\ref{assumption:d}]
    Assumption~\ref{assumption:d} is used to absorb the matrix-dimension prefactor in Freedman's inequality into logarithmic factors. For $(d-r)\times r$ matrix martingales, this prefactor is of order $d$, so $d\lesssim n^{C_d}$ implies $\log d\lesssim \log n$. The assumption could therefore be removed at the cost of replacing the relevant $\log n$ factors by $\log(n+d)$.
    Moreover, Assumption~\ref{assumption:d} imposes no additional restriction in Theorem~\ref{thm:convergence_rate_optimal}. Indeed, conditions \eqref{eqn:cond_nu} and \eqref{eqn:cond_sample_sz} imply
    \[
        n^{1-2\epsilon} \gtrsim \nu_\infty d\, \tfrac{\log(1/\delta_p)}{\delta_p^2} \gtrsim \nu_\infty d \ge c_\nu n^{-C_\nu}d,
    \]
    and hence $d\lesssim c_\nu^{-1}n^{1-2\epsilon+C_\nu}$.
    Thus the polynomial-growth condition on $d$ follows automatically from
    the assumptions of Theorem~\ref{thm:convergence_rate_optimal}.
\end{remark}

\begin{assumption}[Initialization] \label{assumption:init}
    $\cT_0$ exists and satisfies $\norm{\cT_0} \leq C_{\mathsf{init}} \delta_p^{-1}\sqrt{rd\log(1/\delta_p)}$ for $\delta_p \in (0, \frac{1}{e})$ and some absolute constant $C_{\mathsf{init}} > 0$. 
\end{assumption}

\begin{remark}[Role of Assumption~\ref{assumption:init}]
Assumption~\ref{assumption:init} is a deterministic good-initialization
condition used in the intermediate convergence analysis. Its role is
to control the initial tangent radius $\tau_0=\|\cT_0\|$, which enters the stability and epoch-chaining arguments and ensures that the iterates reach the target neighborhood within the available $n$ updates. This assumption imposes no additional restriction in the main theorems under Haar initialization. Indeed, Lemma~\ref{lem:init}
shows that, for $U_0\sim\operatorname{Haar}(\mathbb O_{d,r})$,
\[
    \|\cT_0\|
    \lesssim
    \delta_p^{-1}
    \sqrt{rd\log(1/\delta_p)}
\]
with probability at least $1-O(\delta_p)$.
Thus Assumption~\ref{assumption:init} is simply the deterministic event on
which we carry out the subsequent analysis. In the exactly rank-$r$ case, the geometric contraction itself does not require Assumption~\ref{assumption:init}; it is only used to convert the bound proportional to $\|\cT_0\|$ into the explicit rate in
$d,n$, and $\delta_p$.
\end{remark}

\paragraph{Recursion and one-step bounds.}
For brevity, we write $\cT=\cT_i$, $\cT_+= \cT_{i+1}$, $Y_+=U_*^\top X_{i+1}$ and $Z_+ = U_{*,\perp}^\top X_{i+1}$.

\begin{lemma}[Recurrence] \label{lem:recurrence}
    If $\cT$ is well-defined and $1 + \eta (Y_+^\top + Z_+^\top \cT)Y_+ \neq 0$, then $\cT_+$ is well-defined and satisfies
    \begin{align} \label{eqn:recurrence}
        \cT_+ = \cT + \eta \frac{(-\cT Y_+ + Z_+)(Y_+^\top + Z_+^\top \cT)}{1+\eta (Y_+^\top + Z_+^\top \cT) Y_+}
    \end{align}
\end{lemma}
\begin{proof}
    See Appendix~\ref{sec:prf_recurrence}.
\end{proof}

Let \(\cF_i := \sigma(U_0,X_1,\dots,X_i)\) and write \(\mathbb{E}'[\,\cdot\,] := \EE[\,\cdot\, | \mathcal{F}_i]\). For brevity, denote
\begin{equation*}%\label{eq:zeta-F}
  \zeta_+ := (Y_+^\top + Z_+^\top \cT)\,Y_+ \in \RR, \qquad
  F_+ := (-\cT Y_+ + Z_+)(Y_+^\top + Z_+^\top \cT) \in \RR^{(d-r)\times r},
\end{equation*}
Then \cref{lem:recurrence} gives
\begin{align} \label{eq:dT_split}
    \cT_+ -\cT = \eta \frac{F_+}{1+\eta\zeta_+} = \eta F_+ - \eta^2 \frac{\zeta_+ F_+}{1 + \eta \zeta_+}.
\end{align}
Since $\EE'[F_+] =-\cL \cT$, the conditional drift can be decomposed as
\begin{equation}
    \EE'[\cT_+ - \cT] = -\eta\cL \cT + \underbrace{\eta^2 \, \EE'\!\sbr{-\frac{\zeta_+ F_+}{1+\eta \zeta_+}}}_{=:R_+}. \label{eqn:drift_split}
\end{equation}
Defining the centered one-step noise $D_+ := (\cT_+ - \cT) - \EE'[\cT_+ - \cT]$, we obtain the one-step decomposition
\begin{equation}\label{eqn:linearization}
    \cT_+ = (\cI - \eta\cL)\cT + D_+ + R_+.
\end{equation}
The next lemma quantifies the size of $D_+$ and $R_+$ given $\|\cT\|\le \tau$. The important point is that the martingale term admits a variance scale substantially sharper than its deterministic bound scale; this is what later makes a tight analysis possible.

For the linearization in \cref{lem:linearization}, we can further decompose $D_+$ according to the two terms in~\eqref{eq:dT_split}:
\begin{align}
    D_+ = \underbrace{\eta\rbr{F_+ - \EE'[F_+]}}_{=:\,D_+^{(1)}} + \underbrace{\eta^2\! \rbr{ - \frac{\zeta_+\,F_+}{1+\eta\,\zeta_+}  + \EE'\!\sbr{\frac{\zeta_+\,F_+}{1+\eta\,\zeta_+}}}}_{=:\,D_+^{(2)}}. \label{eqn:D_decomposition}
\end{align}

Recall $\kappa_*:= \tfrac{\lambda_1}{\Delta_{\min}}$.

\begin{lemma}[One step bound] \label{lem:one_step_bound}
    Assume Assumption~\ref{assumption:X_truncation} holds. Assume $\cT$ is well-defined with $\norm{\cT}=\tau$. If $ C_{\psi,n}^2 \teta \kappa_* r (1+\tau \nu_2) \leq \frac{1}{2}$, 
    then the following hold.

    \begin{enumerate}[leftmargin=2em, label=\textup{(\Roman*)}]
        \item Noise $D_+$: $\norm{D_+} \leq 4C_{\psi,n}^2 \teta \kappa_* r (1 + \tau \nu_2) (\tau + \nu_2)$ almost surely and
        \begin{align*}
            \max \cbr[\big]{\norm{\EE'D_+D_+^\top}^{1/2} , \norm{\EE' D_+^\top D_+}^{1/2}} &\lesssim C_\psi^2  \teta  \kappa_* \sqrt{r} \rbr{1 + \tau \nu_\infty} \rbr{\tau + \bnu} .
        \end{align*}

        \item Higher-order noise $D_+^{(2)}$: $\norm[\big]{D_+^{(2)}} \leq 4 C_{\psi,n}^4  \teta^2  \kappa_*^2 r^2 (1+\tau \nu_2)^2 (\tau+\nu_2)$ almost surely and
        \begin{align*}
            \max \cbr[\big]{\norm[\big]{\EE' D_+^{(2)} D_+^{(2)\top}}^{1/2} , \norm[\big]{\EE' D_+^{(2)\top} D_+^{(2)}}^{1/2}} &\lesssim C_{\psi}^4 \teta^2  \kappa_*^2 r^{3/2} \rbr{1+ \tau \nu_\infty}^2 \rbr{\tau + \bnu}.
        \end{align*}

        \item Drift remainder $R_+$: $\norm{R_+} \lesssim C_{\psi}^4 \teta^2  \kappa_*^2 r^{3/2} \rbr{1+ \tau \nu_\infty}^2 \rbr{\tau + \nu_\infty}$.
    \end{enumerate}

\end{lemma}
\begin{proof}
    See Appendix~\ref{sec:prf_one_step_bound}.
\end{proof}

$\cT_k = (\cI - \eta \cL)^k \cT_0 + \cD^{(k)} + \cR^{(k)}$ where
\begin{align*}
    \cD_k = \sum_{i=1}^k (\cI - \eta \cL)^{k-i} D_i \quad \text{and} \quad \cR_k = \sum_{i=1}^k (\cI - \eta \cL)^{k-i} R_i.
\end{align*}

Define 
\begin{equation}\label{eqn:varphi_definitions}
\begin{alignedat}{2}
    \varphi_{\sigma}(\tau) &= \sqrt{\teta} \rho_{\sigma} (1 + \tau\nu_\infty) (\tau+\bnu), \qquad \, \text{where}\quad \rho_\sigma &&= C_{\ref{lem:1epoch_bound}}  C_{\psi}^2 \kappa_* \sqrt{r\log n},\\
    \varphi_B(\tau) &= \teta \rho_B  (1 + \tau\nu_2)(\tau+\nu_2), \qquad \quad \text{where} \quad \rho_B &&= C_{\ref{lem:1epoch_bound}} C_{\psi}^2 \kappa_*  r \log^2 n.
\end{alignedat}
\end{equation}
Here $C_{\ref{lem:1epoch_bound}} > 0$ is some constant to be determined later in Lemma~\ref{lem:1epoch_bound}. 

Let $\varphi_\sigma \vee_\mathrm{cw} \varphi_B$ denote the coefficient-wise maximum of two quadratics. Specifically, after expanding $\varphi_\sigma$ and $\varphi_B$ as
\begin{align*}
    \varphi_\sigma(\tau) &= a_\sigma \tau^2 + b_\sigma \tau + c_\sigma,
        &a_\sigma &= \sqrt{\teta} \rho_{\sigma} \nu_\infty,
        &b_\sigma &= \sqrt{\teta} \rho_{\sigma} (1 + \nu_\infty\bar\nu),
        &c_\sigma &= \sqrt{\teta} \rho_{\sigma} \bar\nu, \\
    \varphi_B(\tau) &= a_B \tau^2 + b_B \tau + c_B,
        &a_B &= \teta \rho_B \nu_2,
        &b_B &= \teta \rho_B (1 + \nu_2^2),
        &c_B &= \teta \rho_B \nu_2.
\end{align*}
we let
\begin{equation}
    \varphi(\tau):= (\varphi_\sigma \vee_{\mathrm{cw}} \varphi_B)(\tau) := \bar a\,\tau^2 + \bar b\,\tau + \bar c,
    \qquad
    \bar a = a_\sigma \vee a_B,\quad
    \bar b = b_\sigma \vee b_B,\quad
    \bar c = c_\sigma \vee c_B.
    \label{eqn:varphi_def}
\end{equation}
Since all coefficients are nonnegative, we have $\varphi(\tau) \ge \varphi_\sigma(\tau) \vee \varphi_B(\tau)$ for all $\tau \ge 0$. Define the \emph{noise floor} $\tau_*$
\begin{align}
    \tau_* &:= 3\bar c = 3\max\rbr[\big]{\sqrt{\teta} \rho_\sigma\bar\nu, \teta \rho_B\nu_2}\label{eqn:tau*_def}
\end{align}

Before establishing convergence, we record several properties of the envelope $\varphi$ that will be used repeatedly.

\begin{lemma}[Properties of $\varphi$] \label{lem:varphi}

    Let $\varphi$ and $\tau_*$ be defined as in \eqref{eqn:varphi_def} and \eqref{eqn:tau*_def}. If $\ba,\bb,\bc>0$ and $\max(\ba\tau_0, \bb) < 1/3$ for some $\tau_0 \ge \tau_*$, then the following hold.

    \textup{(i) (Fixed points)} The equation $\varphi(\tau) = \tau$ has two positive roots $\tau_- < \tau_+$ with
    \begin{align*}
        \bar c < \tau_- < 3\bc, \qquad \frac{1}{3\ba} < \tau_+ < \frac{1}{\ba}.
    \end{align*}

    \textup{(ii) (Stability)} $\tau_- < \tau_* \leq \tau_0 < \tau_+$.

    \textup{(iii) (Contraction between $\tau_\pm$)} $\varphi(\tau) < \tau$ for $\tau \in (\tau_-, \tau_+)$.

    \textup{(iv) (Quantitative contraction)} For all $\tau \in [\tau_*, \tau_0]$: $\varphi(\tau) \leq \max\cbr{(\bar a \tau_0 + 2\bar b)\tau, \tau_*}$.

\end{lemma}
\begin{proof}
    See Appendix~\ref{sec:prf_varphi}.
\end{proof}

\begin{lemma}[One epoch bound] \label{lem:1epoch_bound}
    Assume Assumption~\ref{assumption:X_truncation} and \ref{assumption:d} hold.
    Assume $\|\cT_0\| = \tau_0$. Let $\ttau_0 = \tau_0 \vee \tau_*$. If $\nu_\infty>0$ and  $\max\cbr{\bar a\ttau_0, \bar b} < 1/3$ and
    \begin{equation}
        \teta \kappa_* \in \rbr{0,\tfrac{1}{2}},
        \quad
        \teta C_{\psi,n}^2 \kappa_* r\rbr{1 + 2\ttau_0 \nu_2} \le \tfrac{1}{2},
        \quad
        C_\psi^2 \sqrt{\teta} \, \kappa_* r \rbr{1 + 2 \ttau_0 \nu_\infty} \leq \log^{-3/2} n,
        \label{eqn:1epoch_conditions}
    \end{equation}
    then there exists a constant $C_{\ref{lem:1epoch_bound}} \geq 1$ large enough (only depending on $C_d$) such that, defining $\varphi$ and $\tau_*$ as in \eqref{eqn:varphi_def} and \eqref{eqn:tau*_def}, with probability exceeding $1-n^{-20}$, the following properties hold simultaneously for all $k \in \cbr{0,\ldots,n}$:
    \begin{align}
        \textnormal{(Concentration)}\quad &\norm{\cD_k} + \norm{\cR_k} \leq \frac{\varphi(\ttau_0)}{2},\label{eqn:1epoch_concentration}\\
        \textnormal{(Descent)}\quad &\|\cT_k\| \le (1 - \tilde\eta)^k \ttau_0 + \frac{\varphi(\ttau_0)}{2}.  \label{eqn:1epoch_descent}\\
        \textnormal{(Stability)}\quad &\|\cT_k\| \le 2\ttau_0, \label{eqn:1epoch_stability}
    \end{align}
\end{lemma}
\begin{proof}
    See Appendix~\ref{sec:prf_1epoch_bound}.
\end{proof}

\begin{corollary}\label{cor:epoch_target}
Under the conditions of Lemma~\ref{lem:1epoch_bound}, if $\tau_0 > \tau_*$ and $\delta \ge \varphi(\tau_0)$, then, defining $N(\tau_0,\delta) := \lceil \log(2\tau_0/\delta)/\tilde\eta \rceil$, with probability at least $1-n^{-20}$, $\norm{T_k} \le \delta$ for all $k \in [N(\tau_0,\delta), n]$.
\end{corollary}
\begin{proof}
    See Appendix~\ref{sec:prf_epoch_target}.
\end{proof}

Define the \emph{epoch sequence}: $\otau_{0} := \|\cT_0\|=\tau_0 \vee \tau_*$ and
\begin{equation}
\otau_{j+1} := \varphi(\otau_{j}) \vee \tau_* \quad \text{for } j=0,1,\ldots
\end{equation}
with $\oJ_* := \min\{j \ge 0 : \otau_{j} \le \tau_*\}$, epoch lengths $\on_j := \bigl\lceil \log\bigl(2\otau_{j}/\otau_{j+1}\bigr)/\teta \bigr\rceil$, and cumulative lengths $\oN_0 := 0$, $\oN_{j+1} := \oN_j + \on_j$.

Define epochs to be $[\oN_0, \oN_1), [\oN_1, \oN_2), \ldots, [\oN_{\oJ_*-1}, \oN_{\oJ_*})$ and $[\oN_{\oJ_*}, \infty)$. We call $[\oN_j, \oN_{j+1})$ the $j$-th epoch. An epoch $j$ is called nontruncated if $\varphi(\otau_j) > \tau_*$, and truncated otherwise. By definition, the $(\oJ_* -1)$-th epoch is truncated. 
Lemma~\ref{lem:epochs} shows that all epochs $j \leq \oJ_* -2$ are nontruncated. 
Moreover, if $n$ falls in the final epoch $[\oN_{\oJ_*}, \infty)$ (i.e. $n \geq N_{\oJ_*}$), then $\norm{\cT_n}\leq 2\tau_*$

\begin{lemma}[Epoch chaining]\label{lem:epochs}
Assume the conditions of Lemma~\ref{lem:1epoch_bound}. Suppose further that
\begin{align*}
    \max\cbr{\bar a\otau_0, \bar b} < \frac{1}{3n^{\tepsilon}\log n}
\end{align*}
for some constant $\tepsilon \geq 0$, then the following hold.

\textup{\textbf{I}~(Epoch structure)}

\textup{(a)}~$\otau_j \downarrow \tau_*$, $\oJ_* < \infty$ and $\otau_{\oJ_*} = \tau_*$.

\textup{(b)}~\textup{(Uniform progress)} Every nontruncated epoch $j$ with $\varphi(\otau_j) \geq \tau_*$ satisfies
\begin{align*}
    \log \rbr[\Big]{\frac{\otau_{j}}{\otau_{j+1}}} \ge \tepsilon \log n + \log\log n.
\end{align*}
In particular, each nontruncated epoch reduces $\log \otau_j$ by at least $\tepsilon \log n + \log\log n$.

\textup{(c)}~\textup{(Epoch count)} $\oJ_* \le \log_+(\frac{\tau_0}{\tau_*})/(\tilde\epsilon\log n + \log\log n) + 1$.

\textup{(d)}~\textup{(Total iterations)} $\oN_{\oJ_*} \le \teta^{-1}\log_+(\frac{\tau_0}{\tau_*}) + \oJ_* [\teta^{-1}(\log 2) + 1]$.

\textup{ \textbf{II}~($\norm{\cT_k}$ bound)}
With probability at least $ 1 - \oJ_* n^{-20}$, the following hold uniformly.
\begin{align*}
    \|\cT_k\| \le 2\otau_{j} \quad \text{for any } k \in [\oN_j, n], \text{ for any } j=0,1,\ldots, \oJ_*.
\end{align*}

% \textup{(a) $j=0$:} $\norm{\cT_k}\leq 2\otau_0$ for all $k \in [0,n]$.

% \textup{(b) $j \geq 0$:} $\|\cT_k\| \le 2\otau_{j}$ for any $k \in [\oN_j, n]$ with $1 \leq j \le \oJ_*$.
\end{lemma}
\begin{proof}
    See Appendix~\ref{sec:prf_epochs}.
\end{proof}

\begin{lemma} \label{lem:summarize_dynamics}
    Assume the conditions of Lemma~\ref{lem:1epoch_bound}. Let $\eta = \frac{c_\eta \log n}{n \Delta_{\min}}$ for $c_\eta \geq \frac{1}{2}$. If $c_\eta \log n \ge C_{\mathrm{bud}} \log^+ (\sqrt{2 c_\eta}\tau_0/\tau_*)$, where
    \begin{equation}\label{eq:C_bud}
        C_{\mathsf{bud}} =
        1 \! + \!\rbr[\Big]{\frac{1}{\tepsilon \log n + \log\log n} \!+\! \frac{1}{\log_+ (\sqrt{2 c_\eta} \, \tau_0/\tau_*)}}\rbr{\log 2 + \log^{-3} n }.
    \end{equation}
    Then:

    \textup{(i)} $\oN_{\oJ_*} \le n \cdot \frac{C_{\mathsf{bud}} \log_+ (\sqrt{2 c_\eta} \, \tau_0/\tau_*) }{c_\eta \log n} \leq n$.

    \textup{(ii)} $\tau_* \asymp \sqrt{\teta} \rho_\sigma \bnu$.
    % Let Assumption~\ref{assumption:X_truncation} hold with $n \ge 4$ and $\|\cT_0\| = \tau_0$. Let $\varphi$ be the envelope~\eqref{eqn:varphi_def}. Let 
    % \begin{align*}
    %     \eta = \frac{c_\eta \log n}{n \Delta_{\min}}, \qquad c_\eta \geq \frac{1}{2}.
    % \end{align*}
    % Assume:

    % \textup{(C1)} $\teta \kappa_* \in (0,1/2)$.

    % \textup{(C2)} $\max(\bar a\tau_0, \bar b) \le 1/(3n^{\tilde\epsilon}\log n)$ for some $\tilde\epsilon \ge 0$.

    % \textup{(C3)} $\teta C_{\psi,n}^2 \kappa_* r (1 + 2\tau_0\nu_2) \le 1/2$.

    % \textup{(C4)} $C_\psi^2\, \teta^{1/2} \kappa_* r (1 + \tau_0\nu_\infty) \leq \log^{-3/2} n$.

    % \textup{(C5)} $c_\eta \log n \ge C_{\mathrm{bud}} \log^+ (\sqrt{2 c_\eta}\tau_0/\tau_*)$, where
    % \begin{equation}%\label{eq:C_bud}
    %     C_{\mathsf{bud}} =
    %     1 \! + \!\rbr[\Big]{\frac{1}{\tepsilon \log n + \log\log n} \!+\! \frac{1}{\log_+ (\sqrt{2 c_\eta}\tau_0/\tau_*)}}\rbr{\log 2 + \log^{-3} n }.
    % \end{equation}
    % Then with probability $\ge 1 - n^{-8}$:

    % \textup{(i)} $\oJ_* \le \log_+(\frac{\tau_0}{\tau_*})/(\tepsilon\log n + \log\log n) + 1$.

    % \textup{(ii)} $\oN_{\oJ_*} \le n \cdot \frac{C_{\mathsf{bud}} \log_+ (\sqrt{2 c_\eta}\tau_0/\tau_*) }{c_\eta \log n} \leq n$.

    % \textup{(iii)} $\|\cT_n\| \le 2\tau_*  \lesssim \sqrt{\teta} \rho_\sigma\bnu$ for all $\oN_{\oJ_*}\leq k \leq n$

\end{lemma}
\begin{proof}
    See Appendix~\ref{sec:prf_summarize_dynamics}.
\end{proof}

\begin{lemma}[Convergence] \label{lem:convergence_rate_optimal}
Suppose Assumptions~\ref{assumption:X_truncation}--\ref{assumption:init} hold. Assume
\begin{align}
    \nu_\infty &\ge c_{\nu} n^{-C_\nu} \quad\text{for constants } C_\nu \ge 0, c_\nu >0 \label{eqn:cond_nu_lem}
\end{align}
Let $\varepsilon \in [0,1/2)$ and set 
\begin{align}
  \eta = c_\eta \frac{\log n}{n \Delta_{\min}}, \quad\text{where}\quad c_\eta\geq c_\eta^* := \rbr[\Big]{1  - \varepsilon + \frac{3 C_\nu}{2}} \rbr[\Big]{1 + \frac{\log^{-3} n + \log 2}{\varepsilon \log n + \log\log n}}. \label{eqn:cond_step_sz_lem}
\end{align}
If
\begin{align}
    n^{1-2\varepsilon} &\ge C_{\ref{lem:convergence_rate_optimal}} {c_\eta} {\kappa_*^2} r^3 (\log n)^4 \frac{\log(1/\delta_p)}{\delta_p^2} \max(1, \nu_\infty d) \label{eqn:cond_sample_sz_lem}
\end{align}
for some constant $C_{\ref{lem:convergence_rate_optimal}} > 0$ large enough (depending only on $C_\psi$, $C_\nu$, $c_\nu$ and $C_d$).
Then with probability at least $1- n^{-19}$,
\begin{align}
  \norm{\cT_k} \leq 2\tau_* \asymp  {\kappa_*} \bnu \sqrt{\teta r \log n} = \frac{\lambda_1}{\Delta_{\min}} \sqrt{c_\eta r} \, \bnu  \frac{\log n}{\sqrt{n}}, \qquad \text{for all } \floor{\tfrac{c_\eta^*}{c_\eta} n} \leq k \leq n
  \label{eqn:convergence_rate_Tn}
\end{align}
\end{lemma}
\begin{proof}
    See Appendix~\ref{sec:prf_convergence_rate_optimal}.
\end{proof}

\begin{lemma}[Exact-rank contraction] \label{lem:rankr}
Suppose Assumption~\ref{assumption:X_truncation} holds, $\rank(\Sigma)=r$ and $\cT_0$ is well-defined. Fix $\varepsilon\in[0,1/2)$ and set
\[
    \eta = \frac{c_\eta\log n}{n\lambda_r}, \qquad \text{where} \quad c_\eta\ge 1/2.
\]
If 
\begin{equation}
  n^{1-2\varepsilon} \ge C_{\rm rk}c_\eta\kappa_\ast^2r^2(\log n)^4 \label{eqn:rankr_cond_sample_sz_lem}
\end{equation}
for some constant $C_{\rm rk}>0$ sufficiently large
(depending only on $C_\psi$), then, conditionally on
$\mathcal F_0=\sigma(U_0)$, with probability at least $1-n^{-19}$,
\begin{equation}
    \|\cT_k\| \le 16e \exp\left\{ -\left(1-\frac{1}{n^\varepsilon\log n}\right) k\teta \right\} \|\cT_0\|, \qquad k=0,1,\ldots,n. \label{eqn:rankr_contraction}
\end{equation}
In particular,
\[
    \|\cT_n\| \le 16e\, n^{-c_\eta(1-(n^\varepsilon\log n)^{-1})} \|\cT_0\|.
\]

If, in addition, Assumption~\ref{assumption:init} holds and
\begin{equation}
    n^{1-2\varepsilon} \ge C_{\mathrm{rk}} c_\eta\kappa_\ast^2r^2(\log n)^4 \frac{\log(1/\delta_p)}{\delta_p^2}, \label{eqn:rankr_cond_sample_sz_lem_further}
\end{equation}
then on the same event,
\begin{equation}
    \|\cT_k\| \lesssim \frac{\sqrt d\,n^{1/2-\varepsilon}} {\sqrt{c_\eta}\kappa_\ast\sqrt r(\log n)^2} \exp\left\{ -\left( 1-\frac{1}{n^\varepsilon\log n} \right)k\teta \right\}, \qquad k=0,1,\ldots,n. \label{eqn:rankr_contraction_futher}
\end{equation}
In particular,
\[
    \|\cT_n\|
    \lesssim
    \frac{\sqrt d}
    {\sqrt{c_\eta}\kappa_\ast\sqrt r(\log n)^2}
    n^{1/2-\varepsilon
    -c_\eta(1-(n^\varepsilon\log n)^{-1})}.
\]
\end{lemma}

\begin{proof}
  See Appendix~\ref{sec:prf_rankr}.
\end{proof}

\begin{lemma}[Linearization] \label{lem:linearization}
    Under the assumptions of \cref{lem:convergence_rate_optimal} with
    $c_\eta \geq c_\eta^* + \frac{1}{2}$, with probability exceeding $1-n^{-18}$,
    \begin{align*}
        \norm[\big]{\cT_n - \eta \sum_{i=1}^n (\cI-\eta \cL)^{n-i}(Z_iY_i^\top)} \lesssim  \tau_* \sbr{C_{\ref{lem:1epoch_bound}} C_\psi^2 \kappa_* \sqrt{r \teta \log n} + \tau_*  } \lesssim \frac{\tau_*}{n^{\varepsilon}\log n}.
    \end{align*}
\end{lemma}
\begin{proof}
    See Appendix~\ref{sec:prf_linearization}.
\end{proof}

\subsection{Proof of Lemma~\ref{lem:recurrence}}
\label{sec:prf_recurrence}

Since $U_+=\qr{\tU_+}$, the ratio $\cT_+=(U_{*,\perp}^\top U_+) (U_*^\top U_+)^{-1}$ is invariant under the right factor, hence $\cT_+=\tcS_+ \tcC_+^{-1}$ where $\tcC_+ = U_*^\top \tU_+$ and $\tcS_+ = U_{*,\perp}^\top \tU_+$. 

Since $X_+=U_*Y_+ + U_{*,\perp} Z_+$ and $U=U_* \cC + U_{*,\perp} \cS$, we have
\begin{align}
    X_+^\top U = Y_+^\top \cC + Z_+^\top \cS = \underbrace{(Y_+^\top + Z_+^\top \cT)}_{=:v^\top \in \RR^{1 \times r}} \,\cC \label{eqn:x_plus_u}
\end{align}
Substituting (\ref{eqn:x_plus_u}) into $\tU_+ =  (I_d + \eta X_+ X_+^\top) U$, we can obtain
\begin{align*}
    \tcC_+ &= \cC + \eta Y_+ X_+^\top U  = \sbr{I_r + \eta Y_+ v^\top} \cC\\
    \tcS_+ &= \cS + \eta Z_+ X_+^\top U = \sbr{\cT + \eta Z_+ v^\top} \cC
\end{align*}
Hence
\begin{align*}
    \cT_+ = \tcS_+ \tcC_+^{-1} = \sbr{\cT + \eta Z_+ v^\top} \sbr{I_r + \eta Y_+ v^\top}^{-1}
\end{align*}
Note that the inverse can be computed by Sherman-Morrison formula. Specifically, if $1 + \eta v^\top Y_+  \neq 0$, then
\begin{align*}
    \cT_+ = \sbr{\cT + \eta Z_+ v^\top} \sbr[\Big]{I_r - \eta \frac{Y_+ v^\top}{1 + \eta v^\top Y_+}}
\end{align*}
Expanding the product gives
\begin{align*}
    \cT_+
    &= \cT - \eta \frac{\cT Y_+ v^\top}{1+\eta v^\top Y_+} + \eta Z_+ v^\top - \eta^2\frac{Z_+ v^\top (v^\top Y_+)}{1+\eta v^\top Y_+}\\
    &= \cT - \eta \frac{\cT Y_+ v^\top}{1+\eta v^\top Y_+} + \eta Z_+ v^\top \sbr{1- \frac{\eta  v^\top Y_+}{1+\eta v^\top Y_+}}\\
    &= \cT - \eta \frac{\cT Y_+ v^\top}{1+\eta v^\top Y_+} + \eta \frac{Z_+ v^\top}{1+\eta v^\top Y_+}\\
    &= \cT + \eta \frac{(-\cT Y_+ + Z_+)v^\top}{1+\eta v^\top Y_+}
\end{align*}

Finally, plugging $v^\top$ from~(\ref{eqn:x_plus_u}) back in gives the desired recursion formula.

\subsection{Proof of Lemma~\ref{lem:one_step_bound}}
\label{sec:prf_one_step_bound}

Recall~\eqref{eq:dT_split} that
\begin{align*}
    \cT_+ - \cT = \eta F_+ - \eta^2 \frac{\zeta_+ F_+}{1+\eta\zeta_+},\quad \text{where} \quad &F_+ = (-\cT Y_+ + Z_+)(Y_+^\top + Z_+^\top \cT) \in \RR^{(d-r)\times r}\\
    \text{and} \quad &\zeta_+ = (Y_+^\top + Z_+^\top \cT)Y_+ \in \mathbb{R},
\end{align*}
Define $V_1 = (-\cT Y_+ + Z_+) \in \RR^{d-r}$ and $V_2 = (Y_+ + \cT^\top Z_+) \in \RR^r$, then
\begin{align*}
    F_+ = V_1 V_2^\top, \qquad \zeta_+ = V_2^\top Y_+.
\end{align*}

Recall conditional drift $B_+$ and noise $D_+$ from~\eqref{eqn:drift_split}
\begin{align*}
    B_+ &= E'[\cT_+ - \cT] = -\eta \cL \cT + \underbrace{\eta^2 \, \EE'\!\sbr{-\frac{\zeta_+ F_+}{1+\eta \zeta_+}}}_{=:R_+},\\
    D_+ &= (\cT_+ - \cT) - E'[\cT_+ - \cT] = \underbrace{\eta\rbr{F_+ - \EE'[F_+]}}_{=:\,D_+^{(1)}} + \underbrace{\eta^2\! \rbr{ - \frac{\zeta_+\,F_+}{1+\eta\,\zeta_+}  + \EE'\!\sbr{\frac{\zeta_+\,F_+}{1+\eta\,\zeta_+}}}}_{=:\,D_+^{(2)}}.
\end{align*}
 
Recall $\nu_2 = \lambda_{r+1\sim d}^{1/2}/\lambda_{1\sim r}^{1/2}$ and $\nu_\infty = \lambda_{r+1}^{1/2}/\lambda_1^{1/2}$.

Assumption~\ref{assumption:X_truncation} yields
\begin{align}
    \norm{Y_+}_2^2 \leq C_{\psi,n}^2 \lambda_{1\sim r}, \qquad \norm{Z_+}_2^2 \leq C_{\psi,n}^2\lambda_{r+1\sim d}. \label{eqn:one_step_Y_Z_bound}
\end{align}
which implies
\begin{equation}
    \begin{aligned}
        \norm{V_1}_2 &= \norm{-\cT Y_+ + Z_+}_2 \leq \norm{\cT}_2\norm{Y_+}_2 + \norm{Z_+}_2 \leq C_{\psi,n}\lambda_{1\sim r}^{1/2}(\tau + \nu_2),\\
        \norm{V_2}_2 &= \norm{Y_+ + \cT^\top Z_+}_2 \leq \norm{Y_+}_2 + \norm{\cT}_2\norm{Z_+}_2 \leq  C_{\psi,n} \lambda_{1\sim r}^{1/2}(1 + \tau\nu_2).
    \end{aligned}
    \label{eqn:one_step_V1V2_bound}
\end{equation}

Since $\abs{\zeta_+} = |V_2^\top Y_+| \leq \norm{V_2}_2\norm{Y_+}_2 \leq C_{\psi,n}^2 \lambda_{1\sim r}(1+\tau\nu_2)$, the hypothesis $\eta C_{\psi,n}^2 \lambda_{1\sim r}(1+\tau\nu_2) \leq 1/2$ gives
\begin{align}
    |\eta\zeta_+| \leq \frac{1}{2}, \qquad \text{hence} \quad |1 + \eta\zeta_+| \geq \frac{1}{2}. \label{eqn:one_step_denominator}
\end{align}
 
For later use, we also record the sub-Gaussian norms. Since $Y_+ = \Lambda_*^{1/2}\widetilde{Y}_+$ and $Z_+ = \Lambda_{*,\perp}^{1/2}\widetilde{Z}_+$ with $\|(\widetilde{Y}_+^\top, \widetilde{Z}_+^\top)^\top\|_{\psi_2} \leq C_\psi$, we have $\norm{Y_+}_{\psi_2} \leq \lambda_1^{1/2}C_{\psi}$ and $\norm{Z_+}_{\psi_2} \leq \lambda_{r+1}^{1/2}C_{\psi}$, and thus
\begin{equation}
    \begin{aligned}
        \|V_1\|_{\psi_2} &\leq \norm{\cT} \norm{Y_+}_{\psi_2} + \norm{Z_+}_{\psi_2} \leq C_\psi\lambda_1^{1/2}(\tau + \nu_\infty),\\
        \|V_2\|_{\psi_2} &\leq \norm{Y_+}_{\psi_2} + \norm{\cT} \norm{Z_+}_{\psi_2} \leq C_\psi\lambda_1^{1/2}(1 + \tau\nu_\infty).
    \end{aligned}
    \label{eqn:one_step_V1V2_sub_gaussian}
\end{equation}

\paragraph{Preliminary $[\EE'\zeta^4]^{1/4}$:} Since $\zeta_+ = \norm{Y_+}^2 + Z_+^\top \cT Y_+$, triangle inequality of $L_4$ norm gives
\begin{align*}
    \sbr{E'\zeta_+^4}^{1/4} &\leq \sbr[\big]{E'\norm{Y_+}^8}^{1/4} +  \sbr{E' (Z_+^\top \cT Y_+)^4}^{1/4}\lesssim C_\psi^2 \lambda_{1 \sim r} + \sbr{E' (Z_+^\top \cT Y_+)^4}^{1/4}
\end{align*}
where the last inequality follows from Lemma~\ref{lem:kth_moment}. For the second term, Cauchy--Schwarz gives
\begin{align*}
    \sbr{\EE' (Z_+^\top \cT Y_+)^4}^{1/4} &\leq \sbr[\big]{\EE' \norm{\cT^\top Z_+}^8}^{1/8} \cdot \sbr[\big]{\EE' \norm{Y_+}^8}^{1/8}\\
    &\lesssim C_{\psi} \sbr{\tr (\cT^\top \Lambda_{*,\perp} \cT)}^{1/2} \cdot C_\psi \sbr{\tr \Lambda_*}^{1/2}\\
    &\overset{(i)}{\lesssim} C_\psi \lambda_{r+1}^{1/2} \sqrt{r} \tau \cdot C_\psi \lambda_{1 \sim r}^{1/2}\\
    &\leq C_\psi^2 r \lambda_1  \tau \nu_\infty
\end{align*}
where (i) follows from Lemma~\ref{lem:kth_moment}. Therefore, we have
\begin{align}
    [\EE' \zeta_+^4]^{1/4} \lesssim C_\psi^2 r\lambda_1 (1 + \tau \nu_\infty). \label{eqn:one_step_zeta_4}
\end{align}

\paragraph{(I) Noise $D_+$.} Recall definition
\begin{align*}
    D_+ &= (\cT_+ - \cT) - E'[\cT_+ - \cT], \quad \text{where}\quad \cT_+ - \cT = \eta \frac{F_+}{1+\eta\zeta_+}, \quad F_+ = V_1 V_2^\top.
\end{align*}

\textbf{Deterministic noise bound:} Equation~\eqref{eqn:one_step_V1V2_bound} and~\eqref{eqn:one_step_denominator} give
\begin{align*}
    \|\cT_+ - \cT\| = \eta\|F_+\|/|1+\eta\zeta_+| \leq 2\eta\|F_+\| = 2\eta\|V_1\|\|V_2\| 
    \leq 2C_{\psi,n}^2 \eta \lambda_{1\sim r} (\tau+\nu_2)(1+\tau\nu_2).
\end{align*}
Since $\|D_+\| = \|(\cT_+ - \cT) - E'[\cT_+ - \cT]\| \leq \|\cT_+ - \cT\| + \|\EE'[\cT_+ - \cT]\|$, we obtain
\begin{align*}
    \|D_+\| \leq 4 C_{\psi,n}^2 \eta\lambda_{1\sim r}(1+\tau\nu_2) (\tau+\nu_2) \leq 4 C_{\psi,n}^2 \teta \kappa_* r (1 + \tau \nu_2) (\tau + \nu_2).
\end{align*}

\textbf{Variance bound for $D_+$:} By definition,
\begin{align*}
    D_+ = (\cT_+ - \cT) - \EE' (\cT_+ - \cT), \quad \text{with} \quad \cT_+ - \cT = \eta \frac{F_+}{1+\eta \zeta_+}, \quad F_+ = V_1 V_2^\top.
\end{align*}
The second-moment domination of the variance and \eqref{eqn:one_step_denominator} give
\begin{align*}
    \EE' D_+ D_+^\top &\preceq \EE' (\cT_+ - \cT)(\cT_+ - \cT)^\top \preceq 4\eta^2 \EE' F_+ F_+^\top\\
    \EE' D_+^\top D_+ &\preceq \EE' (\cT_+ - \cT)^\top (\cT_+ - \cT) \preceq 4\eta^2 \EE' F_+^\top F_+
\end{align*}

\textit{Row quadratic variation $\|\EE' D_+D_+^\top\|^{1/2}$.} Since $F_+ F_+^\top =  V_1 V_1^\top \norm{V_2}^2 $, for any unit vector $u \in \RR^{d-r}$, Cauchy-Schwarz inequality yields
\begin{align*}
    \sbr{u^\top E'(F_+F_+^\top)u}^{1/2} = \sbr{E'\rbr{(u^\top V_1)^2 \|V_2\|^2}}^{1/2} &\leq  \sbr{E'(u^\top V_1)^4}^{1/4} \sbr{E'\|V_2\|^4}^{1/4} \\
    &\lesssim C_\psi \lambda_1^{1/2}(\tau+ \nu_\infty) \underbrace{\sbr{E'\|V_2\|^4}^{1/4}}_{=:\tv_{2,4}}
\end{align*}
where the last inequality follows from \eqref{eqn:one_step_V1V2_sub_gaussian}. To bound $\tv_{2,4}$, note that $V_2=Y_++\cT^\top Z_+$ is a linear transformation of $\tX_+$, hence \cref{lem:kth_moment} implies that $\tv_{2,4} \lesssim \sbr{\tr \cov'(V_2)}^{1/2} \cdot \|\tX_+\|_{\psi_2}$ where $\cov'$ denotes the covariance conditional on previous filtration (similar to $\EE'$). Using inequality $\cov'(A+B)\preceq 2\cov'(A) + 2\cov'(B)$ and $(a+b)^{1/2}\leq a^{1/2} + b^{1/2}$ for $a,b \geq 0$ then gives
\begin{align*}
    \sbr{\tr \cov'(V_2)}^{1/2} &\lesssim \sbr{\tr \Lambda_* + \tr \cT^\top \Lambda_{*,\perp} \cT}^{1/2} = \sbr[\Big]{\lambda_{1\sim r} + \Fnorm[\big]{\Lambda_{*,\perp}^{1/2} \cT}^2}^{1/2}\\ 
    &\leq \sbr[\Big]{r\lambda_1 + \norm[\big]{\Lambda_{*,\perp}^{1/2}}^2 \Fnorm{\cT}^2}^{1/2} \leq \sbr{r\lambda_1 + \lambda_{r+1} \cdot r\tau^2}^{1/2}\\
    &\leq (r\lambda_1)^{1/2} \sbr{1+ \nu_\infty^2 \tau^2}^{1/2} \leq (r\lambda_1)^{1/2} (1+ \nu_\infty \tau)
\end{align*}
Thus
\begin{align}
    \tv_{2,4} \lesssim C_{\psi} r^{1/2}\lambda_{1}^{1/2} \rbr{1 + \tau \nu_\infty}. \label{eqn:one_step_v24}
\end{align}
which implies
\begin{align*}
    \|\EE' F_+F_+^\top\|^{1/2} \lesssim C_\psi^2 r^{1/2} \lambda_{1}(\tau + \nu_\infty) \rbr{1 + \tau \nu_\infty} 
\end{align*}
and
\begin{align}
    \|\EE' D_+ D_+^\top\|^{1/2} \lesssim C_\psi^2 \teta \kappa_* \sqrt{r} (1 + \tau\nu_\infty) (\tau + \nu_\infty). \label{eqn:one_step_D_row}
\end{align}

\textit{Column quadratic variation $\|\EE' D_+^\top D_+\|^{1/2}$.} Since $F_+^\top F_+ = \norm{V_1}^2 V_2 V_2^\top$. For any unit vector $v \in \RR^r$, we have
\begin{align*}
    \sbr{v^\top E'(F_+^\top F_+)v}^{1/2} &\leq  \sbr{E'(v^\top V_2)^4}^{1/4} \sbr{E'\|V_1\|^4}^{1/4}\lesssim C_{\psi}\lambda_1^{1/2}(1 + \tau\nu_\infty) \underbrace{\sbr{E'\|V_1\|^4}^{1/4}}_{=:\tv_{1,4}}
\end{align*}
where the last inequality follows from \eqref{eqn:one_step_V1V2_sub_gaussian}.  Recall $V_1=-\cT Y_+ + Z_+$, similar arguments as row quadratic moment give $\tv_{1,4} \lesssim C_\psi \sbr{\tr \cov'(V_1)}^{1/2}$ with
\begin{align*}
    \sbr{\tr \cov'(V_1)}^{1/2} &\lesssim \sbr{\tr \cT \Lambda_* \cT^\top + \tr \Lambda_{*,\perp}}^{1/2} 
    = \sbr[\Big]{\Fnorm[\big]{\Lambda_*^{1/2} \cT^\top}^2 + \tr \Lambda_{*,\perp}}^{1/2}\\
    &\leq \sbr[\Big]{ \Fnorm[\big]{\Lambda_*^{1/2}}^2 \norm{\cT}^2 + \lambda_{r+1\sim d}}^{1/2} \leq \sbr{\lambda_{1 \sim r} \tau^2 + \lambda_{r+1\sim d} }^{1/2}\\
    &\leq \lambda_{1\sim r}^{1/2} (\tau^2 + \nu_2^2)^{1/2} \leq \lambda_{1\sim r}^{1/2} (\tau + \nu_2)
\end{align*}
which implies
\begin{align*}
    \tv_{1,4} \lesssim C_\psi \lambda_{1\sim r}^{1/2} (\tau + \nu_2).
\end{align*}
Thus
\begin{align*}
    \|\EE' F_+^\top F_+\|^{1/2} \lesssim C_\psi^2 \lambda_1^{1/2} \lambda_{1\sim r}^{1/2} (1 + \tau\nu_\infty) (\tau + \nu_2) \leq C_\psi^2 \lambda_1 r^{1/2} (1 + \tau\nu_\infty) (\tau + \nu_2)  
\end{align*}
and
\begin{align}
    \|\EE' D_+^\top D_+\|^{1/2} \lesssim C_\psi^2 \teta \kappa_* \sqrt{r} (1 + \tau\nu_\infty) (\tau + \nu_2). \label{eqn:one_step_D_col}
\end{align}
Combining \eqref{eqn:one_step_D_row} and \eqref{eqn:one_step_D_col} gives the desired variance bound for $D_+$ with $\bnu = \max(\nu_\infty, \nu_2)$.

\paragraph{(II) Higher-order noise $D_+^{(2)}$.} Recall definition
\begin{align*}
    D_+^{(2)} = -\eta^2\! \rbr{ \frac{\zeta_+\,F_+}{1+\eta\,\zeta_+}  - \EE'\!\sbr{\frac{\zeta_+\,F_+}{1+\eta\,\zeta_+}}}, \quad \text{where} \quad F_+ = V_1 V_2^\top, \quad \zeta_+ = V_2^\top Y_+.
\end{align*}

\textbf{Deterministic bound:} Equation~\eqref{eqn:one_step_V1V2_bound} and~\eqref{eqn:one_step_denominator} give
\begin{align*}
    \norm{\frac{\eta^2 \zeta_+\,F_+}{1+\eta\,\zeta_+}} &\leq 2 \eta^2 \abs{\zeta_+} \norm{F_+} \leq 2\eta^2 \abs{\zeta_+} \norm{V_1}\norm{V_2} \\
    &\leq 2\eta^2 C_{\psi,n}^2 \; \lambda_{1\sim r} \rbr{1+\tau \nu_2} \cdot C_{\psi,n}^2 \; \lambda_{1\sim r} (\tau+\nu_2)(1+\tau\nu_2)\\
    &= 2 C_{\psi,n}^4 \eta^2 \lambda_{1\sim r}^2 (1+\tau \nu_2)^2 (\tau+\nu_2)
\end{align*}
which implies
\begin{align*}
    \|D_+^{(2)}\| \leq 4 C_{\psi,n}^4 \eta^2 \lambda_{1\sim r}^2 (1+\tau \nu_2)^2 (\tau+\nu_2) \leq 4 C_{\psi,n}^4 \teta^2 \kappa_*^2 r^2 (1 + \tau\nu_2)^2 (\tau + \nu_2).
\end{align*}

\textbf{Variance bound:} By the second-moment domination of the variance and \eqref{eqn:one_step_denominator}, we have
\begin{align*}
    \EE' \sbr[\big]{D_+^{(2)} (D_+^{(2)})^\top} &\preceq \EE' \sbr[\Big]{\eta^4 \frac{\zeta_+^2}{(1+\eta\zeta_+)^2} F_+ F_+^\top} \preceq 4\eta^4 \EE' \sbr{\zeta_+^2 F_+ F_+^\top} \\
    \EE' \sbr[\big]{(D_+^{(2)})^\top D_+^{(2)}} &\preceq \EE' \sbr[\Big]{\eta^4 \frac{\zeta_+^2}{(1+\eta\zeta_+)^2} F_+^\top F_+} \preceq 4\eta^4 \EE' \sbr{\zeta_+^2 F_+^\top F_+}
\end{align*}

\textit{Row quadratic variation $\|\EE' D_+^{(2)} (D_+^{(2)})^\top\|^{1/2}$.} Since $F_+ F_+^\top = \norm{V_2}^2 V_1 V_1^\top$, for any unit vector $u \in \RR^{d-r}$, Cauchy-Schwarz inequality yields
\begin{align*}
    \sbr{u^\top E'(\zeta_+^2 F_+F_+^\top)u}^{1/2}
    &= \sbr{E'\zeta_+^2\|V_2\|^2(V_2^\top Y_+)^2(u^\top V_1)^2}^{1/2}\\
    &\leq \sbr{E'\zeta_+^4}^{1/4}\;\sbr{E'\bigl[\|V_2\|^4(u^\top V_1)^4\bigr]}^{1/4}\\
    &\leq \sbr{E'\zeta_+^4}^{1/4}\;[E'\|V_2\|^8]^{1/8}[E'(u^\top V_1)^8]^{1/8}\\
    &\lesssim C_\psi^2 r\lambda_1 (1 + \tau \nu_\infty) \underbrace{[E'\|V_2\|^8]^{1/8}}_{=: \tv_{2,8}}  C_\psi\lambda_1^{1/2}(\tau + \nu_\infty)
\end{align*}
where the last inequality follows from \eqref{eqn:one_step_V1V2_sub_gaussian} and \eqref{eqn:one_step_zeta_4}. A similar argument as $\tv_{2,4}$ \eqref{eqn:one_step_v24} gives 
\begin{align}
    \tv_{2,8} &\lesssim C_{\psi} r^{1/2}\lambda_{1}^{1/2} \rbr{1 + \tau \nu_\infty}. \label{eqn:one_step_tv_2_8}
\end{align}
which implies 
\begin{align*}
    \|\EE' \sbr{\zeta_+^2 F_+ F_+^\top}\|^{1/2} \lesssim C_{\psi}^4 \; r^{3/2} \lambda_1^2 (1+\tau \nu_\infty)^2 (\tau + \nu_\infty).
\end{align*}
Thus
\begin{align*}
    \|\EE' D_+^{(2)} (D_+^{(2)})^\top\|^{1/2} \lesssim C_{\psi}^4 \teta^2 \kappa_*^2 r^{3/2} (1+\tau \nu_\infty)^2 (\tau + \nu_\infty).
\end{align*}

\textit{Column quadratic variation $\|\EE' (D_+^{(2)})^\top D_+^{(2)}\|^{1/2}$.} Similar to the row quadratic variation, we have
\begin{align*}
    \|\EE'  (D_+^{(2)})^\top D_+^{(2)}\|^{1/2} \lesssim C_{\psi}^4 \teta^2 \kappa_*^2 r^{3/2} (1+\tau \nu_\infty)^2 (\tau + \nu_2).
\end{align*}
The proof is omitted.

Combining the bounds above gives the desired variance bound for $D_+^{(2)}$ with $\bnu = \max(\nu_\infty, \nu_2)$.

\paragraph{(III) Drift remainder $R_+$.} By definition,
\begin{align*}
    R_+ = -\eta^2  \EE'\!\sbr{\frac{\zeta_+ F_+}{1+\eta \zeta_+}}
\end{align*}
By Lemma~\ref{lem:basic_ineq}, we have $\|E' W\| \leq \|E'WW^\top \|^{1/2}$ for any random matrix $W$. Applying this with $W = \zeta_+ F_+/(1+\eta\zeta_+)$ and using $|1+\eta\zeta_+| \geq 1/2$ give
\begin{align*}
    \|R_+\| \leq \eta^2\norm[\Big]{E'\sbr[\Big]{\frac{\zeta_+^2}{(1+\eta\zeta_+)^2}F_+F_+^\top}}^{1/2} \leq 2\eta^2\left\|E'[\zeta_+^2 F_+F_+^\top]\right\|^{1/2}.
\end{align*}
Since $F_+F_+^\top = \|V_2\|^2 V_1V_1^\top$, for any unit $u \in \mathbb{R}^{d-r}$, we can obtain
\begin{align*}
    \sbr{u^\top E'(\zeta_+^2 F_+F_+^\top)u}^{1/2}
    &\leq \sbr{E'\zeta_+^4}^{1/4}\;\sbr{E'\bigl[\|V_2\|^4(u^\top V_1)^4\bigr]}^{1/4}\\
    &\leq \sbr{E'\zeta_+^4}^{1/4}\;[E'\|V_2\|^8]^{1/8}[E'(u^\top V_1)^8]^{1/8}\\
    &\lesssim C_\psi^2 r\lambda_1 (1 + \tau \nu_\infty) \cdot C_{\psi} r^{1/2}\lambda_{1}^{1/2} \rbr{1 + \tau \nu_\infty} \cdot C_\psi\lambda_1^{1/2}(\tau + \nu_\infty)\\
    &= C_{\psi}^4 r^{3/2} \lambda_1^2 (1 + \tau \nu_\infty)^2 (\tau + \nu_\infty)
\end{align*}
where the last inequality follows from \eqref{eqn:one_step_zeta_4}, \eqref{eqn:one_step_tv_2_8} and \eqref{eqn:one_step_V1V2_sub_gaussian}. Therefore,
\begin{align*}
    \|R_+\| \lesssim C_{\psi}^4 \teta^2 \kappa_*^2 r^{3/2}  (1 + \tau \nu_\infty)^2 (\tau + \nu_\infty).
\end{align*}

\subsection{Proof of Lemma~\ref{lem:varphi}}
\label{sec:prf_varphi}

\textbf{(i)} The fixed-point equation reads $\bar a\tau^2 - (1 - \bar b)\tau + \bar c = 0$.
Since $\tau_0 \geq \tau_* = 3\bar c$: $9\bar a\bar c = 3\bar a \cdot \tau_* \le 3\bar a\tau_0 < 1$.
Hence $4\bar a\bar c < 4/9 = (2/3)^2 < (1 - \bar b)^2$ (using $\bar b < 1/3$), giving positive discriminant. Since $\bar a, \bar c > 0$ and $1 - \bar b > 0$, both roots are positive. Vieta's formulas then give
\begin{align*}
    \tau_+ + \tau_- = \frac{1-\bb}{\ba} \quad\text{and}\quad \tau_+\tau_- = \frac{\bc}{\ba}.
\end{align*}
Since $\tau_+ \geq \tau_- \geq 0$, we then have
\begin{align*}
    \frac{1-\bb}{2\ba}\leq \tau_+ \leq \frac{1-\bb}{\ba} 
\end{align*}
which combined with the equality $\tau_+\tau_- = \bc/\ba$ gives
\begin{align*}
    \frac{\bc}{1-\bb} \leq \tau_- \leq \frac{2 \bc}{1-\bb}
\end{align*}
Plugging in $2/3 < 1 - \bar b < 1$ gives
\begin{align*}
    \frac{1}{3\ba} < \tau_+ < \frac{1}{\ba}, \qquad \bc < \tau_- < 3\bc.
\end{align*}

\textbf{(ii)} By assumption, $\tau_0 \geq \tau_*$. It suffices to prove $\tau_0 < \tau_+$ and $\tau_* > \tau_-$.

$\tau_* > \tau_-$ is implied by $\tau_- < 3\bc$ from (i) since $\tau_*:= 3\bc$.

$\tau_0 < \tau_+$ follows from the assumption $\ba \tau_0 < 1/3$ and $\tau_+ > 1/(3\ba)$ from (i).

\textbf{(iii)} $h(\tau) := \varphi(\tau) - \tau$ is a convex quadratic vanishing at $\tau_\pm$, hence $h < 0$ on $(\tau_-, \tau_+)$.

\textbf{(iv)} Set $t_0=\bc/\bb$. We have $t_0 > \tau_*$ since $\bc/\bb > 3\bc = \tau_*$. 
We consider two cases:

\textit{Case 1 ($\tau \in [t_0, \tau_0]$):}  $\varphi(\tau)/\tau = \ba \tau + \bb +\bc/\tau \leq \ba \tau + 2\bb$ (using $\bc/\tau \leq \bb$).

\textit{Case 2 ($\tau \in [\tau_*,t_0]$):} we have $\varphi(\tau) \leq 3\bc = \tau_*$. This follows from an analysis of the quadratic and linear term in  $\varphi(\tau)=\ba\tau^2 + \bb\tau + \bc$. Linear term: $\bb\tau \leq \bc$ (using $\bb \leq \bc/\tau$). Quadratic term: $\ba \tau^2 \leq \ba \bc^2 / \bb^2 < \bc$ where the last inequality follows from 
\begin{align*}
    \bar b^2 > \bar a\bar c
\end{align*}
which we now prove. Without loss of generality, assume $C_{\ref{lem:1epoch_bound}}=1$. Note that $\nu_2 > 0$ and $\nu_\infty > 0$ otherwise $\ba=0$. 
Since $\bar a\bar c = \max_i a_i \cdot \max_j c_j$ is one of four products $a_i c_j$, it suffices to show $\bar b^2 > a_i c_j$ for all $(i,j) \in \{\sigma, B\}^2$.

Same-index: $\bar b^2 \ge b_i^2 \ge 4a_i c_i > a_i c_i$ (by AM--GM: $(1+x)^2 \ge 4x$).

Cross-index ($a_\sigma c_B$):
            $(1 + \nu_\infty\bar\nu)(1 + \nu_2^2) \ge 2\nu_2 > \nu_\infty\nu_2$ since $\nu_\infty \leq 1$.

Cross-index ($a_B c_\sigma$): note that $\bar b^2 \ge b_\sigma b_B$, hence it suffices to prove
        \begin{align*}
            (1 + \nu_\infty\bar\nu)(1 + \nu_2^2) > \nu_2\bar\nu
        \end{align*} 
        \begin{itemize}
            \item when $\bar\nu = \nu_\infty$ (so $\nu_\infty \ge \nu_2$, both $\leq 1$), $\text{LHS} > 1 \geq \nu_\infty\nu_2 = \nu_2\bar\nu$;
            \item when $\bar\nu = \nu_2$, $\text{LHS} = 1 + \nu_2^2 + \nu_\infty\nu_2 + \nu_\infty\nu_2^3 > \nu_2^2 = \nu_2\bar\nu$.
        \end{itemize}

\subsection{Proof of Lemma~\ref{lem:1epoch_bound}}
\label{sec:prf_1epoch_bound}
Note that $\nu_\infty>0$ implies that $\ba,\bb,\bc>0$. Also, $\eta \lambda_1 = \teta \kappa_*$ which is $\in (0,\frac{1}{2})$ by assumption, hence Lemma~\ref{lem:basic_ineq} can be applied.

The proof has two steps:
Step~1 shows uniform concentration of the stopped accumulated noise and remainder, establishing the existence of $C_{\ref{lem:1epoch_bound}}$ via matrix Freedman's inequality.
Step~2 converts this into \eqref{eqn:1epoch_concentration}-\eqref{eqn:1epoch_stability} via induction.

\paragraph{Step 1: Multi-step concentration and construction of $C_{\ref{lem:1epoch_bound}}$.}  Recall the decomposition
\begin{align*}
    &\cT_k= (\cI - \eta \cL)^k \cT_0 + \cD_k + \cR_k,\\
    \quad\text{where}\quad &\cD_k = \sum_{i=1}^k (\cI - \eta \cL)^{k-i} D_i, \quad \cR_k = \sum_{i=1}^k (\cI - \eta \cL)^{k-i} R_i
\end{align*}
Define the good event
\begin{align*}
    \mathscr{G}_k(\ttau_0, 2)
    :=
    \cbr{\norm{\cT_0} \leq \ttau_0}
    \;\cap\;
    \bigcap_{i=1}^{k} \cbr{\norm{\cT_i} \leq 2 \ttau_0},
    \qquad k \ge 0,
\end{align*}
with the convention $\mathscr{G}_0(\ttau_0, 2) = \cbr{\norm{\cT_0} \leq \ttau_0}$,
and let $\mathbf \bbi_k(\ttau_0)$ denote its indicator.
Define the \emph{stopped} accumulated noise and remainder for $k \in [n]$ as
\begin{align*}
  \cD_k^*(\ttau_0)
  &\;:=\;
  \sum_{i=1}^{k} (\cI - \eta\,\cL)^{k-i}
    \bbi_{i-1}(\ttau_0) D_i,
  \\
  \cR_k^*(\ttau_0)
  &\;:=\;
  \sum_{i=1}^{k} (\cI - \eta\,\cL)^{k-i}\,
    \bbi_{i-1}(\ttau_0)\, R_i\,,
\end{align*}
and
$\cD_k^*(\ttau_0) = \cR_k^*(\ttau_0) = 0$.
On $\mathscr{G}_{k-1}(\ttau_0, 2)$ every indicator $\mathbf \bbi_{i-1} = 1$ for
$i \le k$, so the stopped processes coincide with their unstopped counterparts:
\begin{align*}
    \cD_k^*(\ttau_0) = \cD_k
    \quad\text{and}\quad
    \cR_k^*(\ttau_0) = \cR_k
    \qquad\text{on } \mathscr{G}_{k-1}(\ttau_0, 2).
\end{align*}

\begin{claim}[Multi-step concentration] \label{claim:multi_step_concentration}
    Assume assumptions of Lemma~\ref{lem:1epoch_bound} hold.
    Then there exists a constant $C_{\ref{lem:1epoch_bound}} > 0$ (only depending on $C_d$) such that
    \begin{align*}
        \PP \cbr[\Big]{\norm{\cD_k^*} + \norm{\cR_k^*} \le \frac{\varphi(\tilde\tau_0)}{2} \text{ for all } k \in [n]} \ge 1 - n^{-20}
    \end{align*}
\end{claim}
\begin{proof}
    See Appendix~\ref{sec:prf_multi_step_concentration}.
\end{proof}

Define event
\begin{align*}
    \mathscr{E}_* := \cbr{\norm{\cD_k^*(\ttau_0, 2)} + \norm{\cR_k^*(\ttau_0, 2)}\leq \frac{\varphi(\ttau_0)}{2} \text{ for all } k \in [n]}.
\end{align*}
Claim~\ref{claim:multi_step_concentration} implies $\PP(\mathscr{E}_*) \geq 1 - n^{-20}$.

\paragraph{Step 2: Induction.} We show that on $\mathscr{E}_*$, \eqref{eqn:1epoch_concentration}-\eqref{eqn:1epoch_stability} hold for every $0 \leq k \leq n$ by induction.

\emph{Base case} ($k=0$). Since $\cD_0=\cR_0=0$ and $\|\cT_0\|\le\ttau_0 \le 2\ttau_0$,  \eqref{eqn:1epoch_concentration}-\eqref{eqn:1epoch_stability} hold trivially.

\emph{Inductive step.} Suppose \eqref{eqn:1epoch_concentration}-\eqref{eqn:1epoch_stability} hold for all indices $m=0, 1, \ldots, k-1$ on $\mathscr{E}_*$. Then event $\mathscr{G}_{k-1}(\ttau_0, 2)$ occurs. Thus, by definition of the stopped processes, we have
\begin{align*}
    \cD_k^* = \cD_k, \qquad \cR_k^* = \cR_k \qquad \text{ on } \mathscr{E}_*.
\end{align*}
which implies \eqref{eqn:1epoch_concentration} holds for $m=k$ on $\mathscr{E}_*$.

Triangle inequality then gives
\begin{align*}
    \norm{\cT_k - (\cI - \eta \cL)^k \cT_0} \leq \norm{\cD_k} + \norm{\cR_k} \leq \norm{\cD_k^*} + \norm{\cR_k^*}
    \leq \frac{\varphi(\ttau_0)}{2} \qquad \text{ on } \mathscr{E}_*.
\end{align*}
which is \eqref{eqn:1epoch_descent} for $m=k$.

Moreover, since $\|(\cI - \eta \cL)^k \cT_0\| \leq (1-\eta \Delta_{\min})^k \tau_0 \leq \tau_0 \vee \tau_* = \ttau_0$ by Lemma~\ref{lem:basic_ineq} and $\varphi(\ttau_0) \leq \ttau_0$ (Property~\ref{lem:varphi}(iii)), we have
\begin{align*}
    \norm{\cT_k} &\leq \ttau_0 + \frac{\varphi(\ttau_0)}{2} \leq 2 \ttau_0.
\end{align*}
which is \eqref{eqn:1epoch_stability} for $m=k$.

This completes the induction and hence the proof.

\subsubsection{Proof of Claim~\ref{claim:multi_step_concentration}}
\label{sec:prf_multi_step_concentration}
For brevity, denote $\tau = 2\ttau_0$ and
\begin{align*}
    \cD_k^* = \cD_k^*(\ttau_0), \qquad R_k^* = R_k^*(\ttau_0), \qquad \bbi_k = \bbi_k(\ttau_0).
\end{align*}

\paragraph{(1) Bounds for $\cD_k^*$.} For $i \in [k]$, define
\begin{align*}
    A^{(k)}_i := (\cI - \eta \cL)^{k-i} \bbi_{i-1} D_i.
\end{align*}
Then $\cD_k^* = \sum_{i=1}^k A^{(k)}_i$, and $\cbr[\big]{\sum_{i=1}^j A^{(k)}_i }_{j=0}^k$ is an $\RR^{(d-r) \times r}$-valued martingale. We can calculate
\begin{align}
    \max_{i\in [k]}\, \norm[\big]{A^{(k)}_i} &= \max_{i\in [k]} \norm{(\cI - \eta \cL)^{k-i} \bbi_{i-1} D_i} \nonumber\\ 
    &\overset{(i)}{\leq} \max_{i \in [k]} \norm{\bbi_{i-1} D_i} \overset{(ii)}{\lesssim} C_{\psi,n}^2 \teta  \kappa_* r (1+\tau \nu_2)(\tau + \nu_2) =: B^*_{\cD}(\tau) \label{eqn:prf_freedman_B_D}
\end{align}
and
\begin{align}
    &\max \cbr[\Big]{\norm[\Big]{\sum_{i=1}^k \EE_{i-1} A^{(k)}_i (A^{(k)}_i)^\top}^{1/2}, \norm[\Big]{\sum_{i=1}^k \EE_{i-1} (A^{(k)}_i)^\top A^{(k)}_i}^{1/2}} \nonumber\\
    \overset{(iii)}{\leq} &\sbr[\Big]{\sum_{i=1}^k (1-\teta)^{2(k-i)} \max \cbr{\norm{\EE_{i-1}\bbi_{i-1} D_i D_i^\top}, \norm{\EE_{i-1}\bbi_{i-1} D_i^\top D_i}}}^{1/2} \nonumber \\
    \overset{(iv)}{\lesssim} &\sbr[\Big]{\sum_{i=1}^k (1-\teta)^{2(k-i)}}^{1/2}  C_\psi^2 \teta  \kappa_* \sqrt{r} (1+\tau \nu_\infty)(\tau + \bnu) \nonumber\\
    \lesssim &C_\psi^2 \teta^{1/2}    \kappa_* \sqrt{r} (1+\tau \nu_\infty)(\tau + \bnu) =: \sigma^*_{\cD}(\tau) \label{eqn:prf_freedman_sigma_D}
\end{align}
where (i), (iii) follows from \cref{lem:basic_ineq}; (ii), (iv) follows from \cref{lem:one_step_bound}.

In view of the matrix Freedman's inequality (\cref{thm:freedman}), with probability exceeding $1-n^{-21}$
\begin{align}
    \norm{\cD_k^*} &\lesssim \sigma_{\cD}^*(\tau)\sqrt{\log n} +  B_{\cD}^*(\tau) \log n \label{eqn:prf_freedman_cD}
\end{align}
where the last line follows from the definition of $\varphi_\sigma$ and $\varphi_B$ in \eqref{eqn:varphi_definitions}.

\paragraph{(2) Bound for $\cR_k^*$.} Applying \cref{lem:one_step_bound} and \cref{lem:basic_ineq},  we can obtain
\begin{align}
    \norm{\cR_k^*} &\leq \sum_{i=1}^k (1-\teta)^{k-i} \norm{\bbi_{i-1} R_i} \nonumber\\
    &\lesssim  \sum_{i=1}^k (1-\teta)^{k-i} C_{\psi}^4 \teta^2  \kappa_*^2 r^{3/2}  (\tau + \nu_\infty)(1+\tau \nu_\infty)^2 \nonumber\\
    &\leq C_{\psi}^4 \teta  \kappa_*^2 r^{3/2}  (\tau + \nu_\infty)(1+\tau \nu_\infty)^2 =: B^*_{\cR}(\tau) \label{eqn:prf_freedman_B_R}
\end{align}

\paragraph{(3) Variance domination.} Combining the above the bounds \eqref{eqn:prf_freedman_cD} and \eqref{eqn:prf_freedman_B_R}, we have with probability exceeding $1-n^{-21}$,
\begin{align*}
    \norm{\cD_k^*} + \norm{\cR_k^*} &\lesssim \sigma_{\cD}^*(\tau)\sqrt{\log n}  + B_{\cD}^*(\tau) \log n + B^*_{\cR}(\tau)
\end{align*}
We next show that the remainder term $B^*_{\cR}(\tau)$ is dominated by  $\sigma_{\cD}^*(\tau)\sqrt{\log n}$. To this end, note that
\begin{align*}
    \sigma_{\cD}^*(\tau)\sqrt{\log n} \gtrsim B^*_{\cR}(\tau) \sqrt{\log n} \iff C_{\psi}^2 \; \sqrt{\teta} \, \kappa_* r \frac{(1+\tau \nu_\infty) (\tau + \nu_\infty)}{\tau + \bnu} \lesssim 1
\end{align*}
For $\tau > 0$, we have $\tau+\nu_\infty \leq \tau+\bnu$, thus
\begin{align*}
    \frac{(1+\tau \nu_\infty)(\tau + \nu_\infty)}{\tau + \bnu} 
    &\leq 1 + \tau \nu_\infty
\end{align*}
Hence, a sufficient condition for $\sigma_{\cD}^*(\tau)\sqrt{\log n} \gg B^*_{\cR}(\tau)$ is
\begin{align*}
    C_{\psi}^2  \sqrt{\teta} \kappa_* r ( 1 + \tau \nu_\infty ) \lesssim 1
\end{align*}
which holds under \eqref{eqn:1epoch_conditions}.

\paragraph{(4) Defining $\varphi$.}
Combining the bounds for $\cD_k^*$ and $\cR_k^*$ and the conditions for variance domination, we have with probability exceeding $1-n^{-21}$,
\begin{align*}
    \norm{\cD_k^*} + \norm{\cR_k^*} &\lesssim \sigma_{\cD}^*(\tau)\sqrt{\log n}  + B_{\cD}^*(\tau) \log n
\end{align*}
Thus, if we take
\begin{align*}
    \varphi(\ttau_0) :=  \sbr[\big]{C_{\ref{lem:1epoch_bound}} \sigma_{\cD}^*(\tau)\sqrt{\log n}}  \vee_{\mathrm{cw}} \sbr[\big]{ C_{\ref{lem:1epoch_bound}}B_{\cD}^*(\tau) \log n}
\end{align*}
for some sufficiently large constant $C_{\ref{lem:1epoch_bound}} > 0$, then we have
\begin{align*}
    \PP \cbr[\Big]{\norm{\cD_k^*} + \norm{\cR_k^*} \leq \frac{\varphi(\ttau_0)}{2}} &\geq 1 - n^{-21}
\end{align*}
After plugging in the definition of $B_{\cD}^*$ and $\sigma_{\cD}^*$ from \eqref{eqn:prf_freedman_B_D} and~\eqref{eqn:prf_freedman_sigma_D}, we see that $\varphi$ is exactly the function defined in \eqref{eqn:varphi_def}.

Finally, a union bound over $k \in [n]$ gives the desired result.

% --------------------------------------------------------------    

\subsection{Proof of Corollary~\ref{cor:epoch_target}}
\label{sec:prf_epoch_target}

By Lemma~\ref{lem:1epoch_bound}, with probability at least $1-n^{-20}$,
\[
\|\cT_k\|\le (1-\teta)^k \tau_0 + \frac{\varphi(\tau_0)}{2},
\qquad k=0,1,\dots,n .
\]
For all $k\geq N(\tau_0,\delta)$, $(1-\teta)^k \tau_0 \le \delta/2$, while
$\varphi(\tau_0)/2 \le \delta/2$ by assumption,
hence $\|\cT_k\|\le \delta$.

% --------------------------------------------------------------    

\subsection{Proof of Lemma~\ref{lem:epochs}}
\label{sec:prf_epochs}

\textbf{(I) Epoch structure.} Since $\nu_\infty>0$, we have $\ba,\bb,\bc>0$. Moreover,
\begin{align*}
    \max\cbr{\bar a\otau_0, \bar b} < \frac{1}{3n^{\tepsilon}\log^+ n} < \frac{1}{3}
\end{align*}
Hence the assumption in \cref{lem:varphi} is satisfied.

(a) Let $\varphi^{\circ j}$ be the $j$-th iterate of $\varphi$. Lemma~\ref{lem:varphi}( iii) implies that $\tau_-$ is a stable fixed point over $(\tau_-, \tau_+)$ and the sequence $\{\varphi^{\circ j}(\otau_0)\}_{j \geq 0}$ decreases monotonically to $\tau_-$.  Moreover, we have $\tau_- < \tau_* \leq \otau_0 < \tau_+$ from Lemma~\ref{lem:varphi}(ii), thus (a) follows.

(b) For any nontruncated epoch $j$ with $\varphi(\otau_j) \geq \tau_*$, Lemma~\ref{lem:varphi}(iv) gives
\begin{align*}
    \otau_{j+1} = \varphi(\otau_j) \leq (\ba \otau_0 + 2\bb)\otau_j < \frac{1}{n^\varepsilon \log n} \otau_j
\end{align*}
Thus (b) follows.

(c) Note that
    \begin{align*}
        \sum_{j=0}^{\oJ_*-1} \log(\otau_{j}/\otau_{j+1}) = \log(\otau_0/\otau_{\oJ_*}) = \log(\tau_0/\tau_*) \le \log_+(\tau_0/\tau_*)
    \end{align*}
    Note that the $j$-th epoch must be nontruncated if $j \leq \oJ_*-2$, thereby each contributes $\ge \tepsilon\log n + \log\log n$ to the sum above by (b). Thus, we have
    \begin{align*}
        \log_+(\tau_0/\tau_*) \geq \sum_{j=0}^{\oJ_*-2} \log(\otau_j/\otau_{j+1}) \geq (\oJ_*-1)(\tepsilon\log n + \log\log n)
    \end{align*}
    which implies
    \begin{align*}
        \oJ_* - 1 \leq \frac{\log_+(\tau_0/\tau_*)}{\tepsilon\log n + \log\log n}
    \end{align*}
    and hence (c) follows.

(d) Since Corollary~\ref{cor:epoch_target} implies that
\begin{align*}
    n_j \le \teta^{-1} \log\rbr{\frac{\otau_{j}}{\otau_{j+1}}} + \teta^{-1}\log 2 + 1
\end{align*}
for $j=0,...,\oJ_*-1$.
Summing over $j$ gives the bound on $N_{J_*}$.

\textbf{(II) $\norm{\cT_k}$ bound} follows from iteratively applying Lemma~\ref{lem:1epoch_bound} and Corollary~\ref{cor:epoch_target} over epochs $j=0,1,\ldots,\oJ_*$, and taking a union bound over the failure probabilities of each epoch.

\subsection{Proof of Lemma~\ref{lem:summarize_dynamics}}
\label{sec:prf_summarize_dynamics}
Under assumption $C_\psi^2 \sqrt{\teta} \kappa_* r \rbr{1 + 2 \ttau_0 \nu_\infty} \leq \log^{-3/2} n$, we have
\begin{align}
    \teta \leq \sbr[\Big]{\frac{1}{(\log n)^{3/2} r C_\psi^2 \kappa_*}}^2 \overset{(i)}{\leq} \frac{1}{r^2 \log^3 n} \label{eqn:teta_bd}
\end{align}
where (i) follows from $\kappa_* \geq 1$ and $C_\psi \geq 1$.

\textbf{(i) ($\oN_{\oJ_*} \leq n$)} Note that
    \begin{align}
        n \teta = c_\eta \log n = \frac{c_\eta \log n}{C_{\mathrm{bud}}\log^+ (\sqrt{2 c_\eta} \, \tau_0/\tau_*) } C_{\mathrm{bud}} \log_+ (\sqrt{2 c_\eta} \, \tau_0/\tau_*)
        \label{eqn:n_teta}
    \end{align}

    Then, by Lemma~\ref{lem:epochs}, we can obtain
    \begin{align*}
        &\oN_{\oJ_*} \leq \frac{\log_+(\tau_0/\tau_*)}{\teta} + \rbr{\frac{\log_+(\tau_0/\tau_*)}{\tepsilon \log n + \log\log n} + 1}\rbr{\frac{\log 2}{\teta}+1}\\
        &\overset{(a)}{\leq} \frac{\log_+ (\sqrt{2 c_\eta} \, \tau_0/\tau_*)}{\teta} + \rbr{\frac{\log_+ (\sqrt{2 c_\eta} \, \tau_0/\tau_*)}{\tepsilon \log n + \log\log n} + 1}\rbr{\frac{\log 2}{\teta}+1}\\
        &= n \frac{\log_+ (\sqrt{2 c_\eta} \, \tau_0/\tau_*)}{n \teta} \sbr[\Big]{1 + \rbr[\Big]{\frac{1}{\tepsilon \log n + \log\log n} + \frac{1}{\log_+ (\sqrt{2 c_\eta} \, \tau_0/\tau_*)}}\rbr{\log 2 + \teta } } \\
        &\overset{(b)}{\leq} n \frac{\log_+ (\sqrt{2 c_\eta} \, \tau_0/\tau_*)}{n \teta} \sbr[\Big]{1 + \rbr[\Big]{\frac{1}{\tepsilon \log n + \log\log n} + \frac{1}{\log_+ (\sqrt{2 c_\eta} \, \tau_0/\tau_*)}}\rbr{\log 2 + \log^{-3} n } } \\
        &\overset{(c)}{=} n \frac{\log_+ (\sqrt{2 c_\eta} \, \tau_0/\tau_*) }{c_\eta \log n} C_{\mathrm{bud}}
    \end{align*}
    where (a) is due to $2c_\eta \geq 1$, (b) results from \eqref{eqn:n_teta} and (c) follows from the definition of $C_{\mathrm{bud}}$ \eqref{eq:C_bud}.

\textbf{(ii) ($\tau_* \asymp \sqrt{\teta} \rho_\sigma \bnu$)} By definition \eqref{eqn:tau*_def}, $\tau_* := 3\max\{\sqrt{\teta} \rho_\sigma\bar\nu, \teta \rho_B\nu_2\}$. Thus it suffices to show $\teta \rho_B\nu_2 \lesssim \sqrt{\teta} \rho_\sigma \bnu$. Straightforward computation gives
\begin{align*}
    \frac{\teta \rho_B \nu_2}{\sqrt{\teta} \rho_\sigma \bnu} = \sqrt{\teta r} \log^{3/2} n \leq \frac{1}{\sqrt{r}}
\end{align*}
where the last inequality follows from~\eqref{eqn:teta_bd}. The proof is then complete.

\subsection{Proof of Lemma~\ref{lem:convergence_rate_optimal}}
\label{sec:prf_convergence_rate_optimal}
In view of Lemma~\ref{lem:summarize_dynamics}, it suffices to verify the conditions therein are satisfied. Note that by definition, $c_\eta \geq c_\eta^* \geq \frac{1}{2}$. The conditions to be verified are listed below:

\textup{(C1)} $\teta \kappa_* \in (0,1/2)$.

\textup{(C2)} $\max(\bar a\tau_0, \bar b) \le 1/(3n^{\tilde\epsilon}\log n)$ for some $\tilde\epsilon \ge 0$.

\textup{(C3)} $\teta C_{\psi,n}^2 \kappa_* r (1 + 2\tau_0\nu_2) \le 1/2$.

\textup{(C4)} $C_\psi^2\, \teta^{1/2} \kappa_* r (1 + \tau_0\nu_\infty) \leq \log^{-3/2} n$.

\textup{(C5)} $c_\eta \log n \ge C_{\mathrm{bud}} \log^+ (\sqrt{2 c_\eta}\tau_0/\tau_*)$, where
\begin{align*}%\label{eq:C_bud}
    C_{\mathsf{bud}} &=
    1 \! + \!\rbr[\Big]{\frac{1}{\tepsilon \log n + \log\log n} \!+\! \frac{1}{\log_+ (\sqrt{2 c_\eta}\tau_0/\tau_*)}}\rbr{\log 2 + \log^{-3} n }
\end{align*}

Define
\begin{align*}
    \Phi :&= \frac{\sqrt{c_\eta} \kappa_* r^{3/2}\max(1, \sqrt{\nu_\infty d})  \log^2 n}{\sqrt{n}} \frac{\sqrt{\log(1/\delta_p)}}{\delta_p}\\
    &= \sqrt{\teta} \kappa_* r^{3/2} \max(1, \sqrt{\nu_\infty d}) (\log n)^{3/2} \frac{\sqrt{\log(1/\delta_p)}}{\delta_p}
\end{align*}
Assumption~\ref{eqn:cond_sample_sz_lem} then reads
\begin{equation}
    \Phi \leq \frac{1}{C_{\ref{lem:convergence_rate_optimal}}^{1/2} n^{\varepsilon}}.
\end{equation}

\paragraph{(C1).} Since $\eta \lambda_1 = \teta \kappa_* = \frac{c_\eta \log n}{n} \kappa_*$, we have
\begin{align*}
    \frac{\eta \lambda_1}{\Phi^2} \leq \frac{1}{\log^3 n} \frac{\delta_p^2}{\log(1/\delta_p)} \leq 1
\end{align*}
where we have used $\kappa_* \geq 1$ and $\delta_p \in (0,1/e)$. Thus,
\begin{align*}
    \eta \lambda_1 &\leq \frac{1}{C_{\ref{lem:convergence_rate_optimal}}}\\
    &< \frac{1}{2} \qquad \text{for } C_{\ref{lem:convergence_rate_optimal}} \text{ large enough}.
\end{align*}

\paragraph{(C2).} By definition, $\max(\ba \tau_0, \bb) = \max\cbr{a_\sigma \tau_0, a_B \tau_0, b_\sigma, b_B}$. We consider each term separately.

\textit{Term $b_\sigma$.} By definition, $b_\sigma = \sqrt{\teta} C_{\ref{lem:1epoch_bound}} C_\psi^2 \kappa_* \sqrt{r}(1+\nu_\infty \bnu)\sqrt{\log n}$. Note that $\bnu = \max(\nu_2, \nu_\infty) \asymp \nu_2 +\nu_\infty$, we have
\begin{align*}
    1+\nu_\infty \bnu &\lesssim 1+ \nu_\infty (\nu_2 + \nu_\infty) \leq 2 + \nu_\infty \sqrt{d} \nu_\infty \lesssim 1 \vee \sqrt{\nu_\infty d}
\end{align*}
where the last two inequalities exploit $\nu_\infty \leq 1$ and $\nu_2 \leq \sqrt{d} \nu_\infty$. Thus, we have
\begin{align*}
    \frac{b_\sigma}{\Phi} \lesssim C_{\ref{lem:1epoch_bound}} C_\psi^2 \frac{1}{r \log n} \frac{\sqrt{\log(1/\delta_p)}}{\delta_p}
\end{align*}
thus
\begin{align*}
    b_\sigma &\lesssim \frac{C_{\ref{lem:1epoch_bound}} C_\psi^2}{C_{\ref{lem:convergence_rate_optimal}}^{1/2}} \frac{1}{n^{\varepsilon} \log n}\\
    &< \frac{1}{3 n^{\varepsilon} \log n} \qquad \text{for } C_{\ref{lem:convergence_rate_optimal}} \text{ large enough}.
\end{align*}

\textit{Term $b_B$.} By definition, $b_B = \teta C_{\ref{lem:1epoch_bound}} C_\psi^2 \kappa_* r (\log n)^2 (1+\nu_2^2)$. Since
\begin{align*}
    1+\nu_2^2 \leq 1 + d \nu_\infty^2 \lesssim \max(1, d\nu_\infty)
\end{align*}
we have
\begin{align*}
    \frac{b_B}{\Phi^2} &\lesssim C_{\ref{lem:1epoch_bound}} C_\psi^2 \frac{1}{\log n}
\end{align*}
which implies 
\begin{align*}
    b_B &\lesssim \frac{C_{\ref{lem:1epoch_bound}} C_\psi^2}{C_{\ref{lem:convergence_rate_optimal}}} \frac{1}{n^{2\varepsilon} \log n}\\
    &< \frac{1}{3 n^{\varepsilon} \log n} \qquad \text{for } C_{\ref{lem:convergence_rate_optimal}} \text{ large enough}.
\end{align*}

\textit{Term $a_\sigma \tau_0$.} By definition,
\begin{align*}
    a_\sigma \tau_0 &\lesssim \sqrt{\teta} C_{\ref{lem:1epoch_bound}} C_\psi^2\kappa_* \sqrt{r\log n} \cdot \sqrt{rd \log(1/\delta_p)}\, \delta_p^{-1}\\
    &\lesssim C_{\ref{lem:1epoch_bound}} C_\psi^2 \Phi \log^{-1} n\\
    &< \frac{1}{3 n^{\varepsilon} \log n} \qquad \text{for } C_{\ref{lem:convergence_rate_optimal}} \text{ large enough}.
\end{align*}

\textit{Term $a_B \tau_0$.} By definition and Assumption~\ref{assumption:init}, we can obtain
\begin{align*}
    a_B \tau_0 &\lesssim \teta C_{\ref{lem:1epoch_bound}} C_\psi^2 \kappa_* r (\log n)^2 \sqrt{d}\nu_\infty \cdot \sqrt{rd \log(1/\delta_p)}\, \delta_p^{-1}\\
    &\lesssim C_{\ref{lem:1epoch_bound}} C_\psi^2 \Phi^2 \log^{-1} n \lesssim \frac{C_{\ref{lem:1epoch_bound}} C_\psi^2}{C_{\ref{lem:convergence_rate_optimal}}} \frac{1}{n^{2\varepsilon} \log n}\\
    &< \frac{1}{3 n^{\varepsilon} \log n} \qquad \text{for } C_{\ref{lem:convergence_rate_optimal}} \text{ large enough}.
\end{align*}

\paragraph{(C3).} It suffices to show
\begin{align*}
    \teta C_{\psi,n}^2 \kappa_* r  \le 1/4 \qquad \text{and} \qquad \teta C_{\psi,n}^2 \kappa_* r \cdot 2\tau_0\nu_2 \le 1/4
\end{align*}

\textit{First term.}
\begin{align*}
    \teta C_{\psi,n}^2 \kappa_* r &\lesssim \teta C_\psi^2 (\log n) \kappa_* r \leq \frac{1}{\Phi^2}
\end{align*}
Thus $\teta C_{\psi,n}^2 \kappa_* r < 1/4$ for $C_{\ref{lem:convergence_rate_optimal}}$ large enough.

\textit{Second term.} By Assumption~\ref{assumption:init}, we have
\begin{align*}
    \teta C_{\psi,n}^2 \kappa_* r \cdot 2\tau_0\nu_2 &\lesssim \teta C_\psi^2 (\log n) \kappa_* r \cdot \sqrt{rd \log (1/\delta_p)}\, \delta_p^{-1} \sqrt{d}\nu_\infty\\
    &= C_\psi^2 \teta \kappa_* r^{3/2} d \nu_\infty (\log n) \sqrt{\log(1/\delta_p)} \, \delta_p^{-1} \\
    &\leq \Phi^2  
\end{align*}
Hence $\teta C_{\psi,n}^2 \kappa_* r \cdot 2\tau_0\nu_2 < 1/4$ for $C_{\ref{lem:convergence_rate_optimal}}$ large enough.

\paragraph{(C4).} It suffices to show
\begin{align*}
    C_\psi^2\, \teta^{1/2} \kappa_* r \leq \frac{1}{2 \log^{3/2} n} \quad \text{and} \quad C_\psi^2\, \teta^{1/2} \kappa_* r \tau_0\nu_\infty \leq \frac{1}{2 \log^{3/2} n}
\end{align*}

\textit{First term.} $C_\psi^2 \teta^{1/2} \kappa_* r \log^{3/2} n \leq \Phi < 1/2$ for $C_{\ref{lem:convergence_rate_optimal}}$ large enough.

\textit{Second term.}
\begin{align*}
    C_\psi^2\, \teta^{1/2} \kappa_* r \tau_0\nu_\infty &\lesssim C_\psi^2\, \teta^{1/2} \kappa_* r \sqrt{rd \log(1/\delta_p)}\, \delta_p^{-1}\nu_\infty\\
    &\leq C_\psi^2\, \teta^{1/2} \kappa_* r^{3/2} \sqrt{d \nu_\infty} \sqrt{\log(1/\delta_p)}\, \delta_p^{-1} \leq \Phi \log^{-3/2} n\\
    &< \frac{1}{2 \log^{3/2} n} \qquad \text{for } C_{\ref{lem:convergence_rate_optimal}} \text{ large enough}.
\end{align*}

\paragraph{(C5).} It suffices to show
\begin{align}
    c_\eta &\geq \frac{1}{\log n}C_{\mathrm{bud}} \log^+ (\sqrt{2 c_\eta}\tau_0/\tau_*) \nonumber\\
    &= \frac{1}{\log n} \sbr{\rbr{1 + \frac{\log 2 + \log^{-3} n}{\epsilon \log n + \log\log n}} \log^+ (\sqrt{2 c_\eta}\tau_0/\tau_*) + \log 2 + \log^{-3} n} \label{eqn:c_eta_RHS}
\end{align} 
where the last line follows from the definition of $C_{\mathrm{bud}}$.

By definition of $\tau_*$ \eqref{eqn:tau*_def}, we have
\begin{align*}
    \tau_* \geq 3\sqrt{\teta} \rho_\sigma \nu_\infty &= 3\sqrt{\teta} \, C_{\ref{lem:1epoch_bound}} C_\psi^2 \kappa_* \sqrt{r \log n} \, \nu_\infty
\end{align*}
By \eqref{eqn:cond_nu_lem}, we have
\begin{align}
    \frac{\tau_*}{\sqrt{2 c_\eta}} \geq\frac{3}{\sqrt{2}}C_{\ref{lem:1epoch_bound}} C_\psi^2 \frac{\log n}{\sqrt{n}} \kappa_* \sqrt{r} \, \cdot c_\nu n^{-C_\nu} 
    \label{eqn:tau*_lower}
\end{align}
By Assumption~\ref{assumption:init},
\begin{align}
    \tau_0 &\leq C_{\mathsf{init}} \delta_p^{-1}\sqrt{rd\log(1/\delta_p)} \nonumber\\
    &\leq C_{\mathsf{init}} \sqrt{r} n^{\frac{1}{2}- \varepsilon + \frac{C_\nu}{2}} \frac{1}{\kappa_* r  (\log n)^2 \sqrt{C_{\ref{lem:convergence_rate_optimal}} c_\eta c_\nu} }
    \label{eqn:tau0_upper}
\end{align}
Here the last line follows from assumption \eqref{eqn:cond_sample_sz_lem} and \eqref{eqn:cond_nu_lem}, which together imply
\begin{align*}
    n^{1-2\varepsilon} \geq C_{\ref{lem:convergence_rate_optimal}} c_\eta \kappa_*^2 r^2 \nu_\infty d (\log n)^4 \frac{\log(1/\delta_p)}{\delta_p^2} \geq C_{\ref{lem:convergence_rate_optimal}} c_\eta \kappa_*^2 r^2 (c_\nu n^{-C_\nu}) d (\log n)^4 \frac{\log(1/\delta_p)}{\delta_p^2}
\end{align*}
giving, after rearrangement,
\begin{align*}
    d \frac{\log(1/\delta_p)}{\delta_p^2} \leq n^{1-2\varepsilon + C_\nu} \frac{1}{C_{\ref{lem:convergence_rate_optimal}} c_\eta \kappa_*^2 r^2 c_\nu (\log n)^4}
\end{align*}

Combining \eqref{eqn:tau*_lower} and \eqref{eqn:tau0_upper} gives
\begin{align*}
    \frac{\sqrt{2 c_\eta} \tau_0}{\tau_*} \lesssim C_{\ref{lem:convergence_rate_optimal}}^{-1/2} n^{1-\varepsilon + \frac{3C_\nu}{2}} \log^{-3} n
\end{align*}
Thus for $C_{\ref{lem:convergence_rate_optimal}}$ large enough, we have
\begin{align*}
    \log_+ (\sqrt{2 c_\eta} \tau_0/\tau_*) &\leq \rbr[\Big]{1+\frac{3 C_\nu}{2} - \varepsilon} \log n - 3 \log \log n
\end{align*}
Hence
\begin{align*}
    \text{RHS of \eqref{eqn:c_eta_RHS}} &\leq \sbr{ \rbr{1+\frac{3 C_\nu}{2} - \varepsilon} - \frac{3 \log \log n}{\log n}} \rbr{1 + \frac{\log 2 + \log^{-3} n}{\epsilon \log n + \log\log n}} + \frac{\log 2 + \log^{-3} n}{\log n}\\
    &\leq  \rbr{1+\frac{3 C_\nu}{2} - \varepsilon}\rbr{1 + \frac{\log 2 + \log^{-3} n}{\epsilon \log n + \log\log n}} =: c_\eta^*, \qquad \text{for } n \geq 4.
\end{align*}
where the last line follows from the fact that
\begin{align*}
    \frac{3\log \log n}{\log n} \rbr{1 + \frac{\log 2 + \log^{-3} n}{\frac{1}{2} \log n + \log\log n}} \geq \frac{\log 2 + \log^{-3} n}{\log n} \qquad \text{for } n \geq 4.
\end{align*}
Therefore, we have
\begin{align*}
    \frac{C_{\mathrm{bud}} \log^+ (\sqrt{2 c_\eta}\tau_0/\tau_*) }{c_\eta \log n} \leq \frac{c_\eta^*}{c_\eta} \leq 1
\end{align*}
for $n \geq 4$ and $C_{\ref{lem:convergence_rate_optimal}}$ large enough.

The proof is then complete.

\subsection{Proof of Lemma~\ref{lem:rankr}}
\label{sec:prf_rankr}

\paragraph{Step 0: Preliminaries.}
Since $\operatorname{rank}(\Sigma)=r$, we have $\Lambda_{\ast,\perp}=0$. Therefore $Z_i=0$ almost surely for all $i\in[n]$,
\[
    \nu_2=\nu_\infty=\bar\nu=0, \qquad \Delta_{\min}=\lambda_r, \qquad \teta = \eta \lambda_r = c_\eta \frac{\log n}{n}, \qquad \cL M = M \Lambda_*.
\]
\begin{itemize}[leftmargin=1.5em]
  \item \textit{Monotonicity.} By Lemma~\ref{lem:recurrence}, the tangent recursion reduces to
  \begin{align}
      \cT_{k+1} = \cT_k - \frac{\eta\cT_kY_{k+1}Y_{k+1}^\top} {1+\eta\|Y_{k+1}\|^2}= \cT_k (I_r+\eta Y_{k+1}Y_{k+1}^\top)^{-1}, \label{eqn:rankr_mult}
  \end{align}
  where the last equality follows from the Sherman-Morrison formula.
  Since $I_r+\eta Y_{k+1}Y_{k+1}^\top\succ0$ and $\| (I_r+\eta Y_{k+1}Y_{k+1}^\top)^{-1}\|\le 1$, we have
  \begin{align}
    \operatorname{col}(\cT_{k+1}) = \operatorname{col}(\cT_k), \qquad \|\cT_{k+1}\| \le \|\cT_k\|, \qquad k\ge0. \label{eqn:rankr_monotonicity}
  \end{align}

  \item \textit{Sample-size condition.} Define
  \[
    \Phi_{\rm rk} := \frac{ \sqrt{c_\eta}\,\kappa_\ast r(\log n)^2 }{\sqrt n}.
  \]
  The sample-size condition is equivalent to
  \begin{align}
    \Phi_{\rm rk} \le \frac{1}{\sqrt{C_{\rm rk}}\,n^\varepsilon}. \label{eqn:rankr_cond_sample_sz_Phi}
  \end{align}
  For $C_{\rm rk}$ sufficiently large, this implies
  \begin{align}
    C_{\psi,n}^2\teta\kappa_\ast r\le\frac12, \qquad \eta\lambda_1=\teta\kappa_\ast\le\frac12. \label{eqn:rankr_cond_sample_sz}
  \end{align}
  Since $\nu_2=0$, the first inequality ensures that Lemma~\ref{lem:one_step_bound} applies at every step.
\end{itemize}

\paragraph{Step 1: Contraction over one block.}
For any start time $s\in\{0,\ldots,n-1\}$ and every $t\ge1$ such that $s+t\le n$, iterating~\eqref{eqn:linearization} gives
\begin{equation}
  \cT_{s+t} = (\cI-\eta\cL)^t\cT_s + \sum_{i=s+1}^{s+t} (\cI-\eta\cL)^{s+t-i}D_i + \sum_{i=s+1}^{s+t} (\cI-\eta\cL)^{s+t-i}R_i. \label{eqn:rankr_decomp}
\end{equation}
Specializing Lemma~\ref{lem:one_step_bound} to $\nu_2=\nu_\infty=\bar\nu=0$ and using contraction $\|\cT_{i-1}\|\le\tau_s$, where $\tau_s := \|\cT_s\|$, gives
\begin{equation}
    \begin{aligned}
    \|D_i\| &\lesssim C_{\psi,n}^2\teta\kappa_\ast r\tau_s,\\
    \max \cbr[\big]{\|\EE_{i-1}D_iD_i^\top\|^{1/2}, \|\EE_{i-1}D_i^\top D_i\|^{1/2}} &\lesssim C_\psi^2\teta\kappa_\ast\sqrt r\tau_s,\\
    \|R_i\| &\lesssim C_\psi^4\teta^2\kappa_\ast^2r^{3/2}\tau_s.
    \end{aligned} \label{eqn:rankr_one_step_bd}
\end{equation}
With these bounds, we can bound the three terms in~\eqref{eqn:rankr_decomp} as follows. 
\begin{itemize}[leftmargin=*]
  \item Linear term. By Lemma~\ref{lem:basic_ineq}, $\|(\cI-\eta\cL)^j\| \le (1-\teta)^j$, $j\ge0$. Thus $\|(\cI - \eta \cL)^t \cT_s\| \leq (1-\teta)^t \tau_s$.
  
  \item Martingale term. Define $A_i^{(t)} := (\cI-\eta\cL)^{s+t-i}D_i$ for $i=s+1,\ldots,s+t$.
  Then $\max_i\|A_i^{(t)}\| \lesssim C_{\psi,n}^2\teta\kappa_\ast r\tau_s$. Moreover,
  \begin{align*}
    &\max\cbr[\Big]{\norm[\big]{\sum_{i=s+1}^{s+t} \EE_{i-1}A_i^{(t)}A_i^{(t)\top}}^{1/2}, \norm[\big]{\sum_{i=s+1}^{s+t} \EE_{i-1}A_i^{(t)\top}A_i^{(t)}}^{1/2}} \\
    &\lesssim C_\psi^2\teta\kappa_\ast\sqrt r\,\tau_s \sbr[\Big]{\sum_{j=0}^\infty(1-\teta)^{2j}}^{1/2} \\
    &\lesssim C_\psi^2\kappa_\ast \sqrt{r\teta}\,\tau_s.
  \end{align*}
  We now exploit the multiplicative structure in \eqref{eqn:rankr_mult}, which is specific to the exact-rank case, to reduce the dimension factor in matrix Freedman's inequality from the ambient dimension $d$ to at most $2r$. This removes the need for Assumption~\ref{assumption:d}. Writing
  \begin{align*}
    K_i:= \frac{Y_i Y_i^\top}{1+\eta\|Y_i\|^2}, \qquad i\in[n],
  \end{align*}
  we have $\cT_i - \cT_{i-1} = -\eta \cT_{i-1} K_i$, and hence $D_i = -\eta\cT_{i-1} \bigl(K_i-\EE K_i\bigr)$,
which implies $\operatorname{col}(D_i)\subseteq \operatorname{col}(\cT_s)$.
Moreover, since $\cL D_i=D_i\Lambda_\ast$ in the exact rank-$r$ case, $(\cI-\eta\cL)^j D_i = D_i(I_r-\eta\Lambda_\ast)^j$. Right multiplication cannot enlarge the column space, so
\[
    \operatorname{col}(A_i^{(t)})
    \subseteq \operatorname{col}(\cT_s),
    \qquad i=s+1,\ldots,s+t.
\]
Let
\[
    q_s:=\operatorname{rank}(\cT_s)\le r.
\]
If $q_s=0$, then $\cT_s=0$, and monotonicity implies
$\cT_k=0$ for all $k\ge s$, so the claimed bound is immediate.
Suppose henceforth that $q_s\ge1$, and let
$V_s\in\mathbb O_{d-r,q_s}$ have columns spanning
$\operatorname{col}(\cT_s)$. Conditionally on $\cF_s$, define
\[
    \widetilde A_i^{(t)} := V_s^\top A_i^{(t)} \in\RR^{q_s\times r}.
\]
Then
\[
    A_i^{(t)} = V_s\widetilde A_i^{(t)}, \qquad \|A_i^{(t)}\| = \|\widetilde A_i^{(t)}\|,
\]
and
\[
    \norm[\Big]{\sum_{i=s+1}^{s+t}A_i^{(t)}} = \norm[\Big]{\sum_{i=s+1}^{s+t}\widetilde A_i^{(t)}}.
\]
Furthermore,
\begin{align*}
    \norm[\Big]{\sum_{i=s+1}^{s+t} \EE_{i-1} \widetilde A_i^{(t)} \widetilde A_i^{(t)\top}}
    &=\norm[\Big]{V_s^\top \sbr[\big]{\sum_{i=s+1}^{s+t} \EE_{i-1}A_i^{(t)}A_i^{(t)\top}} V_s} \\
    &\le \norm[\Big]{\sum_{i=s+1}^{s+t} \EE_{i-1}A_i^{(t)}A_i^{(t)\top}},
\end{align*}
while, because
$\operatorname{col}(A_i^{(t)})\subseteq\operatorname{col}(V_s)$,
\[
    \widetilde A_i^{(t)\top}\widetilde A_i^{(t)}
    =
    A_i^{(t)\top}A_i^{(t)}.
\]
Thus the same increment and quadratic-variation bounds hold for the
$q_s\times r$ martingale
$\{\sum_{i=s+1}^j\widetilde A_i^{(t)}\}_{j=s}^{s+t}$. Moreover, for $C_{\rm rk}$ sufficiently large, \eqref{eqn:rankr_cond_sample_sz_lem} implies $r\le n$, and hence $q_s+r\le 2r\le 2n$. Therefore, matrix Freedman's inequality gives,
with probability at least~$1-n^{-20}$,
\[
    \sup_{1\le t\le n-s} \norm[\Big]{\sum_{i=s+1}^{s+t}A_i^{(t)}} = \sup_{1\le t\le n-s} \norm[\Big]{\sum_{i=s+1}^{s+t}\widetilde A_i^{(t)}} \lesssim \sbr[\big]{C_\psi^2\kappa_\ast\sqrt{r\teta\log n} + C_{\psi,n}^2\teta\kappa_\ast r\log n}\tau_s.
\]

  \item Remainder term. \eqref{eqn:rankr_one_step_bd} gives
  \begin{align*}
      \sup_{1\le t\le n-s} \norm[\Big]{\sum_{i=s+1}^{s+t} (\cI-\eta\cL)^{s+t-i}R_i} \lesssim C_\psi^4\teta^2\kappa_\ast^2r^{3/2}\tau_s \sum_{j=0}^\infty(1-\teta)^j \lesssim C_\psi^4\teta\kappa_\ast^2r^{3/2}\tau_s .
  \end{align*}
\end{itemize}
Combining the three bounds, we obtain that, conditionally on $\cF_s$,
with probability at least $1-n^{-20}$,
\begin{equation}
  \|\cT_{s+t}\| \le [(1-\teta)^t+b_{\rm rk}]\tau_s, \qquad 0\le t\le n-s, \label{eqn:rankr_one_block}
\end{equation}
where
\[
    b_{\rm rk} :=C\sbr[\big]{C_\psi^2\kappa_\ast\sqrt{r\teta\log n} + C_{\psi,n}^2\teta\kappa_\ast r\log n + C_\psi^4\teta\kappa_\ast^2r^{3/2}}.
\]

We next bound $b_{\rm rk}$. Since $\teta=c_\eta(\log n)/n$, we have
\begin{align*}
  \kappa_\ast\sqrt{r\teta\log n} &= \frac{\Phi_{\rm rk}}{\sqrt r\,\log n},\\
  \teta\kappa_\ast r(\log n)^2 &= \frac{\Phi_{\rm rk}^2}{\kappa_\ast r\log n},\\
  \teta\kappa_\ast^2r^{3/2} &= \frac{\Phi_{\rm rk}^2} {\sqrt r\,(\log n)^3}.
\end{align*}

Since $\kappa_\ast,r\ge1$ and $C_{\psi,n}^2\lesssim C_\psi^2\log n$, it follows from~\eqref{eqn:rankr_cond_sample_sz_Phi} that
\begin{equation}
  b_{\rm rk} \lesssim \frac{\Phi_{\rm rk}}{\log n} + \frac{\Phi_{\rm rk}^2}{\log n} \lesssim \frac{1}{\sqrt{C_{\rm rk}}\,n^\varepsilon\log n}. \label{eqn:rankr_b_bound}
\end{equation}

\paragraph{Step 2: Block chaining.} Set $m_\eta := \ceil{(\log 4)/\teta}$. Then
\begin{equation}
  m_\eta\teta\ge\log4, \label{eqn:rankr_m_eta_lb}
\end{equation}
which combined with $1-\teta\le e^{-\teta}$ gives
\[
    (1-\teta)^{m_\eta} \le e^{-m_\eta\teta} \le\frac14.
\]
Define
\[
    \rho_{\rm rk} := (1-\widetilde\eta)^{m_\eta}+b_{\rm rk}.
\]
By \eqref{eqn:rankr_b_bound}, taking $C_{\rm rk}$ sufficiently large
also ensures $b_{\rm rk}\le1/4$, and hence
\[
    \rho_{\rm rk}\le\frac12.
\]

Applying the one-block bound \eqref{eqn:rankr_one_block}
successively at the deterministic block endpoints $s_j=jm_\eta$,
and using the tower property and a union bound over at most $n$ blocks,
we obtain, conditionally on $\cF_0$, with probability at least
$1-n^{-19}$,
\[
    \|\cT_{s_j}\| \le \rho_{\rm rk}^j\|\cT_0\|, \qquad \forall s_j\le n.
\]

For arbitrary $k\le n$, let
$j=\lfloor k/m_\eta\rfloor$. By monotonicity,
\[
    \|\cT_k\| \le \|\cT_{jm_\eta}\| \le \rho_{\rm rk}^j\|\cT_0\|.
\]
Since $j\ge k/m_\eta-1$ and $\rho_{\rm rk}\in(0,1)$,
\[
    \rho_{\rm rk}^j \le \rho_{\rm rk}^{-1} \exp\left\{\frac{k}{m_\eta}\log\rho_{\rm rk} \right\} = \rho_{\rm rk}^{-1} \exp\{-c_{\rm rk}k\teta\},
\]
where
\[
    c_{\rm rk} := -\frac{\log\rho_{\rm rk}}{m_\eta\teta}.
\]
To control the prefactor $\rho_{\rm rk}^{-1}$, note that $\teta\le1/2$ and $\log(1-\teta) \ge -\frac{\teta}{1-\teta} \ge -2\teta$.
Moreover, $m_\eta\teta \le \log4+\teta \le \log4+\frac12$.
Hence
\begin{equation}
  \rho_{\rm rk} \ge (1-\teta)^{m_\eta} \ge e^{-2m_\eta\teta} \ge \frac1{16e}. \label{eqn:rankr_rho_lb}
\end{equation}
Thus $\rho_{\rm rk}^{-1}\le16e$, and therefore
\begin{equation}
    \|\cT_k\| \leq 16e \exp\{-c_{\rm rk}k\teta\}\|\cT_0\|, \qquad k\le n.
    \label{eqn:rankr_exp_contraction}
\end{equation}

\paragraph{Step 3: Lower bounding $c_{\rm rk}$.}
By definition,
\begin{align*}
    c_{\rm rk} = -\frac{\log[(1-\teta)^{m_\eta}+b_{\rm rk}]}{m_\eta\teta} &= \frac{-\log(1-\teta)}{\teta} - \frac{1}{m_\eta\teta} \log\left(1+\frac{b_{\rm rk}}{(1-\teta)^{m_\eta}}\right)\\
    &\ge 1- \frac{b_{\rm rk}}{m_\eta\teta\,(1-\teta)^{m_\eta}},
\end{align*}
where the last line follows from $-\log(1-x)\ge x$ for $x\in(0,1)$ and $\log(1+x)\le x$ for $x\ge0$.

Using \eqref{eqn:rankr_m_eta_lb} and \eqref{eqn:rankr_rho_lb}, we obtain
\[
    c_{\rm rk} \ge 1-\frac{16e}{\log4}b_{\rm rk}.
\]
By \eqref{eqn:rankr_b_bound}, choosing $C_{\rm rk}$ sufficiently large yields
\[
    c_{\rm rk} \ge 1-\frac{1}{n^\varepsilon\log n}.
\]
Combining this with~\eqref{eqn:rankr_exp_contraction} proves \eqref{eqn:rankr_contraction}.

\paragraph{Additional conclusion under Assumption~\ref{assumption:init}.}
If, in addition, Assumption~\ref{assumption:init} and \eqref{eqn:rankr_cond_sample_sz_lem_further} hold, then
\begin{align*}
    \|\cT_0\| \le C_{\mathrm{init}} \delta_p^{-1}\sqrt{rd\log(1/\delta_p)}
    \lesssim \frac{\sqrt d\,n^{1/2-\epsilon}} {\sqrt{c_\eta}\kappa_\ast\sqrt r(\log n)^2}.
\end{align*}
Combining this with \eqref{eqn:rankr_contraction} gives
\[
    \|\cT_k\|
    \lesssim
    \frac{\sqrt d\,n^{1/2-\epsilon}}
    {\sqrt{c_\eta}\kappa_\ast\sqrt r(\log n)^2}
    \exp\left\{
        -\left(
            1-\frac{1}{n^\epsilon\log n}
        \right)k\teta
    \right\}.
\]
Since $n\teta=c_\eta\log n$, the stated bound for $\cT_n$ follows.

% ---------------------------------------------------------

\subsection{Proof of Lemma~\ref{lem:linearization}}
\label{sec:prf_linearization}

Denote $N_* = \floor{n c_\eta^*/ c_\eta}$. Lemma~\ref{lem:convergence_rate_optimal} implies that
\begin{align*}
    \PP(\mathscr{G}^c) \leq n^{-19}, \qquad \text{where } \mathscr{G} := \cbr{\norm{\cT_k} \leq 2\tau_* \text{ for all } k \in [N_*,n]}.
\end{align*}
Decomposing $\cT_n$ from $\cT_{N_*}$ gives
\begin{align*}
    \cT_n &= (\cI-\eta \cL)^{n-N_*} \cT_{N_*} + \cD_{N_*,n}^{(1,0)} + \cD_{N_*,n}^{(1,1)} + \cD_{N_*,n}^{(2)} + \cR_{N_*,n}
\end{align*}
where
\begin{align*}
    \cD_{N_*,n}^{(1,0)} &:= \eta \sum_{i=N_* + 1}^n (\cI-\eta \cL)^{n-i} (Z_i Y_i^\top)\\
    \cD_{N_*,n}^{(1,1)} &:= \eta \sum_{i=N_* + 1}^n (\cI-\eta \cL)^{n-i} \sbr{-\cT_{i-1}(Y_i Y_i^\top - \Lambda_*) + (Z_i Z_i^\top - \Lambda_{*,\perp})\cT_{i-1} -\cT_{i-1} Y_i Z_i^\top \cT_{i-1} }\\
    \cD_{N_*,n}^{(2)} &:= \sum_{i=N_* + 1}^n (\cI-\eta \cL)^{n-i} D_i^{(2)} \\
    \cR_{N_*,n} &:= \sum_{i=N_* + 1}^n (\cI-\eta \cL)^{n-i} R_i
\end{align*}
Then
\begin{align*}
    \norm[\Big]{\cT_n - \eta \sum_{i=1}^n (\cI-\eta \cL)^{n-i}(Z_iY_i^\top)} &\leq \underbrace{\norm{(\cI-\eta \cL)^{n-N_*} \cT_{N_*}}}_{=:\alpha_1} + \underbrace{\norm[\Big]{\eta \sum_{i=1}^{N_*} (\cI-\eta \cL)^{n-i} (Z_i Y_i^\top)}}_{\alpha_2}\\
    & + \underbrace{\norm{\cD_{N_*,n}^{(1,1)}}}_{=:\alpha_3} + \underbrace{\norm{\cD_{N_*,n}^{(2)}}}_{=:\alpha_4} + \underbrace{\norm{\cR_{N_*,n}}}_{=:\alpha_5}
\end{align*}
We consider each term separately.

\paragraph{(0) Preliminaries.} We first record some preliminary results that will be used in the subsequent analysis.

\textit{(i)}
Recall definition of $\tau_*$ \eqref{eqn:tau*_def}
\begin{align*}
    \tau_* &= 3C_{\ref{lem:1epoch_bound}} C_\psi^2 \kappa_* \max \rbr{\sqrt{\teta r \log n} \bnu, \teta r(\log n)\nu_2} \asymp C_{\ref{lem:1epoch_bound}} C_\psi^2 \kappa_* \sqrt{\teta r \log n} \; \bnu
\end{align*}
Here the last $\asymp$ follows from Lemma~\ref{lem:summarize_dynamics}(ii).

Since $\kappa_*\geq1$ and $\delta_p \in (0,1/e)$, the sample size condition~\eqref{eqn:cond_sample_sz_lem} implies that
\begin{align*}
    n^{-\varepsilon} \geq C_{\ref{lem:convergence_rate_optimal}}^{1/2} \kappa_* \sqrt{\teta r^3} (\log n)^{3/2}, \quad \text{for } C_{\ref{lem:convergence_rate_optimal}} \text{ large enough}.
\end{align*}
Thus,
\begin{equation}
    \begin{aligned}
        \tau_* \asymp C_{\ref{lem:1epoch_bound}} C_\psi^2 \kappa_* \sqrt{\teta r \log n} &\leq \frac{C_{\ref{lem:1epoch_bound}} C_\psi^2}{C_{\ref{lem:convergence_rate_optimal}}^{1/2}} \frac{1}{r n^{\varepsilon} \log n}\\
        &\leq \frac{1}{n^{\varepsilon} \log n} \qquad \text{for } C_{\ref{lem:convergence_rate_optimal}} \text{ large enough}.
    \end{aligned}
    \label{eqn:tau_leq1}
\end{equation}

\textit{(ii)}
Recall \eqref{eqn:varphi_definitions} and \eqref{eqn:varphi_def}
\begin{align*}
    \varphi_\sigma(\tau_*) &= C_{\ref{lem:1epoch_bound}} C_\psi^2 \kappa_* \sqrt{\teta r \log n}  (1+\tau_* \nu_\infty) (\tau_*+\bnu) \\
    \varphi_B(\tau_*) &= C_{\ref{lem:1epoch_bound}} C_\psi^2 \kappa_* \teta r (\log n)^2  (1+\tau_* \nu_2)(\tau_*+\nu_2) \\
    \varphi(\tau_*) &= \rbr{\varphi_\sigma \vee_{\mathrm{cw}} \varphi_B}(\tau_*)
\end{align*}
and we have shown in Lemma~\ref{lem:varphi}(iii) that $\varphi(\tau_*)< \tau_*$. Thus
\begin{align*}
    \max\rbr{\varphi_\sigma(\tau_*), \varphi_B(\tau_*)} \leq \varphi(\tau_*) < \tau_*
\end{align*}

\paragraph{(1) $\alpha_1$.} Under $\mathscr{G}$, we have
\begin{align*}
    \alpha_1 \leq 2(1-\teta)^{n-N_*} \tau_*
\end{align*}

\paragraph{(2) $\alpha_2$.} Using $n-i = (n-N_*) + (N_*-i)$, we have
\begin{align*}
    \alpha_2 \leq (1-\teta)^{n-N_*} \norm[\Big]{\eta \sum_{i=1}^{N_*} (\cI-\eta \cL)^{N_*-i} (Z_i Y_i^\top)}
\end{align*}
Denote $\Xi_{2,i} = \eta (\cI-\eta \cL)^{N_*-i} (Z_i Y_i^\top)$. To apply Freedman's inequality, similar to the proof of Lemma~\ref{lem:one_step_bound}, we have
\begin{align*}
    \max_{i \in [N_*]} \norm{\Xi_{2,i}} \leq \eta \norm{Z_i} \norm{Y_i} &\leq \eta C_{\psi,n}^2 \sqrt{\lambda_{1\sim r} \lambda_{r+1\sim d}}\\
    &\leq \eta C_{\psi,n}^2 r\lambda_1 \nu_2 =: B_2
\end{align*}
and
\begin{align*}
    &\quad \max \cbr[\Big]{\norm[\Big]{\sum_{i=1}^{N_*} \EE_{i-1} \Xi_{2,i} \Xi_{2,i}^\top}^{1/2}, \norm[\Big]{\sum_{i=1}^{N_*} \EE_{i-1} \Xi_{2,i}^\top \Xi_{2,i}}^{1/2}}\\
    &\leq \eta \max \cbr{\norm{\EE ZY^\top Y Z^\top}^{1/2}, \norm{\EE YZ^\top ZY^\top}^{1/2}} \sbr[\Big]{\sum_{i=1}^{N_*} (1-\teta)^{2(N_*-i)} }^{1/2}\\
    &\leq C_\psi^2 \eta \max \cbr{\sqrt{\lambda_{1 \sim r} \lambda_{r+1}}, \sqrt{\lambda_1 \lambda_{r+1 \sim d}}} \teta^{-1/2}\\ 
    &\leq C_\psi^2 \sqrt{\teta r} \kappa_*\bnu =: \sigma_2 
\end{align*}
Then Freedman's inequality implies that with probability exceeding $1 - n^{-20}$,
\begin{align*}
    \norm[\Big]{\eta \sum_{i=1}^{N_*} (\cI-\eta \cL)^{N_*-i} (Z_i Y_i^\top)} &\overset{(i)}{\leq} C_{\ref{lem:1epoch_bound}} \sbr[\big]{\sigma_2 \sqrt{\log n} + B_2 \log n}\\
    &= C_{\ref{lem:1epoch_bound}} C_\psi^2 \sqrt{\teta r} \, \kappa_* \bnu \sqrt{\log n} + C_{\ref{lem:1epoch_bound}}  C_\psi^2 \teta r \kappa_* \nu_2 (\log n)^2\\
    &= \sqrt{\teta \rho_\sigma \bnu} + \teta \rho_B \nu_2 \asymp \tau_*
\end{align*}
Here the constant in (i) is the same as that in Lemma~\ref{lem:1epoch_bound} since both invoke Freedman's inequality and the matrix has the same dimension $(d-r) \times r$. Therefore, we have
\begin{align*}
    \alpha_2 \lesssim (1-\teta)^{n-N_*} \tau_*
\end{align*}

\paragraph{(3) $\alpha_3$.} Denote
\begin{align*}
    S_1 &= \eta \sum_{i=N_* + 1}^n (\cI-\eta \cL)^{n-i} \sbr{-\cT_{i-1}(Y_i Y_i^\top - \Lambda_*)},\\
    S_2 &= \eta \sum_{i=N_* + 1}^n (\cI-\eta \cL)^{n-i} \sbr{(Z_i Z_i^\top - \Lambda_{*,\perp})\cT_{i-1}},\\
    S_3 &= \eta \sum_{i=N_* + 1}^n (\cI-\eta \cL)^{n-i} \sbr{-\cT_{i-1} Y_i Z_i^\top \cT_{i-1}}.
\end{align*}
We consider $S_1,S_2,S_3$ separately. 

\textit{Term $S_1$:} Denote
    \begin{align*}
        S_1^* &=  \sum_{i=N_* + 1}^n \Xi_{1,i}, \quad \text{where } \Xi_{1,i}:=\eta(\cI-\eta \cL)^{n-i} \sbr{\cT_{i-1}(Y_i Y_i^\top - \Lambda_*) } \bbi\cbr{\norm{\cT_{i-1}}\leq 2\tau_*}
    \end{align*}
    To apply Freedman's inequality on $S_1^*$, we can calculate
    \begin{align*}
        \max_i \|\Xi_{1,i}\|
        \leq 2\eta \tau_* \norm{Y_iY_i^\top - \Lambda_*} \leq 2\eta \tau_* (C_{\psi,n}^2 \lambda_{1 \sim r} + \lambda_1) &\leq 4C_{\psi,n}^2 \teta\tau_* 
        r \kappa_*\\
        &\lesssim C_\psi^2 \teta r \kappa_* \tau_* \log n
         =: B_{3,1}
    \end{align*}
    and
    \begin{align*}
        & \quad \max\cbr[\Big]{\norm[\big]{ \sum_{i=N_* + 1}^n \EE_{i-1} \Xi_{1,i} \Xi_{1,i}^\top}^{1/2}, \norm[\big]{ \sum_{i=N_* + 1}^n \EE_{i-1} \Xi_{1,i}^\top \Xi_{1,i}}^{1/2}} \\
        &\overset{(a)}{\leq} 2\eta \tau_* \sbr[\Big]{\sum_{i=N_* + 1}^n (1-\teta)^{2(n-i)} \norm{\EE_{i-1} (Y_i Y_i^\top - \Lambda_*)^2}}^{1/2} \overset{(b)}{\leq} 2 \eta \tau_* \cdot \teta^{-1/2} C_\psi^2 \sqrt{\lambda_1 \lambda_{1\sim r}}\\
        &\lesssim C_\psi^2 \sqrt{\teta r} \tau_* \kappa_* =: \sigma_{3,1}
    \end{align*}
    where (a) follows from
    \begin{align*}
        \norm{\bbi\cbr{\norm{\cT_{i-1}}\leq 2\tau_*} \EE_{i-1} \cT_{i-1}(Y_i Y_i^\top - \Lambda_*)(Y_i Y_i^\top - \Lambda_*)^\top \cT_{i-1}^\top} &\leq 4\tau_*^2 \norm{\EE_{i-1} (Y_i Y_i^\top - \Lambda_*)^2}\\
        \norm{\bbi\cbr{\norm{\cT_{i-1}}\leq 2\tau_*} \EE_{i-1}  (Y_i Y_i^\top - \Lambda_*)^\top \cT_{i-1}^\top \cT_{i-1} (Y_i Y_i^\top - \Lambda_*) } &\leq 4\tau_*^2 \norm{\EE_{i-1} (Y_i Y_i^\top - \Lambda_*)^2}
    \end{align*}
    and (b) is due to Lemma~\ref{lem:compare}(i).
    Freedman inequality (\cref{thm:freedman}) then implies that with probability at least $1-n^{-20}$,
    \begin{align*}
        \norm{S_1^*} \leq C_{\ref{lem:1epoch_bound}} \sbr{\sigma_{3,1} \sqrt{\log n} + B_{3,1} \log n}.
    \end{align*}
    On good event $\mathscr{G}$, we have $S_1 \bbi_{\mathscr{G}} = S_1^* \bbi_{\mathscr{G}}$. Thus, with probability at least $1-n^{-20}$,
    \begin{align*}
        \norm{S_1} &\leq C_{\ref{lem:1epoch_bound}} \sbr{\sigma_{3,1} \sqrt{\log n} + B_{3,1} \log n}.\\
        &\leq C_{\ref{lem:1epoch_bound}} C_\psi^2 \kappa_* \sqrt{r} \sbr{\sqrt{\teta \log n } + \sqrt{r} \teta \log^2 n} \tau_* \\
        &\lesssim C_{\ref{lem:1epoch_bound}} C_\psi^2 \kappa_* \sqrt{r \teta \log n}\; \tau_*
    \end{align*}
    Here the last line follows from $\sqrt{\teta \log n} \gg \sqrt{r} \teta \log^2 n$, which is a consequence of the sample size condition \eqref{eqn:cond_sample_sz_lem}.
    To see this, note that $\kappa_* \geq 1$ and $\delta_p \in (0,1/e)$, thus \eqref{eqn:cond_sample_sz_lem} implies
    \begin{align*}
        n^{-2\varepsilon} \geq C_{\ref{lem:convergence_rate_optimal}} \teta \kappa_*^2 r^3 (\log n)^3
    \end{align*}
    
    Applying $n^{-2\varepsilon} \geq C_{\ref{lem:convergence_rate_optimal}} \teta \kappa_*^2 r^3 (\log n)^3$ once more gives
    \begin{align*}
        C_{\ref{lem:1epoch_bound}} C_\psi^2 \kappa_* \sqrt{r \teta \log n}\; \tau_* \leq \frac{1}{n^{\varepsilon} \log n} \tau_*
    \end{align*}
    Thus, we have
    \begin{align*}
        \norm{S_1} &\lesssim C_{\ref{lem:1epoch_bound}} C_\psi^2 \kappa_* \sqrt{r \teta \log n}\; \tau_* \lesssim \frac{\tau_*}{n^{\varepsilon} \log n}
    \end{align*}

\textit{Term $S_2$:} Similar to $S_1$, by exploiting $\lambda_{r+1 \sim d} = \lambda_{1\sim r} \nu_2^2$ and $\lambda_{r+1}=\lambda_1 \nu_\infty^2$, we can obtain that with probability at least $1-n^{-20}$, 
    \begin{align*}
        \norm{S_2} &\lesssim C_{\ref{lem:1epoch_bound}} C_\psi^2 \kappa_* \sqrt{r} \sbr{\sqrt{\teta \log n } \,\nu_2 \nu_\infty + \sqrt{r} \teta \nu_2^2 \log^2 n} \tau_* \\
        &\leq C_{\ref{lem:1epoch_bound}} C_\psi^2 \kappa_* \sqrt{r} \sbr{\sqrt{\teta \log n } \,\bnu + \sqrt{r} \teta \bnu^2 \log^2 n} \tau_*\\
        &\lesssim \sbr{\tau_* + \tau_*^2}\tau_* \lesssim \tau_*^2
    \end{align*}
where the last step follows since $\tau_* \lesssim n^{-\varepsilon}\log^{-1} n$ by \eqref{eqn:cond_sample_sz_lem}.

\textit{Term $S_3$:} define
    \begin{align*}
        S_3^*=\sum_{i=N_* + 1}^n \Xi_{3,i}, \quad \text{where } \Xi_{3,i} = \eta(\cI-\eta \cL)^{n-i} \sbr{\cT_{i-1} Y_i Z_i^\top \cT_{i-1} } \bbi\cbr{\norm{\cT_{i-1}}\leq 2\tau_*}
    \end{align*}
    Note that by Lemma~\ref{lem:basic_ineq}, we can obtain that $\max_i \norm{\Xi_{3,i}} \leq \eta (2\tau_*)^2 \norm{Y_i Z_i^\top}$ and 
    \begin{align*}
        \norm{\EE_{i-1} \Xi_{3,i} \Xi_{3,i}^\top}
        &\leq (1-\teta)^{2(n-i)} \eta^2 \norm[\big]{\EE_{i-1} \!\sbr{\cT_{i-1} Y_i Z_i^\top \cT_{i-1} } \! \sbr{\cT_{i-1} Y_i Z_i^\top \cT_{i-1} }^\top \!\!\bbi\cbr{\norm{\cT_{i-1}}\leq 2\tau_*}}\\
        &\leq (1-\teta)^{2(n-i)} \eta^2 (2\tau_*)^4 \norm{\EE_{i-1}\sbr{Y_i Z_i^\top Z_i Y_i^\top}}\\
        \norm{\EE_{i-1} \Xi_{3,i}^\top \Xi_{3,i}} &\leq (1-\teta)^{2(n-i)} \eta^2 \norm[\big]{\EE_{i-1} \sbr{\cT_{i-1} Y_i Z_i^\top \cT_{i-1} }^\top \! \!\sbr{\cT_{i-1} Y_i Z_i^\top \cT_{i-1} } \! \bbi\cbr{\norm{\cT_{i-1}}\leq 2\tau_*}}\\
        &\leq (1-\teta)^{2(n-i)} \eta^2 (2\tau_*)^4 \norm{\EE_{i-1} \sbr{Z_i Y_i^\top Y_i Z_i^\top}}
    \end{align*}
    Using results from (2) $\alpha_2$, we have
    \begin{align*}
        \norm{S_3^*} &\lesssim \tau_*^3
    \end{align*}

\textit{Combining $S_1,S_2,S_3$:} We have
\begin{align*}
    S_1 + S_2 + S_3 &\lesssim \tau_* \sbr{C_{\ref{lem:1epoch_bound}} C_\psi^2 \kappa_* \sqrt{r \teta \log n} + \tau_* + \tau_*^2}
\end{align*}
up to absolute constant.

\paragraph{(4) $\alpha_4$.} Define
\begin{align*}
    \cD_{N_*,n}^{(2) *} &:= \sum_{i=N_* + 1}^n \Xi_{4,i}, \quad \text{where } \Xi_{4,i} = (\cI-\eta \cL)^{n-i} D_i^{(2)} \bbi\cbr{\norm{\cT_{i-1}}\leq 2\tau_*}
\end{align*}
To apply Freedman's inequality, by Lemma~\ref{lem:one_step_bound} and~\ref{lem:basic_ineq}, we can calculate
\begin{align*}
    \max_i \norm{\Xi_{4,i}} &\leq \norm[\big]{D_i^{(2)} \bbi \cbr{\norm{\cT_{i-1}} \leq 2\tau_*}} \lesssim C_\psi^4 \teta^2 r^2 \kappa_*^2 (1+\tau_* \nu_2)^2 (\tau_* + \nu_2) \log^2 n =: B_4
\end{align*}
and
\begin{align*}
    &\quad \max\cbr[\Big]{\norm[\big]{\sum_{i=N_* + 1}^n \EE_{i-1} \Xi_{4,i} \Xi_{4,i}^\top}^{1/2}, \norm[\big]{\sum_{i=N_* + 1}^n \EE_{i-1} \Xi_{4,i}^\top \Xi_{4,i}}^{1/2}}\\ 
    &\leq \frac{1}{\sqrt{\teta}} \max \cbr[\Big]{\norm[\big]{\EE_{i-1} D_i^{(2)} D_i^{(2)\top} \! \bbi\cbr{\norm{\cT_{i-1}}\leq 2\tau_*} }^{1/2}, \norm[\big]{\EE_{i-1} D_i^{(2)\top} \! D_i^{(2)}\bbi\cbr{\norm{\cT_{i-1}}\leq 2\tau_*} }^{1/2} }\\
    &\leq C_\psi^4 \teta^{3/2} r^{3/2} \kappa_*^2 (1+\tau_* \nu_\infty)^2 (\tau_* + \bnu) =: \sigma_4
\end{align*}
Then Freedman's inequality implies that with probability at least $1-n^{-20}$,
\begin{align*}
    &\quad \norm[\big]{\cD_{N_*,n}^{(2)*}} \leq C_{\ref{lem:1epoch_bound}} \sbr{\sigma_4 \sqrt{\log n} + B_4 \log n}
\end{align*}
On good event $\mathscr{G}$, we have $\cD_{N_*,n}^{(2)} \bbi_{\mathscr{G}} = \cD_{N_*,n}^{(2)*} \bbi_{\mathscr{G}}$. Thus, with probability at least $1-n^{-20}$,
\begin{align*}
    \norm{\cD_{N_*,n}^{(2)}} &\leq \varphi_{\sigma,*}^{(2)} + \varphi_{B,*}^{(2)}\\
    \text{where} \qquad \varphi_{\sigma,*}^{(2)}&:=C_{\ref{lem:1epoch_bound}} C_\psi^4 \teta^{3/2}  \kappa_*^2 r^{3/2} (1+\tau_* \nu_\infty)^2 (\tau_* + \bnu) \sqrt{\log n}\\ 
    \varphi_{B,*}^{(2)} &:= C_{\ref{lem:1epoch_bound}} C_\psi^4 \teta^2  \kappa_*^2 r^2 (1+\tau_* \nu_2)^2 (\tau_* + \nu_2) \log^3 n
\end{align*}
Comparing $\varphi_{\sigma,*}^{(2)}$ and $\varphi_{B,*}^{(2)}$ with $\varphi_\sigma(\tau_*)$ and $\varphi_B(\tau_*)$ gives
\begin{align*}
    \frac{\varphi_{\sigma,*}^{(2)}}{\varphi_{\sigma}(\tau_*)} &= C_\psi^2 \teta \kappa_*  r (1+\tau_* \nu_\infty) \leq 2C_\psi^2 \teta r \kappa_*\\
    &= 2\sqrt{\teta} \cdot C_\psi^2 \sqrt{\teta}  \, \kappa_* r
\end{align*}
and
\begin{align*}
    \frac{\varphi_{B,*}^{(2)}}{\varphi_{B}(\tau_*)} &= C_\psi^2 \teta \kappa_* r (1+\tau_* \nu_2) \log n\\
    &\leq \sqrt{\teta} \cdot C_\psi^2 \sqrt{\teta}  \kappa_* r \rbr{1+ \tau_* \sqrt{d \nu_\infty}} \log n
\end{align*}

Combining the above two bounds gives
\begin{align*}
    \alpha_4 &\lesssim \tau_* \sqrt{\teta} \cdot C_\psi^2 \sqrt{\teta} \kappa_* r \max(1, \tau_* \sqrt{d \nu_\infty}) \log n\\
    &\lesssim \tau_* \sqrt{\teta} \cdot \frac{1}{n^{\varepsilon} \log n}
\end{align*}
where the last step follows from $\tau_* \leq 1$ and $C_\psi^2 \sqrt{\teta}\, \kappa_* r \max(1, \sqrt{d \nu_\infty}) \lesssim n^{-\varepsilon} \log^{-3/2} n$, which is a consequence of the sample size condition \eqref{eqn:cond_sample_sz_lem}.

\paragraph{(5) $\alpha_5$.} Similar to the proof of Claim~\ref{claim:multi_step_concentration}, we have
\begin{align*}
    \alpha_5 \lesssim \tau_* \cdot C_\psi^2 \sqrt{\teta} \kappa_* r(1+\tau_* \nu_\infty) &\overset{(i)}{\lesssim} \tau_* \cdot C_\psi^2 \sqrt{\teta} \kappa_* r\\
    &\overset{(ii)}{\lesssim} \tau_* \cdot \frac{1}{n^{\varepsilon} \log n}
\end{align*}
where (i) is due to $\tau_* \leq 1$ and $\nu_\infty \leq 1$; (ii) is a consequence of the sample size condition \eqref{eqn:cond_sample_sz_lem}.
\paragraph{(6) Collecting all terms,} we have with probability at least $1-n^{-18}$,
\begin{align*}
    \sum_{i=1}^5 \alpha_i &\lesssim \tau_* \sbr{e^{-\teta (n-N_*)} + C_{\ref{lem:1epoch_bound}} C_\psi^2 \kappa_* \sqrt{r \teta \log n} + \tau_* + \sqrt{\teta} \cdot \frac{1}{n^{\varepsilon} \log n} }
\end{align*}

Note that $(1-\teta)^{n-N_*}\leq e^{-\teta (n-N_*)}$.
Since $N_* \leq \tfrac{c_\eta^*}{c_\eta} n$, we have $n-N_* \geq n (1 - \frac{c_\eta^*}{c_\eta})$ and
\begin{align*}
    \teta (n-N_*) \geq c_\eta \frac{\log n}{n} \cdot n(1 - c_\eta^*/c_\eta) = (c_\eta - c_\eta^*) \log n \geq \frac{1}{2} \log n.
\end{align*}
Thus, $(1-\teta)^{n-N_*} \leq e^{-\frac{1}{2} \log n} = n^{-1/2} \lesssim \teta^{1/2}$. 
Since $C_{\ref{lem:1epoch_bound}}\geq 1$, $C_\psi\geq 1$, $\kappa_* \geq 1$, we finally have
\begin{align*}
    \sum_{i=1}^5 \alpha_i &\lesssim \tau_* \sbr{C_{\ref{lem:1epoch_bound}} C_\psi^2 \kappa_* \sqrt{r \teta \log n} + \tau_*  } \lesssim \tau_* \cdot \frac{1}{n^{\epsilon}\log n}.
\end{align*}

\section{Analysis of Oja's algorithm: unbounded sub-Gaussian observations}
\label{sec:oja_subG}

This appendix extends the bounded-observation analysis of Appendix~\ref{sec:oja_bdd} to the original sub-Gaussian model in Assumption~\ref{assumption:X}. The argument proceeds by jointly truncating the signal and tail coordinates and controlling the resulting
covariance perturbation.
The proof has three steps.
\begin{itemize}[leftmargin=1.5em]
  \item First, Lemma~\ref{lem:reduction_to_bounded} constructs the truncated observations and shows that they satisfy the required almost-sure signal and tail bounds, up to a small covariance perturbation $H$. 

  \item  Second, we show that this perturbation can be absorbed into the drift remainder, yielding the one-step bounds in Lemma~\ref{lem:one_step_bound-truncation}. These bounds allow the three main arguments of Appendix~\ref{sec:oja_bdd} to be transferred to the truncated process: convergence to the noise floor (\cref{lem:convergence_rate_optimal-trunc}), exact-rank geometric contraction (\cref{lem:rankr-trunc}), and linearization (\cref{lem:linearization-truncation}).

  \item Finally, Lemma~\ref{lem:lower_bound} proves a lower bound for the original linearized term directly under Assumption~\ref{assumption:X} and $Y\perp Z$.
\end{itemize}
Appendix~\ref{sec:prf_sec_convergence} combines these results with the truncation coupling and Haar initialization to prove Theorems~\ref{thm:convergence_rate_optimal} and~\ref{thm:rankr}, together with Propositions~\ref{prop:linearization}--\ref{prop:lower_bd}.

Recall $Y$, $Z$, $\tY$ and $\tZ$ from \eqref{eqn:Y_Z_defn}, we have
\begin{align}
  X = U_* Y + U_{*,\perp} Z = U_* \Lambda_*^{1/2} \tY + U_{*,\perp} \Lambda_{*,\perp}^{1/2} \tZ. \label{eqn:X_decomposition}
\end{align}

\begin{lemma}[Truncation]
\label{lem:reduction_to_bounded}
Under Assumption~\ref{assumption:X}, let $W=\rbr{\begin{smallmatrix} \tY\\ \tZ \end{smallmatrix}}$ where $\tY=U_*^\top \tX$ and $\tZ=U_{*,\perp}^\top \tX$. Then 
\[
\EE W=0,\qquad \EE WW^\top=I_d,\qquad \|W\|_{\psi_2}\le C_\psi.
\]
Moreover, there exist universal constants \(\tC_1, \tC_2, \tC_3>0\) such that, for every
\(\delta_{\tr}\in(0, e^{-1})\), one can construct a random vector
\[
\oW=
\begin{pmatrix}
  \bar{\tY}\\ \overline{\tZ}
\end{pmatrix}
\]
with the following properties.

\begin{enumerate}[leftmargin=2em, label=\textup{(\roman*)}]
  \item $\EE\overline W=0$ and $\PP(\overline W\neq W)\leq \delta_{\tr}$.

  \item Writing $\EE\overline W\overline W^\top = I_d+H$, then $\|H\| \le \tC_1 C_{\psi}^2\delta_{\tr}\log(e/\delta_{\tr})$.

  \item $\|\overline W\|_{\psi_2}\le \tC_2 C_{\psi}$.

  \item Defining $\ooY := \Lambda_*^{1/2}\overline{\tY}$ and $\ooZ := \Lambda_{*,\perp}^{1/2} \overline{\tZ}$, then $\EE \overline{Y} = 0$, $\EE \overline{Z} = 0$ and
    \begin{align*}
      \|\ooY\|_2 &\le \tC_3 C_{\psi} \sqrt{ \lambda_{1\sim r} \log(e/\delta_{\tr}) } \qquad \text{almost surely},\\
      \|\ooZ\|_2 &\le \tC_3 C_{\psi} \sqrt{ \lambda_{r+1\sim d} \log(e/\delta_{\tr}) } \qquad \text{almost surely}.
    \end{align*}
\end{enumerate}

\end{lemma}

\begin{proof}
  See Appendix~\ref{sec:prf_reduction_to_bounded}.
\end{proof}

Fix a constant $\delta_H\in(0,e^{-1})$ such that $\tC_1 C_\psi^2 \delta_H\log(e/\delta_H)\le \tfrac12$ where $\tC_1$ is the universal constant in
Lemma~\ref{lem:reduction_to_bounded}(ii), and set
\begin{equation}
  \delta_{\rm tr}:=n^{-20}\wedge\delta_H.
  \label{eqn:delta_tr}
\end{equation}
Such a constant $\delta_H$ exists because $x\mapsto x\log(e/x)$ is increasing on $(0,1)$, and $x\log(e/x)\to0$ as~$x\downarrow0$.
Apply Lemma~\ref{lem:reduction_to_bounded} with $\delta_{\rm tr}$, and let $H$ be the covariance perturbation in Lemma~\ref{lem:reduction_to_bounded}(ii). Then we have
\begin{equation}
  \begin{aligned}
  \|H\| &\le \frac12,\\
  \|H\| &\lesssim C_\psi^2 n^{-20}\log(en^{20}) \lesssim \teta,\\
  \|H\| &\lesssim C_\psi^4\teta\kappa_*r^{3/2}.
  \end{aligned}
  \label{eqn:H_control}
\end{equation}
Here the second line uses $\teta=c_\eta(\log n)/n$, $c_\eta\ge1/2$, and $n\ge4$, while the third uses $C_\psi\ge1$, $\kappa_*\ge1$, and $r\ge1$.

Let $\ooY$, $\ooZ$, $\overline \tY$, $\overline \tZ$ be the truncated random vectors from Lemma~\ref{lem:reduction_to_bounded} with $\delta_{\tr}$ as given in~\eqref{eqn:delta_tr}. Motivated by \eqref{eqn:X_decomposition}, define the truncated observation
\begin{equation}
  \ooX := U_* \ooY + U_{*,\perp} \ooZ = U_* \Lambda_*^{1/2} \overline{\tY}  + U_{*,\perp} \Lambda_{*,\perp}^{1/2} \overline{\tZ} \label{eqn:X_trunc_decomp}
\end{equation}
and
\begin{equation}
  \ooU_0 = U_0; \qquad \ooU_k = \qr{\ooU_{k-1} + \eta \ooX_k \ooX_k^\top \ooU_{k-1}}, \quad k \in [n]. \label{eqn:oja_update_trunc}
\end{equation}
Then Lemma~\ref{lem:reduction_to_bounded} implies that
\begin{itemize}[leftmargin=1.5em]
  \item simultaneously for all $k \in [n]$,
  \begin{equation}
    \|\ooY_k\|_2 \leq C_{\psi,n} \sqrt{\lambda_{1\sim r}}, \qquad \|\ooZ_k\|_2 \leq C_{\psi,n} \sqrt{\lambda_{r+1\sim d}}, \label{eqn:trunc_boundedness}
  \end{equation}
  where $C_{\psi,n} = C_{\mathsf{noise}} C_{\psi} \sqrt{\log n}$ for some universal constant $C_{\mathsf{noise}}>0$.

  \item \eqref{eqn:delta_tr} implies $n \delta_{\tr} \leq n^{-19}$. Hence
  \begin{align}
    \ooX_k = X_k, \qquad \ooU_k = U_k, \qquad \forall k \in [n], \qquad \text{with probability at least } 1-n^{-19}. \label{eqn:trunc_coupling}
  \end{align}
\end{itemize}

Therefore, it suffices to show $\ooU_k$ is close to $U_*$. To this end, define
\begin{align*}
  \overline \cC_k:=U_\ast^\top\overline U_k, \qquad \ooS_k:=U_{\ast,\perp}^\top\overline U_k,\qquad\overline \cT_k := \ooS_k \overline \cC_k^{-1},
\end{align*}
we show that $\overline \cT_k$ satisfies analogues of Lemma~\ref{lem:recurrence}--\ref{lem:linearization} only up to constant factors.

\paragraph{Tangent recursion for the truncated process.} Applying Sherman--Morrison formula as in Lemma~\ref{lem:recurrence} to the update \eqref{eqn:oja_update_trunc} gives
\begin{align}
  \oocT_+ - \oocT = \eta \frac{ \ooF_+ }{ 1+\eta\oozeta_+ } = \eta\ooF_+ - \eta^2 \frac{ \oozeta_+\ooF_+ }{ 1+\eta\oozeta_+}, \label{eq:dT_split_trunc}
\end{align}
where
\begin{align*}
  \ooF_+ := (-\oocT \ooY_+ + \ooZ_+) (\ooY_+^\top+\ooZ_+^\top\oocT) \in \RR^{(d-r)\times r},\qquad
  \oozeta_+ = (\ooY_+^\top+\ooZ_+^\top\oocT) \ooY_+ \in \RR.
\end{align*}
Similar to \eqref{eqn:drift_split}--\eqref{eqn:D_decomposition}, we can obtain
\begin{align} \label{eqn:linearization_trunc1}
  \oocT_+ - \oocT = (\cI - \eta \oocL) \oocT + \ooD_+ +  \ooR_+,
\end{align}
where
\begin{align*}
  \oocL \oocT &:= -\EE' \ooF_+,\\
  \ooD_+ &:= (\oocT_+-\oocT) - \EE' (\oocT_+-\oocT)= \underbrace{\eta (\ooF_+ - \EE'\ooF_+ )}_{=: \ooD_+^{(1)}} + \underbrace{\eta^2 \left( - \frac{ \oozeta_+ \ooF_+}{1+\eta \oozeta_+}  + \EE' \frac{\oozeta_+  \ooF_+}{1+\eta  \oozeta_+}\right) }_{=: \ooD_+^{(2)}},\\
  \ooR_+ &:= -\eta^2 \EE'\sbr{\frac{ \oozeta_+\ooF_+ }{ 1+\eta\oozeta_+ }}.
\end{align*}
Note that due to covariance perturbation $H$ in Lemma~\ref{lem:reduction_to_bounded}~(ii), $\oocL$ differs from $\cL$. Denote $\ooR_+^H:=\ooR_+ + \eta (\cL-\oocL)\oocT$, we have
\begin{align}
  \oocT_+ = (\cI - \eta \cL) \oocT + \ooD_+ +  \ooR_+^H. \label{eqn:linearization_trunc2}
\end{align}

\paragraph{Effect of the covariance perturbation.}
Unlike the original coordinates, the truncated coordinates need not satisfy $\mathbb E[\ooZ\ooY^\top]=0$. Consequently, the conditional drift operator $\oocL$ differs from $\cL$. We now quantify this difference and show that it can be included in the drift remainder.
Let $H= \rbr{\begin{smallmatrix} H_{11} & H_{12} \\ H_{21} & H_{22}\end{smallmatrix}}$ and
% \begin{align}
%   \ooLambda_*:= \EE\ooY\ooY^\top &= \Lambda_*+A_H, & A_H &:= \Lambda_\ast^{1/2}H_{11}\Lambda_\ast^{1/2},\label{eq:trunc-cov-Y}
%   \\
%   \ooLambda_{*,\perp}:=\EE\ooZ\ooZ^\top &= \Lambda_{\ast,\perp}+B_H, & B_H &:= \Lambda_{\ast,\perp}^{1/2} H_{22} \Lambda_{\ast,\perp}^{1/2},\label{eq:trunc-cov-Z}
%   \\
%   \ooLambda_{*,\times} := \EE\ooZ\ooY^\top&=C_H, &C_H&:= \Lambda_{\ast,\perp}^{1/2} H_{21}\Lambda_\ast^{1/2}.
%   \label{eq:trunc-cov-ZY}
% \end{align}

\begin{equation}
  \begin{aligned}
    \ooLambda_*:= \EE\ooY\ooY^\top &= \Lambda_*+A_H, & A_H &:= \Lambda_\ast^{1/2}H_{11}\Lambda_\ast^{1/2},\\
    \ooLambda_{*,\perp}:=\EE\ooZ\ooZ^\top &= \Lambda_{\ast,\perp}+B_H, & B_H &:= \Lambda_{\ast,\perp}^{1/2} H_{22} \Lambda_{\ast,\perp}^{1/2},\\
    \ooLambda_{*,\times} := \EE\ooZ\ooY^\top&=C_H, &C_H&:= \Lambda_{\ast,\perp}^{1/2} H_{21}\Lambda_\ast^{1/2}.
  \end{aligned}
  \label{eqn:cov_trunc}
\end{equation}
Then
\begin{align*}
  \oocL \oocT = \oocT \ooLambda_* + \oocT \ooLambda_{*,\times}^\top \oocT - \ooLambda_{*,\times} - \ooLambda_{*,\perp} \oocT = \cL \oocT + \oocT A_H + \oocT C_H^\top \oocT - C_H - B_H \oocT.
\end{align*}

Since every block of \(H\) has operator norm at most \(\|H\|\), we have
\begin{align}\label{eq:trunc-ABC-bounds}
  \|A_H\| &\leq \lambda_1\|H\|, \quad \|B_H\| \leq \lambda_{r+1}\|H\| =\lambda_1\nu_\infty^2\|H\|, \quad  \|C_H\| \leq \sqrt{\lambda_1\lambda_{r+1}}\|H\| = \lambda_1\nu_\infty\|H\|,
\end{align}
which implies
\begin{align}
  \norm{(\oocL - \cL) \oocT}
  &\leq \norm{H} \lambda_1 \rbr{\|\oocT\| + \nu_\infty \|\oocT\|^2 + \nu_\infty + \nu_\infty^2 \|\oocT\|}\nonumber \\
  &= \norm{H} \lambda_1 (1+\|\oocT\| \nu_\infty)(\|\oocT\| + \nu_\infty).\label{eqn:L_diff}
\end{align}

\paragraph{One-step bounds for the truncated process.}
We next verify that the centered noise $\ooD_+$, $\ooD_+^{(2)}$ and modified drift remainder $\ooR_+$ in
\eqref{eqn:linearization_trunc2} satisfy the same bounds as their counterparts $D_+$, $D_+^{(2)}$ and $R_+$ in Lemma~\ref{lem:one_step_bound}.

\begin{lemma}[One step bound] \label{lem:one_step_bound-truncation}
    Assume Assumption~\ref{assumption:X} holds, and let the truncated process be
  constructed using \eqref{eqn:delta_tr}. Assume $\overline \cT$ is well-defined with $\|\oocT\|=\tau$. If $ C_{\psi,n}^2 \teta \kappa_* r (1+\tau \nu_2) \leq \frac{1}{2}$, 
    then the following hold.

    \begin{enumerate}[leftmargin=2em, label=\textup{(\Roman*)}]
        \item Noise $\overline D_+$: $\|\overline D_+\| \leq 4C_{\psi,n}^2 \teta \kappa_* r (1 + \tau \nu_2) (\tau + \nu_2)$ almost surely and
        \begin{align*}
            \max \cbr[\Big]{\norm[\big]{\EE' \overline D_+ \overline D_+^\top}^{1/2} , \norm[\big]{\EE' \overline D_+^\top \overline D_+}^{1/2}} &\lesssim C_\psi^2  \teta  \kappa_* \sqrt{r} \rbr{1 + \tau \nu_\infty} \rbr{\tau + \bnu} .
        \end{align*}

        \item Higher-order noise $\overline D_+^{(2)}$: $\norm[\big]{\overline D_+^{(2)}} \leq 4 C_{\psi,n}^4  \teta^2  \kappa_*^2 r^2 (1+\tau \nu_2)^2 (\tau+\nu_2)$ almost surely and
        \begin{align*}
            \max \cbr[\Big]{\norm[\big]{\EE' \overline D_+^{(2)} \overline D_+^{(2)\top}}^{1/2} , \norm[\big]{\EE' \overline D_+^{(2)\top} \overline D_+^{(2)}}^{1/2}} &\lesssim C_{\psi}^4 \teta^2  \kappa_*^2 r^{3/2} \rbr{1+ \tau \nu_\infty}^2 \rbr{\tau + \bnu}.
        \end{align*}

        \item Drift remainder $\overline R_+^H$: $\norm[\big]{\overline R_+^H} \lesssim C_{\psi}^4 \teta^2  \kappa_*^2 r^{3/2} \rbr{1+ \tau \nu_\infty}^2 \rbr{\tau + \nu_\infty}$.
    \end{enumerate}

\end{lemma}

\begin{proof}
  See Appendix~\ref{sec:prf_one_step_bound-truncation}.
\end{proof}

Lemma~\ref{lem:one_step_bound-truncation} provides precisely the inputs used in the convergence analysis
of Appendix~\ref{sec:oja_bdd}: the recursion has the same contractive linear part
$\cI-\eta\cL$, the centered increments satisfy the same
almost-sure and conditional-variance bounds, and the modified drift
remainder has the same order as the remainder in Lemma~\ref{lem:one_step_bound}. Therefore, after enlarging the absolute constants, the stopped-process argument,
matrix-Freedman bound, one-epoch estimate, epoch chaining, and
iteration-budget argument in Lemmas~\ref{lem:1epoch_bound}--\ref{lem:convergence_rate_optimal} apply verbatim to
$\{\overline{\mathcal T}_k\}_{k\le n}$. This yields the following
counterpart of Lemma~\ref{lem:convergence_rate_optimal}.

\begin{lemma}[Convergence] \label{lem:convergence_rate_optimal-trunc}
Suppose Assumptions~\ref{assumption:X},\ref{assumption:d}, \ref{assumption:init} hold. Assume
\begin{align*}
    \nu_\infty &\ge c_{\nu} n^{-C_\nu} \quad\text{for constants } C_\nu \ge 0, c_\nu >0.
    % \label{eqn:cond_nu_lem}
\end{align*}
Let $\varepsilon \in [0,1/2)$ and set 
\begin{align*}
  \eta = c_\eta \frac{\log n}{n \Delta_{\min}}, \quad\text{where}\quad c_\eta\geq c_\eta^* := \rbr[\Big]{1  - \varepsilon + \frac{3 C_\nu}{2}} \rbr[\Big]{1 + \frac{\log^{-3} n + \log 2}{\varepsilon \log n + \log\log n}}. 
  % \label{eqn:cond_step_sz_lem}
\end{align*}
If
\begin{align*}
    n^{1-2\varepsilon} &\ge C_{\ref{lem:convergence_rate_optimal-trunc}} {c_\eta} {\kappa_*^2} r^3 (\log n)^4 \frac{\log(1/\delta_p)}{\delta_p^2} \max(1, \nu_\infty d) 
    % \label{eqn:cond_sample_sz_lem}
\end{align*}
for some constant $C_{\ref{lem:convergence_rate_optimal-trunc}} > 0$ large enough (depending only on $C_\psi$, $C_\nu$, $c_\nu$ and $C_d$).
Then with probability at least $1- n^{-19}$,
\begin{align*}
  \norm{\oocT_k} \leq 2\ootau_* \asymp  {\kappa_*} \bnu \sqrt{\teta r \log n} = \frac{\lambda_1}{\Delta_{\min}} \sqrt{c_\eta r} \, \bnu  \frac{\log n}{\sqrt{n}}, \qquad \text{for all } \floor{\tfrac{c_\eta^*}{c_\eta} n} \leq k \leq n,
  % \label{eqn:convergence_rate_Tn}
\end{align*}
where
\begin{align*}
  \ootau_* := \ooC \max\rbr[\big]{ C_{\psi}^2 \kappa_*\bnu \sqrt{\teta r\log n}, C_{\psi}^2 \kappa_* \nu_2 \teta r (\log n)^2  }
\end{align*}
for some constant $\ooC>0$ (depending only on $C_d$).
\end{lemma}

\paragraph{Exact rank-$r$ contraction.}
We next transfer Lemma~\ref{lem:rankr} to the
truncated process. When $\rank(\Sigma)=r$,
$\Lambda_{\ast,\perp}=0$, so the truncated tail coordinates
$\overline Z_i$ vanish identically. Hence the exact multiplicative
structure and the associated column-space reduction from
Appendix~\ref{sec:prf_rankr} are preserved. Lemma~\ref{lem:one_step_bound-truncation}
controls the additional covariance-perturbation remainder.

\begin{lemma}[Exact-rank contraction] \label{lem:rankr-trunc}
  Suppose Assumption~\ref{assumption:X} holds, $\rank(\Sigma)=r$ and $\overline\cT_0=\cT_0$ is well-defined. Fix $\varepsilon\in[0,1/2)$ and set
\[
    \eta = c_\eta \frac{\log n}{n\lambda_r}, \qquad c_\eta\ge \frac12.
\]
If
\[
    n^{1-2\varepsilon} \ge C_{\mathrm{rk,tr}} c_\eta \kappa_\ast^2 r^2(\log n)^4
\]
for some constant $C_{\mathrm{rk,tr}}>0$ sufficiently large depending only
on $C_\psi$, then, conditionally on $\cF_0=\sigma(U_0)$, with probability at least $1-n^{-19}$,
\[
    \|\overline\cT_k\| \le 16e \exp\left\{ -\left( 1-\frac{1}{n^\varepsilon\log n} \right)k\teta \right\} \|\cT_0\|, \qquad k=0,1,\ldots,n.
\]
In particular,
\[
    \|\overline\cT_n\| \le 16e\, n^{-c_\eta(1-(n^\varepsilon\log n)^{-1})} \|\cT_0\|.
\]
If, in addition, Assumption~\ref{assumption:init} holds and
\begin{equation}
    n^{1-2\varepsilon} \ge C_{\mathrm{rk,tr}} c_\eta\kappa_\ast^2r^2(\log n)^4 \frac{\log(1/\delta_p)}{\delta_p^2}, \label{eqn:rankr_cond_sample_sz_lem2}
\end{equation}
then on the same event,
\begin{equation}
    \|\cT_k\| \lesssim \frac{\sqrt d\,n^{1/2-\varepsilon}} {\sqrt{c_\eta}\kappa_\ast\sqrt r(\log n)^2} \exp\left\{ -\left( 1-\frac{1}{n^\varepsilon\log n} \right)k\teta \right\}, \qquad k=0,1,\ldots,n. \label{eqn:rankr_contraction2}
\end{equation}
In particular,
\[
    \|\cT_n\|
    \lesssim
    \frac{\sqrt d}
    {\sqrt{c_\eta}\kappa_\ast\sqrt r(\log n)^2}
    n^{1/2-\varepsilon
    -c_\eta(1-(n^\varepsilon\log n)^{-1})}.
\]
\end{lemma}

\begin{proof}
  See Appendix~\ref{sec:prf_rankr-trunc}.
\end{proof}

\paragraph{Linearization.} We next transfer the linearization argument of Lemma~\ref{lem:linearization}. The only new term arises because $\ooLambda_{*,\times}:=\mathbb E[\overline Z\overline Y^\top]$
need not vanish. After centering
$\overline Z_i\overline Y_i^\top$, its post-burn-in mean contribution ($i \in (N_*,n]$)
is absorbed into the modified drift remainder, while its pre-burn-in contribution ($i \in [1,N_*]$) is exponentially damped by the contraction between
$N_*$ and $n$. All other terms are controlled exactly as in the proof
of Lemma~\ref{lem:linearization} using Lemma~\ref{lem:one_step_bound-truncation}.

\begin{lemma}[Linearization] \label{lem:linearization-truncation}
Under the assumptions of Lemma~\ref{lem:convergence_rate_optimal-trunc} with $c_\eta\geq c_\eta^*+1/2$, let $N_* := \floor{n c_\eta^*/ c_\eta}$. Then, with probability at least $1-n^{-18}$,
\begin{align}
  \norm[\big]{\oocT_n - \eta\sum_{i=1}^n (\cI-\eta\cL)^{n-i} (\ooZ_i\ooY_i^\top)} \lesssim
  \tau_* \sbr[\big]{C_\psi^2\kappa_* \sqrt{r\teta\log n} +\overline\tau_*}  +e^{-\teta(n-N_*)} \kappa_*\nu_\infty\|H\|.  \label{eq:near-isotropic-linearization-main}
\end{align}
In particular, \eqref{eqn:H_control} gives $\|H\|\lesssim\teta$, and therefore
\begin{align}
  \norm[\big]{\oocT_n - \eta\sum_{i=1}^n (\cI-\eta\cL)^{n-i} (\ooZ_i\ooY_i^\top)} \lesssim \overline\tau_* \sbr[\big]{C_\psi^2\kappa_* \sqrt{r\teta\log n} +\overline\tau_*}.
  \label{eq:near-isotropic-linearization-final}
\end{align}
\end{lemma}

\begin{proof}
  See Appendix~\ref{sec:prf_linearization-truncation}.
\end{proof}

\paragraph{Lower bound for the original linearized term.}
The truncation argument above is needed only for the convergence and
linearization upper bounds. For the lower bound, we return to the
original observations and analyze the linearized term directly. This
argument does not require bounded observations, Assumption~\ref{assumption:d}, or
Assumption~\ref{assumption:init}.

\begin{lemma}[Lower bound for the linearized term] \label{lem:lower_bound}
    Assume Assumption~\ref{assumption:X} holds with $Y$ independent of $Z$. If $\teta \kappa_*^2 \to 0$ and $n\teta \geq \frac{1}{2}\log n$, then
    \begin{align*}
        \PP \cbr[\bigg]{\norm[\Big]{\eta \sum_{i=1}^n (\cI-\eta \cL)^{n-i}(Z_iY_i^\top)}^2 \geq \rbr[\Big]{1-\frac{1}{n}} \frac{\teta \lambda_1 \lambda_{1 \sim r}}{4 \Delta_{\max} \Delta_{\min}} \bnu^2} \geq \frac{1}{12}\sbr{1+o(1)}.
    \end{align*}
    Furthermore, if $\kappa_* \lesssim 1$, then
    \begin{align*}
        \teta \frac{\lambda_1 \lambda_{1 \sim r}}{\Delta_{\max} \Delta_{\min}} \bnu^2 \asymp \frac{\tau_*^2}{\log n}.
    \end{align*}
\end{lemma}
\begin{proof}
    See Appendix~\ref{sec:prf_lower_bound}.
\end{proof}
\begin{remark}[Weaker nondegeneracy condition]
The independence condition $Y\perp Z$ in Lemma~\ref{lem:lower_bound} can be relaxed to $ \lambda_{\min}(\widetilde{\fC}) \gtrsim 1$. Indeed, this condition implies
\[
    \EE[\widetilde Z_\ell^2 \widetilde Y_j^2] \gtrsim 1,
    \qquad \ell\in[d-r],\ j\in[r],
\]
which yields the same row- and column-wise second-moment lower bounds used in the proof, up to constant factors. The fourth-moment upper bounds require only Assumption~1. Hence the conclusion of Lemma~\ref{lem:lower_bound} continues to hold under this weaker nondegeneracy condition.
\end{remark}

\subsection{Proof of Lemma~\ref{lem:reduction_to_bounded}}
\label{sec:prf_reduction_to_bounded}
$\EE W=0$, $\EE WW^\top=I_d$, and $\|W\|_{\psi_2}\le C_\psi$ follow directly from Assumption~\ref{assumption:X} and the properties of orthogonal transformations. The rest of the proof is devoted to constructing $\overline W$ and verifying its properties.

\paragraph{Preliminary moment estimates.}
We first show the following rare-event estimates. If a real random
variable $V$ satisfies $\|V\|_{\psi_2}\le \sigma$ and $\cE$ is an arbitrary event with  $\PP(\cE)\le \delta_{\tr} \le e^{-1}$, then
\begin{align}
  \EE\left[ |V| \bbi_{\cE} \right] &\lesssim \sigma \delta_{\tr} \sqrt{\log(e/\delta_{\tr})} \label{eqn:V1_rare}\\
  \EE\left[ V^2\bbi_{\cE} \right] &\lesssim \sigma^2\delta_{\tr} \log(e/\delta_{\tr}). \label{eqn:V2_rare}
\end{align}

Indeed, the sub-Gaussian tail estimate gives
\[
\mathbb P\left(|V|>t\right) \le 2\exp\left(-\frac{ct^2}{\sigma^2}\right).
\]
Taking $t_{\tr} = \tC \sigma \sqrt{\log(e/\delta_{\tr})}$ for \(\tC\) sufficiently large, we have
\begin{align*}
  \EE[ |V| \bbi_{\cE} ] &\le t_{\tr} \delta_{\tr} + \EE\left[ |V|\bbi_{\{|V|>t_{\tr}\}} \right] \\
  &= t_{\tr}\delta_{\tr} + t_{\tr} \PP\left(|V|>t_{\tr}\right) + \int_{t_{\tr}}^{\infty} \mathbb P\left(|V|>s\right)ds \\
  &\lesssim \sigma \delta_{\tr} \sqrt{\log(e/\delta_{\tr})}.
\end{align*}
The proof of \eqref{eqn:V2_rare} is analogous:
\begin{align*}
  \EE[ V^2\bbi_{\cE} ] &\le t_{\tr}^2 \delta_{\tr} + \EE[ V^2\bbi_{\{|V|>t_{\tr}\}} ] \\
  &= t_{\tr}^2 \delta_{\tr} + t_{\tr}^2\PP(|V|>t_{\tr}) + \int_{t_{\tr}}^{\infty} 2s\PP(|V|>s)ds\\
  &\lesssim \sigma^2\delta_{\tr} \log(e/\delta_{\tr}).
\end{align*}

\paragraph{Construction of $\overline W$.}
Define the two deterministic linear maps
\[
  A_Y := \begin{pmatrix} \Lambda_*^{1/2}&0 \end{pmatrix} \in\mathbb R^{r\times d}, \qquad 
  A_Z := \begin{pmatrix} 0&\Lambda_{*,\perp}^{1/2} \end{pmatrix} \in\RR^{(d-r)\times d}.
\]
Then we have $Y=A_YW$, $Z=A_ZW$ and $\Fnorm{A_Y}^2 = \lambda_{1\sim r}$, $\Fnorm{A_Z}^2 = \lambda_{r+1\sim d}$.
Let 
\begin{align*}
  q:=\log(e/\delta_{\tr}).
\end{align*}
Since \(\delta_{\tr}\leq e^{-1}\), we have \(q \geq 2\). 
By Lemma~\ref{lem:kth_moment}, we have
\[
\EE\|Y\|_2^{2q} \leq \left( CC_\psi\sqrt{q\lambda_{1\sim r}} \right)^{2q}.
\]
Define
\[
L_Y := C_1C_\psi\sqrt{\lambda_{1\sim r} \log \rbr{e/\delta_{\tr}}}, \quad L_Z := C_1C_\psi\sqrt{\lambda_{r+1\sim d} \log \rbr{e/\delta_{\tr}}},
\quad \text{where} \quad C_1 = C e.
\]
Then Markov's inequality gives
\begin{align*}
  \PP(\|Y\|_2>L_Y) &\leq \frac{\EE\|Y\|_2^{2q}}{L_Y^{2q}} \leq \left(\frac{C}{C_1}\right)^{2q} \leq e^{-2q} = \rbr{\frac{\delta_{\tr}}{e}}^2 \leq \frac{\delta_{\tr}}{4},
\end{align*}
where the last inequality holds because $\delta_{\tr} \leq e^{-1}\leq e^2/4$.
Similarly, we have
\begin{align*}
  \PP(\|Z\|_2>L_Z) &\leq \frac{\EE\|Z\|_2^{2q}}{L_Z^{2q}} \leq \left(\frac{C}{C_1}\right)^{2q} \leq e^{-2q} \leq \frac{\delta_{\tr}}{4}.
\end{align*}

Hence the event $\cA := \left\{\|Y\|_2>L_Y\right\} \cup \left\{\|Z\|_2>L_Z\right\}$ satisfies
\begin{equation}
  \mathbb P(\cA)\le\frac{\delta_{\tr}}{2}. \label{eqn:PA}
\end{equation}

Next, we define an event $\cB \supset \cA$ so that \(\mathbb P(\cB)=\delta_{\tr}\). Specifically, we define
\begin{align*}
  \cB := \cA \cup \left( \cA^c\cap\{U\le\theta_{\cA}\} \right),
\end{align*}
where \(U\sim\operatorname{Unif}(0,1)\) is independent of \(W\) and $\theta_{\cA}:=\tfrac{\delta_{\tr}-\PP(\cA)}{1-\PP(\cA)}$. Under this definition, $\cA \subseteq \cB$ and \eqref{eqn:PA} implies that \(\theta_{\cA} \in[0,1]\). Since $U$ is independent of $W$, we have
\begin{align*}
  \PP (\cA^c\cap\{U\le\theta_{\cA}\}) = \PP(\cA^c) \PP(U \leq \theta_{\cA}) = (1-\PP(\cA)) \theta_{\cA} = \delta_{\tr} - \PP(\cA),
\end{align*}
which further implies that
\begin{equation}
  \PP(\cB)=\PP(\cA) + \PP (\cA^c\cap\{U\le\theta_{\cA}\}) = \delta_{\tr}. \label{eqn:PB}
\end{equation}

Then, we define the random vector \(\overline W\) as
\begin{align}
  \overline W := W\bbi_{\cB^c} + b\bbi_{\cB}, \quad \text{where} \quad b:=\frac{\EE\left[ W\bbi_{\cB} \right]}{\delta_{\tr}}. \label{eqn:oW}
\end{align}

Next, we verify that \(\overline W\) satisfies Properties (i)--(iv).

\textbf{Property (i).} By \eqref{eqn:PB}, $\PP(\overline W\neq W) \leq \mathbb P(\cB) = \delta_{\tr}$. Moreover,
\begin{align*}
  \EE\overline W = \EE\left[ W\bbi_{\cB^c} \right] + b \delta_{\tr} = \EE W - \EE\left[ W\bbi_{\cB} \right] + \EE\left[ W\bbi_{\cB} \right] =0.
\end{align*}

\textbf{Property (ii):}
Since \(\EE\overline W=0\), the covariance of \(\overline W\)
equals its second-moment matrix. \eqref{eqn:oW} implies that
\begin{align*}
  \EE\overline W\overline W^\top &= \EE\left[ WW^\top\bbi_{\cB^c} \right] + \delta_{\tr} bb^\top
  \\
  &= I_d - \EE\left[ WW^\top\bbi_{\cB} \right] + \delta_{\tr} bb^\top.
\end{align*}
Hence
\begin{align*}
  H = -\EE\left[ WW^\top\bbi_{\cB} \right] + \delta_{\tr} bb^\top.
\end{align*}

For any unit vector $u\in\RR^d$, by the definition of $b$ in \eqref{eqn:oW}, $\PP(\cB)=\delta_{\tr}$, \eqref{eqn:V1_rare} and \eqref{eqn:V2_rare}, we have
\begin{align}
  |u^\top b|= \frac{1}{\delta_{\tr}} \abs{\EE[u^\top W\bbi_{\cB}]} \leq \frac{1}{\delta_{\tr}} \EE[|u^\top W|\bbi_{\cB}] \lesssim C_{\psi} \sqrt{\log(e/\delta_{\tr})}.
  \label{eq:ub-bound}
\end{align}
and
\begin{align}
  \EE[ (u^\top W)^2\bbi_{\cB} ] \lesssim C_{\psi}^2\delta_{\tr} \log(e/\delta_{\tr}). \label{eq:uW2-bound}
\end{align}
The constants in \eqref{eq:ub-bound} and \eqref{eq:uW2-bound} are independent of the unit vector $u$. Then,
\begin{align*}
  |u^\top Hu| &\le \EE[ (u^\top W)^2\bbi_{\cB} ] + \delta_{\tr}(u^\top b)^2 \lesssim C_{\psi}^2\delta_{\tr} \log(e/\delta_{\tr}).
\end{align*}
Since $H$ is symmetric, taking supremum over all unit vectors $u\in\RR^d$ gives
\[
\|H\| \lesssim C_{\psi}^2\delta_{\tr} \log(e/\delta_{\tr}).
\]

\textbf{Property (iii)} For any unit vector $u \in \RR^d$, we have
\begin{align*}
  \|u^\top \overline{W}\|_{\psi_2} &\leq \|u^\top W \bbi_{\cB^c}\|_{\psi_2} + \|u^\top b \bbi_{\cB}\|_{\psi_2} \leq \|u^\top W\|_{\psi_2} + \|u^\top b \bbi_{\cB}\|_{\psi_2} \leq C_{\psi} +  \|u^\top b\bbi_{\cB}\|_{\psi_2}.
\end{align*}
Then it suffices to bound \(\|u^\top b\bbi_{\cB}\|_{\psi_2}\).

By \eqref{eq:ub-bound}, we can choose a universal constant $C_2>0$ sufficiently large so that
\[
\frac{|u^\top b|^2}{C_2^2C_{\psi}^2} \leq \frac{1}{2} \log\!(e/\delta_{\tr}).
\]
Then
\begin{align*}
  \EE\exp \left(\frac{|u^\top b|^2 \bbi_{\cB}}{C_2^2C_{\psi}^2} \right) 
  &= 1-\delta_{\tr} + \delta_{\tr} \exp\left( \frac{|u^\top b|^2}{C_2^2C_{\psi}^2} \right) \\
  &\leq 1-\delta_{\tr} + \delta_{\tr} (e/\delta_{\tr})^{1/2} \\
  &= 1-\delta_{\tr}+\sqrt{e\delta_{\tr}} \leq 2,
\end{align*}
where the last inequality follows from
\(\delta_{\tr}\leq e^{-1}\). Therefore, we have
\[
\|u^\top b\bbi_{\cB}\|_{\psi_2} \leq C_2C_{\psi}.
\]
which then implies 
\[
\|u^\top\overline W\|_{\psi_2}
\leq
(1+C_2)C_{\psi}.
\]
Taking the supremum over all unit vectors \(u\in\RR^d\) gives
\[
\|\overline W\|_{\psi_2}
\leq
(1+C_2) C_{\psi},
\]
which proves Property~\textup{(iii)}.

\textbf{Property (iv):} we consider two cases:
\begin{itemize}[leftmargin=2em]
  \item On \(\cB^c\), one has \(\cB^c\subseteq \cA^c\), and therefore
  \[
  \norm{\ooY}_2 = \|Y\|_2\le L_Y,
  \qquad
  \norm{\ooZ}_2 = \|Z\|_2\le L_Z.
  \]
  \item On \(\cB\), partition $b = \rbr{\begin{smallmatrix} b_Y\\ b_Z \end{smallmatrix}}$ where
  \[
    b_Y= \frac{\EE [\tY \bbi_{\cB}]}{\delta_{\tr}}, \qquad b_Z = \frac{\EE [\tZ \bbi_{\cB}]}{\delta_{\tr}}
  \]
  according to the signal and tail blocks. Then it suffices to consider $\|\Lambda_*^{1/2}b_Y\|_2$ and $\|\Lambda_{*,\perp}^{1/2}b_Z\|_2$. Note that for any unit vector $u \in \RR^r$, we have
  \begin{align*}
    \abs{u^\top \Lambda_*^{1/2} b_Y} &= \delta_{\tr}^{-1}\abs{\EE u^\top \Lambda_*^{1/2} \tY \bbi_{\cB}}\\
    &\lesssim C_{\psi} \sqrt{ \lambda_1 \log(e/\delta_{\tr}) } \leq C_{\psi} \sqrt{ \lambda_{1\sim r} \log(e/\delta_{\tr}) }.
  \end{align*}
  Here the second line applies \eqref{eqn:V1_rare} with $\|u^\top \Lambda_*^{1/2} \tY\|_{\psi_2}\leq \sqrt{\lambda_1} C_{\psi}$. Taking the supremum over unit~$u$ gives 
  \[
  \|\Lambda_*^{1/2}b_Y\|_2 \lesssim C_{\psi} \sqrt{ \lambda_1 \log(e/\delta_{\tr}) } \le C_{\psi} \sqrt{ \lambda_{1\sim r} \log(e/\delta_{\tr}) }.
  \tag{7}
  \]
  Similarly, we can obtain
  \[ 
  \|\Lambda_{*,\perp}^{1/2}b_Z\|_2 \lesssim C_{\psi} \sqrt{ \lambda_{r+1} \log(e/\delta_{\tr}) } \le C_{\psi} \sqrt{ \lambda_{r+1\sim d} \log(e/\delta_{\tr}) }.
  \tag{8}
  \]
  
\end{itemize}

Combining the two cases above, we conclude that
\begin{align*}
  \norm{\ooY}_2 \lesssim C_{\psi} \sqrt{ \lambda_{1\sim r} \log(e/\delta_{\tr}) }, \quad 
  \norm{\ooZ}_2 \lesssim C_{\psi} \sqrt{ \lambda_{r+1\sim d} \log(e/\delta_{\tr}) }
  \quad \text{almost surely}.
\end{align*}

The proof is then complete.

\subsection{Proof of Lemma~\ref{lem:one_step_bound-truncation}}
\label{sec:prf_one_step_bound-truncation}

\begin{proof}
By~\eqref{eqn:delta_tr}, Lemma~\ref{lem:reduction_to_bounded} and $\log(e/\delta_{\tr})\lesssim\log n$, choosing the constant $C_{\mathsf{noise}}>0$ in $C_{\psi,n} := C_{\mathsf{noise}}C_{\psi}\sqrt{\log n}$ sufficiently large gives
\[
  \|\ooY\|_2 \leq C_{\psi,n}\sqrt{\lambda_{1\sim r}}, \qquad \|\ooZ\|_2 \leq C_{\psi,n}\sqrt{\lambda_{r+1\sim d}}, \qquad \text{almost surely}.
\]

For brevity, define
\[
  \ooV_1 := -\oocT\ooY+\ooZ, \qquad \ooV_2 := \ooY+\oocT^\top\ooZ.
\]
Then
\[
  \overline F_+ = \overline V_1\overline V_2^\top, \qquad \overline\zeta_+ = \overline V_2^\top\ooY.
\]
Since $\|\overline \cT\|=\tau$, the preceding block bounds give
\[
  \|\overline V_1\|_2 \leq C_{\psi,n}\sqrt{\lambda_{1\sim r}}(\tau+\nu_2), \qquad \|\overline V_2\|_2 \leq C_{\psi,n}\sqrt{\lambda_{1\sim r}}(1+\tau\nu_2).
\]
Consequently,
\begin{align*}
  |\eta\overline\zeta_+| \leq \eta\|\overline V_2\|_2\|\ooY\|_2 \leq C_{\psi,n}^2\eta\lambda_{1\sim r}(1+\tau\nu_2) \leq C_{\psi,n}^2\widetilde\eta\kappa_\ast r (1+\tau\nu_2) \leq \frac12, \quad \text{almost surely},
\end{align*}
where we used
$\lambda_{1\sim r}\leq r\lambda_1$ and $\eta\lambda_1=\widetilde\eta\kappa_\ast$. Hence
\[
  |1+\eta\overline\zeta_+|\geq\frac12.
\]

\textbf{Almost-sure bounds in (I) and (II).}
Recall that
\[
  \overline \cT_+-\overline \cT = \eta \frac{\overline F_+} {1+\eta\overline\zeta_+}, \qquad \overline D_+ = (\overline \cT_+-\overline \cT) - \EE'(\overline \cT_+-\overline \cT).
\]
It follows that
\begin{align*}
  \|\overline \cT_+-\overline \cT\| &\leq 2\eta \|\overline V_1\|_2 \|\overline V_2\|_2   \leq 2C_{\psi,n}^2 \widetilde\eta\kappa_\ast r (1+\tau\nu_2)(\tau+\nu_2).
\end{align*}
Centering increases the almost-sure bound by at most a factor of two,
and therefore
\[
  \|\overline D_+\| \leq 4C_{\psi,n}^2 \widetilde\eta\kappa_\ast r (1+\tau\nu_2)(\tau+\nu_2).
\]
proving the almost-sure bound in \textup{(I)}.
Similarly, writing
\[
  \overline Q_+ := -\eta^2 \frac{ \overline\zeta_+\overline F_+ }{ 1+\eta\overline\zeta_+ }, \qquad \overline D_+^{(2)} = \overline Q_+ - \EE'\overline Q_+,
\]
we have
\begin{align*}
  \|\overline Q_+\| &\leq 2\eta^2 |\overline\zeta_+| \|\overline V_1\|_2 \|\overline V_2\|_2 \leq 2C_{\psi,n}^4 \widetilde\eta^2\kappa_\ast^2r^2 (1+\tau\nu_2)^2(\tau+\nu_2).
\end{align*}
Centering again costs at most a factor of two, proving the
almost-sure bound in \textup{(II)}.

\textbf{Quadratic variation bounds in (I) and (II).}
We first verify the moment bound~\eqref{eqn:truncation-kth_moment}, which is the key input for the
quadratic-variation bounds in the proof of Lemma~\ref{lem:one_step_bound}.
By Lemma~\ref{lem:reduction_to_bounded},
\[
  \EE\overline W\overline W^\top = I_d+H, \qquad \|\overline W\|_{\psi_2} \lesssim C_\psi, \qquad \|H\| \lesssim C_\psi^2 \delta_{\tr} \log n.
\]
By \eqref{eqn:H_control}, $\|H\|\le1/2$. Define
\[
  W^\circ:=(I_d+H)^{-1/2}\overline W.
\]
Then
\[
  \EE W^\circ W^{\circ\top}=I_d,\qquad\|W^\circ\|_{\psi_2}\lesssim C_\psi.
\]
Therefore, Lemma~\ref{lem:kth_moment} implies that, for every deterministic matrix $A$
and every $q\geq1$,
\begin{align}
  \sbr{\EE\|A\ooW\|_2^{2q}}^{1/(2q)}
  &= \sbr{\EE \norm[\big]{A(I_d+H)^{1/2}W^\circ}_2^{2q}}^{1/(2q)}\nonumber\\
  &\lesssim C_\psi\sqrt q \|A(I_d+H)^{1/2}\|_{\mathrm F} \lesssim C_\psi\sqrt q \|A\|_{\mathrm F}. \label{eqn:truncation-kth_moment}
\end{align}

Applying~\eqref{eqn:truncation-kth_moment} to the linear maps defining
$\overline V_1$ and $\overline V_2$ gives, for $q\in\{2,4\}$,
\begin{align*}
  \sup_{\|u\|_2=1} \sbr{\EE'|u^\top\ooV_1|^{2q}}^{1/(2q)} &\lesssim C_\psi\sqrt{\lambda_1}(\tau+\nu_\infty),\\
  \sup_{\|v\|_2=1} \sbr{\EE'|v^\top\ooV_2|^{2q}}^{1/(2q)} &\lesssim C_\psi\sqrt{\lambda_1}(1+\tau\nu_\infty),\\
  \sbr{\EE'\|\ooV_1\|_2^{2q}}^{1/(2q)} &\lesssim C_\psi\sqrt{\lambda_{1\sim r}}(\tau+\nu_2),\\
  \sbr{ \EE'\|\ooV_2\|_2^{2q}}^{1/(2q)} &\lesssim C_\psi\sqrt{r\lambda_1}(1+\tau\nu_\infty).
\end{align*}
Moreover, since
\[
  \overline\zeta_+ = \|\ooY\|_2^2 + \ooZ^\top\overline \cT\ooY,
\]
the same moment bound and Cauchy--Schwarz imply
\[
  \sbr{\EE' |\oozeta_+|^4}^{1/4} \lesssim C_\psi^2r\lambda_1 (1+\tau\nu_\infty).
\]

These are precisely the moment estimates used in the proof of
Lemma~\ref{lem:one_step_bound}. Indeed,
\[
  \overline F_+\overline F_+^\top = \|\overline V_2\|_2^2 \overline V_1\overline V_1^\top,
  \qquad
  \overline F_+^\top\overline F_+ = \|\overline V_1\|_2^2 \overline V_2\overline V_2^\top.
\]
Thus, Cauchy--Schwarz gives
\begin{align*}
  \left\|
    \EE'
    \overline F_+\overline F_+^\top
  \right\|^{1/2}
  &\lesssim
  C_\psi^2\sqrt r\,\lambda_1
  (1+\tau\nu_\infty)(\tau+\nu_\infty),
  \\
  \left\|
    \EE'
    \overline F_+^\top\overline F_+
  \right\|^{1/2}
  &\lesssim
  C_\psi^2\sqrt r\,\lambda_1
  (1+\tau\nu_\infty)(\tau+\nu_2).
\end{align*}
Similarly, Holder's inequality gives
\begin{align*}
  \left\|
    \EE'
    \left[
      \overline\zeta_+^2
      \overline F_+\overline F_+^\top
    \right]
  \right\|^{1/2}
  &\lesssim
  C_\psi^4r^{3/2}\lambda_1^2
  (1+\tau\nu_\infty)^2
  (\tau+\nu_\infty),
  \\
  \left\|
    \EE'
    \left[
      \overline\zeta_+^2
      \overline F_+^\top\overline F_+
    \right]
  \right\|^{1/2}
  &\lesssim
  C_\psi^4r^{3/2}\lambda_1^2
  (1+\tau\nu_\infty)^2
  (\tau+\nu_2).
\end{align*}

Since centering reduces the second moment and
$|1+\eta\overline\zeta_+|\geq1/2$,
\begin{align*}
  \EE'
  \overline D_+\overline D_+^\top
  &\preceq
  4\eta^2
  \EE'
  \overline F_+\overline F_+^\top,
  \\
  \EE'
  \overline D_+^\top\overline D_+
  &\preceq
  4\eta^2
  \EE'
  \overline F_+^\top\overline F_+.
\end{align*}
The preceding two moment bounds therefore prove the quadratic-variation
bound in \textup{(I)}. Likewise,
\begin{align*}
  \EE'
  \overline D_+^{(2)}
  (\overline D_+^{(2)})^\top
  &\preceq
  4\eta^4
  \EE'
  \left[
    \overline\zeta_+^2
    \overline F_+\overline F_+^\top
  \right],
  \\
  \EE'
  (\overline D_+^{(2)})^\top
  \overline D_+^{(2)}
  &\preceq
  4\eta^4
  \EE'
  \left[
    \overline\zeta_+^2
    \overline F_+^\top\overline F_+
  \right],
\end{align*}
which proves the quadratic-variation bound in \textup{(II)}.

\textbf{Remainder bound in (III).} Recall that
\[
  \ooR_+
  =
  -\eta^2
  \EE'
  \left[
    \frac{\overline\zeta_+\overline F_+}{1+\eta\overline\zeta_+}
  \right].
\]
By Lemma~\ref{lem:basic_ineq}, the denominator bound, and the preceding row-moment
estimate,
\begin{align*}
  \|\overline R_+\| \leq 2\eta^2 \left\| \EE' \left[ \overline\zeta_+^2 \overline F_+\overline F_+^\top \right] \right\|^{1/2}
  \lesssim C_\psi^4 \widetilde\eta^2\kappa_\ast^2r^{3/2} (1+\tau\nu_\infty)^2 (\tau+\nu_\infty).
\end{align*}

Meanwhile, \eqref{eqn:H_control} and \eqref{eqn:L_diff} give  
\begin{align*}
  \eta\|(\cL - \overline \cL) \overline \cT\| &\leq \widetilde\eta\kappa_\ast\|H\| (1+\tau\nu_\infty)(\tau+\nu_\infty) \lesssim C_\psi^4 \widetilde\eta^2\kappa_\ast^2r^{3/2} (1+\tau\nu_\infty)^2 (\tau+\nu_\infty).
\end{align*}
Recall from~\eqref{eqn:linearization_trunc2} that $\ooR_+^H = \ooR_+ + \eta (\cL - \oocL) \oocT$, combining the last bound with the bound for $\overline R_+$ proves
\textup{(III)}.
\end{proof}

\subsection{Proof of Lemma~\ref{lem:rankr-trunc}}
\label{sec:prf_rankr-trunc}
\begin{proof}
Since $\operatorname{rank}(\Sigma)=r$, we have
$\Lambda_{\ast,\perp}=0$, and hence $\overline Z_i=0$ almost surely
for every $i\in[n]$. Therefore, exactly as in~\eqref{eqn:rankr_mult},
\[
    \overline\cT_{k+1} = \overline\cT_k \bigl(I_r+\eta\overline Y_{k+1} \overline Y_{k+1}^{\top}\bigr)^{-1}.
\]
Consequently, as in~\eqref{eqn:rankr_monotonicity},
\[
    \operatorname{col}(\overline\cT_{k+1})
    =
    \operatorname{col}(\overline\cT_k),
    \qquad
    \|\overline\cT_{k+1}\|
    \le
    \|\overline\cT_k\|,
    \qquad k\ge0.
\]
The sample-size condition implies the analogue of~\eqref{eqn:rankr_cond_sample_sz}, so
Lemma~\ref{lem:one_step_bound-truncation} applies at every step. Specializing Lemma~\ref{lem:one_step_bound-truncation} to
$\nu_2=\nu_\infty=\bar\nu=0$ gives, up to absolute constants, the
same one-step almost-sure, conditional-variance, and drift-remainder
bounds as \eqref{eqn:rankr_one_step_bd}, with $\overline R_i^H$ in place of $R_i$.

It remains to note that the dimension-reduction argument in the proof
of Lemma~\ref{lem:rankr} is preserved by truncation. Indeed, writing
\[
    \overline K_i
    :=
    \frac{\overline Y_i\overline Y_i^\top}
         {1+\eta\|\overline Y_i\|^2},
\]
we have
\[
    \overline\cT_i-\overline\cT_{i-1}
    =
    -\eta\overline\cT_{i-1}\overline K_i,
    \qquad
    \overline D_i
    =
    -\eta\overline\cT_{i-1}
    \bigl(\overline K_i-\EE_{i-1}\overline K_i\bigr).
\]
Hence, for any block starting at time $s$,
\[
    \operatorname{col}(\overline D_i)
    \subseteq
    \operatorname{col}(\overline\cT_s),
    \qquad i>s.
\]
Moreover, since $\cL M=M\Lambda_\ast$ in the exact rank-$r$ case,
right multiplication by $(I_r-\eta\Lambda_\ast)^{s+t-i}$ preserves
this column-space inclusion. Thus the weighted martingale in the
analogue of\eqref{eqn:rankr_decomp} can again be compressed to a
$q_s\times r$ martingale, where
$q_s=\operatorname{rank}(\overline\cT_s)\le r$.
Applying matrix Freedman's inequality for each fixed terminal time
and then taking a union bound over the at most $n$ terminal times
therefore gives the analogue of \eqref{eqn:rankr_one_block}, with the same bound \eqref{eqn:rankr_b_bound} on the block error, up to absolute constants.

The block-chaining argument in Steps~2--3 of the proof of Lemma~\ref{lem:rankr}
then applies unchanged, after enlarging the absolute constants, and
gives the first conclusion.

Finally, if Assumption~\ref{assumption:init} and \eqref{eqn:rankr_cond_sample_sz_lem2} also hold, the same calculation
as in the last paragraph of the proof of Lemma~\ref{lem:rankr} gives
\[
    \|\cT_0\|
    \lesssim
    \frac{\sqrt d\,n^{1/2-\epsilon}}
         {\sqrt{c_\eta}\kappa_\ast\sqrt r(\log n)^2}.
\]
Combining this with the first conclusion yields \eqref{eqn:rankr_contraction2}, and the
stated bound for $\|\overline\cT_n\|$ follows from $n\widetilde\eta=c_\eta\log n$.
\end{proof}

\subsection{Proof of Lemma~\ref{lem:linearization-truncation}}
\label{sec:prf_linearization-truncation}

\begin{proof}

Starting from \eqref{eqn:linearization_trunc2}, the same decomposition as in Lemma~\ref{lem:linearization} gives
\begin{align*}
    \oocT_n &= (\cI-\eta \cL)^{n-N_*} \oocT_{N_*} + \oocD_{N_*,n}^{(1,0)} + \oocD_{N_*,n}^{(1,1)} + \oocD_{N_*,n}^{(2)} + \oocR_{N_*,n}^H,
\end{align*}
where
\begin{align*}
    \oocD_{N_*,n}^{(1,0)} &:= \eta \sum_{i=N_* + 1}^n (\cI-\eta \cL)^{n-i} (\ooZ_i \ooY_i^\top - \ooLambda_{*,\times});\\
    \oocD_{N_*,n}^{(1,1)} &:= \ooS_1 + \ooS_2 + \ooS_3,\\
    \text{where} &\quad \ooS_1 := \eta \sum_{i=N_* + 1}^n (\cI-\eta \cL)^{n-i}[-\oocT_{i-1}(\ooY_i \ooY_i^\top - \ooLambda_*)],\\
    &\quad \ooS_2 := \eta \sum_{i=N_* + 1}^n (\cI-\eta \cL)^{n-i} [(\ooZ_i \ooZ_i^\top - \ooLambda_{*,\perp})\oocT_{i-1}],\\
    &\quad \ooS_3 :=\eta \sum_{i=N_* + 1}^n (\cI-\eta \cL)^{n-i} [-\oocT_{i-1} (\ooY_i \ooZ_i^\top- \ooLambda_{*,\times}^\top) \oocT_{i-1}];\\
    \oocD_{N_*,n}^{(2)} &:= \sum_{i=N_* + 1}^n (\cI-\eta \cL)^{n-i} \ooD_i^{(2)};\\
    \oocR_{N_*,n}^H &:= \sum_{i=N_* + 1}^n (\cI-\eta \cL)^{n-i} \ooR_i^H.
\end{align*}

Then
\begin{align*}
  &\norm[\Big]{\oocT_n - \eta \sum_{i=1}^n (\cI-\eta \cL)^{n-i} (\ooZ_i \ooY_i^\top)}\\ 
  \leq &\underbrace{\norm[\big]{(\cI-\eta \cL)^{n-N_*} \oocT_{N_*}}}_{=:\ooalpha_1} + \underbrace{\norm[\big]{\eta \sum_{i=1}^{N_*} (\cI-\eta \cL)^{n-i} (\ooZ_i \ooY_i^\top - \ooLambda_{*,\times})}}_{=:\ooalpha_2} \\
  + &\underbrace{\norm[\big]{\oocD_{N_*,n}^{(1,1)}}}_{=:\ooalpha_3} + \underbrace{\norm[\big]{\oocD_{N_*,n}^{(2)}}}_{=:\ooalpha_4} + \underbrace{\norm[\big]{\oocR_{N_*,n}^H - \eta \sum_{i=N_* + 1}^n (\cI-\eta \cL)^{n-i} \ooLambda_{*,\times}}}_{=:\ooalpha_5} + \underbrace{\norm[\big]{\eta \sum_{i=1}^{N_*} (\cI-\eta \cL)^{n-i} \ooLambda_{*,\times}}}_{=:\ooalpha_6}.
\end{align*}

Note that by \cref{lem:basic_ineq} and a geometric series bound, we have
\begin{align*}
  \ooalpha_6 &\leq \eta\sum_{i=1}^{N_*} (1-\teta)^{n-i} \|\Lambda_{\ast,\times}\| \leq (1-\teta)^{n-N_*} \frac{\|\Lambda_{\ast,\times}\|}{\Delta_{\min}} \lesssim e^{-\teta(n-N_*)} \kappa_*\nu_\infty\|H\|.
\end{align*}  
Moreover, \eqref{eq:trunc-ABC-bounds} and \eqref{eqn:H_control} gives
\begin{align*}
  \eta\|\Lambda_{\ast,\times}\| \leq \teta\kappa_*\nu_\infty\|H\| \lesssim C_\psi^4\teta^2\kappa_*^2r^{3/2}\nu_\infty \leq C_\psi^4\teta^2\kappa_*^2r^{3/2} (1+\tau\nu_\infty)^2(\tau+\nu_\infty).
\end{align*}
Consequently, $\ooR_i^H - \eta \ooLambda_{*,\times}$ satisfies the same bound (up to
an absolute constant) as $\ooR_i^H$ in Lemma~\ref{lem:one_step_bound-truncation}.

Thus, the same argument as in the proof of Lemma~\ref{lem:linearization}, with Lemma~\ref{lem:one_step_bound-truncation} in place of Lemma~\ref{lem:one_step_bound}, shows that with probability at least $1-n^{-18}$,
\begin{align*}
  \ooalpha_1 + \ooalpha_2 + \ooalpha_3 + \ooalpha_4 + \ooalpha_5 \lesssim \overline\tau_* \left[ C_\psi^2\kappa_* \sqrt{r\teta\log n} +\overline\tau_* \right].
\end{align*}
Combining the above bounds for $\ooalpha_1$, $\ooalpha_2$, $\ooalpha_3$, $\ooalpha_4$, $\ooalpha_5$ and $\ooalpha_6$ gives \eqref{eq:near-isotropic-linearization-main}. 

Finally, since $c_\eta \geq c_\eta^\ast+1/2$ and
$N_\ast=\lfloor (c_\eta^\ast/c_\eta)n\rfloor$, we have
\[ 
  e^{-\teta(n-N_\ast)} \leq e^{-(c_\eta-c_\eta^\ast)\log n} \leq n^{-1/2}.
\]
By \eqref{eqn:H_control}, $\|H\|\lesssim\widetilde\eta$. Hence,
\[
  e^{-\teta(n-N_*)} \kappa_*\nu_\infty\|H\| \lesssim n^{-1/2}\kappa_*\nu_\infty\teta \lesssim \overline\tau_*\,C_\psi^2\kappa_* \sqrt{r\teta\log n}.
\]
Thus, the last term in  \eqref{eq:near-isotropic-linearization-main} is absorbed into the first term on its right-hand side, which proves~\eqref{eq:near-isotropic-linearization-final}.
\end{proof}

\subsection{Proof of Lemma~\ref{lem:lower_bound}}
\label{sec:prf_lower_bound}

We only prove the column bound, as the row bound follows by symmetry. Recall
\begin{align*}
    \cD_n^{(1,0)} = \eta \sum_{i=1}^n (\cI-\eta \cL)^{n-i} (Z_{i} Y_{i}^\top)
\end{align*}

Denote the first column of $\cD_n^{(1,0)}$ by
\begin{align*}
    W:= \sbr{\cD_n^{(1,0)}}_{\cdot,1} = \sum_{i=1}^n \gamma_i, \qquad \text{where} \; \gamma_i := \eta \, \diag \cbr{(1-\eta \Delta_{l1})^{n-i}}_{l=1}^{n-r} Z_{i} \sbr{Y_{i}}_1 \in \RR^{d-r}.
\end{align*}
Note that we have $\EE \gamma_i = 0$ since $Y=U_*^\top X$ and $Z=U_{*,\perp}^\top X$.

\paragraph{Step 1: Column second moment} By independence $\gamma_i\perp \gamma_{i'}$ for $i \ne i'$ and $\EE \gamma_i = 0$, we have
    \begin{align*}
        \EE \norm{W}^2 = \EE \sum_{i=1}^n  \norm{\gamma_i}^2 = \sum_{l=1}^{d-r} \; \underbrace{\EE \sum_{i=1}^n \sbr{\gamma_i}_l^2}_{=:\sigma_{l}^2}
    \end{align*}
    Under $\eta \Delta_{\max} \in (0,1)$, we have $\eta \Delta_{l1} \in (0,1)$, hence $2-\eta \Delta_{l1} \in (1,2)$ and
    \begin{align*}
        \sigma_{l}^2 &= \sum_{i=1}^n \eta^2 (1-\eta \Delta_{l1})^{2(n-i)} \EE \sbr{Z_{i}}_l^2 \sbr{Y_{i}}_1^2 = \sum_{i=1}^n \eta^2 (1-\eta \Delta_{l1})^{2(n-i)} \lambda_{r+l} \lambda_1\\
        &\geq \sum_{i=1}^n \eta^2 (1-\eta \Delta_{\max})^{2(n-i)} \lambda_{r+l} \lambda_1
        = \eta \lambda_{r+l}\lambda_1 \frac{ \sbr{1-(1-\eta \Delta_{\max})^{2n}}}{(2-\eta \Delta_{\max}) \Delta_{\max}} \\
        &\geq  \eta \lambda_{r+l}\lambda_1 \frac{1-(1-\eta \Delta_{\max})^{2n}}{2\Delta_{\max}}
    \end{align*}
    where the last line follows from $\eta \Delta_{\max} \in (0,1)$.

    Moreover, since $n\teta \geq \tfrac{1}{2} \log n$. We have
    \begin{align*}
        (1-\eta \Delta_{\max})^{2n} \leq (1-\eta \Delta_{\min})^{2n}  \leq e^{-2 n \eta \Delta_{\min}} = e^{-2n \teta} \leq e^{-\log n} = \frac{1}{n}.
    \end{align*}
    Thus, we have
    \begin{align}
        \EE \norm{W}^2 =\EE \sum_{i=1}^n  \norm{\gamma_i}^2 \geq \teta \frac{\lambda_1 \lambda_{r+1 \sim d}}{2\Delta_{\max} \Delta_{\min}} \sbr{1- \tfrac{1}{n}} \label{eqn:col_W_2nd_lower}
    \end{align}

\paragraph{Step 2: Column fourth moment}
\begin{itemize}[leftmargin=1em]
    \item First, we show that
    \begin{align*}
        \EE \norm{W}^4 \leq \sum_i\EE \norm{\gamma_i}^4 + 3 \rbr[\Big]{\sum_i\EE \norm{\gamma_i}^2}^2 
    \end{align*}
    To see this, note that
    \begin{align*}
        \EE \norm{W}^4 = \EE \norm[\Big]{\sum_{i=1}^n \gamma_i}^4 = \sum_{i,j,a,b=1}^n \EE \inner[\big]{\gamma_i}{\gamma_j} \inner{\gamma_a}{\gamma_b}.
    \end{align*}
    By independence and $\EE \gamma_i =0$, any term with a singleton index (an index appearing exactly once among $\{i,j,a,b\}$) vanishes: conditioning on all other variables isolates a factor $\EE[\gamma_i]=0$.  The surviving pairings are:
    \begin{enumerate}[label=(\roman*),leftmargin=2em]
    \item $(i=j,\,a=b)$: contributes $\sum_{i,a}\EE \sbr{\norm{\gamma_i}^2 \norm{\gamma_a}^2}$.
    \item $(i=a,\,j=b,\,i\ne j)$: contributes $\sum_{i\ne j}\EE[ \inner{\gamma_i}{\gamma_j}^2]$.
    \item $(i=b,\,j=a,\,i\ne j)$: same value as (ii).
    \end{enumerate}
    Hence
    \begin{align*}
        \EE \norm{W}^4 = \sum_{i,a}\EE \sbr[\big]{\norm{\gamma_i}^2 \norm{\gamma_a}^2}+ 2\sum_{i\ne j}\EE[\inner{\gamma_i}{\gamma_j}^2].
    \end{align*}

    \textit{First sum.}  Splitting the diagonal and off-diagonal:
    \begin{align*}
        \sum_{i,a}\EE \sbr[\big]{\norm{\gamma_i}^2 \norm{\gamma_a}^2} &= \sum_i\EE \norm[\big]{\gamma_i}^4 + \sum_{i\ne a}\EE \norm{\gamma_i}^2\cdot\EE\norm{\gamma_a}^2 \\
        &\le \sum_i\EE \norm{\gamma_i}^4 + \rbr[\Big]{\sum_i\EE \norm{\gamma_i}^2}^2.
    \end{align*}

    \textit{Second sum.}  For $i \neq j$, by independence, 
    \begin{align*}
        \EE\sbr[\big]{\langle\gamma_i,\gamma_j\rangle^2}= \tr(\Sigma_i \Sigma_j), \quad \text{where} \; \Sigma_m = \EE[\gamma_m \gamma_m^{\top}].  
    \end{align*}
    Since $\tr(AB)\le\tr(A)\tr(B)$ for PSD matrices $A,B$, we have
    \begin{align*}
        2\sum_{i\ne j}\tr(\Sigma_i \Sigma_j) \le 2\sum_{i\ne j}\tr(\Sigma_i)\tr(\Sigma_j) \le 2\Bigl(\sum_i\tr\Sigma_i \Bigr)^2 = 2\Bigl(\sum_i\EE\|\gamma_i\|^2\Bigr)^2.
    \end{align*}
    Adding gives
    \begin{align*}
        \EE \norm{W}^4 \leq \sum_{i=1}^n \EE \norm{\gamma_i}^4 + 3 \rbr[\Big]{\sum_{i=1}^n \EE \norm{\gamma_i}^2}^2 
    \end{align*}

    \item Next, we compute $\sum_{i=1}^n\EE \norm{\gamma_i}^4$. To this end, note that
    \begin{align*}
        \norm{\gamma_i}^2 =  \sbr{Y_{i}}_1^2 \sum_{l=1}^{d-r} \eta^2 (1-\eta \Delta_{l1})^{2(n-i)}\sbr{Z_{i}}_l^2 . 
    \end{align*}
    By independence $Y\perp Z$ and $\EE \sbr{Y_i}_1^4 \lesssim \lambda_1^2$ (sub-Gaussian), we have
    \begin{align*}
        \EE \norm{\gamma_i}^4 &= \EE \sbr[\Big]{\sbr{Y_{i}}_1^2  \sum_{l=1}^{d-r} \eta^2 (1-\eta \Delta_{lj})^{2(n-i)} \sbr{Z_{i}}_l^2}^2\\
        &\lesssim \lambda_{1}^2 \EE \sbr[\Big]{\sum_{l=1}^{d-r} \eta^2 (1-\eta \Delta_{l1})^{2(n-i)} \sbr{Z_{i}}_l^2}^2\\
        &\overset{(i)}{=} \lambda_{1}^2  \EE \sbr[\Big]{\sum_{l=1}^{d-r} \eta^2 (1-\eta \Delta_{l1})^{2(n-i)} \lambda_{r+l} \sbr[\big]{\tZ_{i}}_l^2}^2\\
        &\leq \lambda_{1}^2 \eta^4 (1-\eta \Delta_{\min})^{4(n-i)} \EE \sbr[\Big]{\sum_{l=1}^{d-r} \lambda_{r+l} \sbr[\big]{\tZ_{i}}_l^2}^2\\
        &\overset{(ii)}{=} \lambda_{1}^2 \eta^4 (1-\eta \Delta_{\min})^{4(n-i)} \EE \norm[\big]{B \tZ_{i}}^4
    \end{align*}
    where (i) follows from $\sbr{Z_{i}}_j = \lambda_j^{1/2} \sbr[\big]{\tZ_{i}}_j$ and (ii) holds with $B=\diag\rbr{\lambda_{l+r}^{1/2}}_{l=1}^{d-r} \in \RR^{(d-r) \times (d-r)}$. Applying Lemma~\ref{lem:kth_moment} gives
    \begin{align*}
        \sbr[\big]{\EE \norm[\big]{B \tZ_{i}}^4}^{1/4} \lesssim \Fnorm{B} C_\psi = C_\psi \lambda_{r+1\sim d}^{1/2}
    \end{align*}
    Thus,
    \begin{align*}
        \EE \norm{\gamma_i}^4 \lesssim  C_\psi^4  \eta^4 (1-\eta \Delta_{\min})^{4(n-i)} \lambda_{1}^2\lambda_{r+1\sim d}^2 \leq C_\psi^4 \teta^4 (1-\teta)^{4(n-i)} \frac{\lambda_1^2 \lambda_{r+1 \sim d}^2}{\Delta_{\min}^4}
    \end{align*}
    which implies 
    \begin{align*}
        \sum_{i=1}^n \EE \norm{\gamma_i}^4 
        &\lesssim C_\psi^4 \teta^4 \frac{\lambda_1^2 \lambda_{1 \sim r}^2}{\Delta_{\min}^4} \nu_2^4 \sum_{i=1}^n (1-\teta)^{4(n-i)} \leq C_\psi^4 \teta^3 \frac{\lambda_1^2 \lambda_{r+1 \sim d}^2}{\Delta_{\min}^4} 
    \end{align*}
    
    \item With the above two bounds, we have
    \begin{align*}
        \frac{\sum_{i=1}^n \EE \norm{\gamma_i}^4 }{\rbr[\Big]{\sum_{i=1}^n \EE \norm{\gamma_i}^2}^2} \lesssim 
            \frac{C_\psi^4 \teta^3 \frac{\lambda_{1}^2 \lambda_{r+1 \sim d}^2}{\Delta_{\min}^4}}{\teta^2 \frac{\lambda_{1}^2 \lambda_{r+1 \sim d}^2}{4 \Delta_{\max}^2 \Delta_{\min}^2}\sbr{1-\tfrac{1}{n}}^2} 
        \lesssim 
            4 C_\psi^4 \teta \frac{\Delta_{\max}^2}{\Delta_{\min}^2} = o(1)
    \end{align*}
    which implies that
    \begin{align*}
        \EE \norm{W}^4 &\leq \sum_{i=1}^n\EE \norm{\gamma_i}^4 + 3 \rbr[\Big]{\sum_{i=1}^n \EE \norm{\gamma_i}^2}^2 \\
        &= (3+o(1)) \rbr[\Big]{\sum_{i=1}^n \EE \norm{\gamma_i}^2}^2 = (3+o(1)) (\EE \norm{W}^2)^2
    \end{align*}

    Finally, applying the Paley-Zygmund inequality gives
    \begin{align*}
        \PP \cbr{\norm{W}^2 \geq \frac{1}{2} \EE \norm{W}^2} \geq \frac{1}{4} \frac{\rbr{\EE \norm{W}^2}^2}{\EE \norm{W}^4} \geq \frac{1}{12}(1+o(1))
    \end{align*}
    which combined with \eqref{eqn:col_W_2nd_lower} gives
    \begin{align*}
        \PP \cbr[\Big]{\norm{W}^2 \geq \sbr{1- \tfrac{1}{n}} \teta \frac{\lambda_{1} \lambda_{r+1\sim d}}{4\Delta_{\max} \Delta_{\min}} } \geq \frac{1}{4} \frac{\rbr{\EE \norm{W}^2}^2}{\EE \norm{W}^4} \geq \frac{1}{12}(1+o(1))
    \end{align*}
    
\end{itemize}

\section{Proofs for Section~\ref{sec:main_convergence}}
\label{sec:prf_sec_convergence}
\subsection{Proof of Theorem~\ref{thm:convergence_rate_optimal}}
We apply \cref{lem:convergence_rate_optimal-trunc} to the truncated process and then transfer its conclusion to the original process through the truncation coupling~\eqref{eqn:trunc_coupling}. We first record two consequences of the sample-size condition~\eqref{eqn:cond_sample_sz}:
\begin{align*}
    n^{1-2\varepsilon} &\ge C_1 c_\eta \kappa_*^2 r^3 (1 + \nu_\infty d)\, (\log n)^4 \frac{\log(1/\delta_p)}{\delta_p^2}. \tag{\ref{eqn:cond_sample_sz}}
\end{align*}
\begin{itemize}[leftmargin=*]
    \item $n^{1-2\varepsilon} \geq C_1 c_\eta \kappa_*^2 r^3 (\log n)^4 \frac{\log(1/\delta_p)}{\delta_p^2}$ implies that
    \begin{align}
        \delta_p \geq \delta_p^2 \geq \frac{\log(1/\delta_p)}{n^{1-2\varepsilon}} \geq \frac{1}{n}, \label{eqn:cond_delta_p}
    \end{align}
    where the first and third inequalities follow from $\delta_p\le 1/e$, while the second follows from $c_\eta\ge 1/2$, $\kappa_\ast,r\ge 1$, and $C_1$ being sufficiently large.
    \item $n^{1-2\varepsilon} \geq C_1 c_\eta \kappa_*^2 d\nu_\infty r^3 (\log n)^4 \frac{\log(1/\delta_p)}{\delta_p^2}$ similarly implies that
    \begin{align*}
        d \leq n^{1-2\varepsilon} \nu_\infty^{-1} \frac{\delta_p^2}{\log(1/\delta_p)} \leq n^{1-2\varepsilon + C_\nu} c_\nu^{-1}
    \end{align*}
    where the second inequality follows from $\delta_p \leq 1/e$ and \eqref{eqn:cond_nu}. Hence Assumption~\ref{assumption:d} holds.
\end{itemize}

Moreover, by \cref{lem:init},
\begin{align*}
    \PP\cbr[\Big]{\norm[\big]{\cT_0} \leq C_{\mathsf{init}} \delta_p^{-1}\sqrt{rd \log(1/\delta_p)}} \geq 1-\frac{\delta_p}{2}.
\end{align*}
Denote this event by $\cE_0$. On $\cE_0$, Assumption~\ref{assumption:init} holds, so Lemma~\ref{lem:convergence_rate_optimal-trunc} applies to the truncated process. Hence, with probability at least
$1-n^{-19}$,
\[
  \|\overline{\mathcal T}_k\|
  \lesssim \tau_*
  \qquad
  \text{for all }
  \left\lfloor\tfrac{c_\eta^*}{c_\eta}n\right\rfloor
  \le k\le n.
\]
By the truncation coupling~\eqref{eqn:trunc_coupling},
\[
  \mathbb P\{\overline U_k=U_k\text{ for all }k\in[n]\}
  \ge 1-n^{-19}.
\]
Therefore, with probability at least
\begin{align*}
    \PP \cbr{\norm{\sin \Theta(U_k,U_*)} \lesssim \tau_*, \text{ for all } \floor{\tfrac{c_\eta^*}{c_\eta} n}\leq k \leq n} \geq \rbr{1-\delta_p/2}\rbr{1-2n^{-19}} \geq 1-\delta_p
\end{align*}
where the last inequality follows from~\eqref{eqn:cond_delta_p}.

\subsection{Proof of Theorem~\ref{thm:rankr}}
We apply Lemma~\ref{lem:rankr-trunc} to the truncated process and then
transfer its conclusion to the original process through the truncation coupling.
Let $\{\overline X_i\}_{i=1}^n$, $\{\overline U_k\}_{k=0}^n$, and
$\{\overline\cT_k\}_{k=0}^n$ denote the truncated observations, iterates, and
tangent matrices constructed in Appendix~\ref{sec:oja_subG}.

As in the proof of Theorem~\ref{thm:convergence_rate_optimal}, the sample-size condition
\eqref{eqn:cond_sample_sz_rankr_thm} implies, for $C_2$ sufficiently large,
\begin{align}
    \delta_p \geq \delta_p^2 \geq \frac{\log(1/\delta_p)}{n^{1-2\varepsilon}} \geq \frac{1}{n}. \label{eqn:rankr_thm_deltap_lower}
\end{align}

By Lemma~\ref{lem:init}, Haar initialization gives
\[
    \PP\left\{
        \|\cT_0\|
        \le
        C_{\rm init}\delta_p^{-1}
        \sqrt{rd\log(1/\delta_p)}
    \right\}
    \ge 1-\frac{\delta_p}{2},
\]
after enlarging the absolute constant if necessary. Denote this initialization event by $\cE_0$. 

On $\cE_0$, Assumption~\ref{assumption:init} holds. Since $\rank(\Sigma)=r$, Lemma~\ref{lem:rankr-trunc} therefore gives, conditionally on $\cE_0$, with probability at least $1-n^{-19}$,
\[
    \|\overline{\cT}_k\|
    \lesssim
    \frac{\sqrt d\,n^{1/2-\epsilon}}
    {\sqrt{c_\eta}\,\kappa_\ast\sqrt r\,(\log n)^2}
    \exp\left\{
        -\left(1-\frac{1}{n^\epsilon\log n}\right)
        k\widetilde\eta
    \right\},
    \qquad k=0,1,\ldots,n.
\]
Here the sample-size condition in the theorem is precisely
\eqref{eqn:rankr_cond_sample_sz_lem2}, up to enlarging $C_{\rm rk,tr}$.

By the truncation coupling~\eqref{eqn:trunc_coupling} in Appendix~\ref{sec:oja_subG},
\[
    \PP\left\{
        \overline U_k=U_k
        \text{ for all } k\in[n]
    \right\}
    \ge 1-n^{-19}.
\]
Therefore, with probability at least
\[
    \rbr{1-\frac{\delta_p}{2}}\rbr{1-2n^{-19}}
    \ge 1-\delta_p,
\]
where the last inequality follows from~\eqref{eqn:rankr_thm_deltap_lower},
the preceding bound holds with $\cT_k$ in place of
$\overline{\cT}_k$ simultaneously for all $k=0,\ldots,n$.

Finally, Lemma~\ref{lem:alignment_diff} gives
\[
    \|\sin\Theta(U_k,U_\ast)\|
    =
    \bigl\|
        \cT_k(\cC_k\cC_k^\top)^{1/2}
    \bigr\|
    \le \|\cT_k\|,
\]
since $\|(\cC_k\cC_k^\top)^{1/2}\|\le 1$.
This proves the uniform bound.

For $k=n$, using
$n\widetilde\eta=c_\eta\log n$ gives
\[
    \exp\left\{
        -\left(1-\frac{1}{n^\epsilon\log n}\right)
        n\widetilde\eta
    \right\}
    =
    n^{-c_\eta\left(1-(n^\epsilon\log n)^{-1}\right)},
\]
which yields the stated bound for $U_n$ and completes the proof.

% --------------------------------------------------------------------

\subsection{Proof of Proposition~\ref{prop:linearization}}
We apply Lemma~\ref{lem:linearization-truncation} to the truncated process and
then transfer its conclusion to the original process through the truncation coupling.

Let $\cE_0$ be the initialization event defined in the proof of
Theorem~\ref{thm:convergence_rate_optimal}. On $\cE_0$, Assumption~\ref{assumption:init} holds, so Lemma~\ref{lem:linearization-truncation} gives, with probability at least $1-n^{-18}$,
\[
\norm[\Big]{\overline{\cT}_n - \eta\sum_{i=1}^n (\cI-\eta\cL)^{n-i} (\overline Z_i\overline Y_i^\top)} \lesssim \tau_\ast
\sbr[\big]{\kappa_\ast\sqrt{r\widetilde\eta\log n}\,(1+\bar\nu)} \lesssim \frac{\tau_\ast}{n^\epsilon\log n},
\]
where we used Lemma~\ref{lem:linearization-truncation} together with
$\overline\tau_\ast\asymp\tau_\ast$.

On the truncation coupling event~\eqref{eqn:trunc_coupling}, we have
\[
\overline{\cT}_n=\cT_n,
\qquad
\overline Z_i\overline Y_i^\top=Z_iY_i^\top,
\quad i\in[n].
\]
Hence the same bound holds for the original process. The total failure
probability is at most
\[
\frac{\delta_p}{2}+n^{-18}+n^{-19}\le \delta_p
\]
by \eqref{eqn:cond_delta_p}. This proves the proposition.

\subsection{Proof of Proposition~\ref{prop:lower_bd}}
The proof combines Proposition~\ref{prop:linearization} with Lemma~\ref{lem:lower_bound}.

\paragraph{Lower bound for the linearized term.}
Let
\[
L_n := \eta\sum_{i=1}^n(\cI-\eta\cL)^{n-i}(Z_iY_i^\top).
\]
By Lemma~\ref{lem:lower_bound}, under \(Y\perp Z\),
\[
\mathbb P\left\{
\|L_n\|^2
\ge
\left(1-\frac1n\right)
\frac{\widetilde\eta\lambda_1\lambda_{1\sim r}}
{4\Delta_{\max}\Delta_{\min}}\bar\nu^2
\right\}
\ge
\frac1{12}[1+o(1)].
\]
Moreover, when \(\kappa_*\lesssim 1\),
\[
\frac{\widetilde\eta\lambda_1\lambda_{1\sim r}}
{\Delta_{\max}\Delta_{\min}}\bar\nu^2
\asymp
\frac{\tau_*^2}{\log n}.
\]
Therefore,
\[
\mathbb P\left\{
\|L_n\|
\gtrsim
\frac{\tau_*}{\sqrt{\log n}}
\right\}
\ge
\frac1{12}[1+o(1)].
\]

\paragraph{Transfer to the tangent matrix.}
By Proposition~\ref{prop:linearization}, with probability at least $1-\delta_p$,
\[
\|\cT_n-L_n\|
\lesssim
\frac{\tau_\ast}{n^\epsilon\log n}
=
o\!\left(\frac{\tau_\ast}{\sqrt{\log n}}\right).
\]
Intersecting this event with the event above and using the triangle inequality,
\[
\|\cT_n\|
\ge
\|L_n\|-\|\cT_n-L_n\|
\gtrsim
\frac{\tau_\ast}{\sqrt{\log n}}
\]
with probability at least
\[
\frac{1}{12}[1+o(1)]-\delta_p.
\]

\paragraph{Transfer to the sine distance.}
By Lemma~\ref{lem:alignment_diff},
\[
\|\sin\Theta(U_n,U_\ast)\| = \|\cT_n(\cC_n\cC_n^\top)^{1/2}\|.
\]
Moreover, on the same high-probability event used in the proof of
Proposition~\ref{prop:linearization},
\[
\|\cT_n\|\lesssim \tau_\ast=o(1).
\]
Hence
\[
\sigma_{\min}\!\left((\cC_n\cC_n^\top)^{1/2}\right)=1-o(1),
\]
and therefore
\[
\|\sin\Theta(U_n,U_\ast)\|
\ge (1-o(1))\|\cT_n\|
\gtrsim
\frac{\tau_\ast}{\sqrt{\log n}}.
\]
This proves the claim.

\section{Setup for Gaussian approximation and Bootstrap}
\label{sec:clt_setup}
\subsection{Notation and definitions.}
\label{sec:clt_setup_defn}
This section introduces the operator-theoretic framework underlying Theorems~\ref{thm:highD_clt}--\ref{thm:bs_row_clt}.
All the results rest on a common set of objects, which we now define.

\paragraph{Hilbert space $\cH$.}
% -----------------------------------------------------------------------------
Let $\cH = \RR^{(d-r)\times r}$ equipped with the Frobenius inner product $\Finner{\cdot}{\cdot}$. 
Every random element $\xi$ of $\cH$ with $\EE\xi = 0$ and 
$\EE \Fnorm{\xi}^2 < \infty$ induces a \emph{covariance operator} $\cov(\xi)$, which is the unique self-adjoint positive-semidefinite linear operator on $\cH$ satisfying
\begin{equation*}%\label{eq:cov-def}
  \cov(\xi)(A) 
  = \EE \sbr{\Finner{A}{\xi} \xi},
  \qquad \forall\, A \in \cH.
\end{equation*}
Equivalently, for any $A, B \in \cH$, $\Finner{B}{\cov(\xi)(A)}  = \EE[\Finner{A}{\xi} \Finner{B}{\xi}]$.
In particular, $\tr\sbr{\cov(\xi)} = \EE\Fnorm{\xi}^2$ and 
$\Onorm{\cov(\xi)} = \sup_{\Fnorm{A} = 1} \EE\Finner{A}{\xi}^2$.

% -----------------------------------------------------------------------------
\paragraph{Gap operator.}
% -----------------------------------------------------------------------------
Define the linear operator $\cL : \cH \to \cH$ by
\begin{align*}
    [\cL A]_{lj} = \Delta_{lj} A_{lj},
    \qquad \Delta_{lj} := \lambda_j - \lambda_{r+l},
    \quad
    1 \le l \le d-r,\; 1 \le j \le r.
\end{align*}
The operator $\cL$ is diagonal in the standard basis 
$\cbr{E_{lj} := e_l e_j^\top}_{l,j}$ of $\cH$ with eigenvalues 
$\{\Delta_{lj}\}$.
Write $\Delta_{\min} := \min_{l,j} \Delta_{lj} > 0$ and 
$\Delta_{\max} := \max_{l,j} \Delta_{lj}$.
Since $\cL$ is strictly positive, the semigroup $\cbr{e^{-t\cL}}_{t \ge 0}$ is 
exponentially stable: $\Onorm{e^{-t\cL}} \le e^{-t\Delta_{\min}}$ for 
all $t \ge 0$.

% -----------------------------------------------------------------------------
\paragraph{Covariance operator of $ZY^\top$ and $\tZ \tY^\top$.}
% -----------------------------------------------------------------------------
Define the operator $\fC : \cH \to \cH$ as the covariance of the random element $ZY^\top \in \cH$:
\begin{equation} \label{eqn:def_fC}
  \fC(B) := \EE\sbr{\la ZY^\top, B\ra_\cH\; ZY^\top}, \qquad B \in \cH.
\end{equation}
In coordinates, $\Finner{E_{l'j'}}{\fC(E_{lj})} = \EE \rbr[\big]{\sbr{Z}_l \sbr{Y}_j \sbr{Z}_{l'} \sbr{Y}_{j'}}$.
The operator $\fC$ is self-adjoint and positive-semidefinite on $\cH$, with 
$\tr(\fC) = \EE\Fnorm{ZY^\top}^2$ and 
$\HS{\fC} = \sbr[\big]{\sum_{lj,l'j'} \Finner{E_{l'j'}}{\fC(E_{lj})}^2}^{1/2}$.

% \smallskip
Similarly, define the \emph{whitened covariance operator} $\tfC : \cH \to \cH$ using the isotropic coordinates 
$\tX = (\tY, \tZ) \in \RR^r \times \RR^{d-r}$ (where $Y = \Lambda_*^{1/2}\tY$, $Z = \Lambda_{*,\perp}^{1/2}\tZ$):
\begin{equation*}%\label{eq:def-tC}
    \tfC(B) := \EE\sbr[\big]{\Finner[\big]{\tZ\tY^\top}{B} \, \tZ\tY^\top},
    \qquad B \in \cH.
\end{equation*}

% -----------------------------------------------------------------------------
\paragraph{Summands.}
% -----------------------------------------------------------------------------
All the Gaussian approximation and bootstrap theorems rest on the decomposition
\begin{equation*}%\label{eq:linear-rep}
  \frac{1}{\sqrt{\eta}}\sbr{U_n \sgn(U_n^\top U_*) - U_*}
  \approx \frac{1}{\sqrt{\eta}} U_{*,\perp} T_n
  \approx U_{*,\perp} \sbr[\Big]{ \sum_{i=1}^{n} \sqrt{\eta}(I - \eta\cL)^{n-i}\, Z_i Y_i^\top }
\end{equation*}
% where $E_n$ is a remainder controlled by the dynamics 
% (Section~\ref{sec:dynamics}) and $C_n = U_*^\top U_n$.
Thus, we define the \emph{summands} $\xi_i \in \cH$ and their covariances $\Xi_i$ by
\begin{equation} \label{eqn:xi_defn}
  \xi_i := \sqrt{\eta} (\cI - \eta\cL)^{n-i}\, (Z_i Y_i^\top)
  \in \cH,
  \qquad \qquad
  \Xi_i := \cov(\xi_i) = \eta (\cI - \eta\cL)^{n-i} \fC (\cI - \eta\cL)^{n-i}.
\end{equation}
Note that each $\Xi_i$ is a self-adjoint positive-semidefinite linear operator on $\cH$, obtained by conjugating $\fC$ with the contraction $(\cI - \eta\cL)^{n-i}$.

% -----------------------------------------------------------------------------
\paragraph{Cumulative covariances.}
% -----------------------------------------------------------------------------
Define the finite-sample and limiting cumulative covariances $\Gamma_n$ and~$\Gamma$ by
\begin{align*}
  \Gamma_n &:= \sum_{i=1}^{n} \Xi_i 
  = \eta \sum_{k=0}^{n-1} (\cI - \eta\cL)^{k} \fC (\cI - \eta\cL)^{k},\qquad \qquad  \Gamma := \int_0^\infty e^{-t\cL} \fC e^{-t\cL}\, dt.
%   \label{eq:def-Gamma}
\end{align*}
The sum $\Gamma_n$ is an approximation to $\Gamma$.
The operator $\Gamma$ is the unique solution to the \emph{Lyapunov equation} $\cL\Gamma + \Gamma\cL = \fC$ (Lemma~\ref{lem:Gamma_properties}).

% \smallskip
The whitened analogues $\tGamma_n$ and $\tGamma$ are defined identically 
with $\tfC$ replacing $\fC$.

% -----------------------------------------------------------------------------
\paragraph{Projection to $\RR^r$.}
% -----------------------------------------------------------------------------
For $u \in \RR^{d-r}$, define the bounded linear map $T_u : \cH \to \RR^r$ by $T_u(B) := B^\top u$.
Its adjoint is $T_u^* : \RR^r \to \cH$, $T_u^*(v) = uv^\top$.
Given $u \in \RR^{d-r}$, define the \emph{projected summands} $\xi_{u,i}$ and covariances $\Xi_{u,i}$ by
\begin{align*}
    \xi_{u,i} &:= T_u(\xi_i) = \sqrt{\eta} [(\cI - \eta\cL)^{n-i} (Z_i Y_i^\top)]^\top u \in \RR^r, \quad \quad  \Xi_{u,i} := \cov(\xi_{u,i}) = T_u \Xi_i T_u^* \in \RR^{r \times r},
\end{align*}
and the cumulative covariances $\Gamma_{u,n}$ and $\Gamma_u$ by
\begin{align}
    \Gamma_{u,n}:= \sum_{i=1}^{n} \Xi_{u,i} = T_u \Gamma_n T_u^* \in \RR^{r \times r}, \qquad \qquad
    \Gamma_u := T_u \Gamma T_u^* \in \RR^{r \times r}. \label{eqn:Gamma_u_defn}
\end{align}

The whitened analogues $\txi_{u,i}$, $\tXi_{u,i}$, $\tGamma_{u,n}$, and $\tGamma_{u}$ are defined identically with $\tZ_i \tY_i^\top$ replacing $Z_i Y_i^\top$ and $\tfC$ replacing $\fC$, i.e.
\begin{equation}
    \begin{alignedat}{2}
        \txi_{u,i} &:= T_u(\txi_i) = \sqrt{\eta} [(\cI - \eta\cL)^{n-i} (\tZ_i \tY_i^\top)]^\top u \in \RR^r,\quad \quad  &&\tXi_{u,i} := \cov(\txi_{u,i}) = T_u \tXi_i T_u^* \in \RR^{r \times r}, \\
        \tGamma_{u,n}&:= \sum_{i=1}^{n} \tXi_{u,i} = T_u \tGamma_n T_u^* \in \RR^{r \times r}, 
        &&\tGamma_u := T_u \tGamma T_u^* \in \RR^{r \times r}.
    \end{alignedat}
    \label{eqn:whitened_defn}
\end{equation}

% -----------------------------------------------------------------------------
\paragraph{Summary of notation.}
% -----------------------------------------------------------------------------
The table below collects the key objects given $u \in \RR^{d-r}$.
The first column concerns the $\cH$-level analysis (Theorem~\ref{thm:highD_clt}), while 
the last three columns concern the projected analysis (Theorem~\ref{thm:row_clt}).

\begin{table}[ht]
\centering
\caption{Notation for summands and covariances.}
\label{tab:notation_summary}
\renewcommand{\arraystretch}{1.4}
\small
\begin{tabular}{l | c c | c c}
\hline
& \multicolumn{2}{c|}{$\cH$-level} 
& \multicolumn{2}{c}{Projected to $\RR^r$} \\
& unwhitened & whitened & unwhitened & whitened \\
\hline
summand 
  & $\xi_i$ & $\txi_i$ 
  & $\xi_{u,i} = T_u(\xi_i)$ & $\txi_{u,i} = T_u(\txi_i)$ \\
covariance 
  & $\Xi_i$ & $\tXi_i$ 
  & $T_u \Xi_i T_u^*$ & $T_u \tXi_i T_u^*$ \\
cum.\ cov.\ (finite)
  & $\Gamma_n$ & $\tGamma_n$ 
  & $\Gamma_{u,n} = T_u \Gamma_n T_u^*$ 
  & $\tGamma_{u,n} = T_u \tGamma_n T_u^*$ \\
cum.\ cov.\ (limit)
  & $\Gamma$ & $\tGamma$ 
  & $\Gamma_u = T_u \Gamma T_u^*$ 
  & $\tGamma_u = T_u \tGamma T_u^*$ \\
\hline
\end{tabular}
\end{table}

\subsection{Technical lemmas}
In this section, we collect a few useful results used in Appendix~\ref{sec:clt_analysis} and \ref{sec:bootstrap_analysis}.

% =============================================================================
\paragraph{Properties of $\Gamma$, $\Gamma_n$, and $\xi_i$.}
% =============================================================================

\begin{lemma}\label{lem:Gamma_properties}
Assume $0 < \eta\Delta_{\max} < 1$ and let $\opnorm{\cdot}$ denote any 
unitarily invariant norm on $\cH$.

\textup{(i) (Coordinate representation)}  For $l, l' \in [d-r]$ and 
$j, j' \in [r]$,
\begin{align*}
    \Finner{E_{l'j'}}{\Gamma E_{lj}}
    &= \frac{\Finner{E_{l'j'}}{\fC E_{lj}}}
            {\Delta_{lj} + \Delta_{l'j'}}, \\
    \Finner{E_{l'j'}}{\Gamma_n E_{lj}}
    &= \eta \sum_{i=1}^{n} 
    (1 - \eta\Delta_{lj})^{n-i}(1 - \eta\Delta_{l'j'})^{n-i}
    \Finner{E_{l'j'}}{\fC E_{lj}}.
\end{align*}

\textup{(ii) (Bounds from the Lyapunov equation)}   $\Gamma$ is the unique solution to $\cL\Gamma + \Gamma\cL = \fC$, 
with
\begin{align*}
    \frac{\opnorm{\fC}}{2\Delta_{\max}} 
    \le \opnorm{\Gamma} 
    \le \frac{\opnorm{\fC}}{2\Delta_{\min}},
    \qquad \text{and} \qquad
    \Gamma \succeq \frac{\lambda_{\min}(\fC)}{2\Delta_{\max}} \cI.
\end{align*}

\textup{(iii) (Lower bound for $\Gamma_n$)} If $\eta \Delta_{\max} \in (0,1)$, then
\begin{align*}
    \Gamma_n \succeq \frac{\lambda_{\min}(\fC)}{2\Delta_{\max}} \sbr{1-(1-\eta \Delta_{\max})^{2n}} \cI.
\end{align*}
If $\eta \Delta_{\max} \in (0,1)$ and $n \eta \Delta_{\max} \geq 1$, then
\begin{align*}
    \Gamma_n \succeq \frac{\lambda_{\min}(\fC)}{3\Delta_{\max}} \cI.
\end{align*}

\textup{(iv) (Approximation)} Let $\epsilon_n := (1 - \teta)^{2n} 
+ \teta \kappa_\Delta^2$ where $\teta = \eta\Delta_{\min}$.  
Then
\begin{align*}
    &\opnorm{\Gamma_n - \Gamma} \le \epsilon_n \opnorm{\Gamma}
\end{align*}
As a consequence, $\abs{\tr(\Gamma_n - \Gamma)} \le \epsilon_n \tr\Gamma$.

\textup{(v) (Moment bound for $\xi_i$)} For any $q \ge 1$,
\[
  \sum_{i=1}^{n} \EE\Fnorm{\xi_i}^q 
  \;\le\; \frac{\teta^{\,q/2 - 1}}{\Delta_{\min}^{q/2}}\; 
          \EE\Fnorm{ZY^\top}^q.
\]

\textup{(vi) (Matrix moment bound)}
Under Assumption~\ref{assumption:X}, for any $q \ge 1$,
\[
  \sbr{\EE\Fnorm{ZY^\top}^{2q}}^{1/(2q)}
  \le Cq [\tr\Lambda_*]^{1/2} [\tr\Lambda_{*,\perp}]^{1/2}\,
          \norm[\big]{\tX}_{\psi_2}^2,
\]
and the whitened analogue holds with $\Lambda_*$ and $\Lambda_{*,\perp}$ replaced by $I_r$ and $I_{d-r}$, respectively:
\[
  \sbr{\EE\Fnorm[\big]{\tZ\tY^\top}^{2q}}^{1/(2q)}
  \le Cq r^{1/2} (d-r)^{1/2} \norm[\big]{\tX}_{\psi_2}^2.
\]
Here $C>0$ is an absolute constant.

\textup{(vii) (Norm of $\tfC$)} $\Onorm[\big]{\tfC} \lesssim \min\{r, d-r\}$.

\end{lemma}

\begin{proof}
    See Appendix~\ref{sec:prf_Gamma_properties}.
\end{proof}

\begin{remark}[Whitened analogue]\label{rmk:whitened_Gamma}
    Lemma~\ref{lem:Gamma_properties} holds verbatim with  
    $(\tfC, \tGamma_n, \tGamma, \txi_i)$ 
    replacing 
    $(\fC,\allowbreak \Gamma_n, \allowbreak \Gamma, \allowbreak \xi_i)$, 
    since the proof uses only 
    that $\fC$ is positive-semidefinite and $\cL$ is strictly positive.
\end{remark}

\begin{lemma}\label{lem:Gamma_u_properties}
Let $u \in \RR^{d-r}$ and $\tu = \Lambda_{*,\perp}^{1/2} u$.  

\textup{(i) (Coordinate representation)} For any $u \in \RR^{d-r}$,
\begin{align*}
    \inner{e_{j'}}{\Gamma_u e_j}_2
    = \sum_{l, l'=1}^{d-r} 
    \frac{u_l\, u_{l'}\Finner{E_{l'j'}}{\fC(E_{lj})}}
        {\Delta_{lj} + \Delta_{l'j'}}, \qquad \forall j, j' \in [r],
\end{align*}
and in whitened form,
\begin{align*}
    \inner[\big]{e_{j'}}{\tGamma_{u} e_j}_2
    = \sum_{l, l'=1}^{d-r} 
    \frac{\tu_l \, \tu_{l'} \Finner[\big]{E_{l'j'}}{\tfC(E_{lj})}}
            {\Delta_{lj} + \Delta_{l'j'}}, \qquad \forall j, j' \in [r].
\end{align*}

\textup{(ii) (Whitened--unwhitened connection)} For any $u \in \RR^{d-r}$, let $\tu = \Lambda_{*,\perp}^{1/2} u$.  Then
\begin{align*}
    \Gamma_u = \Lambda_*^{1/2} \tGamma_{\tu} \Lambda_*^{1/2}, \qquad \text{and} \qquad
    \Gamma_{u,n} = \Lambda_*^{1/2} \tGamma_{\tu,n} \Lambda_*^{1/2}.
\end{align*}

\textup{(iii) (Loewner order lower bound)}
\begin{align*}
    \tGamma_{\tu}
  \,\succeq\, \frac{\|\tu\|^2}{2\Delta_{\max}} \lambda_{\min}(\tfC)\, I_r,
  \qquad \text{and} \qquad
  \Gamma_u
  \,\succeq\, \frac{\lambda_r \|\tu\|^2}{2\Delta_{\max}} \lambda_{\min}(\tfC)\,I_r.
\end{align*}
If $\eta \Delta_{\max} \in (0,1)$ and $n \eta \Delta_{\max} \geq 1$, then
\begin{align*}
    \tGamma_{\tu,n} &\succeq \frac{\|\tu\|^2}{3\Delta_{\max}} \lambda_{\min}(\tfC)\, I_r, \qquad \text{and} \qquad
    \Gamma_{u,n} \succeq \frac{\lambda_r \|\tu\|^2}{3\Delta_{\max}} \lambda_{\min}(\tfC)\,I_r.
\end{align*}

\textup{(iv) (Approximation)} With $\epsilon_n$ as in 
Lemma~\ref{lem:Gamma_properties}(iv). If $\lambda_{\min}(\tfC) > 0$, then
\begin{align*}
    \opnorm[\big]{\tGamma_{\tu,n} - \tGamma_{\tu}} 
    \;\le\; \epsilon_n \kappa_\Delta\,
    \frac{\|\tfC\|_{\mathrm{op}}}{\lambda_{\min}(\tfC)}\,
    \opnorm[\big]{\tGamma_{\tu}}. 
\end{align*}

\textup{(v) (Independence case)} If $\tY \perp \tZ$, then $\tC = \cI$ and
\begin{align*}
    [\tGamma_{\tu}]_{jj'} = \delta_{jj'} \sum_{l=1}^{d-r} 
    \frac{\tu_l^2}{2\Delta_{lj}},
    \qquad
    [\Gamma_u]_{jj'} = \delta_{jj'} \sum_{l=1}^{d-r} 
    \frac{\lambda_j \tu_l^2}{2 \Delta_{lj}}.
\end{align*}

\end{lemma}

\begin{proof}
    See Appendix~\ref{sec:prf_Gamma_u_properties}.
\end{proof}

\begin{lemma} \label{lem:Latala}
    For independent $X_1,\ldots,X_n \geq 0$ and $p \geq 1$,
    \begin{align*}
        \EE \rbr[\Big]{\sum_{i=1}^n X_i}^{p} \leq C_p \sbr[\bigg]{\rbr[\Big]{\sum_{i=1}^n \EE X_i}^p + \sum_{i=1}^n \EE X_i^p}
    \end{align*}
    Let $\fg \sim \cN(0, I_d)$ be a Gaussian vector in $\RR^d$. If $G \succeq 0$, then for $p \geq 1$,
    \begin{align*}
        \EE \rbr{\fg^\top G \fg}^p \lesssim_p \sbr{\tr G}^p
    \end{align*}
\end{lemma}
\begin{proof}
    See Appendix~\ref{sec:prf_Latala}.
\end{proof}

\begin{lemma}[weighted $\chi^2$ anti-concentration] \label{lem:chi2_anti_concentration}
    Let $a_1 \geq \cdots \geq a_p \geq 0$ such that $\sum_{i=1}^p a_i^2 = 1$ and $\omega_1,\ldots,\omega_p$ be i.i.d. $\chi^2_1$ random variables. Then
    \begin{align*}
        \sup_{t \in \RR} \PP \rbr[\Big]{ t \leq \sum_{i=1}^p a_i \omega_i \leq t + h} \leq \sqrt{\frac{4h}{\pi}}
    \end{align*}
\end{lemma}
\begin{proof}
    See \citep[Proposition~B.7]{nips21_oja}.
\end{proof}

\begin{lemma}[$\chi^2$ concentration] \label{lem:chi2_concentration}
    Let $\fg \sim \cN(0,\cI)$ be a Gaussian random element in $\cH$ with identity covariance. Then for every $t \geq 0$,  we have
    \begin{align*}
        \PP\cbr{\abs{\fg^\top \Gamma \fg - \tr \Gamma} \geq t} \leq 2\exp \cbr[\Big]{- c\rbr[\Big]{\frac{t^2}{\HS{\Gamma}^2} \bigwedge \frac{t}{\Onorm{\Gamma}}}}
    \end{align*}
    As a consequence, with probability exceeding $1-n^{-20}$,
    \begin{align*}
        \abs{\fg^\top \Gamma \fg - \tr \Gamma} \lesssim \HS{\Gamma} \sqrt{\log n} + \Onorm{\Gamma} \log n
    \end{align*}
\end{lemma}
\begin{proof}
    See \cite[Hanson-Wright inequality from][Theorem~6.2.2.]{vershynin}
\end{proof}

\begin{lemma}[$\chi^2$ mixture comparison] 
    \label{lem:chi2_mixture_comparison}
    Suppose $\fg$ and $\tfg$ are centered Gaussian random elements in $\cH$ with $\fg \sim \cN(0, \Gamma)$, $\tfg \sim \cN(0, \tGamma)$. Let $\Delta_\Gamma= \Gamma - \tGamma$. If $\HS{\Gamma}>0$, then for any $t>0$,
    \begin{align*}
        d_K\rbr[\Big]{\Fnorm[\big]{\fg}^2, \Fnorm[\big]{\tfg}^2} \lesssim \sqrt{\frac{\abs{\tr \Delta_\Gamma}+t}{\HS{\Gamma}}} + \exp \cbr[\bigg]{- c\rbr[\bigg]{\frac{t^2}{\HS{\Delta_\Gamma}^2} \bigwedge \frac{t}{\Onorm{\Delta_\Gamma}}}}
    \end{align*}
\end{lemma}
\begin{proof}
    See \citep[Lemma~C.5]{nips21_oja}.
\end{proof}

\begin{lemma}[High-dimensional CLT] \label{lem:clt_biometrika19}
    Let $\cbr{\xi_i}_{i=1}^n$ be independent random variables in $\cH$ such that $\EE \xi_i=0$ and $\cov \xi_i = \Xi_i$. Let $\cbr{\fg_i}_{i=1}^n$ be independent Gaussian analogues of $\cbr{\xi_i}_{i=1}^n$ with $\fg_i \sim \cN(0, \Xi_i)$. Furthermore, let $\Gamma_n = \sum_{i=1}^n \Xi_i$. For $0 < \delta \leq 1$, $q=2+\delta$ and $\beta \geq 2$ define the following quantities:
    \begin{align*}
        L_q^{\xi} &= \sum_{i=1}^n \frac{\EE \Finner{\xi_i}{\Gamma_n \xi_i}^{q/2}}{\HS{\Gamma_n}^q} + \frac{n^2}{\binom{n}{2}} \sum_{1 \leq i < j \leq n} \frac{\EE \abs{\Finner{\xi_i}{\xi_j}}^q}{\HS{\Gamma_n}^q}\\
        L_q^{\fg} &= \sum_{i=1}^n \frac{\EE \Finner{\fg_i}{\Gamma_n \fg_i}^{q/2}}{\HS{\Gamma_n}^q}\\
        K_\beta^{\beta} &= \frac{n^\beta}{n} \sum_{i=1}^n \EE \abs{\frac{\Finner{\xi_i}{\xi_i} - \EE \Finner{\xi_i}{\xi_i}}{\HS{\Gamma_n}}}^{\beta}\\
        J_n &= \frac{\sum_{i=1}^n \fv_i}{\HS{\Gamma_n}^2}, \qquad\text{where}\quad \fv_i=\Var (\Fnorm{\xi_i}^2)
    \end{align*}

    Suppose $L_q^{\xi} \to 0$, $L_q^{\fg} \to 0$, $n^{1-\beta}K_\beta^{\beta} \to 0$, and $J_n \to 0$. Then
    \begin{align*}
        d_K \rbr[\bigg]{\Fnorm[\Big]{\sum_{i=1}^n \xi_i}^2, \Fnorm[\Big]{\sum_{i=1}^n \fg_i}^2} \to 0
    \end{align*}
\end{lemma}
\begin{proof}
    See \citep[Proposition~B.5]{nips21_oja}.
\end{proof}

\begin{lemma}[G-approx additive remainder] \label{lem:g_approx_additive_remainder}
    Let $\mathsf{A}_n, \mathsf{E}_n \in \RR^{d_n}$ be random vectors, and define $\mathsf{B}_n:= \mathsf{A}_n + \mathsf{E}_n$. Let $\mathsf{G}_n \sim N(0,\Gamma_n)$ be a centered Gaussian vector in $\RR^{d_n}$. Define
    \begin{align*}
        \delta_n&:= \sup_{t \in \RR} \abs{\PP\cbr[\big]{\norm{\mathsf{A}_n}^2 \leq t} - \PP\cbr[\big]{\norm{\mathsf{G}_n}^2 \leq t}},\\
        w_n(h) &:= \sup_{t \in \RR} \PP\cbr[\big]{t < \norm{\mathsf{G}_n}^2 \leq t+h}, \quad \forall h \geq 0.
    \end{align*}
    For any deterministic sequences $e_n, L_n > 0$, define $h_n := 2L_n e_n + e_n^2$. Then
    \begin{align*}
        \sup_{t \in \RR} \abs{\PP\cbr[\big]{\norm{{\mathsf{B}}_n}^2 \leq t} - \PP\cbr[\big]{\norm{\mathsf{G}_n}^2 \leq t}}
        \leq \; & 2\delta_n + w_n(h_n)  + \PP\cbr[\big]{\norm{\mathsf{E}_n} > e_n} + \PP\cbr[\big]{\norm{\mathsf{G}_n} > L_n}.
    \end{align*}
\end{lemma}
\begin{proof}
    See Appendix~\ref{sec:prf_g_approx_additive_remainder}.
\end{proof}

\begin{lemma}[G-approx multiplicative remainder] \label{lem:g_approx_multiplicative_remainder}
    Let random $\mathsf{A}_n, \mathsf{B}_n \in \RR^{d_n}$, and suppose $\mathsf{B}_n=(I_{d_n}+\mathsf{E}_n)\mathsf{A}_n$
    for some random $\mathsf{E}_n \in \RR^{d_n \times d_n}$. Let $\mathsf{G}_n \sim N(0, \Gamma_n) \in \RR^{d_n}$. Define 
    \begin{align*}
        \delta_n &:= \sup_{t \in \RR} \abs[\Big]{\PP\cbr[\big]{\norm{\mathsf{A}_n}^2 \leq t} - \PP\cbr[\big]{\norm{\mathsf{G}_n}^2 \leq t}},\\
        w_n(h) &:=\sup_{t \in \RR} \PP\cbr[\big]{t < \norm{\mathsf{G}_n}^2 \leq t+h}, \quad \forall h \geq 0.
    \end{align*}
    Let $e_n \in [0,1/2]$ and $L_n > 0$ be deterministic sequences, and define $h_n := \sbr[\big]{\frac{1}{(1-e_n)^2}-1} L_n^2$. Then
    \begin{align*}
        \sup_{t \in \RR} \abs[\Big]{\PP\cbr[\big]{\norm{\mathsf{B}_n}^2 \leq t} - \PP\cbr[\big]{\norm{\mathsf{G}_n}^2 \leq t}}
        \leq \; &\delta_n + w_n(h_n)  + \PP(\norm{\mathsf{E}_n} > e_n) + \PP(\norm{\mathsf{G}_n} > L_n).
    \end{align*}
\end{lemma}
\begin{proof}
    See Appendix~\ref{sec:prf_g_approx_multiplicative_remainder}.
\end{proof}

% -----------------------------------------------------------------------------

\subsection{Proof of technical lemmas}

\subsubsection{Proof of Lemma~\ref{lem:Gamma_properties}}
\label{sec:prf_Gamma_properties}

\textbf{(i) Coordinate representation.} Note that $\cL$ is diagonal under the basis $\cbr{E_{lj}}_{l,j}$. For any $l,l' \in [d-r]$ and $j,j' \in [r]$, we have $\cL E_{lj} = \Delta_{lj} E_{lj}$ with $\Delta_{lj} = \lambda_j - \lambda_{r+l} > 0$. By the definition of $\Gamma$, we have
\begin{align*}
    \Finner{E_{l'j'}}{\Gamma E_{lj}} &= \int_0^\infty \Finner{e^{-t\cL} E_{l'j'}}{ \fC e^{-t\cL} E_{lj}} dt \\
    &= \int_0^\infty e^{-t(\Delta_{lj} + \Delta_{l'j'})} dt \cdot \Finner{E_{l'j'}}{\fC E_{lj}} = \frac{\Finner{E_{l'j'}}{\fC E_{lj}}}{\Delta_{lj} + \Delta_{l'j'}}
\end{align*}
Similarly for $\Gamma_n$, we have
\begin{align*}
    \Finner{E_{l'j'}}{\Gamma_n E_{lj}} &= \eta \sum_{i=1}^n \Finner{(\cI-\eta \cL)^{n-i} E_{l'j'}}{ \fC (\cI-\eta \cL)^{n-i} E_{lj}}\\
    &= \eta \sum_{i=1}^n (1-\eta \Delta_{lj})^{n-i} (1-\eta \Delta_{l'j'})^{n-i} \cdot \Finner{E_{l'j'}}{\fC E_{lj}}
\end{align*}

\textbf{(ii) Bounds from the Lyapunov equation.} We first show that $\Gamma$ is the unique solution $\cL \Gamma + \Gamma \cL = \fC$. Note that
\begin{align}
    \cL \Gamma_* + \Gamma_* \cL = \fC &\iff \Finner{E_{l'j'}}{(\cL \Gamma_* + \Gamma_* \cL) E_{lj}} = \Finner{E_{l'j'}}{\fC E_{lj}}, \quad \forall l,l' \in [d-r]; j,j' \in [r] \nonumber\\
    &\iff \Finner{\cL E_{l'j'}}{\Gamma_* E_{lj}} + \Finner{E_{l'j'}}{\Gamma_* \cL E_{lj}} = \Finner{E_{l'j'}}{\fC E_{lj}} \nonumber\\
    &\iff \Finner{E_{l'j'}}{\Gamma_* E_{lj}} = \frac{\Finner{E_{l'j'}}{\fC E_{lj}}}{\Delta_{lj} + \Delta_{l'j'}} \nonumber\\
    &\overset{(i)}{\iff} \Finner{E_{l'j'}}{\Gamma_* E_{lj}} = \Finner{E_{l'j'}}{\Gamma E_{lj}}. \label{eqn:lyapunov_coordinate}
\end{align}
Here (i) follows from the coordinate representation of $\Gamma$. Hence $\Gamma$ is the unique solution to the Lyapunov equation.

Next, we prove the two-sided bound for unitarily invariant norm $\opnorm{\cdot}$. For the upper bound, taking $\opnorm{\cdot}$ on the integral definition $\Gamma=\int_0^\infty e^{-t\cL} \fC e^{-t\cL}dt$ gives
\begin{align*}
    \opnorm{\Gamma} &\leq \int_0^\infty \opnorm{e^{-t\cL} \fC e^{-t\cL}} dt \overset{(i)}{\leq} \int_0^\infty \Onorm{e^{-t\cL}}^2 \opnorm{\fC}dt \leq \int_0^\infty e^{-2t\Delta_{\min}} \opnorm{\fC} dt = \frac{\opnorm{\fC}}{2\Delta_{\min}}
\end{align*}
% Here (i) follows from the ideal property of unitarily invariant norms.
For the lower bound, taking $\opnorm{\cdot}$ on both sides of the Lyapunov equation we have
\begin{align*}
    \opnorm{\fC} = \opnorm{\cL \Gamma + \Gamma \cL} \overset{(i)}{\leq} 2\Onorm{\cL} \opnorm{\Gamma} \leq 2\Delta_{\max} \opnorm{\Gamma}
\end{align*}
For both bounds, (i) follows from the ideal property of unitarily invariant norms.

Then, we establish the Loewner order lower bound. Let $\cA = \frac{\lambda_{\min}(\fC)}{2} \cL^{-1}$. Then we have
\begin{align*}
    \cL \cA + \cA \cL = \lambda_{\min}(\fC) \cI \preceq \fC = \cL \Gamma + \Gamma \cL
\end{align*}
which implies that $\cB = \Gamma - \cA$ solves the Lyapunov equation $\cL \cB + \cB \cL = \fC - \lambda_{\min}(\fC) \cI \succeq 0$. The same argument as~\eqref{eqn:lyapunov_coordinate} shows that the unique solution is
\begin{align*}
    \Gamma - \cA = \int_0^\infty e^{-t\cL} (\fC - \lambda_{\min}(\fC) \cI) e^{-t\cL} dt \succeq 0
\end{align*}
Hence $\Gamma \succeq \cA$. Finally, since $\cL^{-1}$ is diagonal with eigenvalues $\Delta_{lj}^{-1} \ge \Delta_{\max}^{-1}$, we have $\cA \succeq \frac{\lambda_{\min}(\fC)}{2\Delta_{\max}} \cI$ and consequently $\Gamma \succeq \frac{\lambda_{\min}(\fC)}{2\Delta_{\max}} \cI$.

\textbf{(iii) Lower bound for $\Gamma_n$.} Since $\cL$ is diagonal with eigenvalues in $\sbr{\Delta_{\min}, \Delta_{\max}}$, we have
\begin{align*}
    (\cI - \eta \cL)^{n-i} \succeq (1-\eta \Delta_{\max})^{n-i} \cI, \qquad \forall i=1,2,\ldots,n
\end{align*}
which combined with $\fC \succeq \lambda_{\min}(\fC) \cI$ gives
\begin{align*}
    (\cI - \eta \cL)^{n-i} \fC (\cI - \eta \cL)^{n-i} &\succeq \lambda_{\min}(\fC) (\cI - \eta \cL)^{2(n-i)} \succeq \lambda_{\min}(\fC) (1-\eta \Delta_{\max})^{2(n-i)} \cI
\end{align*}
Summing over $i$ gives
\begin{align*}
    \Gamma_n = \eta \sum_{i=1}^n (\cI - \eta \cL)^{n-i} \fC (\cI - \eta \cL)^{n-i} &\succeq \eta \lambda_{\min}(\fC) \sum_{i=1}^n (1-\eta \Delta_{\max})^{2(n-i)} \cI
\end{align*}
Therefore, we have
\begin{align*}
    \lambda_{\min}(\Gamma_n) & \geq  \lambda_{\min}(\fC) \frac{1 - (1-\eta\Delta_{\max})^{2n}}{\Delta_{\max} \rbr{2- \eta \Delta_{\max}}} 
\end{align*}
If further $n\eta\Delta_{\max} \geq 1$, then $(1-\eta\Delta_{\max})^{2n} \leq \exp\rbr{-2n\eta\Delta_{\max}} < e^{-2}$. Hence $1-(1-\eta\Delta_{\max})^{2n} \geq 1 - e^{-2}$. Since $2-\eta \Delta_{\max} < 2$, we have
\begin{align*}
    \lambda_{\min}(\Gamma_n) &\geq \lambda_{\min}(\fC) \frac{1-e^{-2}}{2\Delta_{\max}} > \frac{\lambda_{\min}(\fC)}{3\Delta_{\max}}
\end{align*}
where the last inequality uses $(1-e^{-2})/2 > 1/3$.

\textbf{(iv) Approximation.} By definition of $\Gamma_n$, for $n \geq 1$, we have
\begin{align*}
    \Gamma_n = \eta \sum_{i=1}^n (\cI - \eta \cL)^{n-i} \fC (\cI - \eta \cL)^{n-i} \overset{(i)}{=} \eta \sum_{k=0}^{n-1} (\cI - \eta \cL)^{k} \fC (\cI - \eta \cL)^{k}
\end{align*}
Here (i) follows from change of variable $k = n-i$. Let $\Gamma_0 = 0$.
Then we have recursion formula
\begin{align*}
    \Gamma_{n} = (\cI - \eta \cL) \Gamma_{n-1} (\cI - \eta \cL) + \eta \fC, \qquad n=1,2,\ldots
\end{align*}
Subtracting $\Gamma$ on both sides and applying the Lyapunov equation, we obtain that for any $n \geq 1$,
\begin{align*}
    \Gamma_n -\Gamma &= (\cI  -\eta \cL) \rbr{\Gamma_{n-1}-\Gamma} (\cI - \eta \cL) - \eta \underbrace{\sbr{\cL \Gamma + \Gamma \cL - \fC}}_{= 0} + \eta^2 \cL \Gamma \cL\\
    &\overset{(i)}{=} (\cI - \eta \cL)^n \rbr{\Gamma_0-\Gamma} (\cI - \eta \cL)^n + \eta^2 \sum_{k=0}^{n-1} (\cI - \eta \cL)^k \cL \Gamma \cL (\cI - \eta \cL)^k\\
    &= - (\cI - \eta \cL)^n \Gamma (\cI - \eta \cL)^n + \eta^2 \sum_{k=0}^{n-1} (\cI - \eta \cL)^k \cL \Gamma \cL (\cI - \eta \cL)^k
\end{align*}
Here (i) follows by iterating the recursion formula. 

Taking $\opnorm{\cdot}$ on both sides of the above equation, by triangle inequality and the ideal property of~$\opnorm{\cdot}$, we have
\begin{align*}
    \opnorm{\Gamma_n - \Gamma}
    &\leq \Onorm{\cI - \eta \cL}^{2n} \opnorm{\Gamma} + \eta^2 \sum_{k=0}^{n-1} \Onorm{\cI - \eta \cL}^{2k} \Onorm{\cL}^2 \opnorm{\Gamma}\\
    &\leq (1-\eta \Delta_{\min})^{2n} \opnorm{\Gamma} + \eta^2 \sum_{k=0}^{n-1} (1-\eta \Delta_{\min})^{2k} \Delta_{\max}^2 \opnorm{\Gamma}\\
    &\overset{(i)}{\leq} (1-\eta \Delta_{\min})^{2n} \opnorm{\Gamma} + \eta \frac{\Delta_{\max}^2}{\Delta_{\min}} \opnorm{\Gamma}\\
    &= (1-\teta)^{2n} \opnorm{\Gamma} + \teta \frac{\Delta_{\max}^2}{\Delta_{\min}^2} \opnorm{\Gamma}
\end{align*}
Here (i) follows from the fact that $\sum_{k=0}^{n-1} (1-\eta \Delta_{\min})^{2k} \leq \sum_{k=0}^{\infty} (1-\eta \Delta_{\min})^{k} \leq 1/(\eta \Delta_{\min})$.

For the trace bound, we apply the above bound to the nuclear norm $\nuclearnorm{\cdot}$,  which is unitarily invariant. Since $\Gamma_n$ and $\Gamma$ are both PSD, we have
\begin{align*}
    \abs{\tr (\Gamma_n - \Gamma)} &\leq \nuclearnorm{\Gamma_n - \Gamma} \leq \sbr{(1-\teta)^{2n} + \teta \frac{\Delta_{\max}^2}{\Delta_{\min}^2}} \nuclearnorm{\Gamma} = \sbr{(1-\teta)^{2n} + \teta \frac{\Delta_{\max}^2}{\Delta_{\min}^2}} \tr \Gamma
\end{align*}

\textbf{(v) Moment bound for $\xi_i$.} By definition of $\xi_i$, for any $q \geq 1$, we have
\begin{align*}
    \Fnorm{\xi_i}^q &\leq \eta^{q/2} \Onorm{\cI - \eta \cL}^{(n-i)q} \Fnorm{ZY^\top}^q \leq \eta^{q/2} (1-\eta \Delta_{\min})^{(n-i)q} \Fnorm{ZY^\top}^q
\end{align*}
Summing over $i=1,2,\ldots,n$ and taking expectation on both sides gives
\begin{align*}
    \sum_{i=1}^n \EE\Fnorm{\xi_i}^q &\leq \eta^{q/2} \sum_{i=1}^n (1-\eta \Delta_{\min})^{(n-i)q} \EE\Fnorm{ZY^\top}^q \\
    &\leq \frac{\eta^{q/2-1}}{\Delta_{\min}} \EE\Fnorm{ZY^\top}^q = \frac{\teta^{q/2-1}}{\Delta_{\min}^{q/2}} \EE\Fnorm{ZY^\top}^q
\end{align*}

\textbf{(vi) Matrix moment bound.} Applying Lemma~\ref{lem:matrix_kth_moment} to the random matrix $ZY^\top$ gives the desired bound.

\textbf{(vii) $\Onorm[\big]{\tfC}$ upper bound} follows from Lemma~\ref{lem:matrix_kth_moment}.

% --------------------------------------------------------------------------------

\subsubsection{Proof of Lemma~\ref{lem:Gamma_u_properties}}
\label{sec:prf_Gamma_u_properties}

\paragraph{(i) (Coordinate representation.)} By the definition of $\Gamma_u$, for any $j \in [r]$, we have
\begin{align*}
    \Finner{e_{j'}}{\Gamma_{u} e_j} = \Finner{T_{u}^* e_{j'}}{\Gamma T_{u}^* e_j}= \Finner{u e_{j'}^\top}{\Gamma (u e_j^\top) } &= \sum_{l,l'=1}^{d-r} u_lu_{l'} \Finner{E_{l'j'}}{\Gamma E_{lj}}\\
    &\overset{(i)}{=} \sum_{l,l'=1}^{d-r} \frac{u_l u_{l'}}{\Delta_{lj} + \Delta_{l'j'}} \Finner{E_{l'j'}}{\fC E_{lj}}
\end{align*}
Here (i) follows from the coordinate representation of $\Gamma$ in Lemma~\ref{lem:Gamma_properties}(i). The coordinate representation for $\tGamma_{\tu}$ follows similarly.

\paragraph{(ii) (Whitened--unwhitened connection)} For any $j,j' \in [r]$, we have
\begin{align*}
    \inner{e_{j'}}{T_u e^{-t\cL} \fC e^{-t\cL} T_u^* e_j}_2 &= \Finner{ T_u^* e_{j'}}{e^{-t\cL}\fC e^{-t\cL} T_u^* e_j}\\
    &\overset{(i)}{=} \EE \Finner{ u e_{j'}^\top}{e^{-t\cL} ZY^\top}\Finner{e^{-t\cL}ZY^\top}{u e_j^\top}\\
    &\overset{(ii)}{=} \EE \Finner{  u e_{j'}^\top }{\Lambda_{*,\perp}^{1/2} e^{-t\cL} \tZ\tY^\top \Lambda_*^{1/2} }\Finner{\Lambda_{*,\perp}^{1/2} e^{-t\cL} \tZ\tY^\top \Lambda_*^{1/2} }{ u e_j^\top}\\
    &= \EE \Finner{ \Lambda_{*,\perp}^{1/2} u e_{j'}^\top \Lambda_*^{1/2}}{e^{-t\cL} \tZ\tY^\top}\Finner{e^{-t\cL} \tZ\tY^\top }{\Lambda_{*,\perp}^{1/2} u e_j^\top \Lambda_*^{1/2}}\\
    &= \EE \Finner{ \tu \sbr{\Lambda_*^{1/2} e_{j'}}^\top }{e^{-t\cL} \tZ\tY^\top }\Finner{e^{-t\cL} \tZ\tY^\top}{\tu \sbr{\Lambda_*^{1/2} e_j}^\top}\\
    &= \Finner{T_{\tu}^* \sbr{\Lambda_*^{1/2} e_{j'}}}{e^{-t\cL} \tfC e^{-t\cL} T_{\tu}^* \sbr{\Lambda_*^{1/2} e_{j}}}\\
    &= \inner{e_{j'}}{\Lambda_*^{1/2} T_{\tu} e^{-t\cL} \tfC e^{-t\cL} T_{\tu}^* \Lambda_*^{1/2} e_j}_2
\end{align*}
Here (i) follows from the definition of $\fC$ in~\eqref{eqn:def_fC}, $T_u^* e_j = u e_j^\top$ and $e^{-t\cL}$ is PSD; (ii) follows due to the commutation $\Lambda_{*,\perp}^{1/2} e^{-t\cL} (\tZ Y^\top)= e^{-t\cL} (\Lambda_{*,\perp}^{1/2}\tZ Y^\top)$.

Integrating over $t$ from $0$ to $\infty$ gives
\begin{align*}
    \underbrace{T_u \Gamma T_u^*}_{=\Gamma_u} = \Lambda_*^{1/2} \underbrace{T_{\tu} \tGamma T_{\tu}^*}_{=\tGamma_{\tu}} \Lambda_*^{1/2}
\end{align*}

The proof for the connection between $\Gamma_{u,n}$ and $\tGamma_{\tu,n}$ follows similarly.

\paragraph{(iii) (Loewner order lower bound)} For any unit vector $v \in \RR^r$, we have
\begin{align*}
    \inner{v}{\tGamma_{\tu} v}_2 &= \Finner{\tu v^\top}{\tGamma \tu v^\top} \overset{(i)}{\geq} \frac{\lambda_{\min}(\tfC) \Finner{\tu v^\top}{\tu v^\top}}{2\Delta_{\max}} = \frac{\lambda_{\min}(\tfC) \norm{\tu}^2}{2\Delta_{\max}}
\end{align*}
Here (i) follows from the Loewner order lower bound for $\Gamma$ in Lemma~\ref{lem:Gamma_properties}(iii) and Remark~\ref{rmk:whitened_Gamma}. Thus we get the lower bound for $\tGamma_{\tu}$. The lower bound for $\Gamma_u$ follows from the connection between $\Gamma_u$ and $\tGamma_{\tu}$ in part (ii) of this lemma. The lower bound for $\tGamma_{\tu,n}$ and $\Gamma_{u,n}$ follows similarly by applying the same argument to $\Gamma_n$ and using the lower bound for $\Gamma_n$ in Lemma~\ref{lem:Gamma_properties}(iii).

\paragraph{(iv) (Approximation)} By the definition of $\tGamma_{\tu}$ and $\tGamma_{\tu,n}$, we have
\begin{align*}
    \opnorm[\big]{\tGamma_{\tu,n} - \tGamma_{\tu}} &= \opnorm[\big]{T_{\tu} (\tGamma_n - \tGamma) T_{\tu}^*} \overset{(i)}{\leq} \opnorm[\big]{\tGamma_n - \tGamma}  \|\tu\|^2 \leq \varepsilon_n \opnorm[\big]{\tGamma} \|\tu\|^2 \leq \varepsilon_n \Onorm[\big]{\tGamma} \opnorm{\cI} \|\tu\|^2 
\end{align*}
Here (i) follows from the ideal property of unitarily invariant norms and the fact that $\Onorm{T_{\tu}} = \Onorm{T_{\tu}^*} = \|\tu\|$. Hence
\begin{align*}
    \opnorm[\big]{\tGamma_{\tu,n} - \tGamma_{\tu}} &\leq \varepsilon_n \|\tGamma\|_{\mathrm{op}} \opnorm{\cI} \|\tu\|^2  
    = \varepsilon_n \|\tGamma\|_{\mathrm{op}} \lambda_{\min}(\tGamma_{\tu})\opnorm{\cI} \|\tu\|^2 \cdot \frac{1}{\lambda_{\min}(\tGamma_{\tu})} \\
    &\overset{(i)}{\leq}  \varepsilon_n \|\tGamma\|_{\mathrm{op}} \opnorm[\big]{\tGamma_{\tu}} \|\tu\|^2 \cdot \frac{2\Delta_{\max}}{\lambda_{\min}(\tfC) \|\tu\|^2}
    \overset{(ii)}{\leq} \varepsilon_n \frac{\|\tfC\|_{\mathrm{op}}}{2\Delta_{\min}} \opnorm[\big]{\tGamma_{\tu}} \|\tu\|^2 \cdot \frac{2\Delta_{\max}}{\lambda_{\min}(\tfC) \|\tu\|^2}\\
    &= \varepsilon_n \frac{\Delta_{\max}}{\Delta_{\min}} \frac{\|\tfC\|_{\mathrm{op}}}{\lambda_{\min}(\tfC)} \opnorm[\big]{\tGamma_{\tu}}
\end{align*}
Here (i) follows from $\lambda_{\min}(\tGamma_{\tu}) \opnorm{\cI} \leq \opnorm[\big]{\tGamma_{\tu}}$ and the Loewner order lower bound for $\tGamma_{\tu}$ in part (iii) of this lemma; (ii) follows from the upper bound for $\|\tGamma\|_{\mathrm{op}}$ from Lemma~\ref{lem:Gamma_properties}(ii) and Remark~\ref{rmk:whitened_Gamma}.

\paragraph{(v) (Independence case)} When $\tY \perp \tZ$, we have
\begin{align*}
    \fC = \Lambda_* \otimes \Lambda_{*,\perp}, \qquad \tfC = I_r \otimes I_{d-r}.
\end{align*}
or equivalently, 
\begin{align*}
    \Finner{E_{l'j'}}{\fC E_{lj}} = \lambda_j \lambda_{r+l} \delta_{ll'} \delta_{jj'}, \qquad \Finner[\big]{E_{l'j'}}{\tfC E_{lj}} = \delta_{ll'} \delta_{jj'}
\end{align*}
which, combined with (i) coordinate representation, gives the desired results.

% --------------------------------------------------------------------------------

\subsubsection{Proof of Lemma~\ref{lem:Latala}}
\label{sec:prf_Latala}

See \citep[Corollary~3]{Latala}. Without loss of generality, assume $G$ is diagonal with $G_{11} \geq \cdots \geq G_{dd} \geq 0$. Then
\begin{align*}
    \EE \rbr{\fg^\top G \fg}^{p} &= \EE \rbr[\Big]{\sum_{i=1}^d G_{ii} \fg_i^2}^{p} \lesssim_p \sbr{\rbr[\Big]{\sum_{i=1}^d G_{ii}}^p + \sum_{i=1}^d G_{ii}^p \EE \fg_i^{2p}} \lesssim_p \sbr{\tr G}^p
\end{align*}
Here the last inequality follows since $\sum_i G_{ii}^p \leq (\sum_i G_{ii})^p$ for $p \geq 1$ and $G \succeq 0$.

% -------------------------------------------------------------------------

\subsubsection{Proof of Lemma~\ref{lem:g_approx_additive_remainder}}
\label{sec:prf_g_approx_additive_remainder}

Let 
\begin{align*}
    F_{\mathsf{A}_n}(t) &:= \PP\cbr[\big]{\norm{\mathsf{A}_n}^2 \leq t}, \qquad F_{\mathsf{B}_n}(t):= \PP\cbr[\big]{\norm{\mathsf{B}_n}^2 \leq t}, \qquad F_{\mathsf{G}_n}(t):= \PP\cbr[\big]{\norm{\mathsf{G}_n}^2 \leq t}
\end{align*}
Since $\norm{\mathsf{A}_n}, \norm{\mathsf{B}_n}, \norm{\mathsf{G}_n} \geq 0$, it suffices to consider $t \geq 0$.

Define good event
\begin{align*}
    \cG_n := \cbr{\norm{\mathsf{E}_n} \leq e_n, \norm{\mathsf{A}_n} \leq L_n}.
\end{align*}
On $\cG_n$, we have
\begin{align*}
    \abs[\big]{\norm{\mathsf{B}_n}^2 - \norm{\mathsf{A}_n}^2} &= \abs[\big]{2\inner{\mathsf{A}_n}{\mathsf{E}_n} + \norm[\big]{\mathsf{E}_n}^2}\leq 2L_n e_n + e_n^2 =: h_n
\end{align*}
which implies
\begin{align*}
    \cbr[\big]{\norm{\mathsf{A}_n}^2 \leq t - h_n} \cap \cG_n  \subseteq \cbr[\big]{\norm{\mathsf{B}_n}^2 \leq t} &\subseteq \cbr[\big]{\norm{\mathsf{A}_n}^2 \leq t + h_n} \cup \cG_n^c.
\end{align*}
Thus we have
\begin{align*}
    F_{\mathsf{A}_n}(t - h_n) - \PP(\cG_n^c) \leq F_{\mathsf{B}_n}(t) \leq F_{\mathsf{A}_n}(t + h_n) + \PP(\cG_n^c)
\end{align*}
Subtracting $F_{\mathsf{G}_n}(t)$ and using $\abs{F_{\mathsf{G}_n}(t) - F_{\mathsf{A}_n}(t)} \leq \delta_n$ gives
\begin{align*}
    F_{\mathsf{G}_n}(t-h_n)-F_{\mathsf{G}_n}(t) - \PP(\cG_n^c) - \delta_n &\leq F_{\mathsf{B}_n}(t) - F_{\mathsf{G}_n}(t) \\
    &\leq F_{\mathsf{G}_n}(t + h_n) - F_{\mathsf{G}_n}(t) + \PP(\cG_n^c) + \delta_n
\end{align*}
which implies
\begin{align}
    \abs{F_{\mathsf{B}_n}(t) - F_{\mathsf{G}_n}(t)}
    \leq w_n(h_n) + \PP(\cG_n^c) + \delta_n \label{eqn:approx_additive_remainder}
\end{align}
Now by the definition of good event $\cG_n$ and $\delta_n$, one can obtain
\begin{align*}
    \PP(\cG_n^c) \leq \PP \cbr{\norm{\mathsf{E}_n} > e_n} + \PP \cbr{\norm{\mathsf{A}_n} > L_n} \leq \PP \cbr{\norm{\mathsf{E}_n} > e_n} + \PP \cbr{\norm{\mathsf{G}_n} > L_n} + \delta_n
\end{align*}
Plugging the above bound for $\PP(\cG_n^c)$ into~\eqref{eqn:approx_additive_remainder} gives
\begin{align*}
    \abs{F_{\mathsf{B}_n}(t) - F_{\mathsf{G}_n}(t)} &\leq w_n(h_n) + 2\delta_n + \PP \cbr{\norm{\mathsf{E}_n} > e_n} + \PP \cbr{\norm{\mathsf{G}_n} > L_n}
\end{align*}
Taking supremum over $t \in \RR$ gives the desired result.

% -------------------------------------------------------------------------

\subsubsection{Proof of Lemma~\ref{lem:g_approx_multiplicative_remainder}}
\label{sec:prf_g_approx_multiplicative_remainder}

Let
\begin{align*}
    F_{A_n}(t):= \PP\cbr[\big]{\norm{A_n}^2 \leq t}, \qquad F_{B_n}(t):= \PP\cbr[\big]{\norm{B_n}^2 \leq t}, \qquad F_{G_n}(t):= \PP\cbr[\big]{\norm{G_n}^2 \leq t}.
\end{align*}
Since $\norm{A_n}, \norm{B_n}, \norm{G_n} \geq 0$, it suffices to consider $t \geq 0$.  

Define events
\begin{align*}
    \cE_n:= \cbr{\norm{\mathsf{E}_n} \leq e_n}, \qquad \cG_n := \cbr{\norm{\mathsf{G}_n} \leq L_n}.
\end{align*}
On $\cE_n$, we have
\begin{align*}
    (1-e_n)\norm{x} \leq \norm{(I_{d_n} + \mathsf{E}_n) x} \leq (1+e_n) \norm{x}, \quad \forall x \in \RR^{d_n}
\end{align*}
Since $0\leq e_n \leq \frac{1}{2}\leq 1$, by the definition of $B_n$, we then have
\begin{align*}
    (1-e_n)^2 \norm{\mathsf{A}_n}^2 \leq \norm{\mathsf{B}_n}^2 \leq (1+e_n)^2 \norm{\mathsf{A}_n}^2
\end{align*}
So for any $t \geq 0$, we have
\begin{align*}
    \cbr[\big]{\norm{\mathsf{B}_n}^2 \leq t} \cap \cE_n &\subseteq \cbr[\big]{\norm{\mathsf{A}_n}^2 \leq \frac{t}{(1-e_n)^2}}
\end{align*}
and
\begin{align*}
    \cbr[\big]{\norm{\mathsf{A}_n}^2 \leq \frac{t}{(1+e_n)^2}} \cap \cE_n \subseteq \cbr[\big]{\norm{\mathsf{B}_n}^2 \leq t} .
\end{align*}
Taking probabilities yields
\begin{align*}
    F_{\mathsf{A}_n}\rbr{\frac{t}{(1+e_n)^2}} - \PP(\cE_n^c) \leq F_{\mathsf{B}_n}(t) \leq F_{\mathsf{A}_n}\rbr{\frac{t}{(1-e_n)^2}} + \PP(\cE_n^c)
\end{align*}
which together with the definition of $\delta_n$ gives
\begin{align*}
    F_{\mathsf{G}_n}\rbr{\frac{t}{(1+e_n)^2}} - \PP(\cE_n^c) - \delta_n \leq F_{\mathsf{B}_n}(t) \leq F_{\mathsf{G}_n}\rbr{\frac{t}{(1-e_n)^2}} + \PP(\cE_n^c) + \delta_n
\end{align*}
Subtracting $F_{\mathsf{G}_n}(t)$ from the above inequalities gives
\begin{align*}
    F_{\mathsf{B}_n}(t) - F_{\mathsf{G}_n}(t) &\leq \underbrace{F_{\mathsf{G}_n}\rbr{\frac{t}{(1-e_n)^2}} - F_{\mathsf{G}_n}(t)}_{=:\alpha_1} + \PP(\cE_n^c) + \delta_n\\
    F_{\mathsf{G}_n}(t) - F_{\mathsf{B}_n}(t) &\leq \underbrace{F_{\mathsf{G}_n}(t) - F_{\mathsf{G}_n}\rbr{\frac{t}{(1+e_n)^2}}}_{=:\alpha_2} + \PP(\cE_n^c) + \delta_n
\end{align*}
So it remains to bound $\alpha_1$ and $\alpha_2$.

\textup{(i) ($\alpha_1$)}
    \begin{align*}
        \alpha_1 = \PP \cbr[\Big]{ t < \norm{\mathsf{G}_n}^2 \leq \frac{t}{(1-e_n)^2}} 
    \end{align*}
    On the event $\cbr[\big]{t < \norm{\mathsf{G}_n}^2 \leq t/(1-e_n)^2} \cap \cG_n$
    we have $t \leq \norm{\mathsf{G}_n}^2 \leq L_n^2$, hence
    \begin{align*}
        \frac{t}{(1-e_n)^2} - t \leq \rbr{\frac{1}{(1-e_n)^2} - 1} L_n^2 = h_n
    \end{align*}
    Therefore, we have
    \begin{align*}
        \alpha_1 = \PP \cbr[\Big]{ t < \norm{\mathsf{G}_n}^2 \leq \frac{t}{(1-e_n)^2}} \leq w_n(h_n) + \PP (\cG_n^c).
    \end{align*}

\textup{(ii) ($\alpha_2$)}  By definition,
\begin{align*}
    \alpha_2 = \PP\cbr{\frac{t}{(1+e_n)^2}<\|\mathsf{G}_n\|^2\le t}.
\end{align*}
Set $t_\wedge:=t\wedge L_n^2$. On $\cG_n$, the event above forces
$\|\mathsf{G}_n\|^2\le t_\wedge$. Hence
\begin{align*}
    \left\{
    \frac{t}{(1+e_n)^2}<\|\mathsf{G}_n\|^2\le t
    \right\}\cap \cG_n
    \subseteq
    \left\{
    \frac{t}{(1+e_n)^2}<\|\mathsf{G}_n\|^2\le t_\wedge
    \right\}.
\end{align*}
The width of the latter interval satisfies
\begin{align*}
    t_\wedge-\frac{t}{(1+e_n)^2} \le t_\wedge-\frac{t_\wedge}{(1+e_n)^2}
    = t_\wedge\left(1-\frac{1}{(1+e_n)^2}\right)
    \le
    L_n^2\left(1-\frac{1}{(1+e_n)^2}\right)
    \le h_n,
\end{align*}
where the first inequality uses $t_\wedge\le t$, and the last inequality follows from
\begin{align*}
    1-\frac{1}{(1+e_n)^2} \le \frac{1}{(1-e_n)^2}-1.
\end{align*}
Therefore,
\begin{align*}
    \alpha_2 \le w_n(h_n)+\mathbb P(\mathcal G_n^c).
\end{align*}

Combining the preceding bounds gives
\begin{align*}
    \sup_{t \in \RR} \abs{F_{\mathsf{B}_n}(t) - F_{\mathsf{G}_n}(t)} &\leq  \delta_n + \PP\cbr{\norm{\mathsf{E}_n} > e_n} + w_n(h_n) + \PP\cbr{\norm{\mathsf{G}_n} > L_n}
\end{align*}

\section{Proofs for Section~\ref{sec:main_clt}}
\label{sec:clt_analysis}

\subsection{Proof of Theorem~\ref{thm:highD_clt}}
\label{sec:prf_highD_clt}

The proof proceeds in three steps, progressing right-to-left through the chain
\begin{align*}
    \eta^{-1}\Fnorm{\sin \Theta(U_n, U_*)}^2 \longapprox{Step 3.} \eta^{-1}\Fnorm{\cT_n}^2  \longapprox{Step 2.} \Fnorm[\Big]{\sum_{i=1}^n \xi_i}^2 \longapprox{Step 1.} \Fnorm{\mathsf{G}}^2.
\end{align*}

\paragraph{Preliminaries.}
For convenience, recall notation 
\begin{align*}
    d_{\fC} = \frac{\lambda_\times^4}{\HS{\fC}^2} = \frac{\lambda_{1\sim r}^2 \lambda_{r+1 \sim d}^2}{\HS{\fC}^2}, \quad \kappa_* = \frac{\lambda_1}{\Delta_{\min}}, \quad  \kappa_\Delta = \frac{\Delta_{\max}}{\Delta_{\min}}.
\end{align*}

We restate the conditions in Theorem~\ref{thm:highD_clt} here:
\begin{align*}
    &n^{-2\varepsilon} \ge C_1 \teta \kappa_*^2 r^3 (1 + \nu_\infty d)\, (\log n)^3 \frac{\log(1/\delta_p)}{\delta_p^2} \tag{\ref{eqn:cond_sample_sz}}\\
    &\teta  \, \kappa_*^4 r^{3} d_{\fC}^{3/2}  (\log n)^{3} \to 0 \tag{\ref{eqn:highD_clt_cond}(i)}\\
    &\teta  \kappa_*^4  r^3 d_{\fC} \bnu^2 (\log n)^{3} \to 0, \tag{\ref{eqn:highD_clt_cond}(ii)}
\end{align*}
By Lemma~\ref{lem:matrix_kth_moment}, we have
\begin{align*}
    \HS{\fC}\leq \tr \fC = \EE \Fnorm{ZY^\top}^2 \lesssim \lambda_{1\sim r} \lambda_{r+1 \sim d}
\end{align*}
Thus, $d_{\fC} \gtrsim 1$. Since $\lambda_1 \geq \Delta_{\max}$, we have $\kappa_* \geq \kappa_\Delta \geq 1$. Thus \eqref{eqn:cond_sample_sz} implies
\begin{align*}
    \eta \Delta_{\max}=\teta \kappa_\Delta \leq \teta \kappa_* < 1, \quad \text{and}\quad \teta \kappa_\Delta^2 \to 0.
\end{align*}
Note that $\teta = c_\eta\frac{\log n}{n}$ with $c_\eta \geq 1$ implies that $(1-\teta)^{2n}\lesssim \teta$, thus $\epsilon_n = (1-\teta)^{2n} + \teta \kappa_\Delta^2 \to 0$, and hence $< \frac{1}{2}$ for $n$ large enough.

\paragraph{Step 1: Gaussian approximation of $\Fnorm{\sum_{i=1}^n \xi_i}^2$.}
By the triangle inequality for the Kolmogorov distance $d_K$,

\begin{equation}\label{eqn:highD_step1_triangle}
    d_K\! \rbr[\Big]{\Fnorm[\Big]{\sum_{i=1}^n \xi_i}^2,\Fnorm{\mathsf{G}}^2}
    \;\le\; 
    d_K\! \rbr[\Big]{\Fnorm[\Big]{\sum_{i=1}^n \xi_i}^2,\;\Fnorm[\Big]{\sum_{i=1}^n g_i}^2}
    + d_K\! \rbr[\Big]{\Fnorm[\Big]{\sum_{i=1}^n g_i}^2,\;\Fnorm{\mathsf{G}}^2}
\end{equation}

\textbf{First term in \eqref{eqn:highD_step1_triangle}.}

\begin{lemma} \label{lem:highD_check}
    Assume Assumption~\ref{assumption:X} holds.
    Let $L_q^{\xi}$, $L_q^{\fg}$, $K_\beta$, and $J_n$ be defined as in Lemma~\ref{lem:clt_biometrika19} for $q \in (2,3]$ and $\beta \geq 2$.
    If
    \begin{align*}
        \eta \Delta_{\max} \in (0,1), \quad \epsilon_n:=(1-\teta)^{2n} + \teta \kappa_\Delta^2 \leq \frac{1}{2}, \quad \HS{\fC} > 0,
    \end{align*}
    then
    \begin{align*}
        L_{q}^{\xi} &\lesssim \teta^{q/2-1}\kappa_\Delta^{q/2}d_C^{q/4} + \teta^{q-2}\kappa_\Delta^q d_C^{q/2},\\
        % \red{XXX} \kappa_\Delta^3 \sbr{\teta^{1/2} r^{3/2} + \teta} \lambda_\times^6 \HS{\fC}^{-3},\\
        L_q^{\fg} &\lesssim \teta^{q/2-1} \kappa_\Delta^{q/2},\\
        n^{1-\beta} K_\beta^{\beta} &\lesssim \teta^{\beta-1}\kappa_\Delta^\beta d_{\fC}^{\beta/2},\\
        J_n &\lesssim \teta \kappa_\Delta^2 d_{\fC}.
    \end{align*}
    In particular, for $q=\beta=3$, a sufficient condition for all four terms to vanish is 
    \begin{align}
        % \teta^{1/2} \kappa_\Delta^3 r^{3/2} d_{\fC}^{3/2} \to 0. 
        \teta\kappa_\Delta^3d_{\fC}^{3/2}\to0. \label{eqn:highD_step1_cond}
    \end{align}
\end{lemma}
\begin{proof}
    See Appendix~\ref{sec:prf_highD_check}.
\end{proof}

Since $\kappa_* \geq \kappa_\Delta \geq 1$, Condition~\eqref{eqn:highD_clt_cond}(i) implies that \eqref{eqn:highD_step1_cond} holds, thus the first term in \eqref{eqn:highD_step1_triangle} goes to zero.

% We now show $\Phi_{1,1}, \Phi_{1,2}, \Phi_{1,3} \to 0$.

% \textit{(i) $\Phi_{1,1}$.} Since $\kappa_* \geq \kappa_\Delta \geq 1$ and $d_{\fC}\geq 1$, $\Phi_{1,1} \to 0$ is implied by \eqref{eqn:highD_clt_cond}(i).

% \textit{(ii) $\Phi_{1,2}$.} Since $\lambda_1^3 \nu_\infty^3 = \lambda_1^{3/2} \lambda_{r+1}^{3/2} \leq \lambda_\times^3$, we have
% \begin{align*}
%     \Phi_{1,2} \leq \sqrt{\teta} \kappa_*^3 r^{3/2} \frac{\lambda_\times^6}{\HS{\fC}^3} = \sqrt{\teta} \kappa_*^3 r^{3/2} d_{\fC}^{3/2}
% \end{align*}
% Thus $\Phi_{1,2} \to 0$ is implied by \eqref{eqn:highD_clt_cond}(i).

% \textit{(iii) $\Phi_{1,3}$.} Since $\teta \leq 1$ by \eqref{eqn:cond_sample_sz}, $\Phi_{1,3} \to 0$ is also implied by \eqref{eqn:highD_clt_cond}(i).

\textbf{Second term in \eqref{eqn:highD_step1_triangle}.}

\begin{lemma} \label{lem:highD_chi2_mixture_comparison}
    Assume Assumption~\ref{assumption:X} holds and $\HS{\fC}>0$.
    If $\eta \Delta_{\max} \in (0,1)$,
    then
    \begin{align*}
        d_K\! \rbr[\Big]{\Fnorm[\Big]{\sum_{i=1}^n g_i}^2,\;\Fnorm{\mathsf{G}}^2} \lesssim
        \sbr{\epsilon_n (\log n) \kappa_\Delta \frac{C_\psi^4 \lambda_{1\sim r}^2 \nu_2^2}{\|\fC\|_{\mathrm{HS}}}}^{1/2} +n^{-20},
    \end{align*}
    where $\epsilon_n := (1-\teta)^{2n} + \teta \kappa_\Delta^2$.

\end{lemma}
\begin{proof}
    See Appendix~\ref{sec:prf_chi2_mixture_comparison}.
\end{proof}

Note that for $\teta = n^{-1} c_\eta \log n$ with $c_\eta \geq \frac{1}{2}$, we have $(1-\teta)^{2n} \leq e^{-\log n} = n^{-1} \lesssim \teta$. Thus, in order for $d_K\! \rbr[\Big]{\Fnorm[\Big]{\sum_{i=1}^n g_i}^2,\;\Fnorm{\mathsf{G}}^2} \to 0$, it suffices that
\begin{align*}
    \teta \kappa_\Delta^3 (\log n)  \frac{\lambda_\times^2}{\HS{\fC}} \to 0.
    % \label{eqn:highD_clt_term4}
\end{align*}
which is implied by \eqref{eqn:highD_clt_cond}(i), $\teta\leq 1$ and $d_{\fC} \gtrsim 1$.

\paragraph{Step 2: Gaussian approximation of $\eta^{-1}\Fnorm{\cT_n}^2$.}

\begin{lemma}
    Suppose the assumptions of Theorem~\ref{thm:highD_clt} hold. Then
    \begin{align*}
        d_K\! \rbr[\bigg]{\Fnorm[\Big]{\frac{1}{\sqrt{\eta}}\cT_n}^2,\;\Fnorm{\mathsf{G}}^2} \lesssim d_K \! \rbr[\Big]{\Fnorm[\Big]{\sum_{i=1}^n \xi_i}^2,\;\Fnorm{\mathsf{G}}^2} + \sqrt{\frac{e_n (e_n+ L_n)}{\HS{\Gamma}}} + \delta_p,
    \end{align*}
    where $e_n = C_e \Delta_{\min}^{1/2} \kappa_*^2 r^{3/2} \sqrt{\teta}\;  (\log n)  \bnu\sbr{1 + \bnu}$, $L_n = \sbr{C_L \tr (\Gamma) \log n}^{1/2}$ for some constants $C_e, C_L > 0$ large enough.
    \label{lem:highD_step2}
\end{lemma}
\begin{proof}
    See Appendix~\ref{sec:prf_highD_step2}.
\end{proof}

We already have $d_K \! \rbr[\Big]{\Fnorm[\Big]{\sum_{i=1}^n \xi_i}^2,\;\Fnorm{\mathsf{G}}^2} \to 0$ from Step 1. It suffices to show $\frac{e_n(e_n+L_n)}{\HS{\Gamma}} \to 0$, or equivalently, 
\begin{align*}
    \frac{e_n^2}{\HS{\Gamma}} &\asymp \kappa_*^4 r^3 {\teta}\;  (\log n)^2  \bnu^2\sbr{1 + \bnu}^2 \frac{\Delta_{\min}}{\HS{\Gamma}}  \to 0, \\
    \frac{e_n L_n}{\HS{\Gamma}} &\asymp \kappa_*^2 r^{3/2} \sqrt{\teta}\;  (\log n)^{3/2}  \bnu\sbr{1 + \bnu} \frac{\Delta_{\min}^{1/2} \sqrt{\tr \Gamma}}{\HS{\Gamma}} \to 0.
\end{align*}
By Lemma~\ref{lem:Gamma_properties}(ii) and (vi), we have $\HS{\Gamma} \gtrsim \frac{\HS{\fC}}{\Delta_{\max}}$ and 
\begin{align*}
    \tr \Gamma &\lesssim \frac{\tr \fC}{\Delta_{\min}} = \frac{\EE \Fnorm{ZY^\top}^2}{\Delta_{\min}} \lesssim \frac{\lambda_{1\sim r} \lambda_{r+1 \sim d}}{\Delta_{\min}}
\end{align*}
Thus, it suffices that
\begin{align*}
    \Phi_{2,1}&:=\kappa_*^4 r^3 {\teta}\;  (\log n)^2  \bnu^2\sbr{1 + \bnu}^2 \frac{\Delta_{\min} \Delta_{\max}}{\HS{\fC}}  \to 0\\
    \Phi_{2,2}&:=\kappa_*^2 r^{3/2} \sqrt{\teta}\;  (\log n)^{3/2}  \bnu\sbr{1 + \bnu} \frac{\Delta_{\max} \lambda_{1\sim r} \nu_2 }{\HS{\fC}} \to 0
\end{align*}

We now show $\Phi_{2,1}, \Phi_{2,2} \to 0$.

\textit{(i) $\Phi_{2,1}$.} First, note that $\lambda_1^2\bnu^2  \leq \max(\lambda_1^2 \nu_\infty^2, \lambda_1^2 \nu_2^2) \leq \lambda_\times^2$, thus using $\Delta_{\min} \leq \Delta_{\max}\leq \lambda_1$ and $(1+\bnu)^2 \lesssim 1+\bnu^2$ gives
\begin{align*}
    \Phi_{2,1} &\leq \kappa_*^4 r^2 \teta (\log n)^2 \sbr{1+\bnu}^2 \frac{\lambda_\times^2}{\HS{\fC}}\\
    &\lesssim \kappa_*^4 r^2 \teta (\log n)^2 d_{\fC}^{1/2} + \kappa_*^4 r^2 \teta (\log n)^2 \bnu^2 d_{\fC}^{1/2}.
\end{align*}
where \eqref{eqn:highD_clt_cond}(i) and \eqref{eqn:highD_clt_cond}(ii) ensure each term $\to 0$ on the last line.

\textit{(ii) $\Phi_{2,2}$.} Using $\Delta_{\max} \bnu \leq \lambda_\times$, we have
\begin{align*}
    \Phi_{2,2} &\leq \kappa_*^2 r^{3/2} \sqrt{\teta}\;  (\log n)^{3/2} \sbr{1 + \bnu} d_{\fC}^{1/2}\\
    &= \kappa_*^2 r^{3/2} \sqrt{\teta}\;  (\log n)^{3/2} d_\fC^{1/2} + \kappa_*^2 r^{3/2} \sqrt{\teta}\;  (\log n)^{3/2} \bnu d_{\fC}^{1/2}
\end{align*}
where \eqref{eqn:highD_clt_cond}(i) and \eqref{eqn:highD_clt_cond}(ii) ensure each term $\to 0$ on the last line.

\paragraph{Step 3: Gaussian approximation of $\eta^{-1}\Fnorm{\sin \Theta(U_n, U_*)}^2$.}

\begin{lemma}
    Suppose the assumptions of Theorem~\ref{thm:highD_clt} hold. Then
    \begin{align*}
        d_K\! \rbr[\bigg]{\Fnorm[\Big]{\frac{1}{\sqrt{\eta}}\sin \Theta(U_n, U_*)}^2,\;\Fnorm{\mathsf{G}}^2} \lesssim d_K \! \rbr[\Big]{\Fnorm[\Big]{\frac{1}{\sqrt{\eta}} \cT_n}^2 ,\;\Fnorm{\mathsf{G}}^2} + \tau_*  \sqrt{\frac{L_n^2}{\HS{\Gamma}}} + \delta_p,
    \end{align*}
    where $L_n = \sqrt{C_L\sbr{\tr \Gamma} \log n}$ for some constant $C_L > 0$ large enough.
    \label{lem:highD_step3}
\end{lemma}

\begin{proof}
    See Appendix~\ref{sec:prf_highD_step3}.
\end{proof}

We have $d_K \! \rbr[\big]{\|\frac{1}{\sqrt{\eta}} \cT_n\|_{\rm F}^2 ,\;\Fnorm{\mathsf{G}}^2} \to 0$ from step 2. It suffices to show $\tau_*^2 \tr(\Gamma) (\log n)/ \HS{\Gamma} \to 0$.

Recall $\tau_* \asymp  \kappa_*  \bnu \sqrt{\teta r \log n}$. By Lemma~\ref{lem:Gamma_properties}(ii) and (vi), we have $\HS{\Gamma} \gtrsim \frac{\HS{\fC}}{\Delta_{\max}}$ and 
\begin{align*}
    \tr \Gamma &\lesssim \frac{\tr \fC}{\Delta_{\min}} = \frac{\EE \Fnorm{ZY^\top}^2}{\Delta_{\min}} \lesssim \frac{\lambda_{1\sim r} \lambda_{r+1 \sim d}}{\Delta_{\min}}
\end{align*}
Thus, it suffices that
\begin{align*}
    \kappa_*^2 \bnu^2 \teta r (\log n)^2 \cdot \kappa_\Delta \frac{\lambda_\times^2}{\HS{\fC}} \to 0.
\end{align*}
which is implied by \eqref{eqn:highD_clt_cond}(ii). 

% ------------------------------------------------------

\subsection{Proof of Theorem~\ref{thm:row_clt}}
\label{sec:prf_row_clt}

For notation simplicity, we drop the subscript $n$ and write $U_n$ as $U$, $\cC_n$ as $\cC$, and $\cT_n$ as $\cT$. Recall the definition of $\xi_i$ from \eqref{eqn:xi_defn}
\begin{align*}
    \xi_i := \sqrt{\eta} (\cI - \eta\cL)^{n-i} (Z_i Y_i^\top) \in \RR^{(d-r) \times r}.
\end{align*}

Denote
\begin{align*}
    E = \cT - \sqrt{\eta} \sum_{i=1}^n \xi_i \in \RR^{(d-r) \times r}.
\end{align*}
Then \cref{prop:linearization} implies that
\begin{align*}
    \norm{E} &\lesssim \frac{\tau_*}{n^\varepsilon \log n}.
\end{align*}

By Lemma~\ref{lem:alignment_diff}, we can obtain the decomposition
\begin{align*}
    U \sgn(U^\top U_*) - U_* 
    &= \cW + \Psi
\end{align*}
where the main term $\cW$ and remainder $\Psi$ are defined as
\begin{align*}
    \cW &:= U_{*,\perp} \sqrt{\eta} \sum_{i=1}^n \xi_i \in \RR^{d\times r},\\
    \Psi &:= U_*\sbr[\big]{(\cC \cC^\top)^{1/2} - I_r} + U_{*,\perp}\cT \sbr[\big]{(\cC \cC^\top)^{1/2} - I_r} + U_{*,\perp}E \in \RR^{d\times r}.
\end{align*}
and
\begin{align}
    \norm{(\cC \cC^\top)^{1/2} - I_r}\lesssim \norm{\cT}^2 \lesssim \tau_*^2 \lesssim \kappa_*^2 r \teta (\log n) \bnu^2. \label{eqn:row_clt_C_I}
\end{align}

Fix $m \in [d]$ and define
\begin{align*}
    u_m := U_{*,\perp}^\top e_m \in \RR^{d-r}, \qquad \tu_m:= \Lambda_{*,\perp}^{1/2} u_m, \qquad \Sigma^*_{U,m} := \eta \Gamma_{u_m}.
\end{align*}
where $\Gamma_{u_m}$ and $\tGamma_{\tu_m}$ are defined as in \eqref{eqn:Gamma_u_defn} and \eqref{eqn:whitened_defn}, respectively.

Let $e_m$ denote the $m$-th standard basis vector in $\RR^d$. Then the $m$-th row of $U \sgn(U^\top U_*) - U_*$, identified as a column vector via
\begin{align*}
    \sbr{U \sgn(U^\top U_*) - U_*}_{m,\cdot} \leftrightarrow \sbr{U \sgn(U^\top U_*) - U_*}^\top e_m \in \RR^r,
\end{align*}
satisfies
\begin{equation}
    \label{eqn:row_clt_decomposition}
    \begin{aligned}
        &\sbr{U \sgn(U^\top U_*) - U_*}^\top e_m = \cW^\top e_m + \Psi^\top e_m\\
        \text{where} \qquad &\cW^\top e_m = \sqrt{\eta} \sum_{i=1}^n \xi_i^\top U_{*,\perp}^\top e_m = \sqrt{\eta} \sum_{i=1}^n \xi_i^\top u_m =:\cW_{u_m},\\
        &\Psi^\top e_m = \sbr[\big]{(\cC \cC^\top)^{1/2} - I_r} U_*^\top e_m + \sbr[\big]{(\cC \cC^\top)^{1/2} - I_r} \cT^\top u_m +  E^\top u_m =: \Psi_{u_m}.
    \end{aligned}
\end{equation}

\paragraph{Step 1: Gaussian approximation of $\cW_{u_m}$.}

\begin{lemma} \label{lem:W_berry_esseen}
    Suppose $\tu_{u_m} = \Lambda_{*,\perp}^{1/2} u_{u_m} \neq 0$, $\lambda_{\min}(\tfC)>0$, $n \teta \geq \frac{\log n}{2}$. Assume
    \begin{equation}
        \begin{aligned}
            &\teta^{1/2} \kappa_*^3 r^{7/4} \lambda_{\min}^{-3/2}(\tfC) \to 0
            % &\teta \kappa_*^5 r^{5/2} \lambda_{\min}^{-2}(\tfC) \to 0
        \end{aligned}
        \label{eqn:W_berry_esseen_cond}
    \end{equation}

    Then, with $\Sigma^*_{U,m}=\eta \Gamma_{u_m}$, we have
    \begin{align*}
        \sup_{\cA \in \mathscr{A}^r} \abs{\PP \cbr{\cW_{u_m} \in \cA} - \cN(0, \Sigma^*_{U,m})\cbr{\cA} } = o(1).
    \end{align*}
\end{lemma}
\begin{proof}
    See Appendix~\ref{sec:prf_W_berry_esseen}.
\end{proof}

\paragraph{Step 2: Accounting for higher-order term $\Psi_{u_m}$.}

\begin{lemma}
    Assume conditions of Theorem~\ref{thm:row_clt} hold. Then
    \begin{align}
        \PP \cbr{\norm{\Sigma_{U,m}^{*-1/2} \Psi_{u_m}} \leq \zeta} \geq 1-O(\delta_p) \label{eqn:row_clt_P_zeta}
    \end{align}
    where $\zeta$ is defined as
    \begin{align}
        \begin{aligned}
            \zeta &= C_{\zeta} \rbr{\zeta_1 + \zeta_2},\\
            \zeta_1 &= \kappa_*^2 \teta^{1/2} r (\log n) \bnu^2 \frac{\lambda_1^{1/2}\norm{U_*^\top e_m}}{\norm{\tu_{m}}} \lambda_{\min}^{-1/2}(\tfC),\\
            \zeta_2 &= \kappa_*^2 \teta^{1/2} r (\log n) \bnu \;(1+\bnu) \frac{\lambda_1^{1/2}\norm{U_{*,\perp}^\top e_m}}{\norm{\tu_{m}}} \lambda_{\min}^{-1/2}(\tfC).
        \end{aligned}
        \label{eqn:row_clt_zeta_defn}
    \end{align}
    for some sufficiently large constant $C_{\zeta}>0$. Moreover, $r^{1/4} \zeta \to 0$.
    \label{lem:row_clt_higher_order}
\end{lemma}
\begin{proof}
    See Appendix~\ref{sec:prf_row_clt_higher_order}.
\end{proof}

Let $\zeta$ be defined as in \eqref{eqn:row_clt_zeta_defn}. Recall the definition of $\cA^t$ in \eqref{eqn:A_t_defn} from Appendix~\ref{sec:tech_lemmas} for any convex set $\cA \in \mathscr{A}^r$ and any $t$. We have
\begin{align}
    \PP \cbr{\Sigma_{U,m}^{*-1/2} \cW_{u_m} \in \cA^{-\zeta}} &= \PP \cbr{\Sigma_{U,m}^{*-1/2} \cW_{u_m} \in \cA^{-\zeta}, \norm{\Sigma_{U,m}^{*-1/2} \Psi_{u_m}} \leq \zeta } \nonumber\\
    &\; + \PP \cbr{\Sigma_{U,m}^{*-1/2} \cW_u \in \cA^{-\zeta}, \norm{\Sigma_{U,m}^{*-1/2} \Psi_u} > \zeta } \nonumber\\
    &\leq \PP \cbr{\Sigma_{U,m}^{*-1/2} (\cW_u+ \Psi_u) \in \cA} + \PP \cbr{\norm{\Sigma_{U,m}^{*-1/2} \Psi_u} > \zeta } \nonumber\\
    &\leq \PP\cbr{\Sigma_{U,m}^{*-1/2} \sbr{U \sgn(U^\top U_*) - U_*}^\top e_m \in \cA} + O(\delta_p) \label{eqn:row_clt_high_order1}
\end{align}
where the last line follows from the decomposition in \eqref{eqn:row_clt_decomposition} and the bound in \eqref{eqn:row_clt_P_zeta}.

Similarly,
\begin{align}
    \PP\cbr{\Sigma_{U,m}^{*-1/2} \sbr{U \sgn(U^\top U_*) - U_*}^\top e_m \in \cA} \leq \PP \cbr{\Sigma_{U,m}^{*-1/2} \cW_u \in \cA^{\zeta}} + O(\delta_p). \label{eqn:row_clt_high_order2}
\end{align}
Moreover, we can obtain
\begin{align}
    \PP \cbr{\Sigma_{U,m}^{*-1/2} \cW_{u_m} \in \cA^{\zeta}} &= \PP \cbr{\cW_{u_m} \in \Sigma_{U,m}^{*1/2} \cA^{\zeta}} \nonumber\\
    &\overset{(i)}{\leq} \PP \cbr{\cN(0, \Sigma^*_{U,m}) \in \Sigma_{U,m}^{*1/2} \cA^{\zeta}} + o(1) \nonumber\\
    &= \PP \cbr{\cN(0, I_r) \in \cA^{\zeta}} + o(1) \nonumber\\
    &\overset{(ii)}{\leq} \PP \cbr{\cN(0, I_r) \in \cA} + \zeta(0,59 r^{1/4} + 0.21) + o(1) \nonumber\\
    &\overset{(iii)}{=} \PP \cbr{\cN(0, I_r) \in \cA} + o(1) \label{eqn:row_clt_high_order3}
\end{align}
where (i) follows from Lemma~\ref{lem:W_berry_esseen}; (ii) is due to \cref{thm:gaussian_cvx} and (iii) is a result of $r^{1/4} \zeta \to 0$ from \cref{lem:row_clt_higher_order}. Similarly, we have
\begin{align}
    \PP \cbr{\Sigma_{U,m}^{*-1/2} \cW_{u_m} \in \cA^{-\zeta}} &\geq \PP \cbr{\cN(0, I_r) \in \cA} - o(1). \label{eqn:row_clt_high_order4}
\end{align}

Finally, combining \eqref{eqn:row_clt_high_order1}, \eqref{eqn:row_clt_high_order2}, \eqref{eqn:row_clt_high_order3} and \eqref{eqn:row_clt_high_order4} gives
\begin{align}
    \abs{\PP \cbr{\cN(0, I_r) \in \cA} - \PP\cbr{\Sigma_{U,m}^{*-1/2} \sbr{U \sgn(U^\top U_*) - U_*}^\top e_m \in \cA}} \to 0 \label{eqn:row_clt_high_order_final}
\end{align}

Taking supremum over all convex set $\cA \in \mathscr{A}^r$ gives 
\begin{align*}
    &\sup_{\cA \in \mathscr{A}^r} \abs{\PP \cbr{\sbr{U \sgn(U^\top U_*) - U_*}^\top e_m \in \cA}  -  \cN(0, \Sigma^*_{U,m})\cbr{\cA}}\\
    =& \sup_{\cA \in \mathscr{A}^r} \abs{\PP \cbr{\sbr{U \sgn(U^\top U_*) - U_*}^\top e_m \in \Sigma^{*1/2}_{U,m}\cA}  -  \cN(0, \Sigma^*_{U,m})\cbr{\Sigma^{*1/2}_{U,m}\cA}}\\
    =& \sup_{\cA \in \mathscr{A}^r} \abs{\PP \cbr{\Sigma_{U,m}^{*-1/2} \sbr{U \sgn(U^\top U_*) - U_*}^\top e_m \in \cA}  -  \cN(0, I_r)\cbr{\cA}}\\
    \to&\; 0.
\end{align*}
where the last line follows from \eqref{eqn:row_clt_high_order_final}. This completes the proof of Theorem~\ref{thm:row_clt}.

% -------------------------------------------------------

\subsection{Proof of auxiliary lemmas}

\subsubsection{Proof of Lemma~\ref{lem:highD_check}}
\label{sec:prf_highD_check}

\paragraph{Term $L_q^{\xi}$.} Recall $L_q^{\xi} = L^{\xi}_{q,1} + L^{\xi}_{q,2}$ where
\begin{align*}
    L^{\xi}_{q,1} &:=\sum_{i=1}^n \frac{\EE \Finner{\xi_i}{\Gamma_n \xi_i}^{q/2}}{\HS{\Gamma_n}^q}\\
    L^{\xi}_{q,2} &:= \frac{n^2}{\binom{n}{2}} \sum_{1 \leq i < j \leq n} \frac{\EE \abs{\Finner{\xi_i}{\xi_j}}^q}{\HS{\Gamma_n}^q}.
\end{align*}

\textit{Bound for $L_{q,1}^{\xi}$.}
Since $\Gamma_n\succeq0$,
\[
    \langle \xi_i,\Gamma_n\xi_i\rangle_{\rm F}
    \leq
    \|\Gamma_n\|_{\rm op}\|\xi_i\|_{\rm F}^2
    \leq
    \|\Gamma_n\|_{\rm HS}\|\xi_i\|_{\rm F}^2.
\]
Therefore,
\begin{align*}
    L_{q,1}^{\xi} &\leq \frac{\sum_{i=1}^n\EE\|\xi_i\|_{\rm F}^q} {\|\Gamma_n\|_{\rm HS}^{q/2}}.
\end{align*}
By Lemma~\ref{lem:Gamma_properties}, we have
\[
    \|\Gamma_n\|_{\rm HS} \gtrsim \frac{\|\fC\|_{\rm HS}}{\Delta_{\max}}, \qquad \sum_{i=1}^n\EE\|\xi_i\|_{\rm F}^q \lesssim \frac{\teta^{q/2-1}}{\Delta_{\min}^{q/2}} \lambda_\times^q.
\]
Hence
\begin{align*}
L_{q,1}^{\xi} &\lesssim \teta^{q/2-1}\kappa_\Delta^{q/2} \frac{\lambda_\times^q}{\|\fC\|_{\rm HS}^{q/2}} = \teta^{q/2-1}\kappa_\Delta^{q/2}d_{\fC}^{q/4}.
\end{align*}

% \red{XXX} 

% \textit{Bound for $L^{\xi}_{q,1}$.} Note that $\Finner{\xi_i}{\Gamma_n \xi_i} \leq \Onorm{\Gamma_n} \Fnorm{\xi_i}^2$. By Lemma~\ref{lem:Gamma_properties}, we have $\opnorm{\Gamma_n -\Gamma}\leq \frac{1}{2} \opnorm{\Gamma}$ and $\frac{\opnorm{\fC}}{2\Delta_{\max}} \leq \opnorm{\Gamma} \leq \frac{\opnorm{\fC}}{2\Delta_{\min}}$ for any unitarily invariant norm $\opnorm{\cdot}$. Thus, we can obtain
% \begin{align*}
%     L^{\xi}_{q,1} \leq \frac{\Onorm{\Gamma_n}^{q/2}}{\HS{\Gamma_n}^q}\cdot \sum_{i=1}^n \EE \Fnorm{\xi_i}^q \lesssim \frac{\Delta_{\max}^q \Onorm{\fC}^{q/2}}{\Delta_{\min}^{q/2} \HS{\fC}^q} \cdot \sum_{i=1}^n \EE \Fnorm{\xi_i}^q
% \end{align*}
% Lemma~\ref{lem:matrix_kth_moment} gives $\Onorm{\fC} \lesssim r \lambda_1 \lambda_{r+1}\leq r \lambda_\times^2$, while Lemma~\ref{lem:Gamma_properties}(v) and (vi) give that for any $q \in (2,3]$,
% \begin{align*}
%     \sum_{i=1}^n \EE \Fnorm{\xi_i}^q \lesssim  \frac{\teta^{q/2-1}}{\Delta_{\min}^{q/2}} \EE \Fnorm{ZY^\top}^{q} \lesssim \frac{\teta^{q/2-1}}{\Delta_{\min}^{q/2}} \lambda_\times^q.
% \end{align*}
% Therefore, we have
% \begin{align*}
%     L^{\xi}_{q,1} &\lesssim  \frac{\Delta_{\max}^q}{\Delta_{\min}^{q/2}} \cdot \frac{ r^{q/2} \lambda_\times^q}{\HS{\fC}^q} \cdot  \frac{\teta^{q/2-1}}{\Delta_{\min}^{q/2}} \cdot \lambda_\times^q \\
%     &\leq \teta^{q/2-1} \kappa_\Delta^q r^{q/2}\cdot \frac{\lambda_\times^{2q} }{\HS{\fC}^q}
% \end{align*}

\textit{Bound for $L^{\xi}_{q,2}$.} Note that for $i \neq j$,
\begin{align*}
    \EE \abs{\Finner{\xi_i}{\xi_j}}^q \leq \eta^q (1- \teta)^{q(n-i)} (1-\teta)^{q(n-j)}\EE \Fnorm{Z Y^\top}^{2q}
\end{align*}
Since $\sum_{i< j} a_i a_j \leq 2^{-1} (\sum_{i=1}^n a_i)^2$, we have
\begin{align*}
    \sum_{1 \leq i < j \leq n} \EE \abs{\Finner{\xi_i}{\xi_j}}^q &\leq \frac{1}{2} \eta^q \sbr[\bigg]{\sum_{i=1}^n (1-\teta)^{q(n-i)}}^2 \cdot \EE \Fnorm{Z Y^\top}^{2q} \lesssim \eta^q \teta^{-2} \cdot \EE \Fnorm{Z Y^\top}^{2q}
\end{align*}
Therefore, applying Lemma~\ref{lem:Gamma_properties} and Lemma~\ref{lem:matrix_kth_moment} gives
\begin{align*}
    L_{q,2}^{\xi} \lesssim \eta^q \teta^{-2} \cdot \frac{\EE \Fnorm{Z Y^\top}^{2q}}{\HS{\Gamma_n}^q} \lesssim \teta^{q-2} \kappa_\Delta^q \cdot \frac{\lambda_\times^{2q}}{\HS{\fC}^q}
\end{align*}

\textit{Combining two bounds} gives
\[
    L_q^\xi\lesssim \teta^{q/2-1}\kappa_\Delta^{q/2}d_C^{q/4} + \teta^{q-2}\kappa_\Delta^q d_C^{q/2}.
\]
For $q=3$,
\[
    L_3^\xi \lesssim \teta^{1/2}\kappa_\Delta^{3/2}d_C^{3/4} + \teta\kappa_\Delta^3d_C^{3/2}.
\]
% \begin{align*}
%     L_{q}^{\xi} &\lesssim \kappa_\Delta^q  \sbr{\teta^{q/2-1} r^{q/2}  + \teta^{q-2}} \frac{\lambda_\times^{2q}}{\HS{\fC}^q}
% \end{align*}
% For $q=3$, we have
% \begin{align*}
%     L_{3}^{\xi} &\lesssim  \kappa_\Delta^3 \sbr{\teta^{1/2} r^{3/2} + \teta} \frac{\lambda_\times^6}{\HS{\fC}^3}.
% \end{align*}

\paragraph{Term $L_q^{\fg}$.} Recall
\begin{align*}
    L_q^{\fg} &= \sum_{i=1}^n \frac{\EE \Finner{\fg_i}{\Gamma_n \fg_i}^{q/2}}{\HS{\Gamma_n}^q}
\end{align*}
Write $\fg_i = \Xi_i^{1/2} \tfg_i$ where $\tfg_i \sim \cN(\0_{d-r,r}, \cI)$, then $\Finner{\fg_i}{\Gamma_n \fg_i} = \Finner{\tfg_i}{\Xi_i^{1/2} \Gamma_n \Xi_i^{1/2} \tfg_i}$ and we have
\begin{align*}
    \EE \abs{\Finner{\fg_i}{\Gamma_n \fg_i}}^{q/2} \lesssim_q \sbr{\tr \rbr{\Xi_i^{1/2} \Gamma_n \Xi_i^{1/2}}}^{q/2} = \sbr{\tr \rbr{\Gamma_n \Xi_i}}^{q/2} \leq \sbr{\HS{\Gamma_n} \HS{\Xi_i}}^{q/2}
\end{align*}
Therefore,
\begin{align*}
    L_q^{\fg} &\lesssim \frac{1}{\HS{\Gamma_n}^q} \sum_{i=1}^n \sbr{\HS{\Gamma_n} \HS{\Xi_i}}^{q/2} = \frac{1}{\HS{\Gamma_n}^{q/2}} \sum_{i=1}^n \HS{\Xi_i}^{q/2}
\end{align*}
Note that
\begin{align*}
    \HS{\Xi_i} \leq \eta \HS{(\cI-\eta \cL)^{n-i} \fC (\cI-\eta \cL)^{n-i}} \leq \eta (1-\eta \Delta_{\min})^{2(n-i)} \HS{\fC}
\end{align*}
Then
\begin{align*}
    \sum_{i=1}^{n} \HS{\Xi_i}^{q/2} &\lesssim \eta^{q/2} \HS{\fC}^{q/2} \sum_{i=1}^n (1-\eta \Delta_{\min})^{q(n-i)} \lesssim \frac{\eta^{q/2-1}}{\Delta_{\min}} \HS{\fC}^{q/2}
\end{align*}
Note that Lemma~\ref{lem:Gamma_properties} gives $\HS{\Gamma_n} \geq (1-\epsilon_n) \HS{\Gamma} \gtrsim \frac{\HS{\fC}}{\Delta_{\max}}$, thus
\begin{align*}
    L_q^{\fg} \lesssim \frac{\teta^{q/2-1}}{\Delta_{\min}^{q/2}} \cdot \frac{\HS{\fC}^{q/2}}{\HS{\Gamma_n}^{q/2}} \lesssim \teta^{q/2-1} \kappa_\Delta^{q/2}.
\end{align*}

\paragraph{Term $K_{\beta}^{\beta}$.} Recall
\begin{align*}
    K_\beta^{\beta} &= \frac{n^\beta}{n} \sum_{i=1}^n \EE \abs{\frac{\Finner{\xi_i}{\xi_i} - \EE \Finner{\xi_i}{\xi_i}}{\HS{\Gamma_n}}}^{\beta}
\end{align*}
Note that for $\beta \geq 2$,
\begin{align*}
    \EE \abs{\Finner{\xi_i}{\xi_i} - \EE \Finner{\xi_i}{\xi_i}}^{\beta} &\lesssim \EE \abs{\Finner{\xi_i}{\xi_i}}^{\beta} + \abs{\EE \Finner{\xi_i}{\xi_i}}^{\beta} \lesssim \EE \Fnorm{\xi_i}^{2\beta}
\end{align*}
Applying Lemma~\ref{lem:Gamma_properties} and Lemma~\ref{lem:matrix_kth_moment} gives
\begin{align*}
    n^{1-\beta} K_\beta^{\beta} \lesssim \sum_{i=1}^n \frac{\EE \Fnorm{\xi_i}^{2\beta}}{\HS{\Gamma_n}^{\beta}} \lesssim \teta^{\beta-1}\kappa_\Delta^\beta \frac{\lambda_\times^{2\beta}}{\HS{\fC}^{\beta}}
\end{align*}
For $\beta =3$, we have
\begin{align*}
    n^{-2} K_3^3 \lesssim \teta^{2} \kappa_\Delta^3 \cdot \frac{\lambda_\times^{6}}{\HS{\fC}^{3}}
\end{align*}

\paragraph{Term $J_n$.} Recall
\begin{align*}
    J_n &= \frac{\sum_{i=1}^n \fv_i}{\HS{\Gamma_n}^2}, \qquad\text{where}\quad \fv_i=\Var (\Fnorm{\xi_i}^2)
\end{align*}
Note that $\fv_i \leq \EE \Fnorm{\xi_i}^4$, then applying Lemma~\ref{lem:Gamma_properties} and Lemma~\ref{lem:matrix_kth_moment} gives
\begin{align*}
    J_n \leq \frac{\sum_{i=1}^n \EE \Fnorm{\xi_i}^4}{\HS{\Gamma_n}^2} \lesssim \teta \kappa_\Delta^2 \cdot \frac{\lambda_\times^{4} }{\HS{\fC}^2}
\end{align*}

\paragraph{Sufficient condition for $q=\beta=3$.} 
Set
\[
    \rho_n:=\teta\kappa_\Delta^3d_C^{3/2}.
\]
Since $d_C\gtrsim1$, $\kappa_\Delta\geq1$, and
$\teta\leq\teta\kappa_\Delta=\eta\Delta_{\max}<1$, the preceding
bounds give
\[
L_{3,1}^{\xi}\lesssim \rho_n^{1/2},
\qquad
L_{3,2}^{\xi}\lesssim \rho_n,
\qquad
L_3^g\lesssim \rho_n^{1/2},
\qquad
n^{-2}K_3^3\lesssim \teta\rho_n,
\qquad
J_n\lesssim \rho_n.
\]
Hence a sufficient condition for all four quantities in Lemma~\ref{lem:clt_biometrika19} to vanish is
\[
    \teta\kappa_\Delta^3d_C^{3/2}\to0.
\]

% -----------------------------------------------------------------

\subsubsection{Proof of Lemma~\ref{lem:highD_chi2_mixture_comparison}}
\label{sec:prf_chi2_mixture_comparison}

Apply Lemma~\ref{lem:chi2_mixture_comparison} with $\Delta_\Gamma=\Gamma_n-\Gamma$. For any $t>0$,
\begin{align*}
    d_K\rbr[\Big]{\Fnorm[\big]{\fg}^2, \Fnorm[\big]{\tfg}^2} \lesssim \sqrt{\frac{\abs{\tr \Delta_\Gamma}+t}{\HS{\Gamma}}} + \exp \cbr[\bigg]{- c\rbr[\bigg]{\frac{t^2}{\HS{\Delta_\Gamma}^2} \bigwedge \frac{t}{\Onorm{\Delta_\Gamma}}}}
\end{align*}
Since $\|\Delta_\Gamma\|_{\mathrm{op}} \le \|\Delta_\Gamma\|_{\mathrm{HS}} \le\|\Delta_\Gamma\|_*$,
we choose $t=C(\log n)\|\Delta_\Gamma\|_*$ with $C>0$ sufficiently large. Then
\begin{align*}
    \frac{t^2}{\|\Delta_\Gamma\|_{\mathrm{HS}}^2}
    \ge
    C^2(\log n)^2,
    \qquad
    \frac{t}{\|\Delta_\Gamma\|_{\mathrm{op}}}
    \ge
    C\log n.
\end{align*}
Thus, choosing \(C\) large enough, the exponential term is bounded by \(n^{-20}\).

For the first term, using Lemma~\ref{lem:Gamma_properties} gives
\begin{align*}
    |\operatorname{tr}(\Delta_\Gamma)|+t \le \epsilon_n \operatorname{tr}(\Gamma) + C(\log n)\epsilon_n\operatorname{tr}(\Gamma) \lesssim \epsilon_n (\log n)\operatorname{tr}(\Gamma).
\end{align*}
Therefore
\begin{align*}
    d_K\! \rbr[\Big]{\Fnorm[\Big]{\sum_{i=1}^n g_i}^2,\;\Fnorm{G}^2} \lesssim
    \rbr{\epsilon_n (\log n) \kappa_\Delta
        \frac{\tr(\fC)} {\HS{\fC}}}^{1/2}
    +n^{-20}.
\end{align*}
Finally, noticing that $\tr \fC = \EE \Fnorm{ZY^\top}^2$ and applying Lemma~\ref{lem:matrix_kth_moment} gives $\tr \fC \lesssim C_\psi^4 \lambda_{1\sim r}^2 \nu_2^2$, we have
\begin{align*}
    d_K\! \rbr[\Big]{\Fnorm[\Big]{\sum_{i=1}^n g_i}^2,\;\Fnorm{G}^2} \lesssim
    \rbr{\epsilon_n (\log n) \kappa_\Delta
        \frac{C_\psi^4 \lambda_{1\sim r}^2 \nu_2^2} {\HS{\fC}}}^{1/2}
    +n^{-20}.
\end{align*}
This proves the claim.

\subsubsection{Proof of Lemma~\ref{lem:highD_step2}}
\label{sec:prf_highD_step2}
    
We apply Lemma~\ref{lem:g_approx_additive_remainder} with $L_n$ and $e_n$ as defined in the lemma statement, and set
\begin{alignat*}{2}
    &\mathsf{A}_n = \sum_{i=1}^n \xi_i, \qquad &&\mathsf{E}_n=\eta^{-1/2} \sbr[\Big]{\cT_n - \eta \sum_{i=1}^n (\cI-\eta \cL)^{n-i}(Z_iY_i^\top)}
\end{alignat*}
Then it suffices to bound the two probabilities and $w_n(h_n)$ in Lemma~\ref{lem:g_approx_additive_remainder} as follows.

(i) \cref{prop:linearization} implies that with probability at least $1-\delta_p$,
\begin{align*}
    \norm{\mathsf{E}_n}  \lesssim \frac{1}{\sqrt{\eta}} C_1 \tau_* \sbr{C_{\ref{lem:1epoch_bound}}C_\psi^2 \kappa_* \sqrt{r \teta \log n} + \tau_*}
\end{align*}  
Since $\tau_* \asymp \sqrt{\teta} \rho_\sigma \bnu \asymp C_{\ref{lem:1epoch_bound}} C_\psi^2 \kappa_* \sqrt{\teta r \log n} \, \bnu$ and $\eta^{-1/2} = \teta^{-1/2} \Delta_{\min}^{1/2}$, we have
\begin{align*}
    \norm{\mathsf{E}_n} \leq C_e \Delta_{\min}^{1/2}  \kappa_*^2 r \sqrt{\teta}\;  (\log n)  \bnu\sbr{1 + \bnu} 
    %\label{eqn:highD_step2_en}
\end{align*}
for some constant $C_e > 0$ large enough. Since $\mathsf{E}_n$ has at most $r$ columns, $\|\mathsf{E}_n\|_{\rm F}\leq \sqrt r\,\|\mathsf{E}_n\|$.
Therefore, on the event that the above bound holds, we have
\[
    \|E_n\|_{\rm F} \leq C_e\Delta_{\min}^{1/2}\kappa_\ast^2r^{3/2} \sqrt{\teta}(\log n)\bar\nu(1+\bar\nu) =:e_n.
\]
Hence 
\begin{align*}
    \PP \cbr{\Fnorm{\mathsf{E}_n} > e_n} \leq \delta_p.
\end{align*}

(ii) Lemma~\ref{lem:chi2_concentration} implies that with probability at least $1- n^{-20}$,
\begin{align*}
    \norm{\mathsf{G}_n}^2 &\lesssim \tr \Gamma + \HS{\Gamma} \sqrt{\log n} + \Onorm{\Gamma} \log n 
    \leq C_L (\tr \Gamma) \log n
\end{align*}
for some constant $C_L > 0$ large enough. Hence
\begin{align*}
    \PP\cbr{\norm{\mathsf{G}_n} > L_n} \leq n^{-20}.
\end{align*}
(iii) Define $h_n=e_n(2L_n+e_n)$. Lemma~\ref{lem:chi2_anti_concentration} implies that
\begin{align*}
    w_n(h_n) \lesssim \sqrt{\frac{h_n}{\HS{\Gamma}}} \lesssim \sqrt{\frac{e_n(e_n+L_n)}{\HS{\Gamma}}}.
\end{align*}

% --------------------------------------------------------

\subsubsection{Proof of Lemma~\ref{lem:highD_step3}}
\label{sec:prf_highD_step3}

Lemma~\ref{lem:alignment_diff} implies that
\begin{align*}
    \Fnorm{\frac{1}{\sqrt{\eta}}\sin \Theta(U_n, U_*)} = \Fnorm{\frac{1}{\sqrt{\eta}}\cT_n (\cC_n\cC_n^\top)^{1/2} }
\end{align*}
with $\norm{(\cC_n\cC_n^\top)^{1/2} - I_r} \leq \norm{\cT_n}^2$.
    
Thus, we apply Lemma~\ref{lem:g_approx_multiplicative_remainder} with $L_n$ being defined as in the lemma statement, $e_n=C_e \tau_*^2$ fo $C_e>0$ for some sufficiently large constant $C_e>0$, and
\begin{alignat*}{2}
    &\mathsf{A}_n = \frac{1}{\sqrt{\eta}}\cT_n, \qquad \qquad &&\mathsf{E}_n = (\cC_n\cC_n^\top)^{1/2} - I_r.
\end{alignat*}
Then it suffices to bound the two probabilities and $w_n(h_n)$ in Lemma~\ref{lem:g_approx_multiplicative_remainder} as follows.

(i) Lemma~\ref{lem:convergence_rate_optimal} implies that
\begin{align*}
    \PP \cbr{\norm{\mathsf{E}_n} > e_n} \leq \delta_p.
\end{align*}
By the sample size condition in \eqref{eqn:cond_sample_sz_lem}, we have $e_n \leq \frac{1}{2}$.

(ii) Similar to the proof of Lemma~\ref{lem:highD_step2}, we have
\begin{align*}
    \PP \cbr{\norm{\mathsf{G}_n} > L_n} \leq n^{-20}.
\end{align*}
(iii) Define $h_n := \sbr[\big]{\frac{1}{(1-e_n)^2}-1} L_n^2$. Lemma~\ref{lem:chi2_anti_concentration} implies that
\begin{align*}
    w_n(h_n) \lesssim \sqrt{\frac{h_n}{\HS{\Gamma}}} \lesssim \sqrt{\frac{e_n L_n^2}{\HS{\Gamma}}},
\end{align*}
where the last step follows from the fact that $\frac{1}{(1-e_n)^2} -1 \lesssim e_n$ for $e_n \in [0,1/2]$.

% --------------------------------------------------------

\subsubsection{Proof of Lemma~\ref{lem:W_berry_esseen}}
\label{sec:prf_W_berry_esseen}
Recall notation from Appendix~\ref{sec:clt_setup_defn}
\begin{align*}
    \xi_{u,i} = \sqrt{\eta} \sbr{(\cI-\eta \cL)^{n-i} (Z_iY_i^\top)}^\top u  = \xi_i^\top u \in \RR^r, \qquad \Gamma_{u,n} = \sum_{i=1}^n \cov(\xi_{u,i}).
\end{align*}
Then $\cW_u = \sqrt{\eta}\sum_{i=1}^n \xi_{u,i}$ and $\cov(\cW_u) = \eta \Gamma_{u,n}$.

Note that by definition, $\Sigma^*_{U,m} = \eta \Gamma_u$.
By triangle inequality, we have
\begin{align*}
    &\sup_{\cA \in \mathscr{A}^r} \abs{\PP \cbr{\cW_u \in \cA} - \cN(0, \Sigma^*_{U,m})\cbr{\cA} }\\
    \leq &\sup_{\cA \in \mathscr{A}^r} \abs{\PP \cbr{\cW_u \in \cA} - \cN(0, \eta \Gamma_{u,n})\cbr{\cA} } + \sup_{\cA \in \mathscr{A}^r} \abs{\cN(0, \eta \Gamma_{u,n})\cbr{\cA} - \cN(0, \eta \Gamma_u)\cbr{\cA} } 
\end{align*}
The proof proceeds in two steps. In Step~1, we show that the first term vanishes via a Berry--Esseen bound
for a sum of independent random vectors. In Step~2, we show the second term vanishes since $\Gamma_{u,n} \to \Gamma_u$.

\paragraph{Preliminaries.} We first gather some useful results. Lemma~\ref{lem:Gamma_properties}(vii) implies that $\lambda_{\min}(\tfC)\lesssim r$, which combined with condition \eqref{eqn:W_berry_esseen_cond} gives
\begin{align*}
    \eta \Delta_{\max} \leq \teta \kappa_* \lesssim \teta \kappa_*^6 r^{7/2} \lambda_{\min}^{-3}(\tfC) \to 0.
\end{align*}
and
\begin{equation}
    \teta \kappa_*^5 r^{5/2} \lambda_{\min}^{-2}(\tfC) \to 0 \label{eqn:W_berry_esseen_cond2}
\end{equation}

\paragraph{Step 1. Gaussian approximation of $\cW_u$ using the Berry-Esseen theorem.}
The Berry-Esseen theorem (\cref{thm:berry_esseen}) implies that
\begin{align*}
    \sup_{\cA \in \mathscr{A}^r} &\abs{\PP \cbr{\cW_u \in \cA} - \PP \cbr{\cN(0, \eta \Gamma_{u,n})\in \cA}} \lesssim r^{1/4} \gamma_u
\end{align*}
where
\begin{align*}
    \gamma_u &:= \sum_{i=1}^n \EE \norm{\sbr{\eta \Gamma_{u,n}}^{-1/2} \sqrt{\eta}\xi_{u,i}}^3 = \sum_{i=1}^n \EE \norm{\Gamma_{u,n}^{-1/2} \xi_{u,i}}^3\\
    &\leq \lambda_{\min}(\Gamma_{u,n})^{-3/2} \sum_{i=1}^n \EE \norm{\xi_{u,i}}^3
\end{align*}

To upper bound $\EE \norm{\xi_{u,i}}^3$, we first introduce some more notation. For any $j \in [r]$, denote $\fD_j = \diag\cbr{1-\eta \Delta_{lj}}_{l \in [d-r]}$.
Note that $\RR^r \ni \xi_{u,i} = (\xi_{u,i1}, \cdots, \xi_{u,ir})$ with entries
\begin{align*}
    \RR \ni \xi_{u,ij} = \eta^{1/2} u^\top
    \begin{pmatrix}
        Y_{i1} \fD_1^{n-i}Z_i & Y_{i2} \fD_2^{n-i}Z_i & \cdots & Y_{ir} \fD_r^{n-i}Z_i
    \end{pmatrix}
\end{align*}
Applying the power mean inequality gives
\begin{align*}
    \norm{\xi_{u,i}}^3 = \sbr[\Big]{\sum_{j=1}^r \xi_{u,ij}^2}^{3/2} \leq  r^{3/2-1} \sum_{j=1}^r \abs{\xi_{u,ij}}^3 =  r^{1/2} \sum_{j=1}^r \abs{\xi_{u,ij}}^3
\end{align*}
Taking expectation and applying Cauchy-Schwarz inequality gives
\begin{align*}
    \EE \norm{\xi_{u,i}}^3 &\leq \eta^{3/2} r^{1/2} \sum_{j=1}^r \EE \abs{Y_{ij} u^\top \fD_j^{n-i} Z_i}^3 \\
    &\leq \eta^{3/2} r^{1/2} \sum_{j=1}^r \sqrt{\EE Y_{ij}^6} \cdot \sqrt{\EE \sbr{u^\top \fD_j^{n-i} Z_i}^6} \\
    &\overset{(i)}{=}  \eta^{3/2} r^{1/2} \sum_{j=1}^r \sqrt{\EE Y_{ij}^6} \cdot \sqrt{\EE \sbr{(\fD_j^{n-i} \Lambda_{*,\perp}^{1/2} u)^\top  \tZ_i}^6}\\
    &\overset{(ii)}{\lesssim} \eta^{3/2} r^{1/2} \sum_{j=1}^r \lambda_j^{3/2} \cdot \norm{\fD_j^{n-i} \Lambda_{*,\perp}^{1/2} u}_2^3\\
    &\leq \eta^{3/2} r^{1/2} \sum_{j=1}^r \lambda_j^{3/2} \cdot (1-\teta)^{3(n-i)} \cdot \norm{\Lambda_{*,\perp}^{1/2} u}_2^3
\end{align*}
where (i) follows from the commutative property $\fD_j^{n-i} \Lambda_{*,\perp}^{1/2} = \Lambda_{*,\perp}^{1/2} \fD_j^{n-i}$ since both are diagonal; (ii) is due to sub-Gaussianity.

Summing over $i$ gives
\begin{align*}
    \sum_{i=1}^n \EE \norm{\xi_{u,i}}^3 &\lesssim \frac{\teta^{1/2}}{\Delta_{\min}^{3/2}} r^{1/2} \sum_{j=1}^r \lambda_j^{3/2} \cdot \norm[\big]{\Lambda_{*,\perp}^{1/2} u}_2^3 \lesssim \frac{\teta^{1/2}}{\Delta_{\min}^{3/2}} r^{3/2} \lambda_1^{3/2} \cdot \norm[\big]{\Lambda_{*,\perp}^{1/2} u}_2^3 
\end{align*}

Meanwhile, we have lower bound from Lemma~\ref{lem:Gamma_u_properties}(iii) that
\begin{align*}
    \lambda_{\min}(\Gamma_{u,n}) \gtrsim \frac{\lambda_r}{\Delta_{\max}} \lambda_{\min}(\tfC) \norm[\big]{\Lambda_{*,\perp}^{1/2} u}_2^2
\end{align*}

Combining the bounds above gives
\begin{align*}
    \gamma_u &= \sum_{i=1}^n \EE \norm{\Gamma_{u,n}^{-1/2} \xi_{u,i}}^3 \lesssim \teta^{1/2} \kappa_\Delta^{3/2} r^{3/2} \kappa_\lambda^{3/2} \cdot \lambda_{\min}(\tfC)^{-3/2}
\end{align*}

Thus, we have
\begin{align*}
    \sup_{\cA \in \mathscr{A}^r} \abs{\PP \cbr{\cW_u \in \cA} - \PP \cbr{\cN(0, \eta \Gamma_{u,n})\in \cA}} &\lesssim  \teta^{1/2} \kappa_\Delta^{3/2} r^{7/4} \kappa_\lambda^{3/2} \cdot \lambda_{\min}(\tfC)^{-3/2}\\
    &\lesssim \teta^{1/2} \kappa_*^3 r^{7/4} \lambda_{\min}(\tfC)^{-3/2}
\end{align*}
which $\to 0$ under \eqref{eqn:W_berry_esseen_cond}.

\paragraph{Step 2. bound $\mathsf{TV}$-distance between Gaussian distributions.}
\begin{align*}
    \sup_{\cA \in \mathscr{A}^r} &\abs{
        \PP \cbr{\cN(0, \eta\Gamma_{u,n}) \in \cA}-
        \PP \cbr{\cN(0, \eta\Gamma_u) \in \cA}
    } \leq \mathsf{TV} \rbr{\cN(0,\eta\Gamma_{u,n}), \cN(0, \eta\Gamma_u)}\\
    \overset{(i)}{\lesssim} &\Fnorm{\Gamma_u^{-1/2} \Gamma_{u,n} \Gamma_u^{-1/2} - I_d} = \Fnorm{\Gamma_u^{-1/2} (\Gamma_{u,n} - \Gamma_u) \Gamma_u^{-1/2}} \\
    \leq &\norm[\big]{\Gamma_u^{-1/2}}  \Fnorm{\Gamma_{u,n} - \Gamma_u} \norm[\big]{\Gamma_u^{-1/2}} \leq \sqrt{r} \underbrace{\norm{\Gamma_u^{-1}}}_{=:\alpha_1} \underbrace{\norm{\Gamma_{u,n} - \Gamma_u}}_{=:\alpha_2}
\end{align*}
Here (i) follows from \cref{lem:TV_gaussian}. Next, we bound $\alpha_1$ and $\alpha_2$ as follows.

(a) $\alpha_1$: Lemma~\ref{lem:Gamma_u_properties}(iii) gives $\alpha_1 \lesssim \frac{\Delta_{\max}}{\lambda_r \lambda_{\min}(\tfC) \norm{\tu}^2}$. 

(b) $\alpha_2$:
Since $\Gamma_u = \Lambda_*^{1/2} \tGamma_{\tu} \Lambda_*^{1/2}$ and $\Gamma_{u,n} = \Lambda_*^{1/2} \tGamma_{\tu,n} \Lambda_*^{1/2}$ (\cref{lem:Gamma_u_properties}(ii)), we have
\begin{align*}
    \norm{\Gamma_{u,n} - \Gamma_u} \leq \lambda_1  \norm[\big]{\tGamma_{\tu,n} - \tGamma_{\tu}} \overset{(ii)}{\leq} \lambda_1 \epsilon_n \kappa_\Delta \frac{\Onorm[\big]{\tfC}}{\lambda_{\min}(\tfC)} \norm[\big]{\tGamma_{\tu}}
\end{align*} 
Here (ii) follows from Lemma~\ref{lem:Gamma_u_properties}(iv).

From \eqref{eqn:whitened_defn}, we have $\tGamma_{\tu}=T_{\tu} \tGamma T_{\tu}^*$ with $\norm{T_{\tu}}, \norm{T_{\tu}^*}\leq \norm{\tu}$. Thus, we have $\norm[\big]{\tGamma_{\tu}} \leq \norm{\tu}^2 \norm[\big]{\tGamma}$, which combined with Lemma~\ref{lem:Gamma_properties}(ii) and (vii) gives
\begin{align*}
    \norm{\Gamma_{u,n} - \Gamma_u} \lesssim \lambda_1 \epsilon_n \kappa_\Delta \frac{\Onorm[\big]{\tfC}^2}{\Delta_{\min} \lambda_{\min}(\tfC)} \norm{\tu}^2 \lesssim \lambda_1 \epsilon_n \kappa_\Delta \frac{r^2}{\Delta_{\min} \lambda_{\min}(\tfC)} \norm{\tu}^2.
\end{align*}

Combining (a) and (b) gives
\begin{align*}
    &\sup_{\cA \in \mathscr{A}^r} \abs{
        \PP \cbr{\cN(0, \eta\Gamma_{u,n}) \in \cA}-
        \PP \cbr{\cN(0, \eta\Gamma_u) \in \cA}
    }\\ 
    \leq &\sqrt{r} \frac{\Delta_{\max}}{\lambda_r \lambda_{\min}(\tfC) \norm{\tu}^2} \cdot \lambda_1 \epsilon_n \kappa_\Delta \frac{r^2}{\Delta_{\min} \lambda_{\min}(\tfC)} \norm{\tu}^2\\
    =& \sqrt{r} \epsilon_n \kappa_\Delta \frac{\lambda_1}{\lambda_r} \frac{\Delta_{\max}}{\Delta_{\min}} \frac{r^2}{\lambda_{\min}^2(\tfC)}
    = r^{5/2} \epsilon_n  \kappa_\Delta^2 \kappa_\lambda \frac{1}{\lambda_{\min}^2(\tfC)}. 
\end{align*}

Finally, for $\epsilon_n = (1-\teta)^{2n} + \teta \kappa_\Delta^2$ and $n\teta \geq \frac{\log n}{2}$, we have $(1-\teta)^{2n}\leq \frac{1}{n} \lesssim \teta$, thus $\epsilon_n \lesssim \teta \kappa_\Delta^2$, and we have
\begin{align*}
    \sup_{\cA \in \mathscr{A}^r} \abs{
        \PP \cbr{\cN(0, \eta\Gamma_{u,n}) \in \cA}-
        \PP \cbr{\cN(0, \eta\Gamma_u) \in \cA}
    } &\lesssim \teta \kappa_\Delta^4 \kappa_\lambda r^{5/2} \lambda_{\min}^{-2}(\tfC)\\
    &\lesssim \teta \kappa_*^5 r^{5/2} \lambda_{\min}^{-2}(\tfC).
\end{align*}
which $\to 0$ due to \eqref{eqn:W_berry_esseen_cond2}.

% --------------------------------------------------------

\subsubsection{Proof of Lemma~\ref{lem:row_clt_higher_order}}
\label{sec:prf_row_clt_higher_order}

\cref{prop:linearization} implies that with probability at least $1-\delta_p$,
\begin{align*}
    \norm{E} \lesssim \tau_*^2 + \tau_* \kappa_* \sqrt{\teta r \log n} \lesssim 
    \kappa_*^2 r \teta (\log n)\bnu \;(1+\bnu) 
\end{align*}
which combined with $\norm{(\cC \cC^\top)^{1/2} - I_r}\lesssim \norm{\cT}^2 \lesssim \tau_*^2 \lesssim \kappa_*^2 r \teta (\log n) \bnu^2$ gives
\begin{align*}
    \norm{\Psi_{u_m}} &\lesssim \norm{U_*^\top e_m} \tau_*^2 + \norm{U_{*,\perp}^\top e_m} \sbr{\tau_*^2 + \tau_* \kappa_* \sqrt{\teta r \log n} } \nonumber\\
    &\lesssim \norm{U_*^\top e_m} \kappa_*^2 \teta r (\log n) \bnu^2 + \norm{U_{*,\perp}^\top e_m} \kappa_*^2 \teta r (\log n) \bnu \;(1+\bnu).
\end{align*}

By \cref{lem:Gamma_u_properties}, we have
\begin{align*}
    \lambda_{\min}(\Sigma_{U,m}^{*}) = \eta \cdot \lambda_{\min}(\Gamma_u) &\gtrsim \frac{\teta}{\Delta_{\min}}\cdot \frac{\lambda_r \norm{\tu_{u_m}}^2}{\Delta_{\max}} \lambda_{\min}(\tfC) \geq \teta \frac{\norm{\tu_{u_m}}^2}{\lambda_1} \lambda_{\min}(\tfC).
\end{align*}
Here we used $\Delta_{\min}=\lambda_r-\lambda_{r+1}\leq \lambda_r$ and $\Delta_{\max} \leq \lambda_1$.

Thus, we can obtain
\begin{align*}
    &\norm{\Sigma_{U,m}^{*-1/2} \Psi_{u_m}} \lesssim \lambda_{\min}^{-1/2}(\Sigma_{U,m}^{*}) \norm{\Psi_{u_m}} \\
    &\lesssim \sbr[\Big]{\teta \frac{\norm{\tu_{u_m}}^2}{\lambda_1} \lambda_{\min}(\tfC)}^{-1/2} \cdot \sbr{\norm{U_*^\top e_m} \kappa_*^2 \teta r (\log n) \bnu^2 + \norm{U_{*,\perp}^\top e_m} \kappa_*^2 \teta r (\log n) \bnu \;(1+\bnu)}.
\end{align*}
Expanding the last line gives \eqref{eqn:row_clt_P_zeta} and \eqref{eqn:row_clt_zeta_defn}.

% --------------------------------------------------------

\section{Proofs for Section~\ref{sec:main_bs}}
\label{sec:bootstrap_analysis}
This appendix proves the bootstrap results in Section~\ref{sec:main_bs}.
The proof has two main components.

First, Appendix~\ref{sec:prf_bs_iterates} establishes stability and linearization of the bootstrap tangent process. We first work under the joint law of the
truncated observations constructed in Appendix~\ref{sec:oja_subG} and the multiplier
weights. The resulting bounds are then transferred to the original
observations through the truncation coupling.

Second, Appendices~\ref{sec:prf_bs_highD_clt} and~\ref{sec:prf_bs_row_clt} condition on  $\cF_X$, i.e., on the initialization and the original observations, and analyze the resulting frozen multiplier sums. 
Appendix~\ref{sec:prf_bs_highD_clt} proves
the global bootstrap approximation in Theorem~\ref{thm:bs_highD_clt}, while Appendix~\ref{sec:prf_bs_row_clt}
proves the row-wise bootstrap approximation in Theorem~\ref{thm:bs_row_clt}.

% We begin with a deterministic comparison between the tangent difference
% of two subspaces and their aligned difference. This result will be used
% in both Appendices~\ref{sec:prf_bs_highD_clt} and~\ref{sec:prf_bs_row_clt}.

We will also use the deterministic two-subspace alignment identity in Lemma~\ref{lem:alignment_diff_bootstrap}, which relates the tangent difference of two subspaces to their Procrustes-aligned difference. It will be invoked in both Appendices~\ref{sec:prf_bs_highD_clt} and~\ref{sec:prf_bs_row_clt}.

\subsection{Bootstrap stability and linearization}
\label{sec:prf_bs_iterates}

\paragraph{Joint coupling and truncated bootstrap process.}
For the proof, extend the data--bootstrap probability space to include
the truncation coupling from Appendix~\ref{sec:oja_subG}. Let $\PP_{\rm joint}$ denote
the joint law of
\[
  \left(
    U_0,\{(X_i,\overline X_i)\}_{i=1}^n,
    \{w_i^{\bs}\}_{i=1}^n
  \right),
\]
where $\{(X_i,\overline X_i)\}_{i=1}^n$ are the i.i.d.\ coupled
original and truncated observations constructed in~\eqref{eqn:trunc_coupling} Appendix~\ref{sec:oja_subG}, and
$\{w_i^{\bs}\}_{i=1}^n$ is an independent multiplier sequence
satisfying Assumption~\ref{assumption:w}. Let $\EE_{\rm joint}$ denote expectation
under $\PP_{\rm joint}$. The $(U_0,X_1,\ldots,X_n,
w_1^{\bs},\ldots,w_n^{\bs})$ marginal agrees with the
data--bootstrap law of Section~\ref{sec:main_bs}; in particular, $\PP_X$ is its
$(U_0,X_1,\ldots,X_n)$ marginal, and conditional on $\cF_X$ the
multiplier sequence has probability law $\PP^*$.

Define the truncated bootstrap iterates
\[
\overline U_0^{\bs}=U_0, \qquad \overline U_k^{\bs} = \qr{\overline U_{k-1}^{\bs} + \eta w_k^{\bs} \overline X_k\overline X_k^\top \overline U_{k-1}^{\bs}},
\qquad k\in[n],
\]
and let $\oocT_k^{\bs}$ denote the corresponding tangent matrix.

Define the truncated bootstrap filtration
\[
\overline{\mathcal F}_i := \sigma\left( U_0, \overline X_1,w_1^{\bs}, \ldots, \overline X_i,w_i^{\bs} \right),
\qquad
\ooEE_i[\cdot] := \mathbb E_{\rm joint} \left[ \cdot\mid\overline{\mathcal F}_i \right].
\]
We work under $\PP_{\rm joint}$ throughout the joint-law analysis
below and pass to conditional statements given $\cF_X$ only after
establishing the corresponding joint-law bounds.

For brevity, write
\[
\oocT^{\bs} = \oocT_i^{\bs}, \qquad \oocT_+^{\bs} = \oocT_{i+1}^{\bs}, \qquad \overline Y_+ = \overline Y_{i+1}, \qquad \overline Z_+ = \overline Z_{i+1}, \qquad w_+^{\bs} = w_{i+1}^{\bs}.
\]

Similar to \eqref{eq:dT_split_trunc}, the Sherman--Morrison formula gives
\begin{equation}
\oocT_+^{\bs} - \oocT^{\bs} = \frac{ \eta w_+^{\bs}\overline F_+^{\bs} }{ 1+\eta w_+^{\bs}\overline\zeta_+^{\bs} } = \eta w_+^{\bs}\ooF_+^{\bs} - (\eta w_+^{\bs})^2 \frac{ \oozeta_+^{\bs}\ooF_+^{\bs} }{ 1+\eta w_+^{\bs}\oozeta_+^{\bs}}.
\label{eq:dT_split_trunc_bs}
\end{equation}
where
\begin{align*}
  \overline F_+^{\bs} := \left( -\oocT^{\bs}\overline Y_+ +\overline Z_+ \right) \left( \overline Y_+^\top +\overline Z_+^\top\oocT^{\bs} \right) \in \RR^{(d-r)\times r}, \qquad 
  \overline\zeta_+^{\bs} := \left( \overline Y_+^\top +\overline Z_+^\top\oocT^{\bs} \right)\overline Y_+ \in \RR.
\end{align*}

Let $\overline{\mathcal L}$ be the drift operator of the truncated process defined in \eqref{eqn:linearization_trunc1}. Since $w_+^{\bs}$ is independent of $\overline{\mathcal F}_i\vee\sigma(\overline X_{i+1})$ and $\mathbb Ew_+^{\bs}=1$, we have $\ooEE_i [w_+^{\bs}\ooF_+^{\bs}] = \ooEE_i [\ooF_+^{\bs}] = -\oocL \oocT^{\bs}$, and the following decomposition similar to \eqref{eqn:linearization_trunc1}:
\begin{align}
    \oocT_+^{\bs} = (\cI-\eta \oocL) \oocT^{\bs} + \ooD_+^{\bs} + \ooR_+^{\bs} \label{eqn:linearization_trunc_bs1}
\end{align}
where
\begin{align*}
    \ooD_+^{\bs}&:= (\oocT_+^{\bs} - \oocT^{\bs}) - \ooEE_i[\oocT_+^{\bs} - \oocT^{\bs}] = \ooD_+^{\bs,(1)} + \ooD_+^{\bs,(2)},\\ 
    \ooD_+^{\bs,(1)} &:= \eta w_+^{\bs}\ooF_+^{\bs} - \ooEE_i[\eta w_+^{\bs}\ooF_+^{\bs}],\\
    \ooD_+^{\bs,(2)} &:= -(\eta w_+^{\bs})^2 \frac{ \oozeta_+^{\bs}\ooF_+^{\bs} }{ 1+\eta w_+^{\bs}\oozeta_+^{\bs}} + \ooEE_i\left[ (\eta w_+^{\bs})^2 \frac{ \oozeta_+^{\bs}\ooF_+^{\bs} }{ 1+\eta w_+^{\bs}\oozeta_+^{\bs}} \right]\\
    \ooR_+^{\bs}&:=- \ooEE_i  \sbr{(\eta w_+^{\bs})^2 \frac{\zeta_+^{\bs} F_+^{\bs}}{1+\eta w_+^{\bs} \zeta_+^{\bs}}}
\end{align*}
As in \eqref{eqn:linearization_trunc2} of Appendix~\ref{sec:oja_subG}, absorb the covariance perturbation into
\[
\overline R_+^{\bs,H} := \overline R_+^{\bs} + \eta \left( \mathcal L-\overline{\mathcal L} \right)\oocT^{\bs}.
\]
Then
\begin{equation}
\oocT_+^{\bs} = (\mathcal I-\eta\mathcal L)\oocT^{\bs} + \overline D_+^{\bs} + \overline R_+^{\bs,H}.
\label{eqn:linearization_trunc_bs2}
\end{equation}

\begin{lemma}[One-step bounds for the truncated bootstrap process]
\label{lem:one_step_bound-truncation_bs}
Assume Assumptions~\ref{assumption:X} and~\ref{assumption:w} hold. Suppose $\oocT^{\bs}$ is well-defined with $\|\oocT^{\bs}\|=\tau$. If
\begin{equation}
  C_w C_{\psi,n}^2 \teta\kappa_\ast r(1+\tau\nu_2) \le \frac12, \label{eqn:bs_one_step_condition}
\end{equation}
then the following hold, where the implicit constants may depend only on $C_w$.
\begin{enumerate}[label=\textup{(\Roman*)}, leftmargin=*]
\item The full centered noise satisfies $\|\overline D_+^{\bs}\| \lesssim C_{\psi,n}^2 \teta\kappa_\ast r (1+\tau\nu_2)(\tau+\nu_2)$ almost surely, and
\[
\max \cbr[\Big]{\norm[\big]{\overline{\mathbb E}_i \overline D_+^{\bs} (\overline D_+^{\bs})^\top}^{1/2},
\norm[\big]{\overline{\mathbb E}_i (\overline D_+^{\bs})^\top \overline D_+^{\bs}}^{1/2}}
\lesssim C_\psi^2 \teta\kappa_\ast\sqrt r (1+\tau\nu_\infty)(\tau+\bnu).
\]

\item The centered second-order noise satisfies $\|\overline D_+^{\bs,(2)}\| \lesssim C_{\psi,n}^4 \teta^2\kappa_\ast^2r^2 (1+\tau\nu_2)^2(\tau+\nu_2)$ almost surely, and
\[
\max\cbr[\Big]{\norm[\big]{\overline{\mathbb E}_i \overline D_+^{\bs,(2)} \overline D_+^{\bs,(2)\top}}^{1/2}, \norm[\big]{\overline{\mathbb E}_i \overline D_+^{\bs,(2) \top} \overline D_+^{\bs,(2)}}^{1/2}}
\lesssim C_\psi^4 \teta^2\kappa_\ast^2r^{3/2} (1+\tau\nu_\infty)^2(\tau+\overline\nu).
\]

\item The modified drift remainder satisfies
\[ \|\overline R_+^{\bs,H}\| \lesssim C_\psi^4 \teta^2\kappa_\ast^2r^{3/2} (1+\tau\nu_\infty)^2(\tau+\nu_\infty).
\]
\end{enumerate}
\end{lemma}

\begin{proof}
  The proof follows the same argument as that of Lemma~\ref{lem:one_step_bound-truncation}, after enlarging the constants by quantities depending only on $C_w$.
\end{proof}

Lemma~\ref{lem:one_step_bound-truncation_bs} provides the same inputs as Lemma~\ref{lem:one_step_bound-truncation}: the linear part of the recursion is $\mathcal I-\eta\mathcal L$, the centered increments satisfy the same almost-sure and conditional-variance bounds, and the modified drift
remainder has the same order. Hence, after enlarging constants by factors depending only on $C_w$,
the convergence argument of Appendix~\ref{sec:oja_subG} applies verbatim to the truncated bootstrap process.

\begin{lemma}[Convergence of the truncated bootstrap process]
\label{lem:convergence_rate_optimal-trunc_bs}
Suppose Assumptions~\ref{assumption:X}, \ref{assumption:w}, \ref{assumption:d}, and~\ref{assumption:init} hold. Assume
\begin{align*}
  \nu_\infty\ge c_\nu n^{-C_\nu} \qquad \text{for constants } C_\nu \geq 0, c_\nu >0.
\end{align*}
Let $\varepsilon\in[0,1/2)$ and set
\[
\eta = \frac{c_\eta\log n}{n\Delta_{\min}}, \quad\text{where}\quad  c_\eta\ge c_\eta^\ast := \left( 1-\varepsilon+\frac{3C_\nu}{2} \right) \left( 1+ \frac{\log^{-3}n+\log2}{\varepsilon\log n+\log\log n} \right).
\]
If
\[
n^{1-2\varepsilon} \ge C_\bs c_\eta\kappa_\ast^2r^3 (\log n)^4 \frac{\log(1/\delta_p)}{\delta_p^2} \max(1,\nu_\infty d)
\]
for a sufficiently large constant $C_\bs>0$ depending only on $C_\psi,C_\nu,c_\nu,C_d$, and $C_w$, then, under $\mathbb P_{\rm joint}$, with probability at least $1-n^{-19}$,
\begin{align*}
  \norm{\oocT_k^{\bs}} \leq 2\ootau_*^{\bs} \asymp  {\kappa_*} \bnu \sqrt{\teta r \log n} = \frac{\lambda_1}{\Delta_{\min}} \sqrt{c_\eta r}  \bnu  \frac{\log n}{\sqrt{n}}, \qquad \text{for all } \floor{\tfrac{c_\eta^*}{c_\eta} n} \leq k \leq n,
  % \label{eqn:convergence_rate_Tn}
\end{align*}
where
\begin{align*}
  \ootau_*^{\bs} := \ooC_\bs \max\rbr[\big]{ C_{\psi}^2 \kappa_*\bnu \sqrt{\teta r\log n}, C_{\psi}^2 \kappa_* \nu_2 \teta r (\log n)^2  }
\end{align*}
for some constant $\ooC_\bs>0$ (depending only on $C_d$).
\end{lemma}

\begin{proof}
After enlarging the constants by quantities depending only on $C_w$, the stopped-process, matrix-Freedman, epoch-chaining, and iteration-budget arguments in Lemmas~\ref{lem:1epoch_bound}--\ref{lem:convergence_rate_optimal} apply verbatim to $\{\oocT_k^{\bs}\}_{k\le n}$.
\end{proof}

\paragraph{Linearization of the truncated bootstrap process.}
We next establish the bootstrap counterpart of Lemma~\ref{lem:linearization-truncation}. By the definition of $\ooD_+^{(1)}$ and $\ooD_+^{\bs,(1)}$ in \eqref{eqn:linearization_trunc1} and \eqref{eqn:linearization_trunc_bs1}, we have
\begin{equation}
  \begin{aligned}
    \ooD_+^{(1)} &= \eta (\ooF_+ - \ooEE_i \ooF_+),\\
    \ooD_+^{\bs,(1)} &= \eta w_+^{\bs} (\ooF_+^{\bs} - \ooEE_i \ooF_+^{\bs} ) + \eta (w_+^{\bs}-1) \ooEE_i\ooF_+^{\bs}
  \end{aligned}
  \label{eqn:bs_extra_term}
\end{equation}
Thus, the first term in $\overline D_+^{\bs,(1)}$ has the same
centered structure as $\overline D_+^{(1)}$, with the bounded multiplier
$w_+^{\bs}$ affecting only the constants in the corresponding
concentration bounds.  The second term is new: since
$\EE(w_+^{\bs}-1)=0$, it gives a conditionally centered multiplier
martingale.  As shown in the proof below, its accumulated contribution
is of the same order as the existing linearization remainder and can be
absorbed into it.  Consequently, the linearization argument of
Lemma~\ref{lem:linearization-truncation} carries over to the bootstrap
process up to constant factors.

\begin{lemma}[Linearization of the truncated bootstrap process]
\label{lem:linearization-bs-truncation}
Under the assumptions of
Lemma~\ref{lem:convergence_rate_optimal-trunc_bs}, suppose additionally that
$c_\eta \ge c_\eta^\ast + 1/2$, and let $N_*=\floor{n c_\eta^*/c_\eta}$. 
Then, with probability at least $1-O(n^{-18})$,
\[
\norm[\Big]{\overline{\cT}_n^{\bs} - \eta\sum_{i=1}^n w_i^{\bs} (\cI-\eta\cL)^{n-i} (\overline Z_i\overline Y_i^\top)} \lesssim \ootau_\ast^{\bs} \sbr[\big]{C_\psi^2\kappa_\ast \sqrt{r\widetilde\eta\log n} +\ootau_*^\bs} + e^{-\widetilde\eta(n-N_\ast)} \kappa_\ast\nu_\infty\|H\|.
\]
In particular, by \eqref{eqn:H_control},
\[
\norm[\Big]{\overline{\cT}_n^{\bs} - \eta\sum_{i=1}^n w_i^{\bs}(\cI-\eta\cL)^{n-i}(\overline Z_i\overline Y_i^\top)} \lesssim \ootau_\ast^{\bs} \sbr[\big]{C_\psi^2\kappa_\ast \sqrt{r\widetilde\eta\log n} +\ootau_\ast^{\bs}}.
\]
\end{lemma}

\begin{proof}
  See Appendix~\ref{sec:prf_linearization-bs-truncation}.
\end{proof}

\paragraph{Transfer to the original process under $\PP_{\rm joint}$.}
Use the same initialization and multiplier sequence for the original
and truncated bootstrap processes.  On the truncation coupling event
\[
    \mathcal E_{\rm tr}
    :=
    \{
        X_i=\overline X_i,\ i\in[n]
    \},
\]
we have simultaneously
\[
    U_k=\overline U_k,
    \qquad
    U_k^\bs=\overline U_k^\bs,
    \qquad k\in[n].
\]
Moreover,
\[
    \PP_{\rm joint}(\mathcal E_{\rm tr}^c)
    \le n\delta_{\rm tr}
    \le n^{-19}.
\]

The preceding results and the truncation coupling give the following
joint-law stability and linearization event.

\begin{corollary}[Joint stability and linearization]
\label{cor:joint_bs_linearization}
Under the assumptions of
Lemma~\ref{lem:linearization-bs-truncation}, define
\[
    \tau_\ast^\dagger 
    o:= \tau_\ast\vee\tau_\ast^\bs \lesssim \tau_\ast
\]
and linearization remainder
\[
    r_n^{\rm lin} := C_{\rm lin}(C_w)\, \tau_\ast^\dagger \sbr[\big]{C_\psi^2\kappa_\ast \sqrt{r\widetilde\eta\log n} +\tau_\ast^\dagger}.
\]
Then there exists an event $\mathcal A_n^{\rm lin}$ on the joint
data--bootstrap probability space such that
\[
    \PP_{\rm joint}(\mathcal A_n^{\rm lin})
    \lesssim n^{-18},
\]
and on $(\mathcal A_n^{\rm lin})^c$,
\[
    \|\cT_n\| \le 2\tau_\ast^\dagger, \qquad \|\cT_n^\bs\| \le 2\tau_\ast^\dagger,
\]
and
\[
\norm[\Big]{\cT_n^\bs-\cT_n - \eta\sum_{i=1}^n (w_i^\bs-1)(\cI-\eta\cL)^{n-i}(Z_iY_i^\top)}
\le r_n^{\rm lin}.
\]
\end{corollary}

\begin{proof}
Combine Lemma~\ref{lem:linearization-truncation}, Lemma~\ref{lem:linearization-bs-truncation}, and the truncation coupling~\eqref{eqn:trunc_coupling}.
\end{proof}

\paragraph{Transfer to the conditional bootstrap law.}
Recall from Section~\ref{sec:main_bs} that $\cF_X=\sigma(U_0,X_1,...,X_n)$, that $\PP_X$ denotes the
law of $(U_0,X_1,\ldots,X_n)$, and that $\PP^*$ and $\EE^*$ denote the conditional bootstrap probability and expectation given $\cF_X$. Under the joint construction above, equivalently,
\[
    \PP^*(\cdot)=\PP_{\rm joint}(\;\cdot\mid\cF_X),
    \qquad
    \EE^*[\cdot]=\EE_{\rm joint}[\;\cdot\mid\cF_X].
\]

For any event $A$ on the joint probability space and any $a>0$,
the Markov--Fubini argument gives
\begin{equation}
  \PP_X\!\left\{\PP^\ast(\cA)>a\right\} \le \frac{\PP_{\rm joint}(\cA)}{a}. \label{eqn:markov_fubini}
\end{equation}
Indeed, $\EE_X[\PP^\ast(\cA)] = \PP_{\rm joint}(\cA)$, and the claim follows from Markov's inequality.

\begin{corollary}[Conditional bootstrap stability and linearization]
\label{cor:conditional_bs_linearization}
Under the assumptions of Corollary~\ref{cor:joint_bs_linearization}, there exists an $\cF^X$-measurable event $\cG_n^{\rm lin}$ satisfying
\[
    \PP_X(\cG_n^{\rm lin}) \ge 1-O(n^{-9}),
\]
such that, on $\cG_n^{\rm lin}$,
\[
    \|\cT_n\|\le 2\tau_\ast^\dagger, \qquad \PP^\ast\!\left\{ \|\cT_n^{\bs}\|>2\tau_\ast^\dagger \right\}\le n^{-9},
\]
and
\[
    \PP^\ast\! \cbr[\Big]{\norm[\Big]{\cT_n^{\bs}-\cT_n - \eta\sum_{i=1}^n (w_i^{\bs}-1)(\cI-\eta\cL)^{n-i}(Z_iY_i^\top)} > r_n^{\rm lin}} \le n^{-9}.
\]
\end{corollary}

\begin{proof}
Let $\cA_n^{\rm lin}$ be the event in
Corollary~\ref{cor:joint_bs_linearization}, and define
\[
    \cG_n^{\rm lin} := \left\{ \PP_{\rm joint} (\cA_n^{\rm lin}\mid\cF_X) \le n^{-9} \right\}.
\]
Since $\PP_{\rm joint}(\cA_n^{\rm lin}) \lesssim n^{-18}$, \eqref{eqn:markov_fubini} gives
\[
    \PP_X\bigl((\cG_n^{\rm lin})^c\bigr)
    \lesssim n^{-9}.
\]
The claimed conditional bounds follow directly from the
definition of $\cA_n^{\rm lin}$.

Finally, since $\|\cT_n\|$ is $\cF^X$-measurable, if $\|\cT_n\|>2\tau_\ast^\dagger$, then $\PP^\ast(\cA_n^{\rm lin})=1$. Hence $\|\cT_n\|\le2\tau_\ast^\dagger$ on~$\cG_n^{\rm lin}$.
\end{proof}

% ------------------------------------------------

\subsection{Proof of Theorem~\ref{thm:bs_highD_clt}}
\label{sec:prf_bs_highD_clt}

Recall $\cF_X$, $\PP_X$, $\PP^*$, and $\EE^*$ from Section~\ref{sec:main_bs}. Write the Kolmogorov distance between conditional laws (on $\cF_X$) of two random variables $V_1,V_2$ as
\[
  d_K^\ast(V_1,V_2) := \sup_{t\in\RR} \left| \PP^\ast\{V_1\le t\} -\PP^\ast\{V_2\le t\}\right|.
\]
Note that when the law of $V$ is independent of $\cF_X$, $\PP^*\cbr{V \leq t} = \PP\{V \leq t\}$.

Under the assumptions of Theorem~\ref{thm:bs_highD_clt}, the conditions of Proposition~\ref{prop:linearization} hold with $\delta_p\to0$. As in the proof of
Theorem~\ref{thm:convergence_rate_optimal}, condition~\eqref{eqn:cond_sample_sz} implies Assumption~\ref{assumption:d}, while the Haar initialization satisfies Assumption~\ref{assumption:init} with probability $1-o(1)$. Thus, after intersecting with this initialization event, Corollary~\ref{cor:conditional_bs_linearization} yields an $\cF_X$-measurable event $\cG_n^{\rm lin}$ such that $\PP_X(\cG_n^{\rm lin})\to 1$ and, on $\cG_n^{\rm lin}$,
\[
  \|\cT_n\|\le 2\tau_\ast^\dagger, \qquad \PP^\ast\{\|\cT_n^{\bs}\|>2\tau_\ast^\dagger\} \le n^{-9},
\]
and
\[
  \PP^\ast\! \cbr[\Big]{\norm[\Big]{\cT_n^{\bs}-\cT_n - \eta\sum_{i=1}^n (w_i^{\bs}-1)(\cI-\eta\cL)^{n-i}(Z_iY_i^\top)} > r_n^{\rm lin}} \le n^{-9}.
\]

For $i\in[n]$, define the $\cF_X$-measurable matrices
\[
  M_i:=(\cI-\eta\cL)^{n-i}(Z_iY_i^\top),\qquad \xi_i^{\bs} :=  \sqrt{\eta}\,(w_i^{\bs}-1)M_i, \qquad \Xi_n^{\bs} := \sum_{i=1}^n \xi_i^{\bs}.
\]
Also define the conditional covariance operator
\[
\hGamma_n := \eta\sum_{i=1}^n M_i\otimes M_i.
\]
Since the multiplier weights are independent of $\cF_X$ with $\EE(w_i^{\bs}-1)=0$ and $\EE(w_i^{\bs}-1)^2=1$, conditionally on $\cF_X$ the $\xi_i^{\bs}$ are independent and centered, with
\[
  \cov^\ast(\Xi_n^{\bs})=\hGamma_n, \qquad \EE_X\hGamma_n = \Gamma_n = \eta \sum_{i=1}^n (\cI - \eta \cL)^{n-i} \fC (\cI - \eta \cL)^{n-i}.
\]
Let $\mathsf{G} \sim N(0,\Gamma)$, where $\Gamma$ is defined in
Appendix~\ref{sec:clt_setup}. It suffices to prove
\[
  d_K^\ast\left( \eta^{-1}\|\sin\Theta(U_n^{\bs},U_n)\|_{\rm F}^2, \|\mathsf{G}\|_{\rm F}^2 \right)
\xrightarrow{\PP_X}0,
\]
because Theorem~\ref{thm:highD_clt} and the triangle inequality then give the
claim. Similar to the proof of Theorem~\ref{thm:highD_clt}, the proof proceeds in three steps, progressing right-to-left through the chain
\begin{align*}
    \eta^{-1}\Fnorm{\sin \Theta(U_n^\bs, U_n)}^2 \longapprox{Step 3.} \eta^{-1}\Fnorm{\cT_n^{\bs}-\cT_n}^2  \longapprox{Step 2.} \Fnorm[\big]{\Xi_n^{\bs}}^2 \longapprox{Step 1.} \Fnorm{\mathsf{G}}^2.
\end{align*}

\paragraph{Step 1: Conditional Gaussian approximation of $\|\Xi_n^{\bs}\|_{\rm F}^2$.}

We first record an empirical moment bound. For each fixed
$p\in\{2,3,4,6\}$,
\begin{equation}
  S_p:=  \sum_{i=1}^n\|M_i\|_{\rm F}^p = O_{\PP_X}\left(\frac{\lambda_\times^p}{\teta}\right). \label{eqn:bs_highD_clt_Sp}
\end{equation}
Indeed, by Lemma~\ref{lem:basic_ineq}, $\|M_i\|_{\rm F} \le (1-\widetilde\eta)^{n-i}\|Z_iY_i^\top\|_{\rm F}$, and therefore Lemma~\ref{lem:Gamma_properties}(vi) and geometric summation give
\[
  \EE_X\sum_{i=1}^n\|M_i\|_{\rm F}^p \lesssim \lambda_\times^p \sum_{i=1}^n(1-\widetilde\eta)^{p(n-i)} \lesssim \frac{\lambda_\times^p}{\widetilde\eta}.
\]

We next compare $\hGamma_n$ with $\Gamma$. Since $\|M\otimes M\|_{\rm HS}=\|M\|_{\rm F}^2$, independence gives
\[
\EE_X\|\hGamma_n-\Gamma_n\|_{\rm HS}^2 \lesssim \eta^2\sum_{i=1}^n\EE\|M_i\|_{\rm F}^4 \lesssim \frac{\widetilde\eta\,\lambda_\times^4} {\Delta_{\min}^2},
\]
and similarly since $\tr (M \otimes M) = \|M\|_{\rm F}^2$, we have
\[
  \EE_X \abs[\big]{\tr(\hGamma_n-\Gamma_n)}^2 \lesssim \frac{\widetilde\eta\,\lambda_\times^4} {\Delta_{\min}^2}.
\]
Since Lemma~\ref{lem:Gamma_properties}(ii) gives
\begin{equation}
  \|\Gamma\|_{\rm HS} \gtrsim \frac{\|\fC\|_{\rm HS}}{\Delta_{\max}} = \frac{\lambda_\times^2} {\Delta_{\max}\sqrt{d_{\fC}}}, \label{eqn:bs_highD_clt_Gamma_lower}
\end{equation}
we obtain
\[
  \frac{\|\hGamma_n-\Gamma_n\|_{\rm HS}} {\|\Gamma\|_{\rm HS}} = O_{\PP_X}\left( \kappa_\Delta\sqrt{\widetilde\eta\,d_{\fC}} \right),
\]
and the same bound holds with $\|\hGamma_n-\Gamma_n\|_{\rm HS}$ replaced by $|\tr(\hGamma_n-\Gamma_n)|$.

Let $\varepsilon_n := (1-\widetilde\eta)^{2n} + \widetilde\eta \kappa_\Delta^2$.
By Lemma~\ref{lem:Gamma_properties}(iv),
\[
  \|\Gamma_n-\Gamma\|_{\rm HS} \le \varepsilon_n\|\Gamma\|_{\rm HS}, \qquad |\tr(\Gamma_n-\Gamma)| \le \varepsilon_n\tr\Gamma.
\]
Moreover, Lemma~\ref{lem:Gamma_properties}(ii) and (vi) imply
\[
\frac{\tr\Gamma}{\|\Gamma\|_{\rm HS}}
\lesssim
\kappa_\Delta\sqrt{d_{\fC}}.
\]
Consequently, under condition~\eqref{eqn:highD_clt_cond},
\begin{equation}
  \frac{\|\hGamma_n-\Gamma\|_{\rm HS}} {\|\Gamma\|_{\rm HS}} =o_{\PP_X}(1), \qquad \frac{|\tr(\hGamma_n-\Gamma)|} {\|\Gamma\|_{\rm HS}} =o_{\PP_X}(1), \label{eqn:bs_highD_clt_negligible}
\end{equation}
which combined with \eqref{eqn:bs_highD_clt_Gamma_lower} also implies
\begin{equation}
  \|\hGamma\|_{\rm HS} \gtrsim_{\PP_X} \frac{\lambda_\times^2} {\Delta_{\max}\sqrt{d_{\fC}}}, \label{eqn:bs_highD_clt_hGamma_lower}
\end{equation}

We now apply Lemma~\ref{lem:clt_biometrika19} conditionally on $\cF_X$, with
$q=\beta=3$. Let $L_3^{\xi,{\bs}},L_3^{g,{\bs}}, K_3^{\bs}$ and $J_n^{\bs}$
denote the corresponding conditional quantities. More precisely,
\begin{align*}
  L_q^{\xi,\bs} &= \sum_{i=1}^n \frac{\EE^* \Finner{\xi_i^{\bs}}{\hGamma_n \xi_i^{\bs}}^{q/2}}{\HS[\big]{\hGamma_n}^q} + \frac{n^2}{\binom{n}{2}} \sum_{1 \leq i < j \leq n} \frac{\EE^* \abs{\Finner{\xi_i^{\bs}}{\xi_j^{\bs}}}^q}{\HS[\big]{\hGamma_n}^q},\\
  L_q^{\fg,\bs} &= \sum_{i=1}^n \frac{\EE^* \Finner{\fg_i^{\bs}}{\hGamma_n \fg_i^{\bs}}^{q/2}}{\HS[\big]{\hGamma_n}^q}, \quad \text{where} \quad \fg_i^{\bs} \sim N(0, \eta M_i \otimes M_i),\\
  (K_\beta^{\bs})^{\beta} &= \frac{n^\beta}{n} \sum_{i=1}^n \EE^* \abs{\frac{\Finner{\xi_i^{\bs}}{\xi_i^{\bs}} - \EE^* \Finner{\xi_i^{\bs}}{\xi_i^{\bs}}}{\HS[\big]{\hGamma_n}}}^{\beta}\\
  J_n^{\bs} &= \frac{\sum_{i=1}^n \fv_i^{\bs}}{\HS[\big]{\hGamma_n}^2}, \qquad\text{where}\quad \fv_i=\Var^* (\Fnorm{\xi_i^{\bs}}^2)
\end{align*}

\begin{lemma} \label{lem:bs_highD_clt_biometrika19}
  Under \eqref{eqn:bs_highD_clt_Sp}, \eqref{eqn:bs_highD_clt_hGamma_lower} and Assumption~\ref{assumption:w},
  \begin{align*}
    L_3^{\xi,{\bs}} &= O_{\PP_X}\left( \teta^{1/2}\kappa_\Delta^{3/2}d_{\fC}^{3/4} + \teta\kappa_\Delta^3 d_{\fC}^{3/2} \right),\\
    L_3^{g,{\bs}} &= O_{\PP_X}\left( \widetilde\eta^{1/2} \kappa_\Delta^{3/2}d_{\fC}^{3/4} \right),\\
    n^{-2}(K_3^{\bs})^3 &= O_{\PP_X}\left( \widetilde\eta^2 \kappa_\Delta^3d_{\fC}^{3/2} \right),\\
    J_n^{\bs} &= O_{\PP_X}\left( \widetilde\eta\kappa_\Delta^2d_{\fC} \right).
  \end{align*}
\end{lemma}
\begin{proof}
  See Appendix~\ref{sec:prf_bs_highD_clt_biometrika19}.
\end{proof}

All four quantities converge to zero under condition~\eqref{eqn:highD_clt_cond}(i).
Hence the conditional version of Lemma~\ref{lem:clt_biometrika19} gives
\begin{equation}
  d_K^\ast\left( \|\Xi_n^{\bs}\|_{\rm F}^2,\|\mathsf{G}_{\hGamma_n}\|_{\rm F}^2 \right) \xrightarrow{\PP_X}0, \label{eqn:bs_highD_clt_step1_1}
\end{equation}
where, conditionally on the data, $\mathsf{G}_{\hGamma_n}\sim N(0,\hGamma_n)$.

It remains to replace $\hGamma_n$ by $\Gamma$. Set
\[
a_n:= \frac{|\tr(\hGamma_n-\Gamma)|} {\|\Gamma\|_{\rm HS}}, \qquad
b_n:= \frac{\|\hGamma_n-\Gamma\|_{\rm HS}} {\|\Gamma\|_{\rm HS}}.
\]
By \eqref{eqn:bs_highD_clt_negligible}, $a_n=o_{\PP_X}(1)$ and $b_n=o_{\PP_X}(1)$. On $\{b_n>0\}$, apply Lemma~\ref{lem:chi2_mixture_comparison} conditionally with
\[
t_n:=\|\Gamma\|_{\rm HS}\sqrt{b_n}.
\]
Then
\[
\frac{|\tr(\hGamma_n-\Gamma)|+t_n} {\|\Gamma\|_{\rm HS}} = a_n+\sqrt{b_n} =o_{\PP_X}(1),
\]
while
\[
\frac{t_n^2} {\|\hGamma_n-\Gamma\|_{\rm HS}^2} = \frac1{b_n}, \qquad \frac{t_n} {\|\hGamma_n-\Gamma\|_{\rm op}} \ge \frac1{\sqrt{b_n}}.
\]
On $\{b_n=0\}$ the two covariance operators coincide.
Therefore,
\[
d_K^\ast\left( \|\mathsf{G}_{\hGamma_n}\|_{\rm F}^2,\|\mathsf{G}\|_{\rm F}^2 \right) \xrightarrow{\PP_X}0.
\]
Together with \eqref{eqn:bs_highD_clt_step1_1}, this gives
\begin{equation}
  d_K^\ast\left( \|\Xi_n^{\bs}\|_{\rm F}^2,\|\mathsf{G}\|_{\rm F}^2 \right) \xrightarrow{\PP_X}0. \label{eqn:bs_highD_clt_step1}
\end{equation}

\paragraph{Step 2: Conditional Gaussian approximation of $\eta^{-1}\|\cT_n^{\bs}-\cT_n\|_{\rm F}^2$.} 

On $\cG_n^{\rm lin}$, Corollary~\ref{cor:conditional_bs_linearization} gives
\[
\PP^\ast \left\{ \|\mathsf{E}_n^{\bs}\| > \eta^{-1/2}r_n^{\rm lin} \right\} \le n^{-9}.
\]

Applying Lemma~\ref{lem:g_approx_additive_remainder} under $\PP^\ast$ with
\begin{align*}
  \mathsf{A}_n^\bs:= \Xi_n^\bs, \qquad  \mathsf{B}_n^\bs := \eta^{-1/2}(\cT_n^{\bs}-\cT_n), \qquad \mathsf{E}_n^{\bs} := \mathsf{B}_n^{\bs}-\mathsf{A}_n^{\bs}, \qquad \mathsf{G}_n := \mathsf{G}
\end{align*}
and
\begin{align*}
  e_n := \sqrt{r} \eta^{-1/2} r_n^{\rm lin}, \qquad L_n^2 := C_L(\tr\Gamma)\log n, \qquad h_n:= 2L_n e_n + e_n^2,
\end{align*}
where $C_L > 0$ is a sufficiently large constant so that $\PP\{\|\mathsf{G}\|_{\rm F}>L_n\}\le n^{-9}$ (see Lemma~\ref{lem:chi2_concentration}), gives
\begin{align*}
  d^*_K \rbr{\|\eta^{-1/2}(\cT_n^{\bs}-\cT_n)\|_{\rm F}^2, \|\mathsf{G}\|_{\rm F}^2} &\leq 2 d^*_K \rbr{\|\Xi_n^{\bs}\|_{\rm F}^2, \|\mathsf{G}\|_{\rm F}^2} + w_n(h_n) + \PP^* \{\|\mathsf{E}_n^{\bs}\|_{\rm F}>e_n\}\\
  &\qquad + \PP^* \{\|\mathsf{G}\|_{\rm F}>L_n\}.
\end{align*}
We bound each term on the right-hand side. 
\begin{itemize}[leftmargin=1.5em]
  \item \eqref{eqn:bs_highD_clt_step1} gives that $d^*_K \rbr{\|\Xi_n^{\bs}\|_{\rm F}^2, \|\mathsf{G}\|_{\rm F}^2} \xrightarrow{\PP_X}0$.
  
  \item Lemma~\ref{lem:chi2_anti_concentration} gives
  \[
    w_n(h_n) = \sup_{t\in\RR} \PP\{t<\|\mathsf{G}\|_{\rm F}^2\le t+h_n\} \lesssim \sqrt{\frac{h_n}{\|\Gamma\|_{\rm HS}}} = \sqrt{\frac{e_n (2L_n+e_n)}{\|\Gamma\|_{\rm HS}}}.
  \]
  We will show that $e_nL_n/\|\Gamma\|_{\rm HS} \to 0$ and $e_n^2/\|\Gamma\|_{\rm HS} \to 0$ in the sequel, which implies that
  \begin{align*}
    w_n(h_n) = o(1).
  \end{align*}

  Using the definition of $r_n^{\rm lin}$ and $\tau_\ast^\dagger\lesssim\tau_\ast$, we have
  \begin{equation}
    e_n \lesssim \Delta_{\min}^{1/2}\kappa_\ast^2 r^{3/2}\sqrt{\widetilde\eta}\, (\log n) \bar\nu(1+\bar\nu). \label{eqn:bs_highD_clt_step2_e}
  \end{equation}
  By Lemma~\ref{lem:Gamma_properties}(ii) and (vi),
  \[
    \|\Gamma\|_{\rm HS} \gtrsim \frac{\|\fC\|_{\rm HS}}{\Delta_{\max}}, \qquad \tr\Gamma \lesssim \frac{\lambda_\times^2}{\Delta_{\min}}.
  \]
  Together with \eqref{eqn:bs_highD_clt_step2_e}, $\Delta_{\min}\Delta_{\max}\le\lambda_1^2$, $\lambda_1^2\bar\nu^2\le\lambda_\times^2$, and $\Delta_{\max}\bar\nu\le\lambda_\times$, this gives
  \[
    \frac{e_n^2}{\|\Gamma\|_{\rm HS}} \lesssim \widetilde\eta\kappa_\ast^4r^3 (\log n)^2d_{\fC}^{1/2} (1+\bar\nu^2) =o(1),
  \]
  and
  \[
    \frac{e_nL_n}{\|\Gamma\|_{\rm HS}} \lesssim \kappa_\ast^2r^{3/2} \sqrt{\widetilde\eta}\, (\log n)^{3/2}d_{\fC}^{1/2} (1+\bar\nu) =o(1),
  \]
  where the last two relations follow from conditions~\eqref{eqn:highD_clt_cond}(i)
  and~\eqref{eqn:highD_clt_cond}(ii).

  \item Since $E_n^{\bs}$ has at most $r$ columns,
  \begin{align*}
    \PP^* \{\|\mathsf{E}_n^{\bs}\|_{\rm F}>e_n\} \le \PP^* \{\|\mathsf{E}_n^{\bs}\|_{\rm op}> \eta^{-1/2} r_n^{\rm lin} \} \le n^{-9}.
  \end{align*}
  where the last inequality follows from Corollary~\ref{cor:conditional_bs_linearization}.

  \item $\PP^* \{\|\mathsf{G}\|_{\rm F}>L_n\} \le n^{-9}$ by the choice of $C_L$ and Lemma~\ref{lem:chi2_concentration}.

\end{itemize}

Therefore,
\begin{align}
  d_K^\ast\left( \eta^{-1}\|\cT_n^{\bs}-\cT_n\|_{\rm F}^2, \|\mathsf{G}\|_{\rm F}^2 \right) \xrightarrow{\PP_X}0. \label{eqn:bs_highD_clt_step2}
\end{align}

\paragraph{Step 3: Conditional Gaussian approximation of
$\eta^{-1}\|\sin\Theta(U_n^{\bs},U_n)\|_F^2$.} 

We first gather some useful facts.
\begin{itemize}[leftmargin=1.5em]
     \item Define $\cE_{\tau,n} := \{ \|\cT_n\|\le 2\tau_\ast^\dagger,\; \|\cT_n^{\bs}\|\le 2\tau_\ast^\dagger\}$.
     On $\cG_n^{\rm lin}$, Corollary~\ref{cor:conditional_bs_linearization} gives $\PP^\ast(\cE_{\tau,n}^c)\le n^{-9}$.

     \item $\tau_\ast^\dagger$ can be taken uniformly small. Indeed, $\tau_\ast^\dagger\lesssim\tau_\ast$, and the same calculation as in \eqref{eqn:tau_leq1} gives
     \[
     \tau_\ast^\dagger \le \frac{C}{\sqrt{C_{\bs}}}\, \frac{1}{n^\epsilon\log n},
     \]
     where $C_{\bs}$ denotes the sufficiently large constant in \cref{lem:convergence_rate_optimal-trunc_bs}. By enlarging $C_{\bs}$ if necessary, we may
     therefore assume
     \begin{equation}
          \tau_\ast^\dagger\le \frac16. \label{eqn:tau_le16}
     \end{equation}
     In particular, $4(\tau_\ast^\dagger)^2<1$.

     \item Since $\sigma_{\min}(\cC_n)^2 = \lambda_{\min}(\cC_n^\top \cC_n)=1-\|\cS_n^\top \cS_n\|$, and similarly for $\cC_n^{\bs}$, we have, on $\cE_{\tau,n}$,
     \[
     \sigma_{\min}(\cC_n)^2 \wedge \sigma_{\min}(\cC_n^{\bs})^2 \ge 1-4(\tau_\ast^\dagger)^2>0.
     \]
     Moreover,
     \[
     (U_n^{\bs})^\top U_n = (\cC_n^{\bs})^\top \left( I_r+(\cT_n^{\bs})^\top\cT_n \right)\cC_n,
     \]
     where $\|(\cT_n^{\bs})^\top\cT_n\| \le 4(\tau_\ast^\dagger)^2<1$.
     Hence $\cC_n$, $\cC_n^{\bs}$, and $(U_n^{\bs})^\top U_n$ are invertible on $\cE_{\tau,n}$, so Lemma~\ref{lem:alignment_diff_bootstrap} applies, and we can define the self-adjoint positive-semidefinite operator $\cQ_n:\HH\to\HH$ on $\cE_{\tau,n}$ through
     \[
     \begin{aligned}
     \Finner{M}{\cQ_nM} := \left\| \cC_n \cC_n^\top\cT_n^\top M \cC_n^{\bs} \right\|_{\rm F}^2+ \left\| (I_{d-r}-\cS_n \cS_n^\top)M \cC_n^{\bs} \right\|_{\rm F}^2, \qquad \forall M \in \HH.
     \end{aligned}
     \]
     Lemma~\ref{lem:alignment_diff_bootstrap}(ii), with
     $(U_1,U_2)=(U_n,U_n^{\bs})$, then gives
     \begin{equation}
          \eta^{-1} \|\sin\Theta(U_n^{\bs},U_n)\|_{\rm F}^2 = \Fnorm[\big]{\cQ_n^{1/2} \eta^{-1/2}(\cT_n^{\bs}-\cT_n)}^2 \qquad\text{on }\cE_{\tau,n}. \label{eqn:bs_highD_B_relation}
     \end{equation}

     We next bound $\cQ_n^{1/2}-\cI$. On $\cE_{\tau,n}$, for any $M\in\HH$,
     \[
     \left\| \cC_n \cC_n^\top\cT_n^\top M \cC_n^{\bs} \right\|_{\rm F}^2 \le 4(\tau_\ast^\dagger)^2\|M\|_{\rm F}^2.
     \]
     Also, $\sigma_{\min}(I_{d-r}-\cS_n \cS_n^\top) \ge 1-\|\cS_n\|^2 \ge 1-4(\tau_\ast^\dagger)^2$ and $\sigma_{\min}(\cC_n^{\bs})^2 \ge 1-4(\tau_\ast^\dagger)^2$.
     Therefore,
     \[
     (1-4(\tau_\ast^\dagger)^2)^3\|M\|_{\rm F}^2
     \le
     \Finner{M}{\cQ_nM}
     \le
     (1+4(\tau_\ast^\dagger)^2)\|M\|_{\rm F}^2.
     \]
     Since $(1-4x)^3\ge 1-12x$ for $x\in[0,1/4]$, setting $e_n:=12(\tau_\ast^\dagger)^2$ gives
     \[
     (1-e_n)\cI \preceq \cQ_n \preceq (1+e_n)\cI,
     \]
     where by \eqref{eqn:tau_le16}, $e_n \le \frac13 < \frac12$. Thus, on $\cE_{\tau,n}$,
     \begin{align}
          \|\cQ_n^{1/2}-\cI\|_{{\rm op}} \le \max\{1-(1-e_n)^{1/2},(1+e_n)^{1/2}-1\} \le e_n. \label{eqn:bs_highD_E_bound}
     \end{align}

\end{itemize}

Note that $\cQ_n$ is only defined on $\cE_{\tau,n}$. Define the random operator $\tcQ_n:\HH\to\HH$ by
\[
\widetilde\cQ_n:=
\begin{cases}
\cQ_n, & \text{on }\cE_{\tau,n},\\
\cI, & \text{on }\cE_{\tau,n}^c.
\end{cases}
\]

Applying Lemma~\ref{lem:g_approx_multiplicative_remainder} under $\PP^\ast$
with
\[
\mathsf{A}_n^{\bs} := \eta^{-1/2}(\cT_n^{\bs}-\cT_n), \qquad \mathsf{B}_n^{\bs} :=\tcQ_n^{1/2}\mathsf{A}_n^{\bs}, \qquad \mathsf{E}_n^{\bs} := \tcQ_n^{1/2}-\cI, \qquad \mathsf{G}_n:=\mathsf{G},
\]
and
\[
e_n=12(\tau_\ast^\dagger)^2,
\qquad
L_n^2:=C_L(\tr\Gamma)\log n,
\qquad
h_n:= \left( \frac{1}{(1-e_n)^2}-1 \right)L_n^2,
\]
where $C_L>0$ is sufficiently large so that
$\PP\{\|\mathsf{G}\|_{\rm F}>L_n\}\le n^{-9}$, gives
\[
\begin{aligned}
d_K^\ast \left( \|\mathsf{B}_n^{\bs}\|_{\rm F}^2,\|\mathsf{G}\|_{\rm F}^2 \right)
\le{}& d_K^\ast \left( \eta^{-1}\|\cT_n^{\bs}-\cT_n\|_{\rm F}^2, \|\mathsf{G}\|_{\rm F}^2 \right)+ w_n(h_n) + \PP^\ast \left\{ \|\mathsf{E}_n^{\bs}\|_{{\rm op}}>e_n \right\}\\
&+ \PP^*\{\|\mathsf{G}\|_{\rm F}>L_n\}.
\end{aligned}
\]
We consider each term as below.
\begin{itemize}[leftmargin=1.5em]
     \item By \eqref{eqn:bs_highD_B_relation} and $d_K^* \leq 1$,
     \begin{align*}
          d_K^\ast\left( \eta^{-1}\|\sin\Theta(U_n^{\bs},U_n)\|_{\rm F}^2, \|\mathsf{B}_n^{\bs}\|_{\rm F}^2 \right) \le \PP^\ast(\cE_{\tau,n}^c), \qquad \PP_X \text{ almost surely}.
     \end{align*}
     Thus, on $\cG_n^{\rm lin}$,
     \begin{align*}
          d_K^\ast\left( \eta^{-1}\|\sin\Theta(U_n^{\bs},U_n)\|_{\rm F}^2, \|\mathsf{B}_n^{\bs}\|_{\rm F}^2 \right) \le n^{-9}.
     \end{align*}

     \item \eqref{eqn:bs_highD_clt_step2} gives $d_K^\ast\left( \eta^{-1}\|\cT_n^{\bs}-\cT_n\|_{\rm F}^2, \|\mathsf{G}\|_{\rm F}^2 \right) \xrightarrow{\PP_X}0$.
     
     \item Since $e_n<1/2$, we have $h_n \lesssim e_n L_n^2 \lesssim (\tau_\ast^\dagger)^2L_n^2$.
     Thus Lemma~\ref{lem:chi2_anti_concentration} gives
     \[
     w_n(h_n) = \sup_{t\in\RR} \PP\{t<\|G\|_{\rm F}^2\le t+h_n\} \lesssim \sqrt{\frac{h_n}{\|\Gamma\|_{{\rm HS}}}} \lesssim \tau_\ast^\dagger \sqrt{\frac{L_n^2}{\|\Gamma\|_{{\rm HS}}}}.
     \]

     By Lemma~\ref{lem:Gamma_properties}(ii) and (vi),
     \[
     \frac{\tr\Gamma}{\|\Gamma\|_{{\rm HS}}} \lesssim \kappa_\Delta\sqrt{d_{\cC}},
     \]
     which combined with $\tau_\ast^\dagger\lesssim\tau_\ast$ gives
     \[
     \frac{(\tau_\ast^\dagger)^2L_n^2} {\|\Gamma\|_{{\rm HS}}} \lesssim \teta \kappa_\ast^2\kappa_\Delta r\bnu^2 d_{\fC}^{1/2} (\log n)^2 =o(1),
     \]
     where the last relation follows from condition~\eqref{eqn:highD_clt_cond}(ii).
     Consequently, $w_n(h_n)=o(1)$.
     
     \item \eqref{eqn:bs_highD_E_bound} implies $\PP_X$ almost surely, $\PP^\ast \left\{ \|\mathsf{E}_n^{\bs}\|_{{\rm op}}>e_n \right\}=0$.
     
     \item $\PP^* \{\|\mathsf{G}\|_{\rm F}>L_n\} \le n^{-9}$ by the choice of $C_L$ and Lemma~\ref{lem:chi2_concentration}.

\end{itemize}

Combining the above bounds gives
\begin{align*}
     d_K^\ast\left( \eta^{-1} \|\sin\Theta(U_n^{\bs},U_n)\|_{\rm F}^2, \|\mathsf{G}\|_{\rm F}^2 \right) \xrightarrow{\PP_X}0.
\end{align*}

Finally, Theorem~\ref{thm:highD_clt} gives
\[
d_K\left( \eta^{-1}\|\sin\Theta(U_n,U_\ast)\|_{\rm F}^2, \|\mathsf{G}\|_{\rm F}^2 \right) \to0.
\]
Triangle inequality then proves Theorem~\ref{thm:bs_highD_clt}.

% ------------------------------------------------

\subsection{Proof of Theorem~\ref{thm:bs_row_clt}}
\label{sec:prf_bs_row_clt}

For notation simplicity, write
\[
    U:=U_n,\qquad
    U^{\bs}:=U_n^{\bs},\qquad
    \cT:=\cT_n,\qquad
    \cT^{\bs}:=\cT_n^{\bs},
\]
and define
\[
    R:=\sgn(U^\top U_\ast),\qquad
    R^{\bs}:=\sgn((U^{\bs})^\top U_\ast),\qquad
    Q^{\bs}:=\sgn((U^{\bs})^\top U).
\]
Set
\[
    V_1:=UR,\qquad V_2:=U^{\bs}R^{\bs}.
\]
Thus $V_1$ and $V_2$ are both Procrustes-aligned with $U_\ast$.

We apply Lemma~\ref{lem:alignment_diff_bootstrap} with $(U_1,U_2)=(V_1,V_2)$. Note the following facts.
\begin{itemize}[leftmargin=1.5em]
    \item Let $M_{21}:=V_2^\top V_1 = (R^{\bs})^\top (U^{\bs})^\top U R$.
    By orthogonal equivariance of the polar factor,
    \[
        \sgn(M_{21}) = (R^{\bs})^\top \sgn(U^{\bs\top} U) R  = (R^{\bs})^\top Q^{\bs}R.
    \]
    Hence,
    \begin{align*}
        V_2\sgn(M_{21})-V_1 = U^{\bs}R^{\bs}(R^{\bs})^\top Q^{\bs}R-UR = (U^{\bs}Q^{\bs}-U)R.
    \end{align*}

    \item Let $\cC_1,\cS_1$ and $\cC_2,\cS_2$ denote the cosine and sine matrices of $V_1$ and $V_2$ relative to $U_\ast$. Right multiplication by an orthogonal matrix does not change the tangent matrix. Hence the tangent matrices of $V_1$ and $V_2$ are respectively $\cT$ and $\cT^{\bs}$.
\end{itemize}
Then Lemma~\ref{lem:alignment_diff_bootstrap}(i) gives
\begin{equation}
    (U^{\bs}Q^{\bs}-U)R = U_\ast A+U_{\ast,\perp}B, \label{eqn:bs_row_clt_align_diff}
\end{equation}
where
\begin{align*}
    A &= \cC_1\big[(M_{21}^\top M_{21})^{1/2}-I_r\big] - \cC_1\cC_1^\top\cT^\top (\cT^{\bs}-\cT) \cC_2\sgn(M_{21}), \\
    B &= \cS_1\big[(M_{21}^\top M_{21})^{1/2}-I_r\big] + (I_{d-r}-\cS_1\cS_1^\top) (\cT^{\bs}-\cT) \cC_2\sgn(M_{21}).
\end{align*}

We provide some intuition before we proceed. Recall from Appendix~\ref{sec:prf_bs_highD_clt} that $\cT^\bs - \cT \approx \sqrt{\eta} \sum_{i=1}^{n}\xi_i^{\bs}$, where
\[
    \xi_i^{\bs} := \sqrt{\eta}(w_i^{\bs}-1) (\cI-\eta\cL)^{n-i}(Z_iY_i^\top) \in\HH, \qquad i\in[n].
\]
Meanwhile, since $V_1$ and $V_2$ are already aligned with $U_\ast$,  $\cC_1 \approx \cC_2 \approx \sgn(M_{21}) \approx I_r$. These observations motivate the following decomposition of $(U^{\bs}Q^{\bs}-U)R$:
\begin{equation}
    (U^{\bs}Q^{\bs}-U)R = W^{\bs}+\Psi^{\bs}, \label{eqn:bs_row_clt_decomp_joint}
\end{equation}
where
\begin{align*}
    W^{\bs} &= U_{\ast,\perp} \sqrt{\eta}\sum_{i=1}^n\xi_i^{\bs}, \\
    \Psi^{\bs} &= U_\ast A + U_{\ast,\perp}[B-(\cT^\bs - \cT)] + U_{\ast,\perp}E^{\bs}, \qquad E^{\bs} := (\cT^\bs - \cT) - \sqrt{\eta}\sum_{i=1}^n\xi_i^{\bs}.
\end{align*}
Fix $m\in[d]$ and define
\[
    u:=u_m:=U_{\ast,\perp}^\top e_m, \qquad \widetilde u:=\Lambda_{\ast,\perp}^{1/2}u, \qquad \Sigma^\ast_{U,m}:=\eta\Gamma_u.
\]
Identifying the $m$-th row with a column vector as in the proof of
Theorem~\ref{thm:row_clt}, \eqref{eqn:bs_row_clt_decomp_joint} gives
\begin{align}
    \big[(U^{\bs}Q^{\bs}-U)R\big]_{m,\cdot}^\top = W_u^{\bs}+\Psi_u^{\bs}, \label{eqn:bs_row_clt_decomp_row}
\end{align}
where
\begin{align*}
    W_u^{\bs} &:= \sqrt{\eta}\sum_{i=1}^n(\xi_i^{\bs})^\top u,\\
    \Psi_u^{\bs} &:= A^\top U_\ast^\top e_m + [B-(\cT^\bs - \cT)]^\top u + (E^{\bs})^\top u.
\end{align*}

\paragraph{Step 1: Gaussian approximation of $W_u^{\bs}$.}

Define the projected bootstrap summands
\[
    \xi_{u,i}^{\bs} := (\xi_i^{\bs})^\top u = \sqrt{\eta}(w_i^{\bs}-1) \left[ (\cI-\eta\cL)^{n-i}(Z_iY_i^\top) \right]^\top u \in\RR^r.
\]
Then
\[
    W_u^{\bs} = \sqrt{\eta}\sum_{i=1}^n\xi_{u,i}^{\bs}.
\]
Conditionally on $\cF_X$, the variables $\{\xi_{u,i}^{\bs}\}_{i=1}^n$ are independent and centered. Define their conditional cumulative covariance by
\[
    \widehat\Gamma_{u,n}
    :=
    \sum_{i=1}^n
    \cov^\ast(\xi_{u,i}^{\bs}).
\]
Then $\cov^\ast(W_u^{\bs}) = \eta\widehat\Gamma_{u,n}$.
Since $\EE(w_i^{\bs}-1)^2=1$,  we have $\EE_X\widehat\Gamma_{u,n}=\Gamma_{u,n}$ and
\[
    \widehat\Gamma_{u,n} = \eta\sum_{i=1}^n \left[ (\cI-\eta\cL)^{n-i}(Z_iY_i^\top) \right]^\top uu^\top \left[ (\cI-\eta\cL)^{n-i}(Z_iY_i^\top) \right].
\]
We first compare $\widehat\Gamma_{u,n}$ with $\Gamma_u$.
For $p\in\{3,4\}$, we claim that
\[
    \sum_{i=1}^n \left\| \left[ (\cI-\eta\cL)^{n-i}(Z_iY_i^\top) \right]^\top u \right\|^p = O_{\PP_X}\left( \frac{ r^{p/2}\lambda_1^{p/2}\|\widetilde u\|^p }{\teta} \right).
\]
Indeed, for each $j\in[r]$, let
\[
    D_j:=\diag\{1-\eta\Delta_{\ell j}\}_{\ell\in[d-r]}.
\]
The $j$-th coordinate of the vector inside the norm is $[Y_i]_j u^\top D_j^{n-i}Z_i$.
By the power-mean inequality, Cauchy--Schwarz, and Lemma~\ref{lem:kth_moment},
\begin{align*}
    \EE \left| [Y_i]_j\,u^\top D_j^{n-i}Z_i \right|^p &\le \big(\EE|[Y_i]_j|^{2p}\big)^{1/2} \big( \EE|u^\top D_j^{n-i}Z_i|^{2p} \big)^{1/2} \\
    &\lesssim \lambda_j^{p/2} \left\| \Lambda_{\ast,\perp}^{1/2} D_j^{n-i}u \right\|^p \\
    &\le \lambda_1^{p/2} (1-\teta)^{p(n-i)} \|\widetilde u\|^p.
\end{align*}
Therefore,
\[
    \EE \left\| \left[ (\cI-\eta\cL)^{n-i}(Z_iY_i^\top) \right]^\top u \right\|^p \lesssim r^{p/2}\lambda_1^{p/2} (1-\teta)^{p(n-i)} \|\widetilde u\|^p.
\]
Summing the geometric series and applying Markov's inequality proves
the claim.

Using the $p=4$ case and independence gives
\begin{align*}
    \EE_X \|\widehat\Gamma_{u,n}-\Gamma_{u,n}\|_{\rm F}^2 \lesssim \eta^2 \sum_{i=1}^n \EE \left\| \left[ (\cI-\eta\cL)^{n-i}(Z_iY_i^\top) \right]^\top u \right\|^4 \lesssim \teta\kappa_\ast^2r^2 \|\widetilde u\|^4.
\end{align*}
Hence
\[
    \|\widehat\Gamma_{u,n}-\Gamma_{u,n}\|_{\rm F} = O_{\PP_X}\left( \sqrt{\teta}\, \kappa_\ast r\|\widetilde u\|^2 \right).
\]

By Lemma~\ref{lem:Gamma_u_properties}(iii),
\[
    \lambda_{\min}(\Gamma_u) \gtrsim \frac{\lambda_r\|\widetilde u\|^2}{\Delta_{\max}} \lambda_{\min}(\widetilde\fC).
\]
Combining this with the approximation in Lemma~\ref{lem:Gamma_u_properties}(iv), as in
Step~2 of the proof of Lemma~\ref{lem:W_berry_esseen}, gives
\begin{align*}
    \left\| \Gamma_u^{-1/2} (\widehat\Gamma_{u,n}-\Gamma_u) \Gamma_u^{-1/2} \right\|_{\rm F} = O_{\PP_X}\left( \sqrt{\teta}\, \kappa_\ast^2r \lambda_{\min}^{-1}(\widetilde\fC) + \teta\kappa_\ast^5r^{5/2} \lambda_{\min}^{-2}(\widetilde\fC) \right).
\end{align*}
The first condition in Theorem~\ref{thm:row_clt}, $\teta\kappa_\ast^6r^{7/2} \lambda_{\min}^{-3}(\widetilde\fC) \to0$,
implies that the right-hand side is $o_{\PP_X}(1)$. Thus, with
$\PP_X$-probability tending to one, $\widehat\Gamma_{u,n}$ is
positive definite and
\[
    \lambda_{\min}(\widehat\Gamma_{u,n}) \gtrsim \lambda_{\min}(\Gamma_u).
\]

We now apply Theorem~\ref{thm:berry_esseen} conditionally on $\cF_X$. The corresponding
Berry--Esseen quantity satisfies
\begin{align*}
    \gamma_u^{\bs} &:= \sum_{i=1}^n \EE^\ast \left\| (\eta\widehat\Gamma_{u,n})^{-1/2} \sqrt{\eta}\,\xi_{u,i}^{\bs} \right\|^3\\
    &= \sum_{i=1}^n \EE^\ast \left\| \widehat\Gamma_{u,n}^{-1/2} \xi_{u,i}^{\bs} \right\|^3 \\
    &= O_{\PP_X}\left( \sqrt{\teta}\, \kappa_\ast^3r^{3/2} \lambda_{\min}^{-3/2}(\widetilde\fC) \right),
\end{align*}
where the last line follows from Assumption~\ref{assumption:w}, the $p=3$ empirical
moment bound above, and the lower bound for $\lambda_{\min}(\widehat\Gamma_{u,n})$.
Therefore, Theorem~\ref{thm:berry_esseen} yields
\begin{align*}
    \sup_{A\in\mathcal A_r} \left| \PP^\ast\{W_u^{\bs}\in A\} - N(0,\eta\widehat\Gamma_{u,n})\{A\} \right| = O_{\PP_X}\left( \sqrt{\teta}\, \kappa_\ast^3r^{7/4} \lambda_{\min}^{-3/2}(\widetilde\fC) \right) = o_{\PP_X}(1).
\end{align*}
The last convergence follows from the first condition in Theorem~\ref{thm:row_clt}.

Moreover, the preceding relative covariance bound and Lemma~\ref{lem:TV_gaussian} give
\[
    \sup_{A\in\mathcal A_r}
    \left|
      N(0,\eta\widehat\Gamma_{u,n})\{A\}
      -
      N(0,\eta\Gamma_u)\{A\}
    \right|
    =
    o_{\PP_X}(1).
\]
Consequently,
\[
    \sup_{A\in\mathcal A_r}
    \left|
      \PP^\ast\{W_u^{\bs}\in A\}
      -
      N(0,\Sigma^\ast_{U,m})\{A\}
    \right|
    \xrightarrow{\PP_X}0.
\]

\paragraph{Step 2: Accounting for the higher-order term $\Psi_u^{\bs}$.}

Under the assumptions of Theorem~\ref{thm:bs_row_clt}, the conditions of Proposition~\ref{prop:linearization} hold with $\delta_p\to0$. As in the proof of
Theorem~\ref{thm:bs_highD_clt}, condition~\eqref{eqn:cond_sample_sz} implies Assumption~\ref{assumption:d}, while the Haar initialization satisfies Assumption~\ref{assumption:init} with probability $1-o(1)$. Thus, after intersecting with this initialization event, Corollary~\ref{cor:conditional_bs_linearization} gives an $\cF_X$-measurable event $\cG_n^{{\rm lin}}$ such that $\PP_X(\cG_n^{{\rm lin}})\to1$ and, on $\cG_n^{{\rm lin}}$,
\[
    \|\cT\|\le2\tau_\ast^\dagger, \qquad \PP^\ast\{\|\cT^{\bs}\|>2\tau_\ast^\dagger\}\le n^{-9}, \qquad \PP^\ast\{\|E^{\bs}\|>r_n^{{\rm lin}}\}\le n^{-9},
\]
where
\[
    \tau_\ast^\dagger\lesssim\tau_\ast \asymp \kappa_\ast\bar\nu\sqrt{r\teta\log n}, \qquad r_n^{{\rm lin}} \lesssim \tau_\ast^\dagger \sbr[\big]{C_\psi^2\kappa_\ast\sqrt{r\teta\log n} +\tau_\ast^\dagger}.
\]

Define
\[
    \cE_{\tau,n} := \{ \|\cT\|\le2\tau_\ast^\dagger,\, \|\cT^{\bs}\|\le2\tau_\ast^\dagger \}.
\]
On $\cG_n^{{\rm lin}}$, Corollary~\ref{cor:conditional_bs_linearization} gives $\PP^\ast(\cE_{\tau,n}^c)\le n^{-9}$.

To control $A$ and $B-(\cT^\bs -\cT)$ defined in \eqref{eqn:bs_row_clt_align_diff} on $\cE_{\tau,n}$, we use the following bounds.
\begin{itemize}[leftmargin=1.5em]
    \item Since $V_1$ and $V_2$ are Procrustes-aligned with $U_\ast$, their
    cosine matrices $\cC_1$ and $\cC_2$ are symmetric positive definite.
    Using $\cS_i=\cT_i\cC_i$ and $\cC_i^\top \cC_i+\cS_i^\top \cS_i=I_r$ gives
    \[
        \cC_1=(I_r+\cT^\top\cT)^{-1/2}, \qquad \cC_2= (I_r+(\cT^{\bs})^\top\cT^{\bs})^{-1/2}.
    \]
    Therefore,
    \[
        \|\cC_i-I_r\| \lesssim (\tau_\ast^\dagger)^2, \qquad \|\cS_i\| \lesssim \tau_\ast^\dagger, \qquad i=1,2.
    \]
    Moreover, since $M_{21}=\cC_2^\top \cC_1+ \cS_2^\top \cS_1$, we have
    \[
        \|M_{21}-I_r\| \lesssim (\tau_\ast^\dagger)^2.
    \]
    
    \item Lemma~\ref{lem:alignment_diff_bootstrap}(iii) gives
    \[
        \|(M_{21}^\top M_{21})^{1/2}-I_r\| \le \|\sin\Theta(V_2,V_1)\|^2 \le (1+\|\cT\|^2)\|\Delta\|^2.
    \]
    On $\cE_{\tau,n}$,
    \[
        \|\cT^\bs - \cT\| \le \|\cT\|+\|\cT^{\bs}\| \le4\tau_\ast^\dagger,
    \]
    and hence
    \[
        \|(M_{21}^\top M_{21})^{1/2}-I_r\| \lesssim (\tau_\ast^\dagger)^2.
    \]

    \item  Since $\sgn(M_{21})=M_{21}(M_{21}^\top M_{21})^{-1/2}$, we have
    \begin{align*}
        \|\sgn(M_{21})-I_r\| = \|(M_{21}-I_r)(M_{21}^\top M_{21})^{-1/2} + (M_{21}^\top M_{21})^{-1/2}-I_r\| \lesssim (\tau_\ast^\dagger)^2
    \end{align*}
    and consequently
    \[
        \|\cC_2\sgn(M_{21})-I_r\| \leq \|(\cC_2 - I_r) \sgn(M_{21})\| + \|\sgn(M_{21}) - I_r\| \lesssim (\tau_\ast^\dagger)^2.
    \]
\end{itemize}

It follows from the above bounds and the definition in \eqref{eqn:bs_row_clt_align_diff} that
\begin{align*}
    \|A\| &\lesssim (\tau_\ast^\dagger)^2 + \|\cT\|\,\|\Delta\| \lesssim (\tau_\ast^\dagger)^2;\\
    \|B-(\cT^\bs - \cT)\| &= \|\cS_1\big[(M_{21}^\top M_{21})^{1/2}-I_r\big]\| + \|\cS_1 \cS_1^\top(\cT^\bs - \cT) \cC_2\sgn(M_{21})\|\\
    &\quad +\|(\cT^\bs - \cT)(\cC_2\sgn(M_{21})-I_r)\|\\
    &\lesssim (\tau_\ast^\dagger)^3.
\end{align*}

Recall
\[
    \Psi_u^{\bs} = A^\top U_\ast^\top e_m + [B-(\cT^\bs - \cT)]^\top u + (E^{\bs})^\top u.
\]
Hence, on $\cG_n^{{\rm lin}}$, with conditional probability at least
$1-O(n^{-9})$,
\begin{align*}
    \|\Psi_u^{\bs}\| \lesssim (\tau_\ast^\dagger)^2 \|U_\ast^\top e_m\| + (\tau_\ast^\dagger)^3\|u\| + r_n^{{\rm lin}}\|u\| \lesssim \tau_\ast^2 \|U_\ast^\top e_m\| + r_n^{{\rm lin}}\|u\|.
\end{align*}
Here the last line uses
$\tau_\ast^\dagger\lesssim\tau_\ast$,
$\tau_\ast^\dagger=o(1)$, and the fact that the definition of
$r_n^{{\rm lin}}$ contains a term of order
$(\tau_\ast^\dagger)^2$.

By Lemma~\ref{lem:Gamma_u_properties}(iii),
\begin{align*}
    \lambda_{\min}(\Sigma^\ast_{U,m}) = \eta\lambda_{\min}(\Gamma_u) \gtrsim \eta \frac{ \lambda_r\|\widetilde u\|^2 }{ \Delta_{\max} } \lambda_{\min}(\widetilde\fC) \gtrsim \teta \frac{\|\widetilde u\|^2}{\lambda_1} \lambda_{\min}(\tfC).
\end{align*}
Consequently,
\[
    \PP^\ast \cbr[\big]{\norm[\big]{(\Sigma^\ast_{U,m})^{-1/2} \Psi_u^{\bs}} \le \zeta^{\bs}} \ge1-O(n^{-9})
\]
on $\cG_n^{{\rm lin}}$, where $\zeta^{\bs}:=C(\zeta_1+\zeta_2)$ for a sufficiently large constant $C$, with
\begin{align*}
    \zeta_1 &:= \kappa_\ast^2\teta^{1/2} r(\log n)\bar\nu^2 \frac{ \lambda_1^{1/2}\|U_\ast^\top e_m\| }{ \|\widetilde u\| } \lambda_{\min}^{-1/2}(\widetilde\cC), \\
    \zeta_2 &:= \kappa_\ast^2\teta^{1/2} r(\log n)\bar\nu(1+\bar\nu) \frac{ \lambda_1^{1/2}\|u\| }{ \|\widetilde u\| } \lambda_{\min}^{-1/2}(\tfC).
\end{align*}
These are of the same orders as $\zeta_1$ and $\zeta_2$ in
Lemma~\ref{lem:row_clt_higher_order}. In particular, condition~\eqref{eqn:row_clt_cond1} gives
\[
    r^{1/4}\zeta^{\bs}\to0.
\]
Indeed,
\begin{align*}
    r^{1/2}\zeta_1^2&\lesssim \teta\kappa_\ast^4r^{5/2} (\log n)^2\bar\nu^2(1+\bar\nu)^2 \frac{ \lambda_1\|U_\ast^\top e_m\|^2 }{ \|\widetilde u\|^2 } \lambda_{\min}^{-1}(\tfC),\\
    r^{1/2}\zeta_2^2 &\lesssim \teta\kappa_\ast^4r^{5/2} (\log n)^2\bar\nu^2(1+\bar\nu)^2 \frac{ \lambda_1\|u\|^2}{ \|\widetilde u\|^2 } \lambda_{\min}^{-1}(\tfC),
\end{align*}
where for the first bound we used
$\bar\nu^4\le\bar\nu^2(1+\bar\nu)^2$.
Both terms vanish by condition~\eqref{eqn:row_clt_cond1}.

We now proceed as in the proof of Theorem~\ref{thm:row_clt}. For any
$A\in\mathcal A_r$, let $A^{\zeta}$ denote the signed enlargement defined before Theorem~\ref{thm:gaussian_cvx}. The decomposition
\[
    \big[(U^{\bs}Q^{\bs}-U)R\big]_{m,\cdot}^\top = W_u^{\bs}+\Psi_u^{\bs}
\]
and the preceding remainder bound imply
\begin{align*}
    \PP^\ast \cbr[\big]{(\Sigma^\ast_{U,m})^{-1/2}W_u^{\bs} \in A^{-\zeta^{\bs}}} - O(n^{-9})\le{} &\PP^\ast \cbr[\big]{(\Sigma^\ast_{U,m})^{-1/2} \big[(U^{\bs}Q^{\bs}-U)R\big]_{m,\cdot}^\top \in A}\\
    \le{} &\PP^\ast \cbr[\big]{(\Sigma^\ast_{U,m})^{-1/2}W_u^{\bs} \in A^{\zeta^{\bs}}} + O(n^{-9}).
\end{align*}

By Step~1,
\[
    \sup_{A\in\mathcal A_r} \left| \PP^\ast\left\{ (\Sigma^\ast_{U,m})^{-1/2}W_u^{\bs}\in A \right\} - \PP\{N(0,I_r)\in A\} \right| \xrightarrow{\PP_X}0.
\]
Hence, by Theorem~\ref{thm:gaussian_cvx},
\begin{align*}
    \PP^\ast\left\{
      (\Sigma^\ast_{U,m})^{-1/2}W_u^{\bs}
      \in A^{\zeta^{\bs}}
    \right\}
    &\le
    \PP\{N(0,I_r)\in A^{\zeta^{\bs}}\}
    +
    o_{\PP_X}(1)
    \\
    &\le
    \PP\{N(0,I_r)\in A\}
    +
    (0.59r^{1/4}+0.21)\zeta^{\bs}
    +
    o_{\PP_X}(1)
    \\
    &=
    \PP\{N(0,I_r)\in A\}
    +
    o_{\PP_X}(1).
\end{align*}
The corresponding lower bound follows in the same way using
$A^{-\zeta^{\bs}}$. Therefore,
\[
    \sup_{A\in\mathcal A_r}
    \left|
      \PP^\ast\left\{
        \big[(U^{\bs}Q^{\bs}-U)R\big]_{m,\cdot}
        \in A
      \right\}
      -
      N(0,\Sigma^\ast_{U,m})\{A\}
    \right|
    \xrightarrow{\PP_X}0.
\]

Finally, Theorem~\ref{thm:row_clt} gives
\[
    \sup_{A\in\mathcal A_r}
    \left|
      \PP\{[UR-U_\ast]_{m,\cdot}\in A\}
      -
      N(0,\Sigma^\ast_{U,m})\{A\}
    \right|
    \to0.
\]
Combining the last two displays by the triangle inequality yields
\[
    \sup_{A\in\mathcal A_r}
    \left|
      \PP^\ast\left\{
        \big[(U^{\bs}Q^{\bs}-U)R\big]_{m,\cdot}
        \in A
      \right\}
      -
      \PP\{[UR-U_\ast]_{m,\cdot}\in A\}
    \right|
    \xrightarrow{\PP_X}0.
\]
Recalling
\[
    U=U_n,\qquad
    U^{\bs}=U_n^{\bs},\qquad
    R=\sgn(U_n^\top U_\ast),\qquad
    Q^{\bs}=\sgn((U_n^{\bs})^\top U_n),
\]
this is exactly the conclusion of Theorem~\ref{thm:bs_row_clt}.

% ------------------------------------------------

\subsection{Proof of Lemma~\ref{lem:linearization-bs-truncation}}
\label{sec:prf_linearization-bs-truncation}

\begin{proof}

Starting from \eqref{eqn:linearization_trunc_bs2} and \eqref{eqn:bs_extra_term}, the same decomposition as in Lemma~\ref{lem:linearization-truncation} gives
\begin{align*}
    \oocT_n^{\bs} &= (\cI-\eta \cL)^{n-N_*} \oocT_{N_*}^{\bs} + \oocD_{N_*,n}^{\bs,(1,0)} + \oocD_{N_*,n}^{\bs,(1,1)} + \oocD_{N_*,n}^{\bs,(1,2)} + \oocD_{N_*,n}^{\bs,(2)} + \oocR_{N_*,n}^{\bs,H},
\end{align*}
where
\begin{align*}
    \oocD_{N_*,n}^{\bs,(1,0)} &:= \eta \sum_{i=N_* + 1}^n w_i^{\bs} (\cI-\eta \cL)^{n-i} (\ooZ_i \ooY_i^\top - \ooLambda_{*,\times});\\
    \oocD_{N_*,n}^{\bs,(1,1)} &:= \ooS_1^{\bs} + \ooS_2^{\bs} + \ooS_3^{\bs},\\
    \text{where} &\quad \ooS_1^{\bs} := \eta \sum_{i=N_* + 1}^n w_i^{\bs} (\cI-\eta \cL)^{n-i}[-\oocT_{i-1}^{\bs}(\ooY_i \ooY_i^\top - \ooLambda_*)],\\
    &\quad \ooS_2^{\bs} := \eta \sum_{i=N_* + 1}^n w_i^{\bs} (\cI-\eta \cL)^{n-i} [(\ooZ_i \ooZ_i^\top - \ooLambda_{*,\perp})\oocT_{i-1}^{\bs}],\\
    &\quad \ooS_3^{\bs} :=\eta \sum_{i=N_* + 1}^n w_i^{\bs} (\cI-\eta \cL)^{n-i} [-\oocT_{i-1}^{\bs} (\ooY_i \ooZ_i^\top- \ooLambda_{*,\times}^\top) \oocT_{i-1}^{\bs}];\\
    \oocD_{N_*,n}^{\bs,(1,2)} &:= - \eta \sum_{i=N_* + 1}^n (w_i^{\bs}-1) (\cI-\eta \cL)^{n-i} (\oocL \oocT_{i-1}^{\bs}) ;\\
    \oocD_{N_*,n}^{\bs,(2)} &:= \sum_{i=N_* + 1}^n (\cI-\eta \cL)^{n-i} \ooD_i^{\bs,(2)};\\
    \oocR_{N_*,n}^{\bs,H} &:= \sum_{i=N_* + 1}^n (\cI-\eta \cL)^{n-i} \ooR_i^{\bs,H}.
\end{align*}

Then
\begin{align*}
  &\norm[\Big]{\oocT_n^{\bs} - \eta \sum_{i=1}^n w_i^{\bs} (\cI-\eta \cL)^{n-i} (\ooZ_i \ooY_i^\top)}\\ 
  \leq &\underbrace{\norm[\big]{(\cI-\eta \cL)^{n-N_*} \oocT_{N_*}^{\bs}}}_{=:\ooalpha_1^{\bs}} + \underbrace{\norm[\big]{\eta \sum_{i=1}^{N_*} w_i^{\bs} (\cI-\eta \cL)^{n-i} (\ooZ_i \ooY_i^\top - \ooLambda_{*,\times})}}_{=:\ooalpha_2^{\bs}} \\
  + &\underbrace{\norm[\big]{\oocD_{N_*,n}^{\bs, (1,1)}}}_{=:\ooalpha_3^{\bs}} + \underbrace{\norm[\big]{\oocD_{N_*,n}^{\bs,(2)}}}_{=:\ooalpha_4^{\bs}} + \underbrace{\norm[\big]{\oocR_{N_*,n}^{\bs, H} - \eta \sum_{i=N_* + 1}^n w_i^{\bs} (\cI-\eta \cL)^{n-i} \ooLambda_{*,\times}}}_{=:\ooalpha_5^{\bs}}\\
  + &\underbrace{\norm[\big]{\eta \sum_{i=1}^{N_*} w_i^{bs} (\cI-\eta \cL)^{n-i} \ooLambda_{*,\times}}}_{=:\ooalpha_6^{\bs}} + \underbrace{\norm[\big]{\oocD_{N_*,n}^{\bs, (1,2)}}}_{=:\ooalpha_7^{\bs}}.
\end{align*}
The terms $\ooalpha_1^\bs,\ldots,\ooalpha_6^\bs$ are controlled exactly as their counterparts
$\alpha_1,\ldots,\alpha_6$ in the proof of Lemma~\ref{lem:linearization-truncation}, up to constants depending only on $C_w$. In particular, we have
\begin{align}
  \sum_{l=1}^6 \ooalpha_l^{\bs} \lesssim \ootau_\ast^\bs \sbr[\big]{C_\psi^2\kappa_\ast \sqrt{r\widetilde\eta\log n} +\ootau_*^\bs} + e^{-\widetilde\eta(n-N_\ast)} \kappa_\ast\nu_\infty\|H\|. \label{eqn:bs_alpha_1_to_6}
\end{align}
It remains to control the new multiplier term $\ooalpha_7^\bs$. Note that $\|\cL \oocT_{i-1}^\bs\| \leq 2\lambda_1 \|\oocT_{i-1}^\bs\|$. Moreover, similar to \eqref{eqn:tau_leq1}, we have $\ootau_*^{\bs} \leq \frac{1}{n^{\varepsilon} \log n}$ for sufficiently large constant $C_\bs$. Then on event $\{\|\overline\cT_{i-1}^\bs\|\le 2\tau_\ast^\bs\}$, \eqref{eqn:H_control} and \eqref{eqn:L_diff} give
\begin{align*}
  \|(\oocL-\cL) \oocT_{i-1}^\bs\| &\leq \norm{H}\lambda_1 (1+\|\oocT_{i-1}^\bs\| \nu_\infty)(\|\oocT_{i-1}^\bs\| + \nu_\infty)\\
  &\lesssim \norm{H} \lambda_1 (\ootau_*^\bs + \nu_\infty)\\
  &\lesssim \lambda_1 \ootau_*^\bs,
\end{align*}
where the last inequality uses $\nu_\infty\|H\| \lesssim \nu_\infty\teta \lesssim \tau_\ast^\bs$. By triangle inequality, on event $\{\|\overline\cT_{i-1}^\bs\|\le 2\tau_\ast^\bs\}$, we have
\begin{align*}
  \|\oocL \oocT_{i-1}^\bs\| \lesssim \lambda_1 \ootau_*^\bs.
\end{align*}
Define, for $i=N_\ast+1,\ldots,n$,
\[
\Xi_i^\bs
:= -\eta(w_i^\bs-1)(\cI-\eta\cL)^{n-i} (\oocL\oocT_{i-1}^\bs) \bbi \cbr[\big]{\|\oocT_{i-1}^{\bs}\| \le 2\tau_\ast^\bs}.
\]
Since $\overline\cT_{i-1}^\bs$ is $\overline\cF_{i-1}$-measurable and $w_i^\bs$ is independent of $\overline\cF_{i-1}$ with $\EE(w_i^\bs-1)=0$, $\{\Xi_i^\bs\}$ is a matrix martingale-difference sequence.
Moreover,
\[
\|\Xi_i^\bs\| \lesssim \eta\lambda_1\tau_\ast^\bs (1-\widetilde\eta)^{n-i}.
\]
Therefore,
\[
\max_{i>N_\ast}\|\Xi_i^\bs\| \lesssim \widetilde\eta\kappa_\ast\tau_\ast^\bs.
\]
Similarly,
\[
\begin{aligned}
&\max\cbr[\Big]{\norm[\Big]{\sum_{i=N_\ast+1}^n \overline\EE_{i-1} \bigl[ \Xi_i^\bs(\Xi_i^\bs)^\top \bigr]}^{1/2}, \norm[\Big]{\sum_{i=N_\ast+1}^n \overline\EE_{i-1} \bigl[ (\Xi_i^\bs)^\top\Xi_i^\bs \bigr]}^{1/2}}\\
\lesssim&\; \eta\lambda_1\tau_\ast^\bs \sbr[\Big]{\sum_{i=N_\ast+1}^n (1-\widetilde\eta)^{2(n-i)}}^{1/2} \\
\lesssim& \; \kappa_\ast\tau_\ast^\bs \sqrt{\widetilde\eta}.
\end{aligned}
\]
Thus, matrix Freedman's inequality implies that, with probability at
least $1-O(n^{-18})$,
\[
\norm[\Big]{\sum_{i=N_\ast+1}^n\Xi_i^\bs} \lesssim
\kappa_\ast\tau_\ast^\bs \rbr[\big]{\sqrt{\widetilde\eta\log n}+\widetilde\eta\log n}
\lesssim \tau_\ast^\bs C_\psi^2\kappa_\ast \sqrt{r\widetilde\eta\log n}.
\]
On the stability event from
Lemma~\ref{lem:convergence_rate_optimal-trunc_bs}, the stopping indicators equal one for all $i>N_\ast$, and therefore
\[
\overline\alpha_7^\bs \lesssim \tau_\ast^\bs C_\psi^2\kappa_\ast \sqrt{r\widetilde\eta\log n}
\]
with probability at least $1-O(n^{-18})$. Combining this with \eqref{eqn:bs_alpha_1_to_6} gives the desired bound.

\end{proof}

% -----------------------------------------------

\subsection{Proof of Lemma~\ref{lem:bs_highD_clt_biometrika19}}
\label{sec:prf_bs_highD_clt_biometrika19}

\begin{proof}
Recall \eqref{eqn:bs_highD_clt_Sp}, \eqref{eqn:bs_highD_clt_hGamma_lower}:
\begin{align*}
  S_p :=\sum_{i=1}^n \|M_i\|_{\rm F}^p &= O_{\PP_X}\left(\frac{\lambda_\times^p}{\widetilde\eta}\right),\\
  \|\hGamma_n\|_{\rm HS} &= \Omega_{\PP_X} \left(\frac{\lambda_\times^2}{\Delta_{\max}\sqrt{d_{\fC}}}\right).
\end{align*}

For $L_3^{\xi,\bs}$, since $\xi_i^{\bs} = \sqrt{\eta}(w_i^{\bs}-1)M_i$ and the multipliers are uniformly bounded,
\[
\begin{aligned}
  \frac{\sum_{i=1}^n \EE^\ast \langle \xi_i^{\bs},\hGamma_n\xi_i^{\bs} \rangle_{\rm F}^{3/2}}{\|\hGamma_n\|_{\rm HS}^3 } 
  \lesssim \frac{ \eta^{3/2} \|\hGamma_n\|_{\rm op}^{3/2} S_3 }{ \|\hGamma_n\|_{\rm HS}^3} 
  \leq  \frac{ \eta^{3/2} S_3 }{ \|\hGamma_n\|_{\rm HS}^{3/2}}  = O_{\PP_X}\left( \teta^{1/2} \kappa_\Delta^{3/2} d_{\fC}^{3/4} \right).
\end{aligned}
\]
Also, conditional independence gives
\[
\begin{aligned}
  \frac{ \sum_{1\le i<j\le n} \EE^\ast \big| \langle\xi_i^{\bs},\xi_j^{\bs}\rangle_{\rm F} \big|^3 }{ \|\hGamma_n\|_{\rm HS}^3 } \lesssim \frac{ \eta^3 S_3^2 }{ \|\hGamma_n\|_{\rm HS}^3 }= O_{\PP_X}\left( \teta \kappa_\Delta^3 d_{\fC}^{3/2} \right).
\end{aligned}
\]
Therefore,
\[
  L_3^{\xi,\bs} = O_{\PP_X}\left( \teta^{1/2} \kappa_\Delta^{3/2} d_{\fC}^{3/4} + \teta \kappa_\Delta^3 d_{\fC}^{3/2} \right).
\]

For $L_3^{g,\bs}$, conditionally on $\cF_X$, $g_i^{\bs} \sim N(0,\eta M_i\otimes M_i)$,
so we may write $g_i^{\bs} = \sqrt{\eta}\gamma_i M_i$, $\gamma_i\sim N(0,1)$. Hence
\[
\begin{aligned}
  L_3^{g,\bs} &= \frac{ \sum_{i=1}^n \EE^\ast \langle g_i^{\bs}, \hGamma_ng_i^{\bs} \rangle_{\rm F}^{3/2} }{ \|\hGamma_n\|_{\rm HS}^3 }\lesssim \frac{ \eta^{3/2} \|\hGamma_n\|_{\rm op}^{3/2} S_3 }{ \|\hGamma_n\|_{\rm HS}^3 } \leq \frac{ \eta^{3/2} S_3 }{ \|\hGamma_n\|_{\rm HS}^{3/2} } = O_{\PP_X}\left( \teta^{1/2} \kappa_\Delta^{3/2} d_{\fC}^{3/4} \right).
\end{aligned}
\]

Next, by the definition of $K_3^{\bs}$,
\[
\begin{aligned}
n^{-2}(K_3^{\bs})^3 \lesssim \frac{ \sum_{i=1}^n \EE^\ast\|\xi_i^{\bs}\|_{\rm F}^6 }{ \|\hGamma_n\|_{\rm HS}^3 }\lesssim \frac{ \eta^3 S_6 }{ \|\hGamma_n\|_{\rm HS}^3 } = O_{\PP_X}\left( \teta^2 \kappa_\Delta^3 d_{\fC}^{3/2} \right).
\end{aligned}
\]

Finally,
\[
\begin{aligned}
J_n^{\bs} = \frac{ \sum_{i=1}^n \Var^\ast(\|\xi_i^{\bs}\|_{\rm F}^2) }{ \|\hGamma_n\|_{\rm HS}^2 } \le \frac{ \sum_{i=1}^n \EE^\ast\|\xi_i^{\bs}\|_{\rm F}^4 }{ \|\hGamma_n\|_{\rm HS}^2 } \lesssim \frac{ \eta^2 S_4 }{ \|\hGamma_n\|_{\rm HS}^2 } = O_{\PP_X}\left( \widetilde\eta \kappa_\Delta^2 d_{\fC} \right).
\end{aligned}
\]

\end{proof}

% \section{Proofs for Section~\ref{sec:main_bs}}
% % \label{sec:bootstrap_analysis}
% \input{bootstrap_analysis.tex}

\section{Technical lemmas}
\label{sec:tech_lemmas}

Let $\cI$ be the identity operator and $\cL$ be the gap operator on $\RR^{(d-r) \times r}$ defined as
\begin{align*}
    [\cL M]_{l,j} = \Delta_{lj} M_{l,j}, \quad \text{where} \quad \Delta_{lj} = \lambda_j - \lambda_{r+l}  \text{ for any } l \in [d-r], j \in [r].
\end{align*}
\begin{lemma} \label{lem:basic_ineq}
    Assume $\eta \lambda_1 \in (0,1)$.
    Let $D$ be a random matrix in $\RR^{(d-r) \times r}$. Then the following hold.

    \textup{(i)} $(\EE D)^\top (\EE D) \preceq \EE D^\top D$ which implies $\norm{\EE D} \leq \min (\|\EE DD^\top\|^{1/2}, \|\EE D^\top D\|^{1/2})$.

    \textup{(ii)} for any matrix $M \in \RR^{(d-r) \times r}$, $\norm{(\cI - \eta \cL) M} \leq (1-\eta \Delta_{\min}) \norm{M}$.

    \textup{(iii)} $\norm[\big]{\EE \sbr[\big]{ \sbr{(\cI - \eta \cL)^m D}^\top \sbr{(\cI-\eta \cL)^m D} }}^{1/2} \leq (1-\eta \Delta_{\min})^m  \|\EE D^\top D\|^{1/2}$ for $m=1,2,\ldots$.

    \textup{(iv)} $\norm[\big]{\EE \sbr[\big]{ \sbr{(\cI - \eta \cL)^m D} \sbr{(\cI-\eta \cL)^m D}^\top }}^{1/2} \leq (1-\eta \Delta_{\min})^m  \|\EE DD^\top\|^{1/2}$ for $m=1,2,\ldots$.
\end{lemma}
\begin{proof}
    See Appendix~\ref{sec:prf_basic_ineq}.
\end{proof}

\begin{lemma}[Moment bound] \label{lem:kth_moment}
    Let $\tX \in \RR^d$ be a random vector with  $\EE \tX \tX^\top = I_d$ and $\|\tX\|_{\psi_2} \leq C_\psi$. Then for any $A \in \RR^{p \times d}$ and any $k \geq 1$,
    \begin{align}
        \sbr[\big]{\EE \norm[\big]{A\tX}^{2k}}^{1/(2k)} &\leq C \sqrt{k}  \Fnorm{A} C_\psi \label{eqn:k_th_moment_1}\\
        \sbr[\big]{\EE \norm[\big]{A\tX}^{2k}}^{1/(2k)} &\leq C \sqrt{kd}  \|A\tX\|_{\psi_2} \label{eqn:k_th_moment_2}
    \end{align}
    Equivalently, let $X=A\tX$, we have
    \begin{align*}
        \sbr[\big]{\EE \norm[\big]{X}^{2k}}^{1/(2k)} &\leq C \sqrt{k \tr \EE (XX^\top)} \;C_\psi\\
        \sbr[\big]{\EE \norm[\big]{X}^{2k}}^{1/(2k)} &\leq C \sqrt{kd}  \norm{X}_{\psi_2}
    \end{align*}
    Here $C>0$ is an absolute constant.
\end{lemma}
\begin{proof}
    See Appendix~\ref{sec:prf_kth_moment}.
\end{proof}

\begin{theorem} \label{thm:freedman}
    Consider a matrix martingale $\cbr{Y_k : k= 0, 1,\ldots}$ whose values are matrices in dimension $d_1 \times d_2$ and $Y_0= 0$. Let $\cbr{X_k:k=1,\ldots}$ be the difference sequence $X_k = Y_k - Y_{k-1}$. Assume that the difference sequence is uniformly bounded:
    \begin{align*}
        \norm{X_k} \leq B \qquad \text{almost surely} \quad \forall \, k=1,2,3,\ldots
    \end{align*}
    Define two predictable quadratic variation processes for the martingale:
    \begin{align*}
        W_{\mathrm{col},k} &:= \sum_{j=1}^k \EE_{j-1} X_j X_j^\top, \qquad W_{\mathrm{row},k} := \sum_{j=1}^k \EE_{j-1} X_j^\top X_j, \qquad \forall \, k=1,2,3,\ldots
    \end{align*}
    Then, for all $t \geq 0$ and $\sigma^2$,
    \begin{align*}
        \PP \cbr{\exists \, k \geq 0: \norm{Y_k} \geq t \text{ and } \norm{W_{\mathrm{col},k}} \vee \norm{W_{\mathrm{row},k}} \leq \sigma^2} &\leq (d_1 + d_2) \exp \rbr{\frac{-t^2/2}{\sigma^2 + Bt/3}}
    \end{align*}
\end{theorem}
\begin{proof}
    See \citep[Corollary~1.3]{tropp}.
\end{proof}

\begin{lemma}[Matrix moments] \label{lem:matrix_kth_moment}
    Assume Assumption~\ref{assumption:X} holds. Then the following hold.
    
    \textup{(i)} For any $k \geq 1$,
    \begin{align*}
        \sbr[\Big]{\EE \Fnorm[\big]{ZY^\top}^{2k}}^{1/(2k)} &\leq C_1 k \sbr{\tr \Lambda_*}^{1/2} \sbr{\tr \Lambda_{*,\perp}}^{1/2} \norm[\big]{\tX}_{\psi_2}^2\\
        \sbr[\Big]{\EE \Fnorm[\big]{\tZ\tY^\top}^{2k}}^{1/(2k)} &\leq C_1 k  r^{1/2} (d-r)^{1/2} \norm[\big]{\tX}_{\psi_2}^2
    \end{align*}
    for an absolute constant $C_1>0$. As a consequence, $d_{\fC} \gtrsim 1$.

    \textup{(ii)} Let $\fC= \cov(ZY^\top)$ and $\tfC = \cov(\tZ\tY^\top)$. Then
    \begin{align*}
        \Onorm[\big]{\fC} &\leq C_2 C_\psi^4 \min(r,d-r) \lambda_1 \lambda_{r+1},\\
        \Onorm[\big]{\tfC} &\leq C_2 C_\psi^4 \min(r,d-r)
    \end{align*}
    for an absolute constant $C_2>0$.
\end{lemma}
\begin{proof}
    See Appendix~\ref{sec:prf_matrix_kth_moment}.
\end{proof}

\begin{lemma}[Alignment with the target subspace]\label{lem:alignment_diff}
    Let $U \in \OO_{d,r}$ with $\cC = U_*^\top U \in \RR^{r \times r}$ invertible,
    $\cS = U_{*,\perp}^\top U \in \RR^{(d-r) \times r}$, and $\cT = \cS\cC^{-1}$. Then
    \begin{align}
        \opnorm{\sin\Theta(U, U_*)} &= \opnorm[\big]{\cT\,(\cC\cC^\top)^{1/2}}, \label{eqn:sinTheta_repre}\\
        U \sgn(U^\top U_*) - U_* 
        &= U_*\sbr[\big]{(\cC\cC^\top)^{1/2} - I_r} + U_{*,\perp} \cT (\cC\cC^\top)^{1/2}. \label{eqn:aligned_diff_repre}
    \end{align}
    where $\opnorm{\cdot}$ is either operator norm or Frobenius norm. Denote $\tau = \Onorm{\cT}$. Then 
    \begin{equation}
        \norm[\big]{(\cC\cC^\top)^{1/2} - I_r} \le \norm[\big]{\cT (\cC\cC^\top)^{1/2}}^2 \le \tau^2 \label{eqn:repre_order}
    \end{equation}
    If $\tau<1$, then $\Onorm[\big]{(\cC\cC^\top)^{1/2} - I_r} \geq \tau^2/4$.
\end{lemma}
\begin{proof}
    See Appendix~\ref{sec:prf_alignment_diff}.
\end{proof}

The preceding lemma concerns alignment with the target subspace $U_\ast$. We next record its two-subspace analogue, which will be used to compare the original and bootstrap iterates.

\begin{lemma}[Relative alignment of two subspaces] \label{lem:alignment_diff_bootstrap}
    Let $U_*, U_1,U_2 \in \OO_{d,r}$ with $U_i = U_* \cC_i + U_{*,\perp} \cS_i$. Denote $M:= U_2^\top U_1 = \cC_2^\top \cC_1 + \cS_2^\top \cS_1$. Assume $\cC_1$, $\cC_2$ and $M$ are invertible, and let $\cT_i=\cS_i \cC_i^{-1}$ for $i=1,2$. Then the following hold.
    \begin{enumerate}[label=\textup{(\roman*)}, leftmargin=*]
        \item \textup{(Decomposition)} $U_2 \sgn(U_2^\top U_1)-U_1 = U_* A + U_{*,\perp}B$ where
        \begin{align*}
            A &= \cC_1 \sbr[\big]{(M^\top M)^{1/2} - I_r} + \cC_1 \cC_1^\top \cT_1^\top (\cT_1 - \cT_2) \cC_2 \sgn(M),\\
            B &= \cS_1\sbr[\big]{(M^\top M)^{1/2} - I_r} - (I_{d-r} - \cS_1 \cS_1^\top) (\cT_1-\cT_2) \cC_2 \sgn(M).
        \end{align*}
        \item $\Fnorm{\sin\Theta(U_2, U_1)}^2 = \Fnorm{\cC_1\cC_1^\top \cT_1^\top (\cT_1 - \cT_2) \cC_2}^2 + \Fnorm{(I_{d-r} - \cS_1 \cS_1^\top) (\cT_1-\cT_2) \cC_2}^2$
        \item Let $\tau_1 = \norm{\cT_1}$. Then
        \begin{align*}
            \norm[\big]{(M^\top M)^{1/2} - I_r} \leq \norm{\sin \Theta(U_2, U_1)}^2 \leq (1+\tau_1^2) \norm{\cT_1 - \cT_2}^2
        \end{align*}
        % In addition, $\|(M^\top M)^{1/2}-I_r\|\ge \frac12\|\sin\Theta(U_2,U_1)\|^2 .$
    \end{enumerate}
    
\end{lemma}

\begin{proof}
    See Appendix~\ref{sec:prf_alignment_diff_bootstrap}.
\end{proof}

\begin{lemma}[Initialization] \label{lem:init}

    \textup{(i) \citep[Lemma~28]{colt21} and \citep[Theorem~II.13]{Davidson}} Let $G\in\mathbb R^{d\times r}$ be a standard Gaussian matrix. For any $t \geq 0$,
        \begin{align*}
            \PP \cbr[\big]{\norm{G} \geq \sqrt{d} + \sqrt{r} + t} \leq 2e^{-t^2/2}.
        \end{align*}

    \textup{(ii) \citep[Lemma~29]{colt21} and \citep[Lemma i.A.3]{zhiyuan17}} Let $G\in\mathbb R^{r\times r}$ be a standard Gaussian matrix. For any $\delta \in (0,1)$, 
        \begin{align*}
            \PP \cbr[\Big]{\norm{G^{-1}} \geq \frac{6\sqrt{r}}{\delta}} \leq \delta.
        \end{align*}

    \textup{(iii) (Haar initialization)} Fix $U_* \in \OO_{d,r}$ and let $U_0 \sim \Haar    {\OO_{d,r}}$. Then for any $\delta_p \in (0,1/e]$, any $c_p \geq 1$, we have
        \begin{align*}
            \PP \cbr[\Big]{\norm{(U_{*,\perp}^\top U_0)(U_*^\top U_0)^{-1}}\leq \frac{48c_p \sqrt{2\log(4c_p)} \sqrt{rd \log(1/\delta_p)}}{\delta_p}  } \geq 1 - \frac{\delta_p}{c_p}.
        \end{align*}
\end{lemma}

\begin{proof}
    See Appendix~\ref{sec:prf_init}.
\end{proof}

\begin{theorem}[Berry-Esseen] \label{thm:berry_esseen}
    Let $\xi_1, \ldots, \xi_n$ be independent,  $\RR^d$-valued random variables with zero means. Let $W= \sum_{i=1}^n \xi_i$ and assume $\Sigma=\cov(W)$ is invertible. Let $Z$ be a $d$-dimensional Gaussian random vector with zero mean and covariance $\Sigma$. Then we have
    \begin{align*}
        \sup_{\cA \in \mathscr{A}^d } \abs{\PP \cbr{W \in \cA} - \PP \cbr{Z \in \cA}} \leq (42d^{1/4} + 16) \gamma, \qquad\text{where} \; \gamma = \sum_{i=1}^n \EE \norm[\big]{\Sigma^{-1/2} \xi_i}^3
    \end{align*}
    Here $\mathscr{A}^d$ represents the set of all convex sets in $\RR^d$.
\end{theorem}
\begin{proof}
    See \citep[Theorem~11]{yuling}.
\end{proof}

\begin{lemma} \label{lem:TV_gaussian}
    Let $\mu \in \RR^d$, and let $\Sigma_1$, $\Sigma_2$ be positive definite $d\times d$ matrices. Then the total variation distance between $\cN(\mu, \Sigma_1)$ and $\cN(\mu, \Sigma_2)$ -- denoted by $\mathsf{TV}(\cN(\mu, \Sigma_1), \cN(\mu, \Sigma_2))$--satisfies
    \begin{align*}
        \frac{1}{100} \leq \frac{\mathsf{TV}(\cN(\mu, \Sigma_1), \cN(\mu, \Sigma_2))}{\min \cbr{1, \Fnorm[\big]{\Sigma_1^{-1/2} \Sigma_2 \Sigma_1^{-1/2} - I_d}}} \leq \frac{3}{2}
    \end{align*}
\end{lemma}
\begin{proof}
    See \citep[Theorem~12]{yuling}.
\end{proof}

For any point $x\in \RR^r$ and any non-empty convex set $\cA \in \mathscr{A}^r$ satisfying $\cA \neq \RR^r$, define the signed distance function as
\begin{align*}
    \delta_{\cA}(x) := \begin{cases}
        - \mathsf{dist}(x, \RR^r\setminus \cA), & \text{if } x \in \cA;\\
        \mathsf{dist}(x, \cA), & \text{if } x \notin \cA.
    \end{cases}
\end{align*}
Here, $\mathsf{dist}(x,\cA)$ is the Euclidean distance between $x\in \RR^r$ and a non-empty set $\cA \subset \RR^r$. Also, for any $t \in \RR$, define
\begin{align}
    \cA^{t} := \cbr{x \in \RR^r: \delta_{\cA}(x) \leq t} \label{eqn:A_t_defn}
\end{align}
for any non-empty convex set $\cA \in \mathscr{A}^r$ satisfying $\cA \neq \RR^r$, and define $\emptyset^{t}=\emptyset$ and $(\RR^r)^{t} = \RR^r$.

For any convex set $\cA \in \mathscr{A}^r$, define the following quantity related to Gaussian distributions
\begin{align*}
    &\gamma_r := \sup_{\cA \in \mathscr{A}^r} \gamma(\cA),\\
    &\quad \text{where} \quad
    \gamma(\cA):= \sup_{t > 0} \max \cbr{\frac{1}{t}\cN(0,I_r)\cbr{\cA^t \setminus \cA}, \frac{1}{t} \cN(0, I_r)\cbr{\cA \setminus \cA^{-t}}},
\end{align*}
\begin{theorem}
    For all $r\in \NN$, we have $\gamma_r < 0.59 r^{1/4} + 0.21$. \label{thm:gaussian_cvx}
\end{theorem}
\begin{proof}
    See \citep[Theorem~13]{yuling}.
\end{proof}

% \begin{lemma} \label{lem:reduction_to_bounded}
%     There exist two universal constants $C_\delta$, $C_\sigma>0$ such that: for any sub-Gaussian random variable $X$ with $\EE X=0$, $\Var{X} = \sigma^2$, $\|X\|_{\psi_2}\lesssim \sigma$, and any $\delta_p \in (0, C_\delta \sigma)$, one can construct a random variable $\tX$ satisfying the following properties:
%     \begin{enumerate}[label=(\roman*), leftmargin=2em]
%         \item $\tX$ is equal to $X$ with probability at least $1-\delta_p$;
%         \item $\EE \tX = 0$;
%         \item $\tX$ is a bounded random variable: $|\tX|\leq C_\sigma \sigma \sqrt{\log(\delta^{-1})}$;
%         \item The variance of $\tX$ obeys $\Var{\tX} = (1+O(\sqrt{\delta}))\sigma^2$;
%         \item $\tX$ is a sub-Gaussian random variable with $\|\tX\|_{\psi_2}\lesssim \sigma$.
%     \end{enumerate}
% \end{lemma}

% \begin{proof}
%     See \citep[Lemma~31]{yuling}.
% \end{proof}

\subsection{Proof of technical lemmas}

\subsubsection{Proof of Lemma~\ref{lem:basic_ineq}}
\label{sec:prf_basic_ineq}

\textbf{(i)} for any $u \in \RR^r$, we have
\begin{align*}
    u^\top \sbr{\EE D^\top D - (\EE D)^\top(\EE D) } u & = \EE \norm{D u}^2 - \norm{\EE D u}^2 \geq 0
\end{align*}
Thus $(\EE D)^\top (\EE D) \preceq \EE D^\top D$, which combined with $\|M\|^2 = \|MM^\top\|=\|M^\top M\|$
gives the desired result.

\textbf{(ii)} By the definition of $\cL$, we have
\begin{align*}
    (\cI -\eta \cL)M = \eta \Lambda_{*,\perp} M + M(I_r - \eta \Lambda_*)
\end{align*}
Then the triangle inequality gives
\begin{align*}
    \norm{(\cI - \eta \cL) M} &\leq \eta \norm{\Lambda_{*,\perp} M} + \norm{M(I_r - \eta \Lambda_*)} \\
    &\leq \eta \lambda_{r+1} \norm{M} + (1-\eta \lambda_r) \norm{M}\\
    &= (1-\eta (\lambda_r -\lambda_{r+1})) \norm{M} = (1-\eta \Delta_{\min}) \norm{M}.
\end{align*}

\textbf{(iii)} Without loss of generality, assume $m=1$ since $m>1$ follows by iterating the $m=1$ case. 
For any unit vector $u \in \RR^r$, we have
\begin{align}
    \sbr{u^\top \EE \sbr[\big]{ \sbr{(\cI - \eta \cL) D}^\top \sbr{(\cI-\eta \cL) D} } u}^{1/2} &= \norm{\sbr{(\cI-\eta \cL) D} u}_{L_2} \nonumber\\
    &= \norm{\eta \Lambda_{*,\perp} D u + D (I_r - \eta \Lambda_*) u}_{L_2} \nonumber \\
    &\leq \norm{\eta \Lambda_{*,\perp} D u}_{L_2} + \norm{D (I_r - \eta \Lambda_*) u}_{L_2} \label{eqn:basic_ineq_1}
\end{align}
Note that 
\begin{align*}
    \norm{\eta \Lambda_{*,\perp} D u}_{L_2}^2 \leq \eta^2 \lambda_{r+1}^2 \norm{D u}_{L_2}^2 = \eta^2 \lambda_{r+1}^2 \norm{\EE u^\top D^\top D u} \leq \eta^2 \lambda_{r+1}^2 \norm{\EE D^\top D}
\end{align*}
and
\begin{align*}
    \norm{D (I_r - \eta \Lambda_*) u}_{L_2}^2 &=  u^\top (I_r - \eta \Lambda_*) (\EE D^\top D) (I_r - \eta \Lambda_*) u\\
    &\leq  \norm{(I_r - \eta \Lambda_*) u}^2 \norm{\EE D^\top D} \leq (1-\eta \lambda_r)^2 \norm{\EE D^\top D}
\end{align*}
which combined with \eqref{eqn:basic_ineq_1} gives
\begin{align*}
    \sbr{u^\top \EE \sbr[\big]{ \sbr{(\cI - \eta \cL) D}^\top \sbr{(\cI-\eta \cL) D} } u}^{1/2} &\leq \rbr{1-\eta \lambda_r + \eta \lambda_{r+1}} \norm{\EE D^\top D}^{1/2}\\ 
    &= (1-\eta \Delta_{\min})  \|\EE D^\top D\|^{1/2}
\end{align*}

\textbf{(iv)} is similar to (iii).

\subsubsection{Proof of Lemma~\ref{lem:kth_moment}}
\label{sec:prf_kth_moment}
If $A=0$, the bound is trivial. Assume $A \neq 0$. 

\paragraph{Proof of \eqref{eqn:k_th_moment_1}.} Let $A_{i,\cdot} \in \RR^{1 \times d}$ denote the $i$-th row of $A$, then
\begin{align*}
    \|A\tX\|^{2} = \sum_{i=1}^p (A_{i,\cdot} \tX)^2 = \Fnorm{A}^2 \cdot \sum_{A_{i,\cdot}\neq 0}  \frac{\norm{A_{i,\cdot}}^2}{\Fnorm{A}^2} \frac{(A_{i,\cdot} \tX)^2}{\norm{A_{i,\cdot}}^2}
\end{align*}
Note that $\sum_{A_{i,\cdot}\neq 0}  \norm{A_{i,\cdot}}^2/\Fnorm{A}^2 = 1$.
Taking $k$-th power and applying Jensen's inequality with weights $\norm{A_{i,\cdot}}^2 / \Fnorm{A}^2$ gives
\begin{align*}
    \EE \|A \tX\|^{2k} 
    &\leq \Fnorm{A}^{2k} \sum_{A_{i,\cdot}\neq 0} \frac{\norm{A_{i,\cdot}}^2}{\Fnorm{A}^2} \EE \sbr[\bigg]{\frac{(A_{i,\cdot} \tX)^2}{\norm{A_{i,\cdot}}^2}}^k
\end{align*}
Since $[\EE (v^\top \tX)^{2k}]^{1/2k} \leq C\sqrt{k} C_{\psi}$ holds uniformly for all unit  vector $v \in \RR^d$ for some absolute constant $C>0$ \citep{vershynin}, taking $2k$-th root gives
\begin{align*}
    \sbr[\big]{\EE \|A\tX\|^{2k}}^{1/(2k)} \leq C \sqrt{k} \; \Fnorm{A} \norm[\big]{\tX}_{\psi_2}
\end{align*}

\paragraph{Proof of \eqref{eqn:k_th_moment_2}.} Denote $X=A\tX$. For any $k \geq 1$, 
\begin{align*}
    \norm{X}^{2k} = \sbr[\Big]{\sum_{i=1}^d X_i^2}^k \leq d^{k-1} \sum_{i=1}^d X_i^{2k}
\end{align*}
Since $\norm{X_i}_{\psi_2} \leq \norm{X}_{\psi_2}$ for all $i \in [d]$, taking expectation and $2k$-th root, we have
\begin{align*}
    \sbr[\big]{\EE \norm{X}^{2k}}^{1/(2k)} \leq \sbr[\Big]{\EE d^{k-1} \sum_{i=1}^d X_i^{2k}}^{1/(2k)} \leq C \sqrt{kd} \; \norm{X}_{\psi_2}
\end{align*}
for the same absolute constant $C>0$ as before.

\subsubsection{Proof of Lemma~\ref{lem:matrix_kth_moment}}
\label{sec:prf_matrix_kth_moment}

\paragraph{(i) }
Note that $\Fnorm{ZY^\top}^2=\norm{Z}^2\norm{Y}^2$, then we have
\begin{align*}
    \Fnorm{ZY^\top}^{2k} &= \rbr[\Big]{\norm{\Lambda_{*,\perp}^{1/2}\tZ}^2 \norm{\Lambda_*^{1/2}\tY}^2 }^k = \rbr[\Big]{\sum_{l=1}^{d-r} \sum_{j=1}^r \lambda_{r+l} \lambda_j \tZ_l^2 \tY_j^2}^k 
\end{align*}
Let $S=\sbr{\tr \Lambda_{*,\perp}} \sbr{\tr \Lambda_*}$. Note that $\sum_{lj} \lambda_{r+l} \lambda_j = S$. Since $k \geq 1$, applying Jensen's inequality gives
\begin{align*}
    \Fnorm{ZY^\top}^{2k} \leq S^k \cdot \sum_{l=1}^{d-r} \sum_{j=1}^r \frac{\lambda_{r+l} \lambda_j}{S} \tZ_l^{2k} \tY_j^{2k}
\end{align*}
For sub-exponential norm, we have
\begin{align*}
    \norm[\big]{\tZ_l \tY_j}_{\psi_1} \leq \norm[\big]{\tZ_l}_{\psi_2} \norm[\big]{\tY_j}_{\psi_2} \leq \norm[\big]{\tZ}_{\psi_2} \norm[\big]{\tY}_{\psi_2} \leq \norm[\big]{\tX}_{\psi_2}^2
\end{align*}
Hence
\begin{align*}
    \EE \Fnorm{ZY^\top}^{2k} \leq S^k \cdot \sum_{l=1}^{d-r} \sum_{j=1}^r \frac{\lambda_{r+l} \lambda_j}{S} \EE (\tZ_l \tY_j)^{2k} &\leq S^k \cdot \sum_{l=1}^{d-r} \sum_{j=1}^r \frac{\lambda_{r+l} \lambda_j}{S} \sbr{C k\norm[\big]{\tX}_{\psi_2}^2 }^{2k}\\
    &= S^k \cdot \sbr{C k\norm[\big]{\tX}_{\psi_2}^2 }^{2k} 
\end{align*}
thus
\begin{align*}
    \sbr{\EE \Fnorm{ZY^\top}^{2k}}^{1/(2k)} &\leq C k S^{1/2} \norm[\big]{\tX}_{\psi_2}^2 = C k \cdot \sbr{\tr \Lambda_*}^{1/2} \sbr{\tr \Lambda_{*,\perp}}^{1/2} \norm[\big]{\tX}_{\psi_2}^2
\end{align*}
Applying the above inequality with $k=1$ gives
\begin{align*}
    d_{\fC} := \frac{\lambda_\times^4}{\HS{\fC}^2} \gtrsim \sbr{\frac{\EE \|ZY^\top\|_{\rm F}^2}{\|\fC\|_{\rm HS}}}^2 \geq 1,
\end{align*}
where the last inequality follows from the fact that $\|\fC\|_{\rm HS} \leq \tr(\fC) = \EE \|ZY^\top\|_{\rm F}^2$.

\paragraph{(ii)} For any $A \in \RR^{(d-r)\times r}$ with $\Fnorm{A}=1$,
\begin{align*}
    \Finner{A}{\fC A} = \EE \Finner{ZY^\top}{A}^2 &= \EE \Finner{\tZ \tY^\top}{ \Lambda_{*,\perp}^{1/2} A \Lambda_*^{1/2} }^2\\
    &\leq \Onorm[\big]{\tfC} \Fnorm{\Lambda_{*,\perp}^{1/2} A \Lambda_*^{1/2}}^2 \leq \Onorm[\big]{\tfC} \lambda_1 \lambda_{r+1}
\end{align*}
Thus
\begin{align*}
    \Onorm{\fC} \leq \Onorm[\big]{\tfC} \lambda_1 \lambda_{r+1}
\end{align*}
It suffices to show $\Onorm[\big]{\tfC}\lesssim C_\psi^4 \min(r,d-r)$. For any $A \in \RR^{(d-r)\times r}$ with $\Fnorm{A}=1$, we have
\begin{align*}
    \Finner[\big]{\tZ\tY^\top}{A}= \tZ^\top A \tY = (A^\top \tZ)^\top \tY = \tZ^\top (A \tY)
\end{align*}

\textit{(a)} By Cauchy-Schwarz inequality, we have
\begin{align*}
    \Finner[\big]{A}{\tfC A} = \EE \Finner[\big]{\tZ\tY^\top}{A}^2 = \EE \sbr[\big]{(A^\top \tZ)^\top \tY}^2 \leq \sbr{\EE \norm[\big]{A^\top \tZ}^4}^{1/2} \sbr{\EE \norm[\big]{\tY}^4}^{1/2}
\end{align*}
Applying Lemma~\ref{lem:kth_moment} to each factor gives
\begin{align*}
    \sbr{\EE \norm[\big]{A^\top \tZ}^4}^{1/4} \lesssim C_\psi \Fnorm{A^\top} = C_\psi, \quad \sbr{\EE \norm[\big]{\tY}^4}^{1/4} \lesssim C_\psi \Fnorm{I_r} = C_\psi \sqrt{r}
\end{align*}
Thus, we have
\begin{align*}
    \Finner[\big]{A}{\tfC A} \lesssim C_\psi^4 r.
\end{align*}
Taking supremum over $A$ gives $\Onorm[\big]{\tfC} \lesssim C_\psi^4 r$.

\textit{(b)} Applying the same argument as in (a) to $\tZ^\top (A \tY)$ gives $\Onorm[\big]{\tfC} \lesssim C_\psi^4 (d-r)$.

% ----------------------------------------------

\subsubsection{Proof of Lemma~\ref{lem:alignment_diff}}
\label{sec:prf_alignment_diff}

We first establish the following identities for $\cC$ and $\cS$:
\begin{equation}\label{eq:polar-C}
  \cC \sgn(\cC^\top) = (\cC\cC^\top)^{1/2}, \qquad
  \cS \sgn(\cC^\top) = \cT (\cC\cC^\top)^{1/2}.
\end{equation}
To see this, note that for any matrix $M$ with singular value decomposition $M= A \Lambda B^\top$, the sign matrix of $M$ is defined as $\sgn(M) = A B^\top$. Hence $M= \sgn(M)(M^\top M)^{1/2}$.
Applying this to $M = \cC^\top$ gives $\cC^\top = \sgn(\cC^\top)\,(\cC\cC^\top)^{1/2}$. Since $\cC$ is invertible by assumption, so is $(\cC \cC^\top)^{1/2}$, hence we have $\sgn(\cC^\top) = \cC^\top(\cC\cC^\top)^{-1/2}$. Thus, for the first identity in~\eqref{eq:polar-C}, we have
\begin{align*}
    \cC \sgn(\cC^\top)
    = \cC \cC^\top(\cC\cC^\top)^{-1/2}
    = (\cC\cC^\top)^{1/2}.
\end{align*}
For the second, using $\cS = \cT\cC$ gives
\begin{align*}
    \cS \sgn(\cC^\top)
  = \cT\cC \cC^\top(\cC\cC^\top)^{-1/2}
  = \cT(\cC\cC^\top)^{1/2}.
\end{align*}

With \eqref{eq:polar-C} in hand, we can obtain a representation of $U \sgn(U^\top U_*)$. To see this, since $U^\top U_* = \cC^\top$, plugging in \eqref{eq:polar-C} gives
\begin{align} \label{eq:align-decomp}
    U \sgn(U^\top U_*) = (U_* \cC + U_{*,\perp} \cS) \sgn(\cC^\top) = U_* (\cC \cC^\top)^{1/2} + U_{*,\perp} \cT (\cC \cC^\top)^{1/2}.
\end{align}
which implies \eqref{eqn:sinTheta_repre} and \eqref{eqn:aligned_diff_repre} directly by noticing that $\Fnorm{\sin \Theta(U,U_*)} = \Fnorm{\cS}$.

\paragraph{Proof of \eqref{eqn:repre_order}.} Since $\cC \in \RR^{r \times r}$ is a square matrix, $\cC\cC^\top$ and $\cC^\top \cC$ share eigenvalues.
From $\cC^\top \cC = I_r - \cS^\top \cS$, the eigenvalues of $\cC\cC^\top$ are
$\{1 - \sigma_j^2\}_{j=1}^r$, where $\sigma_1 \ge \cdots \ge \sigma_r \ge 0$
are the singular values of~$\cS$.
Hence $(\cC\cC^\top)^{1/2}$ has eigenvalues $\{\sqrt{1 - \sigma_j^2}\}_{j=1}^r$,
all lying in $[\sqrt{1 - \sigma_1^2}, 1]$.

For the operator norm bound \eqref{eqn:repre_order}, since $1 - \sqrt{1-x} \le x$ for all $0 \le x \le 1$, we have
\begin{equation*} %\label{eq:sqrt-perturb}
  \Onorm{I_r - (\cC\cC^\top)^{1/2}}
  = \max_{j} \cbr{1 - \sqrt{1 - \sigma_j^2}}
  \le \max_j \sigma_j^2
  = \Onorm{\cS}^2 = \Onorm{\cT (\cC \cC^\top)^{1/2}}^2.
\end{equation*}
where the last equality follows since $\Onorm{\cS} = \Onorm{\cS \sgn(\cC^\top)}$ and \eqref{eq:polar-C}.
Finally, $\Onorm{\cS} = \Onorm{\cT\cC} \le \Onorm{\cT}\Onorm{\cC} \le \tau$,
so $\Onorm{(\cC\cC^\top)^{1/2} - I_r} \le \tau^2$.

\paragraph{Proof of Lower bound.} Since $1-\sqrt{1-x} \geq x/2$ for all $x \in [0,1]$, we have
\begin{align*}
    \Onorm{I_r - (\cC \cC^\top)^{1/2}} = 1 - \sqrt{1-\sigma_1^2} \geq \frac{\sigma_1^2}{2} = \frac{\Onorm{\cS}^2}{2}.
\end{align*}
Finally, we connect $\Onorm{\cS}$ to $\tau$ via
\begin{align*}
    \Onorm{\cS} = \Onorm{\cT\cC} \geq \sigma_{\min}(\cC) \tau = \sqrt{1-\Onorm{\cS}^2} \; \tau
\end{align*}
which rearranges to $\Onorm{\cS}^2 \geq \tau^2/(1 + \tau^2)$. Hence
\begin{align*}
    \Onorm{I_r - (\cC \cC^\top)^{1/2}} \geq \frac{\Onorm{\cS}^2}{2} \geq \frac{\tau^2}{2(1 + \tau^2)}.
\end{align*}
When $\tau \in [0,1]$, we have $(1+\tau^2)\leq 2$, so
\begin{align*}
    \Onorm{I_r - (\cC \cC^\top)^{1/2}} \geq \frac{\tau^2}{4}.
\end{align*}

\subsection{Proof of Lemma~\ref{lem:alignment_diff_bootstrap}}
\label{sec:prf_alignment_diff_bootstrap}

\paragraph{Preliminaries.} Using $\cS_i= \cT_i\cC_i$, we have $\cS_1^\top \cS_2 = \cC_1^\top \cT_1^\top \cT_2 \cC_2$ and
\begin{equation}
    \begin{aligned}
        U_1^\top U_2 = \cC_1^\top \cC_2 + \cS_1^\top \cS_2  = \cC_1^\top (I_r + \cT_1^\top \cT_2) \cC_2&= \cC_1^\top (I_r + \cT_1^\top \cT_1) \cC_2 - \cC_1^\top \cT_1^\top (\cT_1 - \cT_2) \cC_2\\
        &\overset{(a)}{=} \cC_1^{-1} \cC_2 - \cC_1^\top \cT_1^\top (\cT_1 - \cT_2) \cC_2 
    \end{aligned}
    \label{eqn:prf_alignment_diff_BS_prelim}
\end{equation}
Here (a) follows from the fact that since $\cC_1^\top \cC_1 + \cS_1^\top \cS_1 = I_r$, we have $\cC_1^\top (I_r + \cT_1^\top \cT_1) \cC_1 = I_r$, which implies $I_r + \cT_1^\top \cT_1 = \cC_1^{\top -1}\cC_1^{-1}$.

\paragraph{(i) Decomposition.}
We first decompose the aligned difference into its components parallel and
orthogonal to $\operatorname{col}(U_1)$:
\begin{equation}
    U_2\sgn(M)-U_1 = U_1U_1^\top\sbr{U_2\sgn(M)-U_1} +(I-U_1U_1^\top)U_2\sgn(M). \label{eqn:prf_alignment_diff_BS_decomp}
\end{equation}
\begin{itemize}[leftmargin=*]
    \item First term: Since $U_1^\top U_2=M^\top$ and $M^\top \sgn(M)=(M^\top M)^{1/2}$,
        the first term equals $U_1[(M^\top M)^{1/2}-I_r]$. Using $U_1 = U_* \cC_1 + U_{*,\perp} \cS_1$ gives
        \begin{align*}
            U_1 [(M^\top M)^{1/2}-I_r] = U_* \cC_1 [(M^\top M)^{1/2}-I_r] + U_{*,\perp}\cS_1 [(M^\top M)^{1/2}-I_r].
        \end{align*}

    \item Second term: we first compute $U_*$ and $U_{*,\perp}$ components of $(I-U_1U_1^\top)U_2$:

    \begin{itemize}[leftmargin=*]
        \item $U_*$-component is $\cC_2-\cC_1(\cC_1^\top \cC_2+\cS_1^\top \cS_2)$. Using \eqref{eqn:prf_alignment_diff_BS_prelim} gives
        \begin{align*}
            \cC_2-\cC_1(\cC_1^\top \cC_2+\cS_1^\top \cS_2)
            &= \underbrace{\cC_2 - \cC_1 \cC_1^\top (I_r + \cT_1^\top \cT_1) \cC_2}_{=0} + \cC_1 \cC_1^\top \cT_1^\top (\cT_1 - \cT_2) \cC_2\\
            &= \cC_1 \cC_1^\top \cT_1^\top (\cT_1 - \cT_2) \cC_2
        \end{align*}
        where the last line follows from $\cC_1^\top \cC_1 + \cS_1^\top \cS_1 = I_r$.

        \item $U_{*,\perp}$-component is $\cS_2-\cS_1(\cC_1^\top \cC_2+\cS_1^\top \cS_2)$. Using \eqref{eqn:prf_alignment_diff_BS_prelim} gives
        \begin{align*}
            \cS_2-\cS_1(\cC_1^\top \cC_2+\cS_1^\top \cS_2) &= \cT_2 \cC_2 - \cT_1 (\cC_2-\cC_1 \cC_1^\top \cT_1^\top (\cT_1 - \cT_2) \cC_2 )\\
            &=-(\cT_1 -\cT_2) \cC_2  + \cS_1 \cS_1^\top (\cT_1 -\cT_2) \cC_2\\
            &= -(I_{d-r} - \cS_1 \cS_1^\top) (\cT_1 -\cT_2) \cC_2
        \end{align*}
    \end{itemize}

\end{itemize}
Combining the above gives the desired decomposition.

\paragraph{(ii) $\Fnorm{\sin\Theta(U_2, U_1)}^2$.} This follows directly from \eqref{eqn:prf_alignment_diff_BS_decomp} and the fact that $\Fnorm{\sin\Theta(U_2, U_1)}^2= \|(I-U_1U_1^\top)U_2\|_{\rm F}^2$ \citep[Lemma~2.5]{spectral_methods}.

\paragraph{(iii)}
It remains to prove the operator-norm bounds. Let $U_{1,\perp}\in
\mathbb O_{d,d-r}$ be an orthogonal complement of $U_1$, and set $\tM=U_{1,\perp}^\top U_2$.
Since $[U_1,U_{1,\perp}]$ is orthogonal,
\[
    U_2^\top U_2
    =
    U_2^\top U_1U_1^\top U_2
    +
    U_2^\top U_{1,\perp}U_{1,\perp}^\top U_2,
\]
that is, $M M^\top+\tM^\top \tM=I_r$. Thus the singular values of $M$ are $\sqrt{1-\sigma_j^2(\tM)}$,
while the singular values of $\tM$ are
$\sin\theta_j(U_2,U_1)$. Therefore
\[
    \|(M^\top M)^{1/2}-I_r\|
    =
    \max_j\{1-\sqrt{1-\sin^2\theta_j}\}.
\]
Using $\frac{x}{2}\le 1-\sqrt{1-x}\le x$ for $x\in[0,1]$,
with $x=\sin^2\theta_j$, gives
\[
    \frac12\|\sin\Theta(U_2,U_1)\|^2
    \le
    \|(M^\top M)^{1/2}-I_r\|
    \le
    \|\sin\Theta(U_2,U_1)\|^2 .
\]

Finally, since $\|\sin\Theta(U_2,U_1)\| = \|(I-U_1U_1^\top)U_2\|$ \citep[Lemma~2.5]{spectral_methods}, analysis of the second term in \eqref{eqn:prf_alignment_diff_BS_decomp} gives
\[
    \|\sin\Theta(U_2,U_1)\|^2
    \le
    \|\cC_1\cC_1^\top \cT_1^\top\Delta \cC_2\|^2+\|(I_{d-r}-\cS_1\cS_1^\top) (\cT_1 - \cT_2) \cC_2\|^2,
\]
Since $\|\cC_i\|\le 1$ and $\|I_{d-r}-\cS_1\cS_1^\top\|\le 1$, we have
\[
    \|\sin\Theta(U_2,U_1)\|^2
    \le
    (1+\tau_1^2)\|\cT_1 - \cT_2\|^2 .
\]
The proof is complete.

% ---------------------------------------------------------------------- 

\subsubsection{Proof of Lemma~\ref{lem:init}}
\label{sec:prf_init}

\textbf{Proof of Item (ii).} This is Lemma~29 of \citep{colt21}, which follows from Lemma~i.A.3 of \citep{zhiyuan17} upon setting $p=\delta^2/4$ and observing that $\tr[(G^\top G)^{-1}] = \|G^{-1}\|_{\mathrm{F}}^2.$

\textbf{Proof of Item (iii).}
Without loss of generality, we can assume $U_* = [I_r; 0]$ and $U_0 = G(G^\top G)^{-1/2}$ where $G \in \RR^{d \times r}$ is a standard Gaussian matrix. Write 
\begin{align*}
    G= \begin{pmatrix}
        G_1\\
        G_2
    \end{pmatrix}
\end{align*}
where $G_1 \in \RR^{r \times r}$ and $G_2 \in \RR^{(d-r) \times r}$ are independent standard Gaussian matrices. Then
\begin{align*}
    (U_{*,\perp}^\top U_0)(U_*^\top U_0)^{-1} = G_2 G_1^{-1}.
\end{align*}
Item~(ii) gives $\PP(\norm{G_1^{-1}} \geq 12\sqrt{r}/\delta) \leq \delta/2$. Item~(i) with $t=\sqrt{2\log(4/\delta)}$ gives
\begin{align*}
    \PP \cbr[\big]{\norm{G_2} \geq \sqrt{d-r} + \sqrt{r} + \sqrt{2\log(4/\delta)}} \leq \delta/2.
\end{align*}
Combining the two bounds gives
\begin{align*}
    \PP \cbr[\Big]{\norm{(U_{*,\perp}^\top U_0)(U_*^\top U_0)^{-1}}\leq \frac{12 \sqrt{r}}{\delta}\rbr{\sqrt{d-r} + \sqrt{r} + \sqrt{2 \log (4/\delta)}} } \geq 1 - \delta.
\end{align*}

Since $\sqrt{d-r}+ \sqrt{r} \leq \sqrt{2d}$, and
\begin{align*}
    \sqrt{2d} + \sqrt{2\log(4/\delta)} \leq 2 \sqrt{2d} \sqrt{2\log(4/\delta)} = 4\sqrt{d \log (4/\delta)}
\end{align*}
We have
\begin{align*}
    \PP \cbr[\Big]{\norm{(U_{*,\perp}^\top U_0)(U_*^\top U_0)^{-1}}\leq \frac{48 \sqrt{rd \log(4/\delta)}}{\delta}  } \geq 1 - \delta.
\end{align*}

Set $\delta= \delta_p / c_p$ for $c_p \geq 1$, since $\log(4c_p/\delta_p)\leq 2\log(4c_p) \log(1/\delta_p)$ for $\delta_p \leq 1/e$, we have
\begin{align*}
    \PP \cbr[\Big]{\norm{(U_{*,\perp}^\top U_0)(U_*^\top U_0)^{-1}}\leq \frac{48c_p \sqrt{2\log(4c_p)} \sqrt{rd \log(1/\delta_p)}}{\delta_p}  } \geq 1 - \frac{\delta_p}{c_p}.
\end{align*}

\section{Additional experiments}
\label{sec:additional_experiments}
We provide additional row-wise diagnostics for Theorem~\ref{thm:bs_row_clt} under the same simulation settings as in Section~\ref{sec:experiments}.

\paragraph{Row-specific coverage.}
Table~\ref{tab:row_bs_quantile} summarizes the empirical coverage rates across all $d=100$ rows by reporting their quantiles. The coverage rates are concentrated near the nominal level, particularly for faster spectral decay, while some rows exhibit undercoverage for smaller $\beta$.
The variation across rows is natural, as Theorem~\ref{thm:bs_row_clt} is a fixed-row result with row-dependent conditions, so finite-sample accuracy may differ across $m\in[d]$.
\begin{table}[!htb]
\centering
\caption{Quantiles across rows of the row-specific empirical coverage rates at nominal level \(0.95\) over 200 Monte Carlo trials.}
\label{tab:row_bs_quantile}
\resizebox{\linewidth}{!}{\begin{tabular}{cccccccccccc}
\toprule
$\beta$ & Min & 10\% & 20\% & 30\% & 40\% & 50\% & 60\% & 70\% & 80\% & 90\% & Max \\
\midrule
$0.5$ & 0.840 & 0.880 & 0.890 & 0.897 & 0.906 & 0.910 & 0.920 & 0.930 & 0.930 & 0.940 & 0.970 \\
$1$ & 0.860 & 0.899 & 0.910 & 0.920 & 0.920 & 0.930 & 0.934 & 0.940 & 0.950 & 0.960 & 0.990 \\
$4$ & 0.900 & 0.920 & 0.930 & 0.930 & 0.940 & 0.940 & 0.950 & 0.950 & 0.960 & 0.970 & 0.970 \\
\bottomrule
\end{tabular}}
\end{table}

\paragraph{Row-specific Q--Q diagnostics.} To further assess row-wise calibration, for each row $m\in[d]$ we compare the standardized quadratic statistic
\begin{align*}
  \hQ_m:= ([U_\ast\sgn(U_\ast^\top U_n)]_{m,\cdot} - [U_n]_{m,\cdot}  )^\top [\widehat\Sigma_m^{\rm bs}]^{-1} ( [U_\ast\sgn(U_\ast^\top U_n)]_{m,\cdot} - [U_n]_{m,\cdot} )
\end{align*}
with the $\chi_r^2$ distribution using Q--Q plots. Table~\ref{tab:row_bs_qq_slope} summarizes, across all $d=100$ rows, the quantiles of the fitted Q--Q slopes, with slopes near one indicating good agreement. Figure~\ref{fig:row_qq} additionally shows the Q--Q plots for the first two rows under each $\beta\in\{0.5,1,4\}$. Overall, the slopes become more concentrated around one as the spectral decay becomes faster, consistent with the coverage results above.

\begin{table}[!htb]
\centering
\caption{Quantiles across rows of the fitted slopes of the row-specific Q--Q plots against the \(\chi_r^2\) distribution.}
\label{tab:row_bs_qq_slope}
\resizebox{\linewidth}{!}{\begin{tabular}{cccccccccccc}
\toprule
$\beta$ & Min & 10\% & 20\% & 30\% & 40\% & 50\% & 60\% & 70\% & 80\% & 90\% & Max \\
\midrule
$0.5$ & 0.943 & 1.052 & 1.107 & 1.121 & 1.156 & 1.189 & 1.214 & 1.257 & 1.292 & 1.348 & 1.431 \\
$1$ & 0.899 & 1.017 & 1.049 & 1.063 & 1.082 & 1.107 & 1.138 & 1.169 & 1.197 & 1.240 & 1.343 \\
$4$ & 0.912 & 0.964 & 0.987 & 1.002 & 1.018 & 1.033 & 1.058 & 1.075 & 1.102 & 1.140 & 1.228 \\
\bottomrule
\end{tabular}}
\end{table}

\begin{figure}[ht] 
  \centering 
  \includegraphics[width=\linewidth]{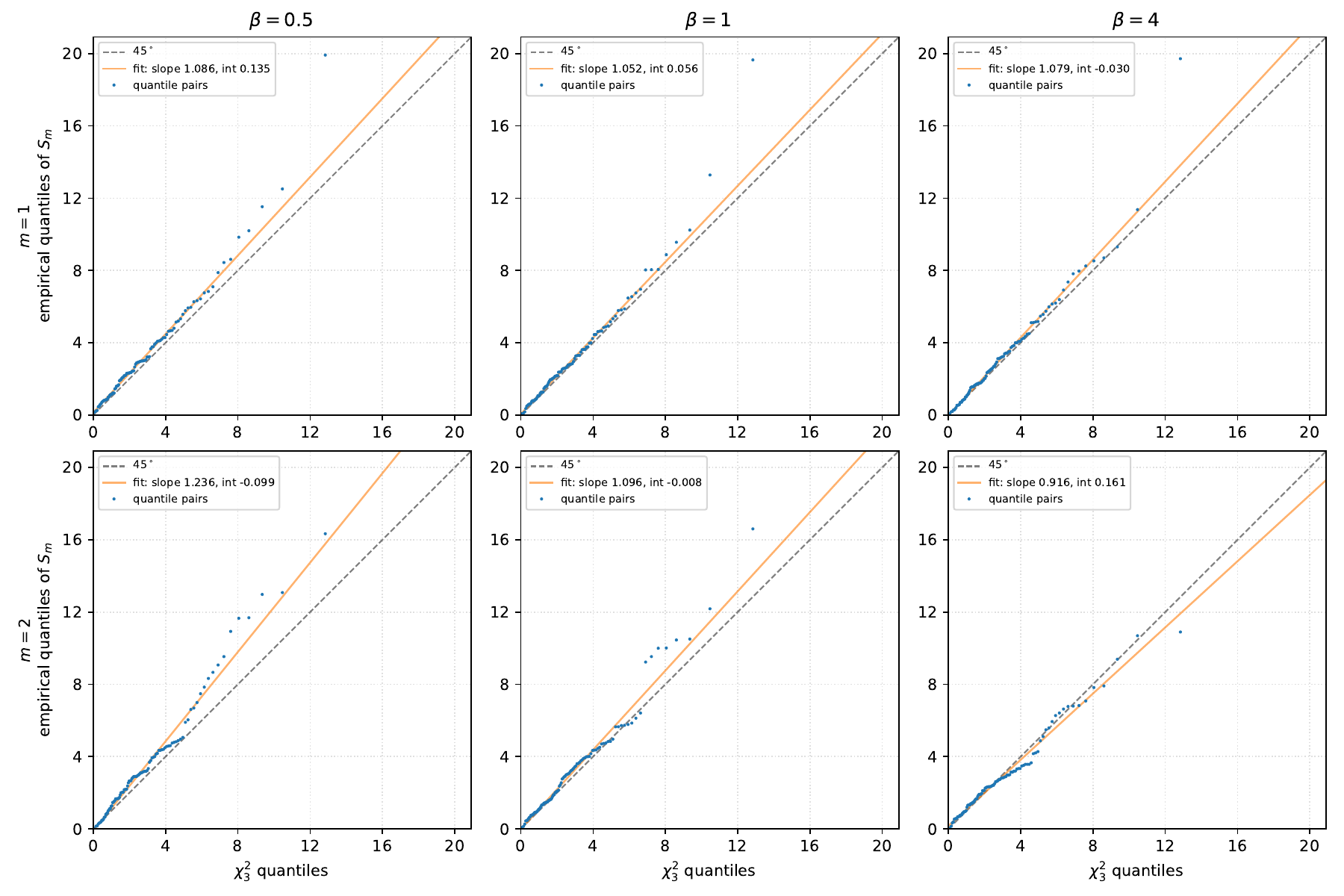} \caption{Row-specific Q--Q plots of the standardized quadratic statistic against the $\chi_r^2$ distribution for $m=1,2$ and $\beta\in\{0.5,1,4\}$. The diagonal line represents exact agreement. Columns correspond to $\beta=0.5,1,4$, and rows correspond to $m=1,2$.} \label{fig:row_qq} 
\end{figure}

%%%%%%%%%%%%%%%%%%%%%%%%%%%%%%%%%%%%%%%%%%%%%%%%%%%%%%%%%%%%

\newpage
\section*{NeurIPS Paper Checklist}

\begin{enumerate}

\item {\bf Claims}
    \item[] Question: Do the main claims made in the abstract and introduction accurately reflect the paper's contributions and scope?
    \item[] Answer: \answerYes{} % Replace by \answerYes{}, \answerNo{}, or \answerNA{}.
    \item[] Justification: The main claims in the abstract and introduction are supported by the theoretical results in \cref{sec:main_convergence}-\cref{sec:main_bs}. We also provide detailed comparisons with previous work.

\item {\bf Limitations}
    \item[] Question: Does the paper discuss the limitations of the work performed by the authors?
    \item[] Answer: \answerYes{} % Replace by \answerYes{}, \answerNo{}, or \answerNA{}.
    \item[] Justification: The Discussion section identifies limitations of the current theory, including its focus on the constant-step-size setting, and points to extensions to broader stepsize schedules and additional streaming PCA settings. The main theoretical results also explicitly state the assumptions under which the guarantees hold.


\item {\bf Theory assumptions and proofs}
    \item[] Question: For each theoretical result, does the paper provide the full set of assumptions and a complete (and correct) proof?
    \item[] Answer: \answerYes{} % Replace by \answerYes{}, \answerNo{}, or \answerNA{}.
    \item[] Justification: The paper states the assumptions for the theoretical results in the main text and provides full proofs in the appendix. \cref{thm:convergence_rate_optimal}--\ref{thm:bs_row_clt}, and the supporting lemmas and technical results are stated and referenced throughout the appendix.
    

    \item {\bf Experimental result reproducibility}
    \item[] Question: Does the paper fully disclose all the information needed to reproduce the main experimental results of the paper to the extent that it affects the main claims and/or conclusions of the paper (regardless of whether the code and data are provided or not)?
    \item[] Answer: \answerYes{} % Replace by \answerYes{}, \answerNo{}, or \answerNA{}.
    \item[] Justification: The experiments are based on synthetic Gaussian data, and \cref{sec:experiments} specifies the data-generating model, eigenvalue structure, parameter choices, stepsize, multiplier distribution, and the number of Monte Carlo and bootstrap replications used to produce the reported results.


\item {\bf Open access to data and code}
    \item[] Question: Does the paper provide open access to the data and code, with sufficient instructions to faithfully reproduce the main experimental results, as described in supplemental material?
    \item[] Answer: \answerNo{} % Replace by \answerYes{}, \answerNo{}, or \answerNA{}.
    \item[] Justification: This is mainly a theoretical paper. The experiments use synthetic data generated from the model described in \cref{sec:experiments},  no external dataset is required. Code is not included with the initial submission, but the paper provides the algorithmic description and experimental settings needed to reproduce the main results.

\item {\bf Experimental setting/details}
    \item[] Question: Does the paper specify all the training and test details (e.g., data splits, hyperparameters, how they were chosen, type of optimizer) necessary to understand the results?
    \item[] Answer: \answerYes{} % Replace by \answerYes{}, \answerNo{}, or \answerNA{}.
    \item[] Justification: \cref{sec:experiments} specifies the simulation model, including the Gaussian design, eigenvalue decay model, dimensions, rank, stepsize, bootstrap multiplier distribution, and the number of sampling/bootstrap repetitions used in each experiment.


\item {\bf Experiment statistical significance}
    \item[] Question: Does the paper report error bars suitably and correctly defined or other appropriate information about the statistical significance of the experiments?
    \item[] Answer: \answerYes{} % Replace by \answerYes{}, \answerNo{}, or \answerNA{}.
    \item[] Justification: For the row-wise coverage experiment, \cref{tab:rowwise_bootstrap} reports mean coverage together with standard deviations across Monte Carlo trials.


\item {\bf Experiments compute resources}
    \item[] Question: For each experiment, does the paper provide sufficient information on the computer resources (type of compute workers, memory, time of execution) needed to reproduce the experiments?
    \item[] Answer: \answerNo{} % Replace by \answerYes{}, \answerNo{}, or \answerNA{}.
    \item[] Justification: The experiments are based on synthetic data and simple streaming algorithms, so they do not require specialized hardware or large compute resources. The current version nevertheless does not provide detailed hardware specifications, memory usage, running time, or total compute.
    
\item {\bf Code of ethics}
    \item[] Question: Does the research conducted in the paper conform, in every respect, with the NeurIPS Code of Ethics \url{https://neurips.cc/public/EthicsGuidelines}?
    \item[] Answer: \answerYes{} % Replace by \answerYes{}, \answerNo{}, or \answerNA{}.
    \item[] Justification: The research is theoretical and simulation-based, uses no human-subject data or sensitive datasets, and we have reviewed the NeurIPS Code of Ethics.

\item {\bf Broader impacts}
    \item[] Question: Does the paper discuss both potential positive societal impacts and negative societal impacts of the work performed?
    \item[] Answer: \answerNA{} % Replace by \answerYes{}, \answerNo{}, or \answerNA{}.
    \item[] Justification: We do not see any direct societal impacts of the work performed. 
    
\item {\bf Safeguards}
    \item[] Question: Does the paper describe safeguards that have been put in place for responsible release of data or models that have a high risk for misuse (e.g., pre-trained language models, image generators, or scraped datasets)?
    \item[] Answer: \answerNA{} % Replace by \answerYes{}, \answerNo{}, or \answerNA{}.
    \item[] Justification: The paper only uses synthetic data. We do not release any models or datasets, so we do not see any risks for misuse of the work performed.

\item {\bf Licenses for existing assets}
    \item[] Question: Are the creators or original owners of assets (e.g., code, data, models), used in the paper, properly credited and are the license and terms of use explicitly mentioned and properly respected?
    \item[] Answer: \answerNA{} % Replace by \answerYes{}, \answerNo{}, or \answerNA{}.
    \item[] Justification: The experiments use synthetic data generated from the model described in \cref{sec:experiments} and do not rely on external datasets, pretrained models, or third-party assets.


\item {\bf New assets}
    \item[] Question: Are new assets introduced in the paper well documented and is the documentation provided alongside the assets?
    \item[] Answer: \answerNA{} % Replace by \answerYes{}, \answerNo{}, or \answerNA{}.
    \item[] Justification: The paper does not introduce or release new datasets, models, or other research assets.

\item {\bf Crowdsourcing and research with human subjects}
    \item[] Question: For crowdsourcing experiments and research with human subjects, does the paper include the full text of instructions given to participants and screenshots, if applicable, as well as details about compensation (if any)? 
    \item[] Answer: \answerNA{} % Replace by \answerYes{}, \answerNo{}, or \answerNA{}.
    \item[] Justification: The paper does not involve crowdsourcing, user studies, or research with human subjects.

\item {\bf Institutional review board (IRB) approvals or equivalent for research with human subjects}
    \item[] Question: Does the paper describe potential risks incurred by study participants, whether such risks were disclosed to the subjects, and whether Institutional Review Board (IRB) approvals (or an equivalent approval/review based on the requirements of your country or institution) were obtained?
    \item[] Answer: \answerNA{} % Replace by \answerYes{}, \answerNo{}, or \answerNA{}.
    \item[] Justification: The paper does not involve human subjects research.

\item {\bf Declaration of LLM usage}
    \item[] Question: Does the paper describe the usage of LLMs if it is an important, original, or non-standard component of the core methods in this research? Note that if the LLM is used only for writing, editing, or formatting purposes and does \emph{not} impact the core methodology, scientific rigor, or originality of the research, declaration is not required.
    %this research? 
    \item[] Answer: \answerNA{} % Replace by \answerYes{}, \answerNo{}, or \answerNA{}.
    \item[] Justification: LLMs are not used as an important, original, or non-standard component of the core methodology, theory, experiments, or scientific claims of this work.

\end{enumerate}

\end{document}